%% file: draft.tex
\documentclass[10pt,journal,compsoc]{IEEEtran}

\usepackage{times}
\usepackage{epsfig}
\usepackage{graphicx}

\usepackage{amssymb}
\usepackage{rotating}

\usepackage{amsmath}

\usepackage[bf,font=footnotesize]{caption}
\usepackage{multirow}
\usepackage{subfig}
\usepackage{comment}
\usepackage{url}
\usepackage{amsthm}
\usepackage{xspace,eucal}

\usepackage{algorithm}
\usepackage[algo2e,vlined,algoruled,resetcount]{algorithm2e}
\usepackage{cases}
\usepackage{array}
\usepackage{booktabs}

\newcolumntype{L}[1]{>{\raggedright\arraybackslash}p{#1}}

\def\NA{{$-$}}
\def\NAA{{$\dagger$}}
\newcommand{\fastsdp}{SDCut-QN\xspace}
\newcommand{\lowrank}{LowRank\xspace}
\newcommand{\smooth}{SDCut-SN\xspace}
\newcommand{\lbfgsb}{SDCut-QN\xspace}
\newcommand{\sdcut}{SDCut\xspace}
\newcommand{\parat}{$\gamma$\xspace}
\newcommand{\LBFGSB}{\mbox{L-BFGS-B}\xspace} %

\newcommand{\ba}{\mathbf a}

\newcommand{\bu}{\mathbf u}
\newcommand{\bx}{\mathbf x}
\newcommand{\bt}{\mathbf t}
\newcommand{\bh}{\mathbf h}
\newcommand{\by}{\mathbf y}

\newcommand{\bv}{\mathbf v}
\newcommand{\bb}{\mathbf b}
\newcommand{\bp}{\mathbf p}
\newcommand{\bq}{\mathbf q}
\newcommand{\bw}{\mathbf w}
\newcommand{\bz}{\mathbf z}
\newcommand{\bk}{\mathbf k}
\newcommand{\bbf}{\mathbf f}
\newcommand{\bs}{\mathbf s}

\newcommand{\blambda}{\boldsymbol \lambda}

\newcommand{\bX}{\mathbf X}

\newcommand{\bP}{\mathbf P}

\newcommand{\bW}{\mathbf W}
\newcommand{\bA}{\mathbf A}
\newcommand{\bI}{\mathbf I}
\newcommand{\bK}{\mathbf K}
\newcommand{\bH}{\mathbf H}
\newcommand{\bC}{\mathbf C}

\newcommand{\bB}{\mathbf B}
\newcommand{\bZ}{\mathbf Z}
\newcommand{\bY}{\mathbf Y}
\newcommand{\bV}{\mathbf V}

\newcommand{\sst}{\mathrm{s.t.}}
\def\T{{\!\top}}

\ifx\theorem\undefined
\newtheorem{theorem}{Theorem}
\newenvironment{theorem*}{\par\noindent{\bf Theorem\ }}{\hfill\\[2mm]}

\newtheorem{proposition}[theorem]{Proposition}

\newenvironment{corollary*}{\par\noindent{\bf Corollary\ }}{\hfill\\[2mm]}

\fi

\makeatletter
\DeclareRobustCommand\onedot{\futurelet\@let@token\@onedot}
\def\@onedot{\ifx\@let@token.\else.\null\fi\xspace}

\def\eg{for example\xspace} 
\def\ie{that is\xspace} 
\def\wrt{with respect to\xspace}

 \def\vs{\emph{vs}\onedot}

\def\etal{\emph{et al}\onedot}
\def\psd{p.s.d\onedot}
\makeatother

\usepackage{xr}
\usepackage{color}

\begin{document}

\title{Large-scale Binary Quadratic Optimization Using Semidefinite Relaxation and Applications}

\author{Peng Wang$^1$, Chunhua Shen$^{1}$, Anton van den Hengel$^{1}$, Philip H. S. Torr$^{2}$
\thanks{$^1$University of Adelaide, Australia;
$^2$University of Oxford, United Kingdom.
Correspondence should be addressed to C. Shen (chunhua.shen@adelaide.edu.au).
}
}

\markboth{Appearing in IEEE Trans. Pattern Analysis and Machine Intelligence, Feb.\ 2016}%
{Wang \MakeLowercase{\textit{et al.}}:
    Semidefinite Binary Quadratic Optimization
}

\IEEEtitleabstractindextext{%
\begin{abstract}
    In computer vision, many problems can be formulated as binary quadratic programs (BQPs),
    which are in general NP hard.
    Finding a solution when the problem is of large size to be of practical interest typically requires relaxation.
    Semidefinite relaxation usually yields tight bounds, but its computational complexity is high.
    In this work, we present a semidefinite programming (SDP) formulation for BQPs,
    with two desirable properties. First, it produces similar bounds to the standard SDP formulation.  Second, compared with
    the conventional SDP formulation, the proposed SDP formulation leads to a
    considerably more efficient and scalable dual optimization
    approach.
    We then propose two solvers, namely, quasi-Newton and smoothing Newton methods,
    for the simplified dual problem. Both of them are significantly more efficient than standard interior-point methods.
    Empirically the smoothing Newton solver is faster than the quasi-Newton solver
    for dense or medium-sized problems, while
    the quasi-Newton solver is preferable for large sparse/structured problems.
\end{abstract}

}

\maketitle

\onecolumn\tableofcontents
\twocolumn

\clearpage

\section{Introduction}

Binary quadratic programs (BQPs) are a class of combinatorial optimization problems
with binary variables, quadratic objective function and linear/quadratic constraints.
They appear in a wide variety of applications in computer vision,
such as image segmentation/pixel labelling, image registration/matching, image denoising/restoration.
Moreover, Maximum a Posteriori (MAP) inference problems
for Markov Random Fields (MRFs) can be formulated as BQPs too.
There are a long list of references to applications formulated as BQPs or specifically MRF-MAP problems.
Readers may refer to \cite{van2013solution,kochenberger2014unconstrained,li1994markov,wang2013markov,kappes2013comparative,Givry2014exp} %
and the references therein for detailed studies.

Unconstrained BQPs with  submodular pairwise terms
can be solved exactly and efficiently using
graph cuts~\cite{kolmogorov2004energy,kolmogorov2007minimizing,rother2007optimizing}.
However solving general BQP problems is known to be NP-hard (see \cite{li2010polynomially} for exceptions).
In other words, it is unlikely to find polynomial time algorithms to exactly solve these problems.
Alternatively, relaxation approaches can be used to produce a feasible solution close to the global optimum in polynomial time.
In order to accept such a relaxation we require a guarantee that the divergence between the solutions
to the original problem and the relaxed problem is bounded.
The quality of the relaxation thus depends upon the tightness of the bounds.
Developing an efficient relaxation algorithm with a tight relaxation bound
that can achieve a good solution (particularly for large problems) is thus of great practical importance.
There are a number of relaxation methods for BQPs (in particular MRF-MAP inference problems) in the literature,
including linear programming (LP) relaxation~\cite{wainwright2005map,kolmogorov2006convergent,Globerson2007,sontag2012efficiently}, %
quadratic programming relaxation ~\cite{ravikumar2006quadratic},
second order cone relaxation~\cite{kumar2006solving,kim2003exact,ghaddar2011second},
spectral relaxation~\cite{Shi2000normalized,Yu2004segmentation,Cour_solvingmarkov,Gour2006}
and
SDP relaxation~\cite{kim2003exact,torr2003solving}.

Spectral methods are effective for many computer vision applications, such as
image segmentation~\cite{Shi2000normalized,Yu2004segmentation} and motion segmentation~\cite{Lauer09spectralmotion}.
The optimization of spectral methods eventually lead to the computation of top eigenvectors.
Nevertheless, spectral methods may produce loose relaxation bounds in many cases~\cite{Guattery98onthe,Lang05fixingtwo,Kannan00onclusterings}.
Moreover, the inherent quadratic programming formulation of spectral methods
is difficult to incorporate certain types of additional constraints~\cite{Cour_solvingmarkov}.

{%
SDP relaxation has been shown that it leads to tighter approximation than other relaxation methods
for many combinatorial optimization problems~\cite{vandenberghe1996semidefinite,wolkowicz2000handbook,wainwright2008graphical,Kumar2009Ananalysis}.
In particular for the max-cut problem, Goemans and Williamson~\cite{Goemans95improved} achieve the state-of-the-art $0.879$ approximation ratio using SDP relaxation.
SDP relaxation has also been used in a range of vision problems, such as image segmentation~\cite{Heiler2005semi}, restoration~\cite{Keuchel2003binary,Olsson07solvinglarge},
graph matching~\cite{Torr2003MRFSDP,Schellewald05} and co-segmentation~\cite{Joulin2010dis}.
In a standard SDP problem, a linear function of a symmetric matrix $\bX$ is optimized,
subject to linear (in)equality constraints and the constraint of
$\bX$ being positive semidefinite (\psd).
The standard SDP problem and its Lagrangian dual problem are written as:
\begin{align}
\mbox{(SDP-P) \quad}
\min_{\bX \in \mathcal{S}^n_+} &\quad \mathrm{p}(\bX) := \langle \bX, \bA \rangle, \label{eq:psdp0} \\
\sst \,\,&\quad \langle \bX, \bB_i \rangle = b_i, \ i \in \mathcal{I}_{eq}, \notag \\
     &\quad \langle \bX, \bB_i \rangle \leq b_i, \ i \in \mathcal{I}_{in}, \notag \\
\mbox{(SDP-D) \quad}
\max_{\bu \in \mathbb{R}^m} &\quad \mathrm{d} (\bu) := - \bu^\T \bb, \label{eq:dsdp0} \\
\sst \,\,&\quad \bA + \textstyle{\sum_{i=1}^m} u_i \bB_i \in \mathcal{S}^n_+, \notag \\
         &\quad u_i \geq 0, \ i \in \mathcal{I}_{in}, \notag
\end{align}
where $m = |\mathcal{I}_{eq}| + |\mathcal{I}_{in}|$, and $\mathcal{I}_{eq}$ ($\mathcal{I}_{in}$) denotes the indexes of linear (in)equality constraints.
The \psd constraint $\bX \in \mathcal{S}^n_+$ is convex, so SDP problems are convex optimization problems and the above two formulations are equivalent if a feasible solution exists.
The SDP problem~\eqref{eq:psdp0} can be considered as a semi-infinite LP problem, as the \psd constraint
can be converted to an {\em infinite} number
of linear constraints: $\langle \bX, \ba \ba^\T \rangle \geq 0, \forall \ba \in \mathbb{R}^n $.
Through SDP, these infinite number of linear constraints can be handled in {\em finite} time. }

It is widely accepted that interior-point methods~\cite{alizadeh1995interior,nesterov1994interior}
are very robust and accurate for general SDP problems up to a moderate size
(see SeDuMi~\cite{Sturm98usingsedumi}, SDPT3~\cite{Toh99sdpt3} and MOSEK~\cite{mosek} for implementations).
However, its high computational complexity and memory requirement hampers the application of SDP methods to large-scale problems.
Approximate nonlinear programming methods~\cite{burer2003nonlinear,journee2010low,NIPS2014_5341}
are proposed for SDP problems based on low-rank factorization,
which may converge to a local optimum.
Augmented Lagrangian methods~\cite{malick2009regularization,zhao2010newton}
and the variants~\cite{wen2010alternating,guibasscalable}
have also been developed.
As gradient-descend based methods~\cite{rockafellar1973dual}, they may converge slowly.
The spectral bundle method~\cite{helmberg2000spectral} and the log-barrier algorithm~\cite{burer2003computational}
can be used for large-scale problems as well.
A drawback is that they can fail to solve some SDP problems to satisfactory accuracies~\cite{zhao2010newton}.

{%
In this work, we propose a regularized SDP relaxation approach to BQPs.
Preliminary results of this paper appeared in \cite{wang2013fast}.
Our main contributions are as follows.
}
\begin{enumerate}
\item
{
      Instead of directly solving the standard SDP relaxation to BQPs,
      we propose a quadratically regularized version of the original SDP formulation,
       which can be solved efficiently and achieve a solution quality comparable to the standard SDP relaxation.
}

\item We proffer two algorithms to solve the dual problem,
      based on quasi-Newton (referred to as \lbfgsb) and smoothing Newton (referred to as \smooth) methods respectively.
      The sparse or low-rank structure of specific problems are also exploited to speed up the computation.
      The proposed solvers require much lower computational cost and storage memory than standard interior-point methods.
      In particular, \lbfgsb has a lower computational cost in each iteration while
      needs more iterations to converge.
      On the other hand, \smooth converges quadratically with  higher
      computational complexity per iteration.
      In our experiments, \smooth is faster for dense or medium-sized problems,
      and \lbfgsb is more efficient for large-scale sparse/structured problems.

\item
  {
    We demonstrate the efficiency and flexibility of our proposed algorithms by applying them to
    a variety of computer vision tasks. We show that due to the capability of accommodating various constraints,
    our methods can encode problem-dependent information.
    More specifically,
    the formulation of SDCut allows multiple additional linear and quadratic constraints,
    which enables a broader set of applications than
    what spectral methods and graph-cut methods can be applied to.
}
\end{enumerate}

\noindent {\bf Notation}
A matrix (column vector) is denoted by a bold capital (lower-case) letter.
$\mathbb{R}^n$ denotes the space of real-valued $n \times 1$ vectors.
$\mathbb{R}^n_+$ and $\mathbb{R}^n_-$ represent the
non-negative and non-positive orthants of $\mathbb{R}^n$ respectively.
$\mathcal{S}^n$ denotes the space of $n \times n$ symmetric matrices, and
$\mathcal{S}^n_+$ represents the corresponding cone of positive semidefinite (p.s.d.) matrices.
For two vectors, $\bx \leq \by$ indicates the element-wise inequality;
$\mathrm{trace}(\bX)$, $\mathrm{rank}(\bX)$ and $\mathrm{diag}(\bX)$
denote the trace, rank and the main diagonal elements of $\bX$ respectively.
$\mathrm{Diag}(\bx)$ denotes a diagonal matrix with the elements of vector $\bx$ on the main diagonal.
$\lVert \bX \rVert_F^2$ denotes the Frobenius norm of $\bX$.
The inner product of two matrices is defined as $\langle \bX, \bY \rangle$.
$\bI_n$ indicates the $n \times n$ identity matrix.
$\mathbf{0}$ and $\mathbf{1}$ denote all-zero and all-one column vectors respectively.%
$\nabla \mathrm{f}(\cdot)$ and $\nabla^2 \mathrm{f}(\cdot)$ stand for
the first-order and second-order derivatives of function $\mathrm{f}(\cdot)$ respectively.

\section{BQPs and their SDP relaxation}
\label{sec:bqp}
{ %
Let us consider a binary quadratic program of the following form:
\begin{subequations}
\label{eq:bqp}
\begin{align}
\min_{\bx \in \{+1,-1 \}^n } \,\, &\bx^\T \bA_0 \bx + \ba_0^\T \bx, \,\,\,\, \\
\sst  \quad\,\,\,\,\,      &\bx^{\T} \bA_i \bx + \ba_i^\T \bx    = b_i, \, i \in \mathcal{I}_{eq}, \label{eq:bqp_cons1}\\
                 &\bx^{\T} \bA_i \bx + \ba_i^\T \bx \leq b_i, \, i \in \mathcal{I}_{in},
\end{align}
\end{subequations}
where $\bA_i \in \mathcal{S}^n, \ba_i \in \mathbb{R}^n, \forall i \in \mathcal{I}_{eq} \cup \mathcal{I}_{in}$;
$\bb \in \mathbb{R}^{\lvert \mathcal{I}_{eq} \rvert +  \lvert \mathcal{I}_{in} \rvert}$.
Note that BQP problems can be considered as special cases of quadratically constrained quadratic program (QCQP),
as the constraint $\bx \in \{1,-1 \}^n$ is equivalent to $x_i^2 = 1, \forall i = 1, \cdots, n$.
Problems over $\bx \in \{0,1\}^n$ can be also expressed as $\{1,-1\}$-problems \eqref{eq:bqp} by replacing $\bx$ with $\by = 2\bx - \mathbf{1}$.

Solving \eqref{eq:bqp} is in general NP-hard, so relaxation methods are considered in this paper.
Relaxation to \eqref{eq:bqp} can be done by extending the feasible set to a larger set,
such that the optimal value of the relaxation is a lower bound on the optimal value of \eqref{eq:bqp}.
The SDP relaxation to \eqref{eq:bqp} can be expressed as:
\begin{subequations}
\label{eq:backgd_sdp1}
\begin{align}
\min_{\bx, \bX } \quad   &\langle \bX,\bA_0 \rangle + \ba_0^\T \bx, \\
 \sst \quad      & \mathrm{diag}(\bX) = \mathbf{1}, \\
                 &\langle \bX, \bA_i \rangle + \ba_i^\T \bx   = b_i, \, i \in \mathcal{I}_{eq}, \,\, \\
                 &\langle \bX, \bA_i \rangle + \ba_i^\T \bx   \leq b_i, \, i \in \mathcal{I}_{in}, \\
                 &{\scriptsize \begin{bmatrix} 1 & \bx^\T \\ \bx & \bX \end{bmatrix} }\in \mathcal{S}^{n+1}. \label{eq:sdp_cons}
\end{align}
\end{subequations}
Note that constraint \eqref{eq:sdp_cons} is equivalent to $\bX - \bx \bx^\T \in \mathcal{S}^n$, which is
the convex relaxation to the nonconvex constraint $\bX - \bx \bx^\T = \mathbf{0}$.
In other words, \eqref{eq:backgd_sdp1} is equivalent to \eqref{eq:bqp}, by replacing constraint \eqref{eq:sdp_cons} with $\bX = \bx \bx^\T$
or by adding the constraint $\mathrm{rank}({\scriptsize \begin{bmatrix} 1 & \bx^\T \\ \bx & \bX \end{bmatrix} }) = 1$.

The objective function and constraints (apart from the \psd constraint) of the problem \eqref{eq:backgd_sdp1} are all linear
\wrt $\overline{\bX}  = {\scriptsize \begin{bmatrix} 1 & \bx^\T \\ \bx & \bX \end{bmatrix}}$,
so \eqref{eq:backgd_sdp1} can be expressed in the homogenized form shown in \eqref{eq:psdp0} \wrt $\overline{\bX}$.
For simplicity, we consider the homogeneous problem \eqref{eq:psdp0}, instead of \eqref{eq:backgd_sdp1}, in the sequel.

Note that the SDP solution does not offer a feasible solution to the BQP~\eqref{eq:bqp} directly, unless it is of rank $1$.
A rounding procedure is required to extract a feasible BQP solution from the SDP solution, which will be discussed in Section~\ref{sec:rounding}.
}

\section{\sdcut Formulation}

A regularized SDP formulation is considered in this work:
\begin{subequations}
\label{eq:sdcut_sdp1}
\begin{align}
\mbox{(SDCut-P) \,\,}  \min_{\bX \in \mathcal{S}^n_+} &\,  \mathrm{p}_{\gamma}(\bX) \!:=\! \langle \bX,\bA \rangle + \frac{1}{2\gamma} \lVert \bX \rVert_F^2, \label{eq:sdcut_sdp1_obj} \\
\sst    \,\,   &\, \langle \bB_i, \bX \rangle = b_i, i \in \mathcal{I}_{eq}, \\
               &\, \langle \bB_i, \bX \rangle \leq b_i, i \in \mathcal{I}_{in},
\end{align}
\end{subequations}
where $\gamma > 0$ is a prescribed parameter (its practical value is discussed in Section~\ref{sec:graph_bisection}).

Compared to \eqref{eq:psdp0}, the formulation \eqref{eq:sdcut_sdp1}
adds into the objective function a Frobenius-norm term \wrt $\bX$.
The reasons for choosing this particular formulation are two-fold:
$i$) The solution quality of \eqref{eq:sdcut_sdp1} can be as close to that of
\eqref{eq:backgd_sdp1}
as desired
by making $\gamma$ sufficiently large.
$ii$) A simple dual formulation can be derived from \eqref{eq:sdcut_sdp1}, which can be optimized using quasi-Newton or inexact generalized Newton approaches.

In the following, a few desirable properties of \eqref{eq:sdcut_sdp1} are demonstrated,
where $\bX^\star$ denotes the optimal solution to \eqref{eq:psdp0} and $\bX^\star_\gamma$ denotes the optimal solution to \eqref{eq:sdcut_sdp1} \wrt $\gamma$.
The proofs can be found in Section~\ref{sec:appdx}.

{%
\begin{proposition}
\label{thm:prop1}
The following results hold:
($ i $)
$\forall \ \epsilon > 0$, $\exists \ \gamma >0$
such that
$ | \mathrm{p}(\bX^\star) - \mathrm{p}(\bX_{\gamma}^\star) | \leq \epsilon$;
($ ii $) $\forall \gamma_2 > \gamma_1 > 0$, we have $\mathrm{p}(\bX^\star_{\gamma_1}) \geq \mathrm{p}(\bX^\star_{\gamma_2})$.
\end{proposition}
}

The above results show that
the solution quality of \eqref{eq:sdcut_sdp1} can be monotonically improved
towards that of \eqref{eq:backgd_sdp1},
by making $\gamma$ sufficiently large.

\begin{proposition}
\label{thm:dual_simple}
The dual problem of~\eqref{eq:sdcut_sdp1} can be simplified to
\begin{align}
\mbox{(SDCut-D)\,\,} \max_{\bu \in \mathbb{R}^m} \,
     &\mathrm{d}_\gamma(\bu) \!:=\! -\! \bu^{\T} \bb \!-\! \frac{\gamma}{2} \lVert \Pi_{\mathcal{S}^n_+} \!(\bC(\bu))\! \rVert_F^2, \notag \\
\sst      \,\, &\, u_i \geq 0, i \in \mathcal{I}_{in}, \label{eq:fastsdp_dual}
\end{align}
where
\[ \bC(\bu) := - \bA - \textstyle{\sum_{i=1}^{m}} u_i \bB_i,
\] and
\[ \Pi_{\mathcal{S}^n_+}(\bC(\bu)) := \textstyle{\sum_{i=1}^n} \max(0,\lambda_i) \bp_i \bp_i^\T.
\]
$\lambda_i$, $\bp_i$, $i = 1, \cdots, n$ are eigenvalues and the corresponding eigenvectors of $\bC(\bu)$.
Supposing problem \eqref{eq:sdcut_sdp1} is feasible and denoting $\bu^\star$ as the dual optimal solution, we have:
\begin{align}
\bX^{\star} = \gamma \Pi_{\mathcal{S}^n_+} (\bC(\bu^\star)). \label{eq:fastsdp_primal_dual}
\end{align}
\end{proposition}

The simplified dual~\eqref{eq:fastsdp_dual} is convex and contains only simple box constraints.
Furthermore, its objective function $\mathrm{d}_\gamma(\cdot)$ has the following important properties.

\begin{proposition}
\label{thm:dual_objective}
$\mathrm{d}_\gamma(\cdot)$ is continuously differentiable but not necessarily twice differentiable,
and its gradient is given by
\begin{align}
\label{eq:dual_gradient}
\nabla \mathrm{d}_\gamma(\bu) = - \gamma \Phi \left[ \Pi_{\mathcal{S}^n_+} \left(\bC(\bu)\right) \right] - \bb.
\end{align}
where $\Phi: \mathcal{S}^n \rightarrow \mathbb{R}^m $ denotes the linear transformation $\Phi[\bX] := [\langle \bB_1, \bX\rangle, \cdots, \langle \bB_m, \bX\rangle]^\T$.
\end{proposition}

Based on the above result, the dual problem can be solved by quasi-Newton methods directly.
Furthermore, we also show in Section~\ref{sec:newton} that, the second-order derivatives of $\mathrm{d}_\gamma(\cdot)$
can be smoothed such that inexact generalized Newton methods can be applied.

\begin{proposition}
$\forall \bu \in \mathbb{R}^{|\mathcal{I}_{eq}|} \!\times\! \mathbb{R}_{+}^{|\mathcal{I}_{in}|}$,
$\forall \gamma \!> \!0$, $\mathrm{d}_\gamma(\bu)  \!-\! \frac{n^2}{2\gamma}$
yields a lower-bound on the optimum of the BQP \eqref{eq:bqp}.
\label{Remark:4}
\end{proposition}

The above result is important as the lower-bound can be used to examine how close
between an approximate binary solution and the global optimum.

\subsection{Related Work}
\label{sec:ralate}

{ %
Considering the original SDP dual problem \eqref{eq:dsdp0},
we can find that its \psd constraint, \ie $\bA + \sum_{i=1}^m u_i \bB_i \in \mathcal{S}^n_+ $,
is penalized in \eqref{eq:fastsdp_dual} by minimizing $\lVert \Pi_{\mathcal{S}^n_+} (\bC(\bu)) \rVert_F^2 = \sum_{i=1}^n \max(0,\lambda_i)^2$,
where $\lambda_i, \cdots, \lambda_n$ are the eigenvalues of $- \bA - \sum_{i=1}^m u_i \bB_i$.
The \psd constraint is satisfied if and only if the penalty term equals to zero.

Other forms of penalty terms may be employed in the dual.
The spectral bundle method of \cite{helmberg2000spectral} penalizes $\lambda_{max} (\bX)$ and
the log-barrier function is used in \cite{burer2003computational}.
It is shown in \cite{zhao2010newton} that these two first-order methods may converge slowly for some SDP problems.
Note that the objective function of the spectral bundle methods is not necessarily differentiable
($\lambda_{max} (\cdot)$ is differentiable if and only if it has multiplicity one).
The objective function of our formulation is differentiable and its twice derivatives can be smoothed,
such that classical methods can be easily used for solving our problems, using quasi-Newton and inexact generalized Newton methods.

Consider a proximal algorithm for solving SDP with only equality constraints
(see \cite{malick2009regularization,zhao2010newton,wen2010alternating,guibasscalable,henrion2012projection}):
\begin{align}
\min_{\bY \in \mathcal{S}^n} \big( \mathrm{G}_{\gamma} (\bY)
 \!:=\! \min_{\bX \in \mathcal{S}^n_{+}, \Phi[\bX] = \bb} \langle \bX, \bA \rangle + \frac{1}{2\gamma} || \bX \!-\! \bY ||^2_F \big),
\end{align}
where $\Phi[\bX] := [\langle \bB_1, \bX \rangle, \cdots, \langle \bB_{m},\bX \rangle ]^{\T}$.
Our algorithm is equivalent to solving the inner problem, \ie, evaluating $\mathrm{G}_{\gamma}(\bY)$, with a fixed $\gamma$ and $\bY = \mathbf{0}$.
In other words, our methods attempt to solve the original SDP relaxation approximately, with a faster speed.
After rounding, typically,
the resulting solutions of our algorithms are already close to those of the original SDP relaxation.
}

Our method is mainly motivated by the work of Shen
\etal~\cite{Shen2011scalabledual}, which presented a fast dual SDP
approach to Mahalanobis metric learning.
They, however, focused on learning a real-valued metric for nearest neighbour
classification. Here, in contrast, we are interested in discrete
combinatorial optimization problems arising in computer vision.
Krislock \etal~\cite{krislock2012improved}
have independently formulated a similar SDP problem for the
max-cut problem, which is simpler than the problems that we solve here.
Moreover, they focus on globally solving the max-cut problem using branch-and-bound.

\section{Solving the Dual Problem}

Based on Proposition~\ref{thm:dual_objective},
first-order methods (\eg gradient descent, quasi-Newton), which only require the calculation
 of the objective function and its gradients,
can be directly applied to solving \eqref{eq:fastsdp_dual}.
It is difficult in employing standard Newton methods, however, as they require the calculation of second-order derivatives.
In the following two sections,
 we present two algorithms for solving the dual~\eqref{eq:fastsdp_dual},
which are based on quasi-Newton and inexact generalized Newton methods respectively.

\subsection{Quasi-Newton Methods}

One main advantage of quasi-Newton methods over Newton methods is that
the inversion of the Hessian matrix is approximated
 by analyzing successive gradient vectors,
and thus that there is no need to explicitly compute the Hessian matrix and its inverse,
which can be very expensive.
Therefore the per-iteration computation cost of quasi-Newton methods is less than that of standard Newton methods.

The quasi-Newton algorithm for \eqref{eq:fastsdp_dual} (referred to as \lbfgsb) is summarized in Algorithm~\ref{alg:lbfgsb}.
In Step~1, the dual problem~\eqref{eq:fastsdp_dual} is solved using L-BFGS-B~\cite{Zhu94lbfgsb},
which only requires the calculation of the dual objective function~\eqref{eq:fastsdp_dual} and its gradient~\eqref{eq:dual_gradient}.
At each iteration, a descent direction for $\Delta \bu$ is computed based on the gradient $\nabla \mathrm{d}_\gamma(\bu)$
and the
approximated inverse
of the Hessian matrix: $\bH \approx (\nabla^2 \mathrm{d}_\gamma(\bu))^{-1}$.
A step size $\rho$ is found using line search.
The algorithm is stopped when the difference between successive dual objective values is smaller than a pre-set tolerance.

After solving the dual using L-BFGS-B, the primal optimal variable $\bX^\star$
is calculated from the dual optimal $\bu^\star$ based on Equation~\eqref{eq:fastsdp_primal_dual} in Step~2.

Finally in Step~3, the primal optimal variable $\bX^\star$ is discretized
and factorized to produce
the feasible binary solution $\bx^\star$, which will be described in Section~\ref{sec:rounding}.

Now we have an upper-bound and a lower-bound (see Propsition~\ref{Remark:4}) on the optimum of the original BQP~\eqref{eq:bqp} (referred to as $p^\star$):
$\mathrm{p}(\bx^\star \bx^{\star\T}) \geq p^\star \geq \mathrm{d_\gamma}(\bu^\star) - \frac{n^2}{2\gamma}$.
These two values are used to measure the solution quality in the experiments.

\begin{algorithm}[t]
\footnotesize
\setcounter{AlgoLine}{0}
\caption{\lbfgsb: Solving \eqref{eq:fastsdp_dual} using quasi-Newton methods.}
\begin{minipage}[]{1\linewidth}
   \KwIn{$\bA$, $\Phi$, $\bb$, $\gamma$, $\bu_{0}$, $K_{\mathrm{max}}$, $\tau > 0$.}
   {\bf Step~1:} Solving the dual using L-BFGS-B

   \Indp\For{$k = 0, 1, 2, \dots, K_{\mathrm{max}}$}
   {
     {\bf Step~1.1:} Compute $\nabla \mathrm{d}_\gamma(\bu_k)$ and update  ${\bH}$. \\
     {\bf Step~1.2:} Compute the descent direction $\Delta \bu = - {\bH} \nabla \mathrm{d}_\gamma (\bu_k)$. \\
     {\bf Step~1.3:} Find a step size $\rho$, and $\bu_{k+1} = \bu_k + \rho \Delta \bu$. \\
     {\bf Step~1.4:} Exit, if $\frac{(\mathrm{d}_\gamma(\bu_{k+1}) - \mathrm{d}_\gamma(\bu_{k}))}
                               {\max\{|\mathrm{d}_\gamma(\bu_{k+1})|, |\mathrm{d}_\gamma(\bu_{k})|,1\}} \leq \tau$.
   }

   \Indm{\bf Step~2:} $\bu^\star = \bu_{k+1}$, $\bX^{\star} =  \gamma\Pi_{\mathcal{S}^n_+} (\bC(\bu^\star))$.

   {\bf Step~3:} $\bx^\star = \mathrm{Round}(\bX^\star)$.

   \KwOut{$\bx^\star$, $\bu^\star$, upper-bound: $\mathrm{p}(\bx^\star \bx^{\star\T})$ and lower-bound: $\mathrm{d_\gamma}(\bu^\star) - \frac{n^2}{2\gamma}$.}
\end{minipage}
\label{alg:lbfgsb}
\end{algorithm}

\subsection{Smoothing Newton Methods}
\label{sec:newton}

As $\mathrm{d}_{\gamma}(\bu)$ is a concave function,
the dual problem~\eqref{eq:fastsdp_dual} is equivalent to finding $\bu^\star \in \mathcal{D}$ such that
$\langle \bu - \bu^\star, - \nabla \mathrm{d}_{\gamma}(\bu^\star) \rangle \geq 0$, $\forall \bu \in \mathcal{D}$, which is known as variational inequality~\cite{harker1990finite}.
$\mathcal{D} := \mathbb{R}^{|\mathcal{I}_{eq}|} \!\times\! \mathbb{R}_{+}^{|\mathcal{I}_{in}|}$ is used to denote the feasible set of the dual problem.
Thus \eqref{eq:fastsdp_dual} is also equivalent to finding a root of the following equation:
\begin{align}
\mathrm{F}(\bu) \!:=\! \bu \!-\! \Pi_{\mathcal{D}}\!\big(\bu \!-\! \gamma \Phi \left[ \Pi_{\mathcal{S}^n_+} \!\left(\bC(\bu)\right) \right] \!-\! \bb  \big) \!=\! \mathbf{0}, \bu \!\in\! \mathbb{R}^{m},
\label{eq:smooth_F_equation}
\end{align}
where
$\left[\Pi_{\mathcal{D}}(\bv)\right]_i := {\scriptsize \left\{ \begin{array}{ll}
                                           v_i                 & \mbox{if } i \in \mathcal{I}_{eq} \\
                                           \max(0,v_i)  & \mbox{if } i \in \mathcal{I}_{in}
                                       \end{array}
                               \right. }$
can be considered as a metric projection from $\mathbb{R}^{m}$
to $\mathbb{R}^{|\mathcal{I}_{eq}|} \times \mathbb{R}^{|\mathcal{I}_{in}|}_+$.
Note that $\mathrm{F}(\bu)$ is continuous but not continuously differentiable,
as both $\Pi_{\mathcal{D}}$ and $\Pi_{\mathcal{S}^n_+}$ have the same smoothness property.
Therefore, standard Newton methods cannot be applied directly to solving \eqref{eq:smooth_F_equation}.
In this work, we use the inexact smoothing Newton method in \cite{gao2009calibrating} to solve the smoothed Newton equation:
\begin{align}
\mathrm{E}(\epsilon,\bu) := \left[ {\epsilon}; \tilde{\mathrm{F}}(\epsilon,\bu) \right]  = \mathbf{0},
\quad (\epsilon,\bu) \in \mathbb{R}\times\mathbb{R}^m, \label{eq:smooth_E_equation}
\end{align}
where $\tilde{\mathrm{F}}(\epsilon,\bu)$ is a smoothing function of $\mathrm{F}(\bu)$, which is constructed as follows.

Firstly, the smoothing functions for $\Pi_{\mathcal{D}}$ and $\Pi_{\mathcal{S}^n_+}$ are respectively written as:
\begin{align}
  &\left[\tilde{\Pi}_{\mathcal{D}}(\epsilon,\bv)\right]_i \!:=\! \left\{ \!\begin{array}{ll}
                                    v_i                 & \mbox{if } i \in \mathcal{I}_{eq}, \\
                                    \phi(\epsilon,v_i)  & \mbox{if } i \in \mathcal{I}_{in},
                                  \end{array}
                          \right. (\epsilon, \bv) \!\in\! \mathbb{R} \!\times\! \mathbb{R}^{m}, \\
  &\tilde{\Pi}_{\mathcal{S}^n_+}(\epsilon,\bX)   :=
                           \sum_{i=1}^n \phi(\epsilon, \lambda_i) \bp_i \bp_i^\T,
                           \quad (\epsilon, \bX) \in \mathbb{R} \times \mathcal{S}^n,
\end{align}
where $\lambda_i$ and $\bp_i$ are the $i$th eigenvalue and the corresponding eigenvector of $\bX$.
$\phi(\epsilon,v)$ is the Huber smoothing function that we adopt here to replace $\max(0,v)$:
\begin{align}
\phi(\epsilon,v) := \left\{ \begin{array}{ll}
                  v & \mbox{if } v > 0.5\epsilon, \\
                  (v+0.5\epsilon)^2 / 2\epsilon, &\mbox{if } -0.5\epsilon \leq v \leq 0.5\epsilon, \\
                  0 & \mbox{if } v < -0.5\epsilon.
                  \end{array} \right.
\end{align}
Note that at $\epsilon = 0$,
$\phi(\epsilon,v) = \max(0,v)$,
$\tilde{\Pi}_{\mathcal{D}}(\epsilon,\bv) = \Pi_{\mathcal{D}}(\bv)$ and
$\tilde{\Pi}_{\mathcal{S}^n_+}(\epsilon,\bX)   = \Pi_{\mathcal{S}^n_+}(\bX)$.
$\phi$, $\tilde{\Pi}_{\mathcal{D}}$, $\tilde{\Pi}_{\mathcal{S}^n_+}$ are Lipschitz continuous on
$\mathbb{R}$, $\mathbb{R} \times \mathbb{R}^m$, $\mathbb{R} \times \mathcal{S}^n$ respectively,
and they are continuously differentiable when $\epsilon \neq 0$.
Then $\tilde{\mathrm{F}}(\epsilon,\bu)$ is defined as:
\begin{align}
\tilde{\mathrm{F}}(\epsilon,\bu) := \bu - \tilde{\Pi}_{\mathcal{D}} \left(\epsilon,
   \bu -  \gamma \Phi \left[ \tilde{\Pi}_{\mathcal{S}^n_+} \left( \epsilon,\bC(\bu) \right) \right] - \bb \right),
\label{eq:Upsilon_smooth_F}
\end{align}
which has the same smoothness property as $\tilde{\Pi}_{\mathcal{D}}$ and $\tilde{\Pi}_{\mathcal{S}^n_+}$.

The presented inexact smoothing Newton method (referred to as \smooth) is shown in Algorithm~\ref{alg:smooth}.
In Step~1.2, the Newton linear system \eqref{eq:cg_equation}
is solved approximately using conjugate gradient (CG) methods when $|\mathcal{I}_{in}| = 0$
and using biconjugate gradient stabilized (BiCGStab) methods~\cite{van1992bi} otherwise.
In Step~1.3,
we carry out a search in the direction $\left[ \Delta \epsilon_k ; \Delta \bu_k  \right]$ for an appropriate step size $\rho^{l}$ such that the norm of $\mathrm{E}(\epsilon,\bu)$ is decreased.

   \begin{algorithm}[t]
   \footnotesize
   \setcounter{AlgoLine}{0}
   \caption{\footnotesize \smooth: Solving \eqref{eq:fastsdp_dual} using smoothing Newton methods.}
   \centering
   \begin{minipage}[]{1\linewidth}
   \KwIn{$\bA$, $\Phi$, $\bb$, $\gamma$, $\bu_{0}$, ${\epsilon_0}$, $K_{\mathrm{max}}$, $\tau > 0$, $\mu \in (0,1)$, $\rho \in (0,1)$.}

   {\bf Step~1:} Solving the dual using smoothing Newton methods

   \Indp\For{$k = 0, 1, 2, \dots, K_{\mathrm{max}}$}
   {
     {\bf Step~1.1:} $\bar{\epsilon} \leftarrow {\epsilon_k}$ or $\mu{\epsilon_k}$. \\
     {\bf Step~1.2:} Solve the following linear system up to certain accuracy
              \begin{align}
                 \mathrm{E}(\epsilon_k,\bu_k) + \nabla \mathrm{E} (\epsilon_k, \bu_k)
                      \left[  \Delta \epsilon_k ; \Delta \bu_k  \right]
                 = \left[  \bar{\epsilon} ; \mathbf{0}  \right].
              \label{eq:cg_equation}
              \end{align}

     {\bf Step~1.3:} Line Search

           \Indp $l = 0$; \\
           \lWhile(){$\lVert \mathrm{E}(\epsilon_k \!+\! \rho^l \Delta \epsilon_k,\bu_k \!+\! \rho^l \Delta \bu_k) \rVert^2_2
                           \geq \lVert\mathrm{E}(\epsilon_k,\bu_k)\rVert^2_2$}
                     {\hspace{0.5cm} $\quad\,\,\, l = l + 1$}

            $\epsilon_{k+1} = \epsilon_{k} + \rho^l \Delta \epsilon_k$, $\bu_{k+1} = \bu_{k} + \rho^l \Delta \bu_k$.

     \Indm{\bf Step~1.4:} If $\frac{|\mathrm{d}_\gamma(\bu_{k+1}) - \mathrm{d}_\gamma(\bu_{k})|}
                      {\max\{|\mathrm{d}_\gamma(\bu_{k+1})|,|\mathrm{d}_\gamma(\bu_{k})|,1\}} \leq \tau$, break.

   }

   \Indm{\bf Step~2:} $\bu^\star = \bu_{k+1}$, $\bX^{\star} =  \gamma\Pi_{\mathcal{S}^n_+} (\bC(\bu^\star))$.

   {\bf Step~3:} $\bx^\star = \mathrm{Round}(\bX^\star)$.

   \KwOut{$\bx^\star$, $\bu^\star$, upper-bound: $\mathrm{p}(\bx^\star \bx^{\star\T})$ and lower-bound: $\mathrm{d_\gamma}(\bu^\star) - \frac{n^2}{2\gamma}$.}
   \end{minipage}
   \label{alg:smooth}
   \end{algorithm}

   \begin{algorithm}[t]
   \footnotesize
   \setcounter{AlgoLine}{0}
   \caption{ { \footnotesize Randomized Rounding Procedure: $\bx^\star = \mathrm{Round}(\bX^\star)$}}
   \centering
   \begin{minipage}[]{1\linewidth}
{ %
   \KwIn{The SDP solution $\bX^\star$, which is decomposed to a set of vectors
         $\bv_1 \dots \bv_n \in \mathbb{R}^r$ where $r = \mathrm{rank}(\bX^\star)$.}

   \Indp\For{$k = 0, 1, 2, \dots, K$}
   {
     {\bf Step~1:} Random sampling: obtain a real $1$-dimensional vector $\bz = [\bv_1 \dots \bv_n]^\T \by$, where $\by \sim \mathcal{N}(\mathbf{0},\bI_r)$. \\
     {\bf Step~2:} Discretization: $\bz$ is discretized to a feasible BQP solution (see Table~\ref{tab:formulation} for problem-specific methods).
   }
   \KwOut{$\bx^\star$ is assigned to the best feasible solution.}
}
   \end{minipage}
   \label{alg:rounding}
   \end{algorithm}

\subsection{Randomized Rounding Procedure}
\label{sec:rounding}

{ %
In this section, we describe a randomized rounding procedure (see Algorithm~\ref{alg:rounding}) for obtaining a
feasible binary solution from the relaxed SDP solution $\bX^\star$. %

Suppose that $\bX^\star$ is decomposed into a set of $r$-dimensional vectors $\bv_1 \dots \bv_n$,
such that $\bX^\star_{ij} = \bv_i^\T \bv_j$.
This decomposition can be easily obtained through the eigen-decomposition of $\bX^\star$:
$\bX = \bV \bV^\T$ and $\bV = [\bv_1 \dots \bv_n]^\T$.
We can see that these vectors reside on the $r$-dimensional unit sphere $\mathcal{S}_r := \{ \bv \in \mathbb{R}^r, \bv^\T \bv = 1 \}$,
and the angle between two vectors $\bv_i$ and $\bv_j$ defines how likely
the corresponding two variables $x_i$ and $x_j$ will be separated (assigned with different labels).
To transform these vectors into binary solutions,
they are firstly projected onto a random $1$-dimensional line $\by \sim \mathcal{N}(\mathbf{0},\bI_r)$ in Step~$1$ of Algorithm~\ref{alg:rounding},
\ie, $\bz = [\bv_1 \dots \bv_n]^\T \by$.
Note that Step~$1$ is equivalent to sampling $\bz$ from the Gaussian distribution $\mathcal{N}(\mathbf{0},\bX^\star)$,
which has a probabilistic interpretation~\cite{d2003relaxations,luo2010semidefinite}:
$\bX^\star$ is the optimal solution to the problem
\begin{align}
\min_{\Sigma} \quad &\mathbb{E}_{\bz \sim \mathcal{N}(\mathbf{0},\Sigma)}[ \bz^\T \bA \bz], \label{equ:prob_inter} \\
         \sst \quad %
                    &\mathbb{E}_{\bz \sim \mathcal{N}(\mathbf{0},\Sigma)}[\bz^\T \bB_i \bz] =    b_i, \, i \in \mathcal{I}_{eq}, \notag \\
                    &\mathbb{E}_{\bz \sim \mathcal{N}(\mathbf{0},\Sigma)}[\bz^\T \bB_i \bz] \leq b_i, \, i \in \mathcal{I}_{in}, \notag
\end{align}
where $\Sigma$ denotes a covariance matrix.
{
The proof is simple: since $\mathbb{E}_{\bz \sim  \mathcal{N}(\mathbf{0},\Sigma)} [ \bz^\T \bA \bz]
= \sum_{i,j} A_{ij} \mathbb{E}_{\bz \sim  \mathcal{N}(\mathbf{0},\Sigma)}[z_i z_j ]
= \sum_{i,j} A_{ij} \Sigma_{ij}$ for any $\bA \in \mathcal{S}^{n}$, \eqref{equ:prob_inter} is equivalent to \eqref{eq:psdp0}.
}
In other words, $\bz$ solves the BQP in expectation. %
As the eigen-decomposition of $\bX^\star$ is already known when computing $\Pi_{\mathcal{S}^n_+} (\bC(\bu^\star))$ at the last descent step,
there is no extra computation for obtaining $\bv_1 \dots \bv_n$.
Due to the low-rank structure of SDP solutions (see Section~\ref{sec:analysis1}),
the computational complexity of sampling $\bz$ is linear in the number of variables $n$.

{%
Note that the above random sampling procedure does not guarantee that a feasible solution can always be found.
In particular, this procedure will certainly fail when equality constraints are imposed on the problems~\cite{d2003relaxations}.
But for all the problems considered in this work, each random sample $\bz$ can be discretized to a ``nearby'' feasible solution (Step $2$ of Algorithm \ref{alg:rounding}).
The discretization step is problem dependant, which is discussed in Table~\ref{tab:formulation}.

}
}
\subsection{Speeding Up the Computation}
\label{sec:analysis1}

In this section, we discuss
several techniques  for the eigen-decompostion of $\bC(\bu)$,
which is one of the computational bottleneck for our algorithms.

\noindent
{\bf Low-rank Solution}
In our experiments, we observe that the final \psd solution typically has a low-rank structure
and $r = \mathrm{rank}(\Pi_{\mathcal{S}^n_+}(\bC(\bu)))$ usually decreases sharply such that $r \ll n$ for most of descent iterations
in both our algorithms.
Actually, it is known (see \cite{barvinok1995problems} and \cite{pataki1998rank})
that any SDP problem with $m$ linear constraints has an optimal solution $\bX^\star \in \mathcal{S}^n_+$,
such that $\mathrm{rank}(\bX^\star)(\mathrm{rank}(\bX^\star)+1)/2 \leq m$.
It means that the rank of $\bX^\star$ is roughly bounded by $\sqrt{2m}$.
Then Lanczos methods can be used to efficiently calculate the $r$ positive eigenvalues of $\bC(\bu)$ and the corresponding eigenvectors.
Lanczos methods rely only on the product of the matrix $\bC(\bu)$ and a column vector.
This simple interface allows us to exploit specific structures of the coefficient matrices $\bA$ and $\bB_i$, $i = 1, \cdots, m$.

\noindent
{\bf Specific Problem Structure}
In many cases, $\bA$ and $\bB_i$ are sparse or structured.
Such that the computational complexity and memory requirement of the matrix-vector product \wrt $\bC(\bu)$ can be considered as linear in $n$,
which are assumed as $\mathcal{O}(nt_1)$ and $\mathcal{O}(nt_2)$ respectively.
The iterative Lanczos methods are faster than standard eigensolvers when $r \ll n$ and $\bC(\bu)$ is sparse/structured,
which require $\mathcal{O}(nr^2+nt_1r)$ flops and $\mathcal{O}(nr+nt_2)$ bytes at each iteration of Lanczos factorization,
given that the number of Lanczos basis vectors is set to a small multiple ($1\!\sim\!3$) of $r$.
ARPACK \cite{lehoucq1997arpack}, an implementation of Lanczos algorithms,
is employed in this work for the eigen-decomposition of sparse or structured matrices.
The DSYEVR function in LAPACK~\cite{anderson1999lapack} is used for dense matrices.

\noindent
{\bf Warm Start}
A good initial point is crucial for the convergence speed of iterative Lanczos methods.
In quasi-Newton and smoothing Newton methods,
the step size $\Delta \bu = \bu_{k+1} - \bu_{k}$ tends to decrease with descent iterations.
It means that $\bC(\bu_{k+1})$ and $\bC(\bu_{k})$ may have similar eigenstructures,
which inspires us to use a random linear combination of eigenvectors of $\bC(\bu_{k})$
as the starting point of the Lanczos process for $\bC(\bu_{k+1})$.

\noindent
{\bf Parallelization}
Due to the importance of eigen-decomposition, its parallelization has been well studied and there are
several off-the-shelf parallel eigensolvers (such as SLEPc~\cite{Hernandez:2005:SSF}, PLASMA~\cite{PLASMA} and MAGMA~\cite{MAGMA}).
Therefore, our algorithms can also be easily parallelized by using these off-the-shelf parallel eigensolvers.

\subsection{Convergence Speed, Computational Complexity and Memory Requirement}
\label{sec:analysis}

\begin{table*}[t]
  \centering
  \scriptsize
{ %
  \begin{tabular}{l@{\hspace{0cm}}|@{\hspace{0.1cm}}c@{\hspace{0cm}}c@{\hspace{0.1cm}}c@{\hspace{0.1cm}}c}
  \hline
& &\\ [-2ex]
   Algorithms  & Convergence & Eigen-solver & Computational Complexity & Memory Requirement\\
& &\\ [-2ex]
  \hline
& &\\ [-2ex]
   \lbfgsb \begin{tabular}{@{\hspace{0cm}}c@{\hspace{0.1cm}}}
                         Dense \\
                         Sparse/Structured
                         \end{tabular}
   & unknown  & \begin{tabular}{c}
                         LAPACK-DSYEVR \\
                         ARPACK
                         \end{tabular}
             & \begin{tabular}{c}
                         $\mathcal{O}(m + n^3)$ \\
                         $\mathcal{O}(m) + \mathcal{O}(nr^2 + nt_1r) \times \#\mbox{Lanczos-iters}$
                         \end{tabular}
                       & \begin{tabular}{c}
                         $\mathcal{O}(m+n^2)$\\
                         $\mathcal{O}(m+nr+nt_2)$
                         \end{tabular}\\
& &\\ [-2ex]
  \hline
& &\\ [-2ex]
   \smooth & quadratic & LAPACK-DSYEVR & $\mathcal{O}(n^3)+ \mathcal{O}(m+n^2r)\times \#\text{CG-iters}$ & $\mathcal{O}(m+n^2)$ \\
& &\\ [-2ex]
  \hline
& &\\ [-2ex]
   Interior Point Methods    & quadratic  & $-$ & $\mathcal{O}(m^3+mn^3+m^2n^2)$ & $\mathcal{O}(m^2+n^2)$\\
  \hline
  \end{tabular}
}
\caption{
{ %
The comparison of our algorithms and interior-point algorithms on convergence rate, computational complexity and memory requirement.
\lbfgsb is considered in two cases: the matrix $\bC(\bu)$ is dense or sparse/structured and different eigen-solvers are applied. $n$ and $m$ denotes the primal \psd matrix size and the number of dual variables.
The definition of $r$, $t_1$ and $t_2$ can be found in Section~\ref{sec:analysis1}.
}}
\label{tab:complexity}
\end{table*}

\noindent
{\bf \lbfgsb}
In general, quasi-Newton methods converge superlinearly given that the objective function is at least twice differentiable
(see \cite{broyden1973local,dennis1974characterization,qi1997superlinear}).
However, the dual objective function in our case~\eqref{eq:fastsdp_dual} is not necessarily twice differentiable.
So {\em the
theoretical
convergence speed of \lbfgsb is unknown}.

At each iteration of \LBFGSB, both of the computational complexity and memory requirement of \LBFGSB itself are $\mathcal{O}(m)$.
The only computational bottleneck of \lbfgsb is on the computation of the projection
$\Pi_{\mathcal{S}^n_+}(\bC(\bu))$, which is discussed in Section~\ref{sec:analysis1}.

\noindent
{\bf \smooth}
	The inexact smoothing Newton method \smooth is quadratically convergent under the assumption that the constraint nondegenerate condition holds at the optimal solution (see \cite{gao2009calibrating}).
There are two computationally intensive aspects of \smooth: $i$). the
CG algorithms for solving the linear system~\eqref{eq:cg_equation}.
In the appendix,
we show
that the Jacobian-vector product requires
$\mathcal{O}(m+n^2r)$
flops at each CG iteration, where $r = \mathrm{rank}(\Pi_{\mathcal{S}^n_+}(\bC(\bu)))$.
$ii$). All eigenpairs of $\bC(\bu)$ are needed to obtain Jacobian matrices implicitly, which takes $\mathcal{O}(n^3)$ flops using DSYEVR function in LAPACK.

From Table~\ref{tab:complexity}, we can see that the computational costs and memory requirements for both \lbfgsb and \smooth are linear in $m$,
which means that our methods are much more scalable to $m$ than interior-point methods.
In terms of $n$, our methods is also more scalable than interior-point methods and comparable to spectral methods.
Especially for sparse/structured matrices, the computational complexity of \lbfgsb is linear in $n$.
As \smooth cannot significantly benefit from sparse/structured matrices, it needs more time than \lbfgsb in each descent iteration for such matrices.
However, \smooth has a fast convergence rate than \lbfgsb.
In the experiment section, we compare the speeds of \smooth and \lbfgsb in different cases.

\section{Applications}
\label{sec:app}

\begin{table*}[t]
  \centering
  \scriptsize
{ %
  \begin{tabular}{p{1.5cm}|L{6cm}|p{9.3cm}}
  \hline
     Application & BQP formulation  & Comments\\
  \hline
  \hline
& &\\ [-2ex]
Graph bisection

(Sec.~\ref{sec:graph_bisection})
&
\vspace{-0.6cm}
{
\begin{subequations}
\label{eqa:app_gb}
\begin{flalign}
&\min_{\bx \in \{-1,1\}^n} \,\,- \bx^{\T} \bW \bx, &\\
&\quad\quad\,\,\sst \quad\,\,\, \bx^{\T} \mathbf{1} = 0. &
\end{flalign}
\end{subequations}
}
\vspace{-0.5cm}
&
$W_{ij} \!=\! { \left\{
                \begin{array}{ll}
                \! \exp(-\mathrm{d}_{ij}^2 / \sigma^2)  & \!\mbox{if } (i,j) \in \mathcal{E},\\
                \!0 &\! \mbox{otherwise,}  \end{array}
         \right.} $
where
$\mathrm{d}_{ij}$ denotes the Euclidean distance between $i$ and $j$.

{\em Discretization}: $\bx = \mathrm{sign}(\bz - \mathrm{median}(\bz))$.
\\
& &\\ [-2ex]
\hline
& &\\ [-2ex]
Image segmentation with partial grouping constraints

(Sec.~\ref{sec:segm})
&
\vspace{-0.6cm}
{
{%
\begin{subequations}
\label{equ:app_segm1}
\begin{flalign}
&\min_{\bx \in \{-1,1\}^n} \,\, - \bx^{\T} \bW \bx,  &\\
&\quad\quad\,\, \sst \quad\,\,    (\bs_f^{\T} \bx)^2  \geq  \kappa^2,  &\\
&\quad\quad\,\,\quad\quad\quad    (\bs_b^{\T} \bx)^2  \geq  \kappa^2,  &\\
&\quad\quad\,\,\quad\quad\,\,\,    - \frac{1}{2} \bx^\T (\bs_f \bs_b^\T + \bs_b \bs_f^\T) \bx \geq \kappa^2. &
\end{flalign}
\end{subequations}
}
}
&
{%
$W_{ij} \!=\! { \left\{
                \begin{array}{ll}
                \!\mathrm{exp} \left(- \lVert \mathbf{f}_i - \mathbf{f}_j \rVert_2^2 / \sigma_f^2
                               - \mathrm{d}_{ij}^2 / \sigma_d^2 \right) & \!\mbox{if }   \mathrm{d}_{ij} \!<\! r,\\
                \!0 &\! \mbox{otherwise,}  \end{array}
         \right.} $
where $\mathbf{f}_i$ denotes the local feature of pixel $i$.
The weighted partial grouping pixels are defined as $\bs_f = \bP \bt_f / (\mathbf{1}^\T \bP \bt_f) $
and $\bs_b = \bP \bt_{b} / (\mathbf{1}^\T \bP \bt_b)$ for foreground and background respectively,
where
$\bt_f, \bt_b \in \{0,1\}^n$ are two indicator vectors for manually labelled pixels and
$\bP \!=\! \mathrm{Diag}(\bW \mathbf{1})^{-1}\!\bW$ is the normalized affinity matrix used as smoothing terms~\cite{Yu2004segmentation}.
The overlapped non-zero elements between $\bs_f$ and $\bs_b$ are removed.
$\kappa \in (0,1]$ denotes the degree of belief.

{\em Discretization}: see \eqref{eq:round_partial}. %
}
\\
& &\\ [-2ex]
\hline
& &\\ [-2ex]
Image segmentation with histogram constraints

(Sec.~\ref{sec:segm})
&
\vspace{-0.4cm}
{ %
\begin{subequations}
\label{eq:app_hist}
\begin{flalign}
&\min_{\bx \in \{-1,1\}^n} \,\, - \bx^{\T} \bW \bx, &\\
&\quad\quad\,\, \sst \,\,\,\,\,  \sum_{i=1}^{K} \left(\frac{\langle \bt_{i}, \bx+\mathbf{1} \rangle}{\langle \mathbf{1} , \bx+\mathbf{1} \rangle} - q_i \right)^2 \leq \delta^2, & \label{eq:hist_cons} \\
&\quad\quad\quad\quad\quad (\bx^\T \mathbf{1})^2 \leq \kappa^2 n^2.  & \label{eq:hist_cons2}
\end{flalign}
\end{subequations}
}
\vspace{-0.5cm}
&
{ %
$\bW$ is the affinity matrix as defined above.
$\bq$ is the target $K$-bin color histogram;
$\bt_i \in \{0,1\}^n$ is the indicator vector for every color bin;
$\delta$ is the prescribed upper-bound on the Euclidean distance between the obtained histogram and $\bq$.
Note that \eqref{eq:hist_cons} is equivalent to a quadratic constraint on $\bx$ and can be expressed as $\bx \bB \bx + \ba^\T \bx \leq b$.
Constraint \eqref{eq:hist_cons} is penalized in the objective function with a weight (multiplier) $\alpha > 0$ in this work:
$\min_{\bx \in \{-\!1\!,\!1\}^{\!n}} \bx^{\T} (\alpha \bB - \bW) \bx + \alpha \ba^\T \bx, \, \sst \, \eqref{eq:hist_cons2}$.
Constraint \eqref{eq:hist_cons2} is used to avoid trivial solutions.

{\em Discretization}: $\bx = \mathrm{sign}(\bz - \theta)$. See \eqref{eq:hist-round} for the computation of the threshold $\theta$. %
}
\\
& &\\ [-2ex]
\hline
& &\\ [-2ex]
Image
co-segmentation

(Sec.~\ref{sec:cosegm})
&
\vspace{-0.4cm}
{%
\begin{subequations}
\begin{flalign}
&\min_{\bx \in \{-1,1\}^n} \quad\bx^{\T}\! \bA \bx,  &\\
&\quad\quad\,\, \sst  \quad\,\, ( \bx^{\T} \bt_i)^2 \leq \kappa^2 n^2_i,  \ i = 1 ,\dots, s. &
\end{flalign}
\end{subequations}
}
\vspace{-0.4cm}
&
{%
The definition of $\bA$ can be found in \cite{Joulin2010dis}.
$s$ is the number of images, $n_i$ is the number of pixels for $i$-th image, and $n = \sum_{i=1}^{s} n_i$.
$\bt_i \in \{0,1\}^n$ is the indicator vector for the $i$-th image. $\kappa \in (0,1]$.

{\em Discretization}: see \eqref{eq:cosegm-round}.
}
\\
& &\\ [-2ex]
\hline
& &\\ [-2ex]
Graph

matching

(Sec.~\ref{sec:matching})
&
\vspace{-0.4cm}
{
\begin{subequations}
\begin{flalign}
&\min_{\bx \in \{0,1\}^{KL} } \,\,\,\, \bh^{\T} \bx + \bx^{\T} \bH \bx,  &\\
&\quad\quad \sst  \,\,   \textstyle{\sum_{j=1}^L} \bx_{\tiny(i-1)\!L+j}=1, i \!=\! 1,\! \dots,\! K,  &\\
&             \quad\quad\quad\quad  \textstyle{\sum_{i=1}^K} \bx_{\tiny(i-1)\!L+j} \leq 1, j \!=\! 1,\! \dots,\! L. &
\end{flalign}
\end{subequations}
}
\vspace{-0.4cm}
&
$x_{\tiny(i-1)\!L+j}\!=\!1$ if the $i$-th source point is matched to the $j$-th target point; otherwise it equals to $0$.
$h_{\tiny(i-1)\!L+j}$ records the local feature similarity between source point $i$ and target point $j$;
$H_{\tiny(i-1)\!L+j,(k-1)\!L+l} = \mathrm{exp} (-(\mathrm{d}_{ij} - \mathrm{d}_{kl})^2 / \sigma^2)$
encodes the structural consistency of source point $i$, $j$ and target point $k$,~$l$.
See \cite{Schellewald05} for details.

{\em Discretization}: see \eqref{eq:graph_match_round}.
\\
& &\\ [-2ex]
\hline
& &\\ [-2ex]
Image

deconvolution

(Sec.~\ref{sec:deconv})
&
\vspace{-0.4cm}
{%
\begin{flalign}
\label{eq:deconv}
&\,\, \min_{\bx \in \{0,1\}^n} \quad\,\, \lVert \bq - \bK \bx \rVert^2_2 + \mathrm{S}(\bx). &
\end{flalign}
}
&
{%
$\bK$ is the convolution matrix corresponding to the blurring kernel $\bk$;
$\mathrm{S}$ denotes the smoothness cost;
$\bx$ and $\bq$ represent the input image and the blurred image respectively.
See \cite{raj2005graph} for details.

{\em Discretization}: $\bx = (\mathrm{sign}(\bz) + \mathbf{1}) / 2$.
}
\\
& &\\ [-2ex]
\hline
& &\\ [-2ex]
Chinese

character

inpainting

(Sec.~\ref{sec:dtf})
&
\vspace{-0.4cm}
{
\begin{flalign}
& \min_{\bx \in \{-1,1\}^n } \quad \bh^{\T} \bx + \bx^{\T} \bH \bx. &
\end{flalign}
}
&
The unary terms ($\bh \in \mathbb{R}^n$) and pairwise terms ($\bH \in \mathbb{R}^{n \times n}$)
are learned using decision tree fields~\cite{nowozin2011decision}.

{\em Discretization}: $\bx = \mathrm{sign}(\bz)$.
\\
\hline
\end{tabular}
\caption{ BQP formulations for different applications considered in this paper.
          The discretization step in Algorithm~\ref{alg:rounding} for each application is also described.
        }
\label{tab:formulation}
}
\end{table*}

We now show how we can attain good solutions on various vision tasks with the proposed methods.
The two proposed methods,  \lbfgsb and \smooth, are evaluated on several computer vision applications.
The BQP formulation of different applications and the corresponding rounding heuristics are demonstrated in Table~\ref{tab:formulation}.
The corresponding SDP relaxation can be obtained based on \eqref{eq:backgd_sdp1}.
In the experiments, we also compare our methods to spectral methods~\cite{Shi2000normalized,Yu2004segmentation,Cour_solvingmarkov,Gour2006},
graph cuts based methods~\cite{kolmogorov2004energy,kolmogorov2007minimizing,rother2007optimizing}
and interior-point based SDP methods~\cite{Sturm98usingsedumi,Toh99sdpt3,mosek}.
The upper-bounds (\ie, the objective value of BQP solutions) and the lower-bounds (on the optimal objective value of BQPs)
achieved by different methods are demonstrated, and the runtimes are also compared.

The code is written in Matlab, with some key subroutines
implemented in C/MEX.
We have used the \LBFGSB~\cite{Zhu94lbfgsb} for the optimization in \lbfgsb.
All of the experiments are evaluated on a core of Intel Xeon E$5$-$2680$ $2.7$GHz CPU ($20$MB cache).
The maximum number of descent iterations of \lbfgsb and \smooth are set to $50$ and $500$ respectively.
As shown in Algorithm~\ref{alg:lbfgsb} and Algorithm~\ref{alg:smooth},
the same stopping criterion is used for \lbfgsb and \smooth,
and the tolerance $\tau$ is set to $10^7 \mathrm{eps}$ where $\mathrm{eps}$ is the machine precision.
The initial values of the dual variables $u_i, i \in \mathcal{I}_{eq}$ are set to $0$,
and $u_i, i \in \mathcal{I}_{in}$ are set to a small positive number.
The selection of parameter $\gamma$ will be discussed in the next section.

\subsection{Graph Bisection}
\label{sec:graph_bisection}

Graph bisection is a problem of separating the nodes of a weighted graph $\mathcal{G} = (\mathcal{V}, \mathcal{E})$ into two disjoint sets with equal cardinality,
while minimizing the total weights of the edges being cut. $\mathcal{V}$ denotes the set of nodes and $\mathcal{E}$ denotes the set of non-zero edges.
The BQP formulation of graph bisection can be found in \eqref{eqa:app_gb} of Table~\ref{tab:formulation}.
To enforce the feasibility (two partitions with equal size), the randomized score vector $\bz$ in Algorithm~\ref{alg:rounding}
is dicretized by thresholding the median value (see Table~\ref{tab:formulation}).

\begin{figure}[t]
\begin{minipage}[c]{0.48\textwidth}
\centering
\subfloat{
\includegraphics[width=0.24\textwidth]{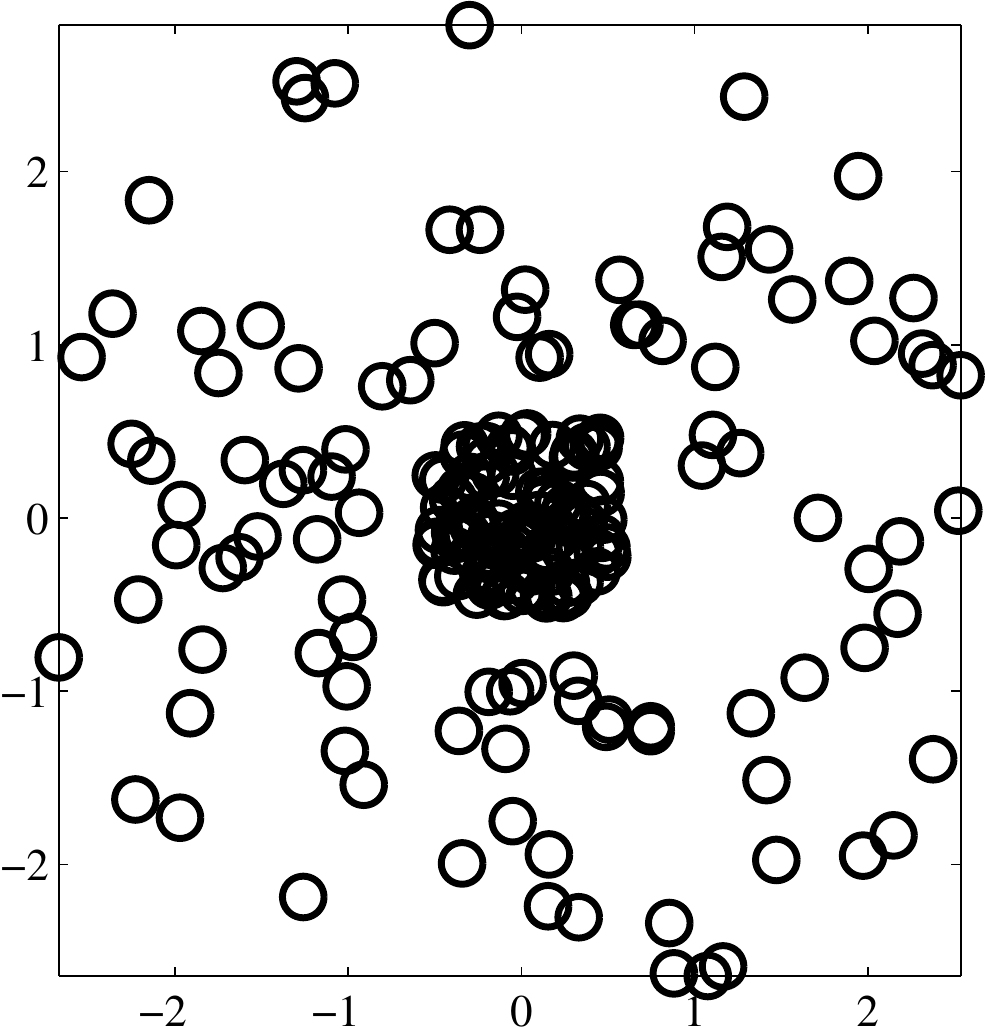}
\includegraphics[width=0.24\textwidth]{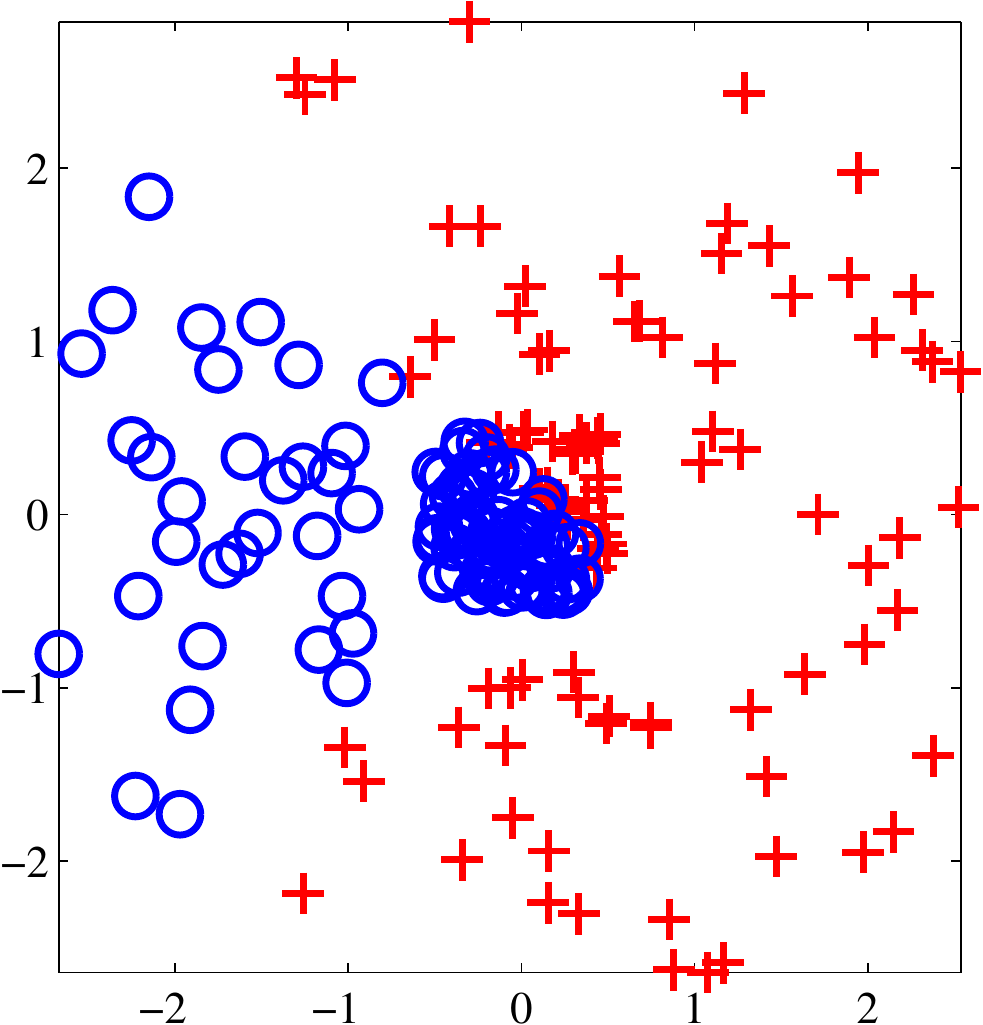}
\includegraphics[width=0.24\textwidth]{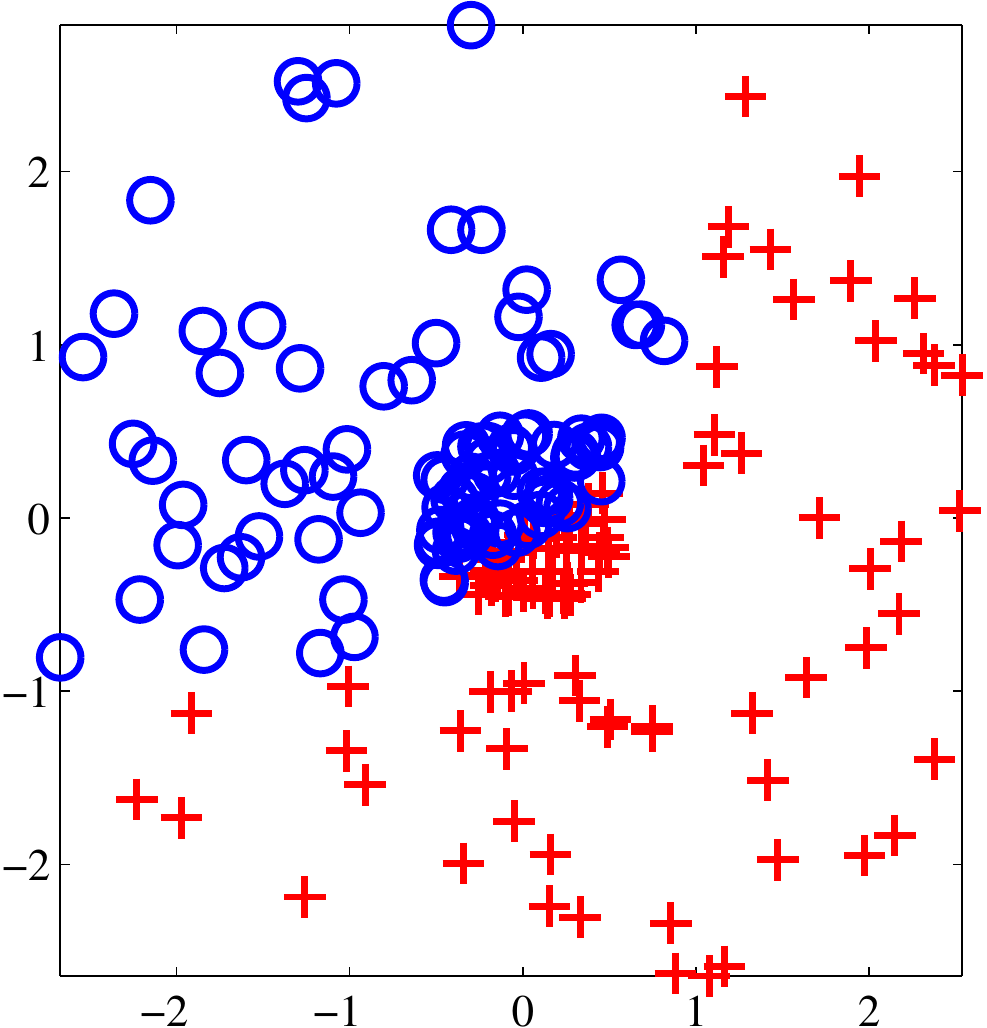}
\includegraphics[width=0.24\textwidth]{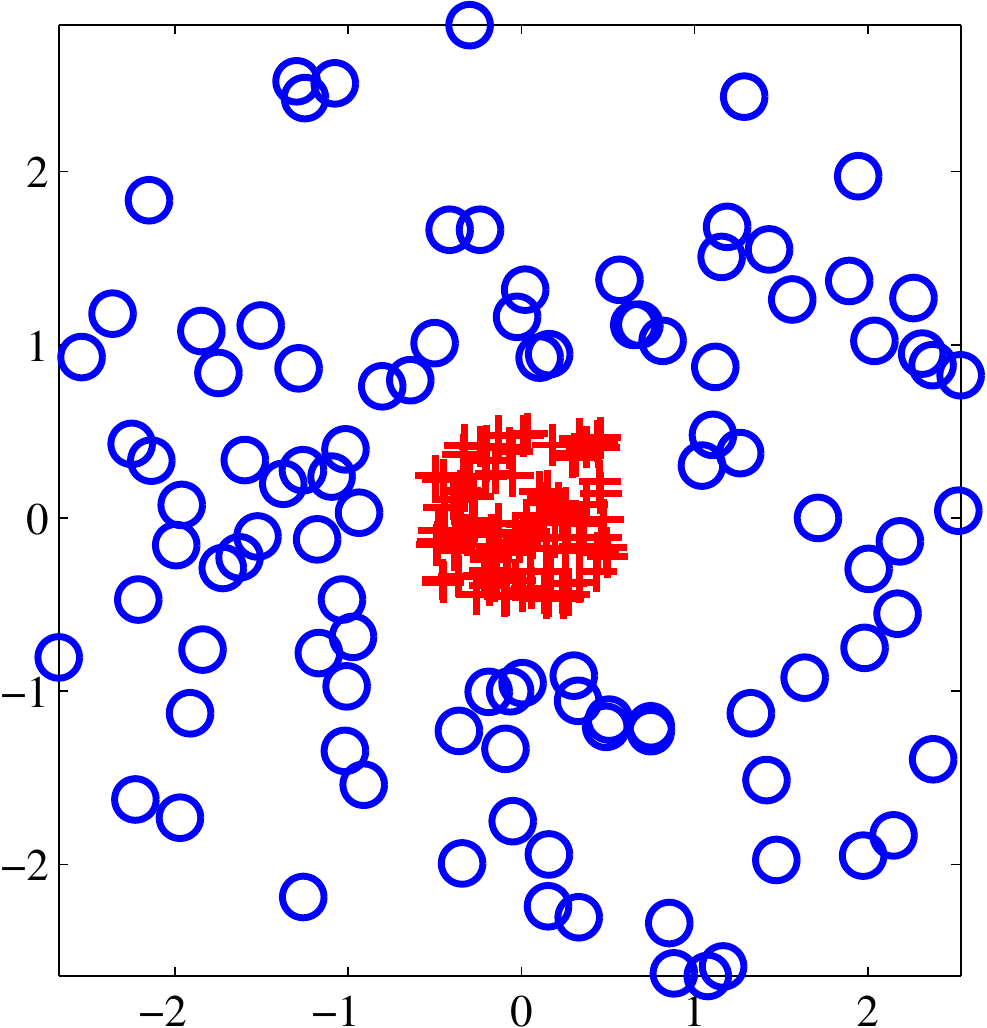}
\centering
}\\
\subfloat{
\centering
\includegraphics[width=0.24\textwidth]{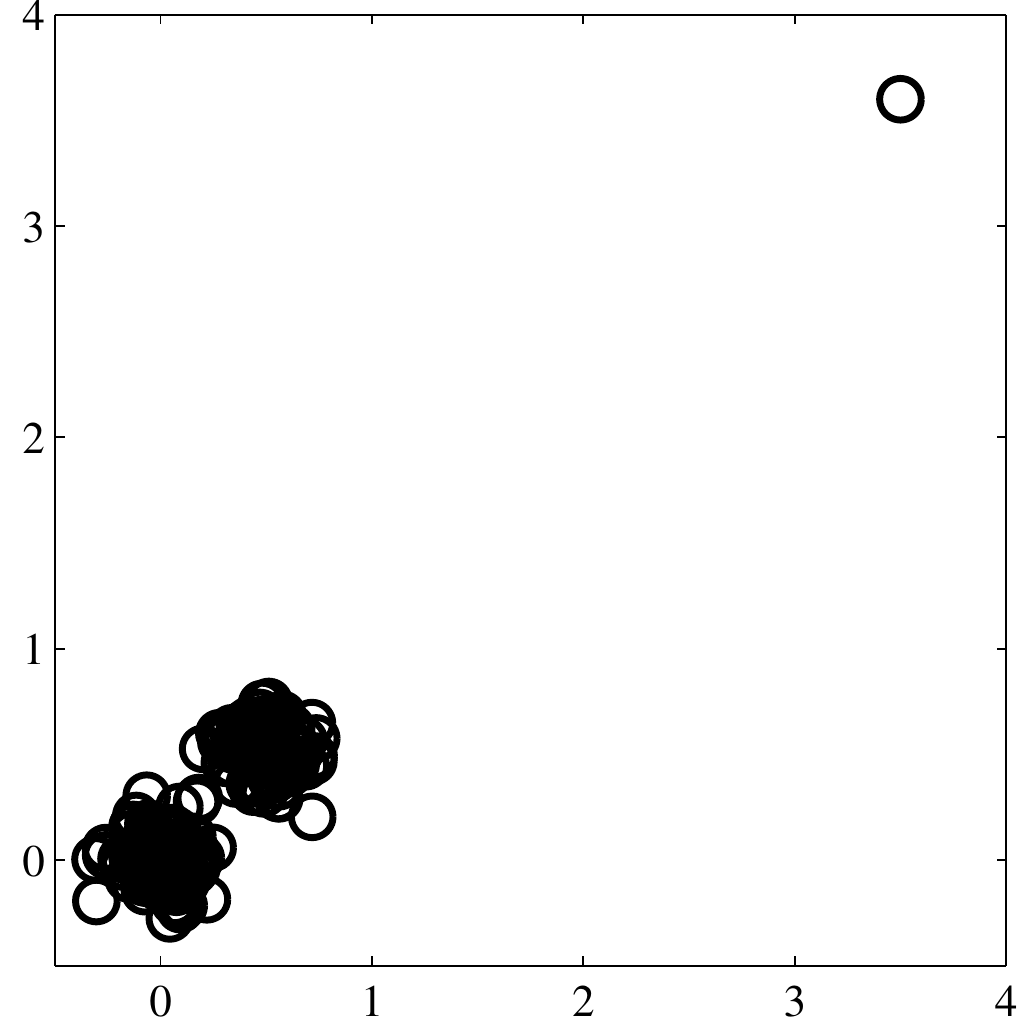}
\includegraphics[width=0.24\textwidth]{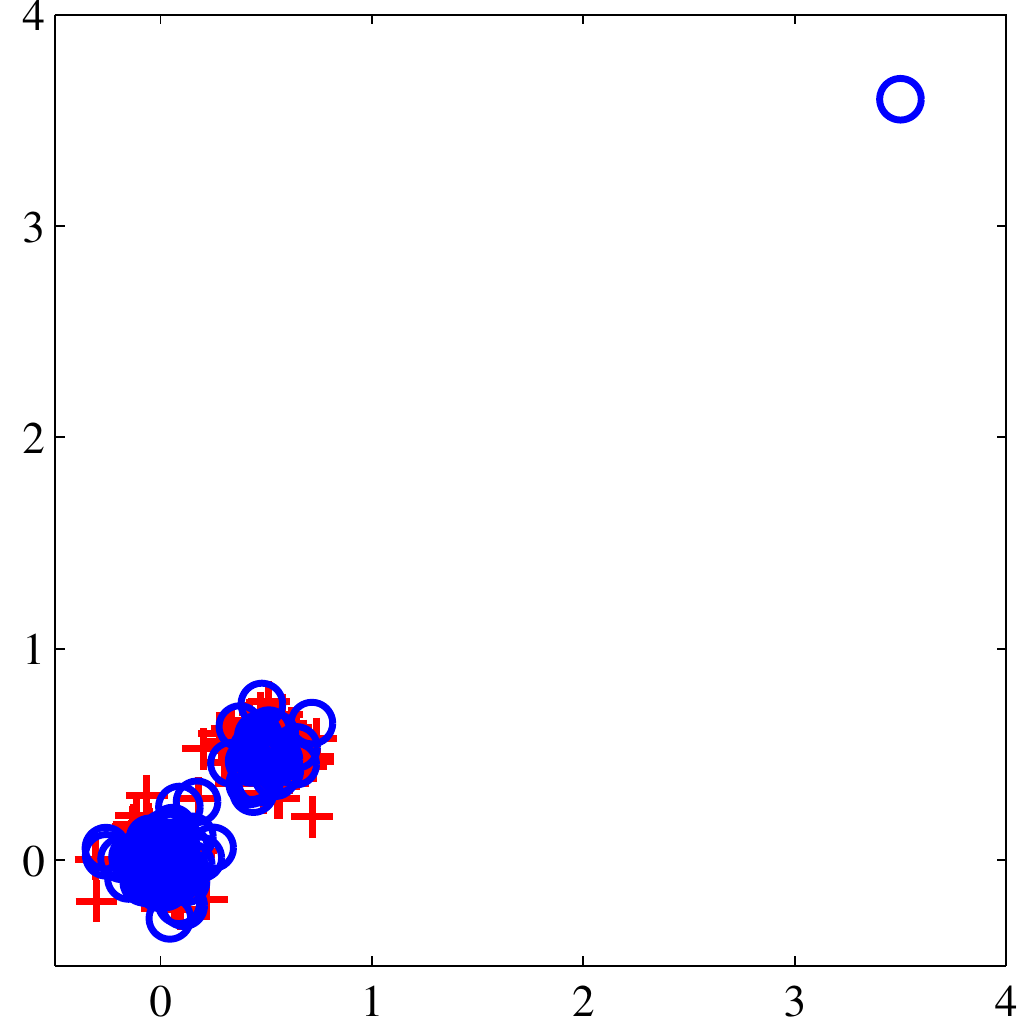}
\includegraphics[width=0.24\textwidth]{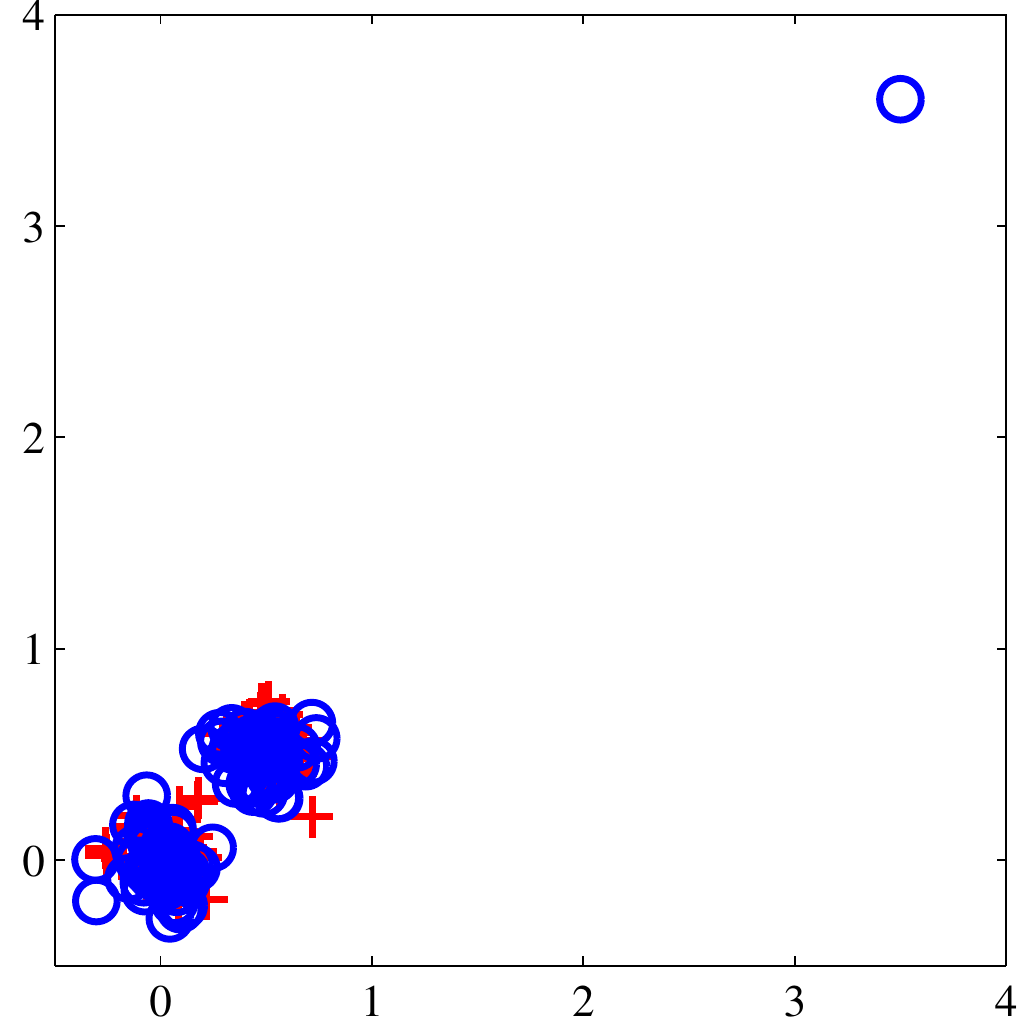}
\includegraphics[width=0.24\textwidth]{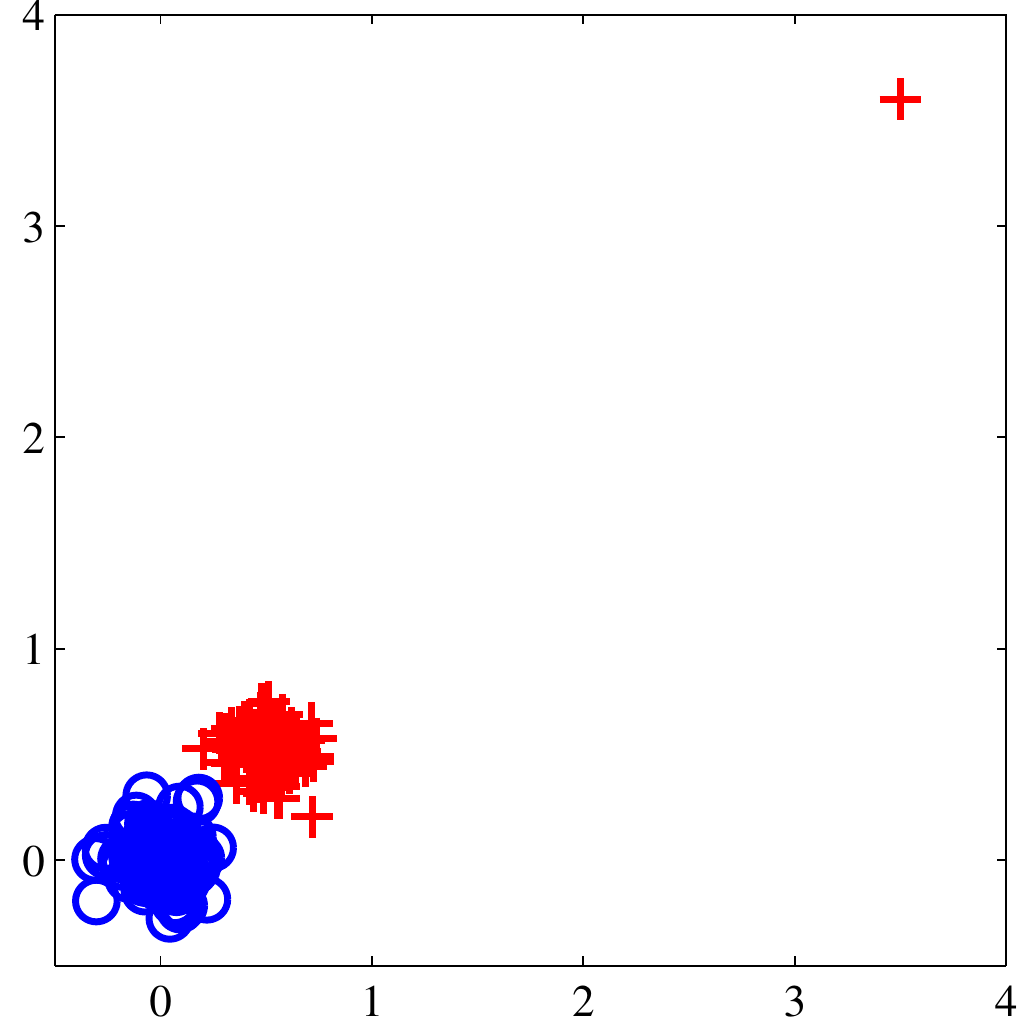}
}\\
\subfloat{
\centering
\begin{minipage}[c]{0.24\textwidth}
\centering
{\scriptsize Original data}
\end{minipage}
\begin{minipage}[c]{0.24\textwidth}
\centering
{\scriptsize NCut}
\end{minipage}
\begin{minipage}[c]{0.24\textwidth}
\centering
{\scriptsize RatioCut}
\end{minipage}
\begin{minipage}[c]{0.24\textwidth}
\centering
{\scriptsize \lbfgsb}
\end{minipage}
}
\end{minipage}
\caption{Results for $2$-demensional points bisection. %
The resulting two classes
of points are shown in red  `+'  and blue `$ \circ $' respectively.
\lbfgsb succeeds in clustering the points as desired,
while both RatioCut and NCut failed in these two cases.}
\label{fig:2d-cluster}
\end{figure}

To show that the proposed SDP methods have better solution quality than spectral methods
we compare the graph-bisection results of RatioCut~\cite{hagen1992new}, Normalized Cut (NCut)~\cite{Shi2000normalized} and \lbfgsb %
on two artificial 2-dimensional datasets.

As shown in Fig.~\ref{fig:2d-cluster}, the first data set (the first row) contains two sets of points with different densities,
and the second set contains an outlier.
RatioCut and NCut fail to offer satisfactory results on both of the
data sets, possibly due to the poor approximation of spectral relaxation.
In contrast, our \lbfgsb achieves desired results on these data sets.

Secondly, to demonstrate the impact of the parameter~\parat,
we test \lbfgsb and \smooth on a random graph with \parat ranging from $10^2$ to $10^4$
({%
$\bA$ and $\bB_i$ in \eqref{eq:psdp0} are scaled such that $\lVert \bA \rVert_F^2 = \lVert \bB \rVert_F^2 = 1$}).
The graph is generated with $1000$ vertices and all possible edges are assigned a non-zero weight uniformly sampled from $(0,1]$.
As the resulting affinity matrices are dense, the DSYEVR routine in LAPACK package is used for eigen-decomposition.
In Fig.~\ref{fig:graph_bisection_t}, we show the upper-bounds, lower-bounds, number of iterations and time achieved by \lbfgsb and \smooth,
\wrt different values of {\parat}.
There are several observations:
$i$) With the increase of \parat, upper-bounds become smaller and lower-bounds become larger,
which implies a tighter relaxation.
$ii$) Both \lbfgsb and \smooth take more iterations to converge when $\gamma$ is larger.
$iii$) \smooth uses fewer iterations than \lbfgsb.
The above observations coincide with the analysis in Section~\ref{sec:analysis}.
Using a larger parameter \parat yields better solution quality, but at the cost of slower convergence speed.
{%
The choice of a good \parat is data dependant.
To reduce the difficulty of the choice of $\gamma$, the matrices $\bA$ and $\bB_i$ of Equation~\eqref{eq:psdp0}
are scaled such that the Frobenius norm is $1$ in the following experiments.
}

Thirdly, experiments are performed to evaluate another two factors affecting the speed of our methods:
the sparsity of the affinity matrix $\bW$ and the matrix size $n$.
The numerical results corresponding to dense and sparse affinity matrices
are shown in Table~\ref{tab:graph_bisection_dense} and Table~\ref{tab:graph_bisection_sparse} respectively.
The sparse affinity matrices are generated from random graphs with $8$-neighbour connection.
In these experiments, the size of matrix $\bW$ is varied from $200$ to $5000$.
ARPACK is used by \lbfgsb for partial eigen-decomposition of sparse problems,
and DSYEVR is used for other cases.
For both \lbfgsb and \smooth, the number of iterations does not grow significantly with the increase of $n$.
However, the running time is still correlated with $n$,
since an eigen-decompostion of an $n \times n$ matrix needs to be computed at each iteration for both of our methods.
We  also find that the second-order method \smooth
uses significantly fewer iterations than the first-order method \lbfgsb.
For dense affinity matrices, \smooth runs consistently faster than \lbfgsb.
In contrast for sparse affinity matrices, \smooth is only faster than \lbfgsb
on problems of size up to $n \geq 2000$.
That is because the Lanczos method used by \lbfgsb (for partial eigen-decompostion)
scales much better for large sparse matrices than the standard factorization method (DSYEVR) used by \smooth (for full eigen-decomposition).
The upper-/lower-bounds yielded by our methods are similar to those of the interior-point methods.
Meanwhile, NCut and RatioCut run much faster than other methods, but offer significantly worse upper-bounds.

{ %
Finally, we evaluate \lbfgsb on a large dense graph with $10000$ nodes.
The speed performance is compared on a single CPU core (using DSYEVR function of LAPACK as eigensolver)
and a hybrid CPU+GPU workstation (using the DSYEVDX$\_2$STAGE function of MAGMA as eigensolver).
The results are shown in Table~\ref{tab:graph_bisection_large} and we can see that the parallelization brings a $10$-fold speedup over running on a single CPU core.
The lower-/upper-bounds are almost identical as there is no difference apart from the implementation of eigen-decompostion.
}

\begin{table*}[t]
\begin{minipage}[t]{0.74\textwidth}
  \centering
  \scriptsize
  \begin{tabular}{l|l|ccccc|cc}
  \hline
& & & & & & & &\\ [-2ex]
     $n$, $m$  & Methods  & \lbfgsb  & \smooth & SeDuMi & SDPT3 & MOSEK & NCut & RatioCut\\
  \hline
  \hline
\multirow{3}{*}{\begin{tabular}{c} $200$, \\ $201$ \end{tabular}}
& & & & & & & &\\ [-2ex]
& Time/Iters
& $0.7$s/$67.7$	& $\mathbf{0.6}\bf{s}$/$\mathbf{11.0}$	& $10.4$s	& $7.0$s	& $5.5$s	& $0.2$s	& $0.2$s	\\
& Upper-bound
& $1.03$	& $1.04$	& $1.04$	& $1.03$	& $1.04$	& $1.82$	& $4.61$	\\
& Lower-bound
& $-0.63$	& $-0.63$	& $-0.58$	& $-0.58$	& $-0.58$	& \NA	& \NA	\\
& & & & & & & &\\ [-2ex]
\hline\multirow{3}{*}{\begin{tabular}{c} $500$, \\ $501$ \end{tabular}}
& & & & & & & &\\ [-2ex]
& Time/Iters
& $1.9$s/$43.2$	& $\mathbf{1.8}\bf{s}$/$\mathbf{9.7}$	& $01$m$21$s	& $33.9$s	& $36.0$s	& $0.3$s	& $0.4$s	\\
& Upper-bound
& $2.94$	& $2.96$	& $2.93$	& $2.92$	& $2.93$	& $4.01$	& $9.23$	\\
& Lower-bound
& $-0.31$	& $-0.31$	& $-0.20$	& $-0.20$	& $-0.20$	& \NA	& \NA	\\
& & & & & & & &\\ [-2ex]
\hline\multirow{3}{*}{\begin{tabular}{c} $1000$, \\ $1001$ \end{tabular}}
& & & & & & & &\\ [-2ex]
& Time/Iters
& $22.6$s/$39.9$	& $\mathbf{13.0}\bf{s}$/$\mathbf{9.0}$	& $08$m$21$s	& \NAA %
                                                                                              & $02$m$36$s	& $0.5$s	& $0.9$s	\\
& Upper-bound
& $5.06$	& $5.10$	& $5.07$	& \NAA %
                                                                & $5.04$	& $6.10$	& $13.28$	\\
& Lower-bound
& $-0.19$	& $-0.19$	& $0.02$	& \NAA %
                                                                & $0.02$	& \NA	& \NA	\\
& & & & & & & &\\ [-2ex]
\hline\multirow{3}{*}{\begin{tabular}{c} $2000$, \\ $2001$ \end{tabular}}
& & & & & & & &\\ [-2ex]
& Time/Iters
& $01$m$54$s/$34.9$	& $\mathbf{54.3}\bf{s}$/$\mathbf{9.0}$	& $55$m$45$s	& \NAA %
                                                                                               & $22$m$25$s	& $2.1$s	& $2.9$s	\\
& Upper-bound
& $8.02$	& $7.99$	& $7.94$	& \NAA %
                                                                & $7.95$	& $9.00$	& $20.85$	\\
& Lower-bound
& $-0.18$	& $-0.18$	& $0.21$	& \NAA %
                                                                & $0.21$	& \NA	& \NA	\\
& & & & & & & &\\ [-2ex]
\hline\multirow{3}{*}{\begin{tabular}{c} $5000$, \\ $5001$ \end{tabular}}
& & & & & & & &\\ [-2ex]
& Time/Iters
& $20$m$39$s/$27.1$	& $\mathbf{11}$\bf{m}$\mathbf{05}$\bf{s}/$\mathbf{8.1}$	& $14$h$55$m	& \NAA %
                                                                                                              & $04$h$40$m	& $24.2$s	& $15.4$s	\\
& Upper-bound
& $13.89$	& $13.87$	& $13.78$	& \NAA %
                                                                & $15.60$	& $14.91$	& $33.46$	\\
& Lower-bound
& $-0.32$	& $-0.32$	& $0.51$	& \NAA %
                                                                & $2.66$	& \NA	& \NA	\\
\hline
\end{tabular}
\end{minipage}\hfill
\begin{minipage}[t]{0.23\textwidth}
\vspace{-27mm}
  \caption{Numerical results for graph bisection with dense affinity matrices.
           All the results are the average over $10$ random graphs.
           SDP based methods (the left five columns) achieve better upper-bounds than spectral methods (NCut and RatioCut).
           \smooth uses fewer iterations than \lbfgsb and achieves the fastest speed of the five SDP based methods.
           \NAA~ denotes the cases where SDPT3 fails to output feasible solutions.
          }
\label{tab:graph_bisection_dense}
\end{minipage}
\end{table*}

\begin{table*}[t]
\begin{minipage}[t]{0.74\textwidth}
  \centering
  \scriptsize
  \begin{tabular}{l|l|ccccc|cc}
  \hline
& & & & & & & &\\ [-2ex]
     $n$, $m$  & Methods  & \lbfgsb  & \smooth & SeDuMi & SDPT3 & MOSEK & NCut & RatioCut\\
  \hline
  \hline
\multirow{3}{*}{\begin{tabular}{c} $200$, \\ $201$ \end{tabular}}
& & & & & & & &\\ [-2ex]
& Time/Iters
& $6.0$s/$76.5$	& $\mathbf{0.6}\bf{s}$/$\mathbf{11.0}$	& $9.8$s	& $7.3$s	& $3.5$s	& $0.1$s	& $0.1$s	\\
& Upper-bound
& $-0.57$	& $-0.57$	& $-0.57$	& $-0.57$	& $-0.57$	& $8.38$	& $-0.48$	\\
& Lower-bound
& $-1.32$	& $-1.32$	& $-1.28$	& $-1.28$	& $-1.28$	& \NA	& \NA	\\
& & & & & & & &\\ [-2ex]
\hline\multirow{3}{*}{\begin{tabular}{c} $500$, \\ $501$ \end{tabular}}
& & & & & & & &\\ [-2ex]
& Time/Iters
& $12.3$s/$65.3$	& $\mathbf{3.1}\bf{s}$/$\mathbf{11.0}$	& $01$m$36$s	& $54.0$s	& $40.5$s	& $0.1$s	& $0.2$s	\\
& Upper-bound
& $0.65$	& $0.64$	& $0.65$	& $0.64$	& $0.64$	& $19.20$	& $0.73$	\\
& Lower-bound
& $-0.41$	& $-0.41$	& $-0.30$	& $-0.30$	& $-0.30$	& \NA	& \NA	\\
& & & & & & & &\\ [-2ex]
\hline\multirow{3}{*}{\begin{tabular}{c} $1000$, \\ $1001$ \end{tabular}}
& & & & & & & &\\ [-2ex]
& Time/Iters
& $28.5$s/$73.3$	& $\mathbf{24.0}\bf{s}$/$\mathbf{11.8}$	& $11$m$36$s	& \NAA %
                                                                                              & $02$m$43$s	& $0.1$s	& $0.3$s	\\
& Upper-bound
& $1.35$	& $1.35$	& $1.35$	& \NAA %
                                                                & $1.34$	& $28.32$	& $1.41$	\\
& Lower-bound
& $0.25$	& $0.25$	& $0.46$	& \NAA %
                                                                & $0.46$	& \NA	& \NA	\\
& & & & & & & &\\ [-2ex]
\hline\multirow{3}{*}{\begin{tabular}{c} $2000$, \\ $2001$ \end{tabular}}
& & & & & & & &\\ [-2ex]
& Time/Iters
& $\mathbf{01}$\bf{m}$\mathbf{12}$\bf{s}/$72.5$	& $02$m$38$s/$\mathbf{12.5}$	& $42$m$19$s	& \NAA %
                                                                                                                                & $23$m$12$s	& $0.3$s	& $0.5$s	\\
& Upper-bound
& $2.43$	& $2.43$	& $2.41$	& \NAA %
                                                                & $2.41$	& $41.18$	& $2.51$	\\
& Lower-bound
& $1.01$	& $1.01$	& $1.40$	& \NAA %
                                                                & $1.40$	& \NA	& \NA	\\
& & & & & & & &\\ [-2ex]
\hline\multirow{3}{*}{\begin{tabular}{c} $5000$, \\ $5001$ \end{tabular}}
& & & & & & & &\\ [-2ex]
& Time/Iters
& $\mathbf{04}$\bf{m}$\mathbf{43}$\bf{s}/$90.3$	& $26$m$19$s/$\mathbf{13.2}$	& $15$h$48$m	& \NAA %
                & $05$h$18$m	& $1.2$s	& $0.9$s	\\
& Upper-bound
& $4.00$	& $3.99$	& $3.95$	& \NAA %
                                                                & $3.95$	& $64.98$	& $4.02$	\\
& Lower-bound
& $2.24$	& $2.24$	& $3.12$	& \NAA %
                                                                & $3.12$	& \NA	& \NA	\\
\hline\end{tabular}
\end{minipage}\hfill
\begin{minipage}[t]{0.23\textwidth}
\vspace{-27mm}
\centering
\caption{Numerical results for graph bisection with sparse affinity matrices.
         All the results are the average over $10$ random graphs.
         The upper-bounds achieved by SDP based methods are close to each other
         and significantly better than spectral methods (NCut and RatioCut).
         The number of iterations for \smooth is much less than \lbfgsb.
         For problems with $n \leq 1000$, \smooth is faster than \lbfgsb.
         While for larger problems ($n \geq 2000$), \lbfgsb achieves faster speeds than \smooth.
         \NAA~ denotes the cases where SDPT3 fails to output feasible solutions.
}
\label{tab:graph_bisection_sparse}
\end{minipage}
\end{table*}

\begin{table}[t]
  \centering
  \scriptsize
{ %
  \begin{tabular}{l|cc}
  \hline
& & \\ [-2ex]
   & CPU  & CPU+GPU\\
  \hline
  \hline
& \\ [-2ex]
 Time/Iters &$3$h$8$m/$24.0$	& ${18}{m}$/$24.0$			\\
Upper-bound & $20.42$	& $20.36$			\\
Lower-bound & $-4.15$	& $-4.15$			\\
\hline\end{tabular}
}
\caption{{ %
Graph bisection on large dense graphs ($n = 10000, m = 10001$).
\lbfgsb is tested on $1$ core of Intel Xeon E$5$-$2680$ $2.7$GHz CPU ($20$MB cache) and a workstation with
$1$ Intel Xeon E$5$-$2670$ $2.30$GHz CPU ($8$ cores and $20$MB cache) and $1$ NVIDIA Tesla K$40$c GPU.
A $10$-fold speedup is achieved by using CPU+GPU compared with using CPU only.}
\label{tab:graph_bisection_large}}
\end{table}

\begin{figure}
\centering
\includegraphics[width=0.4\textwidth]{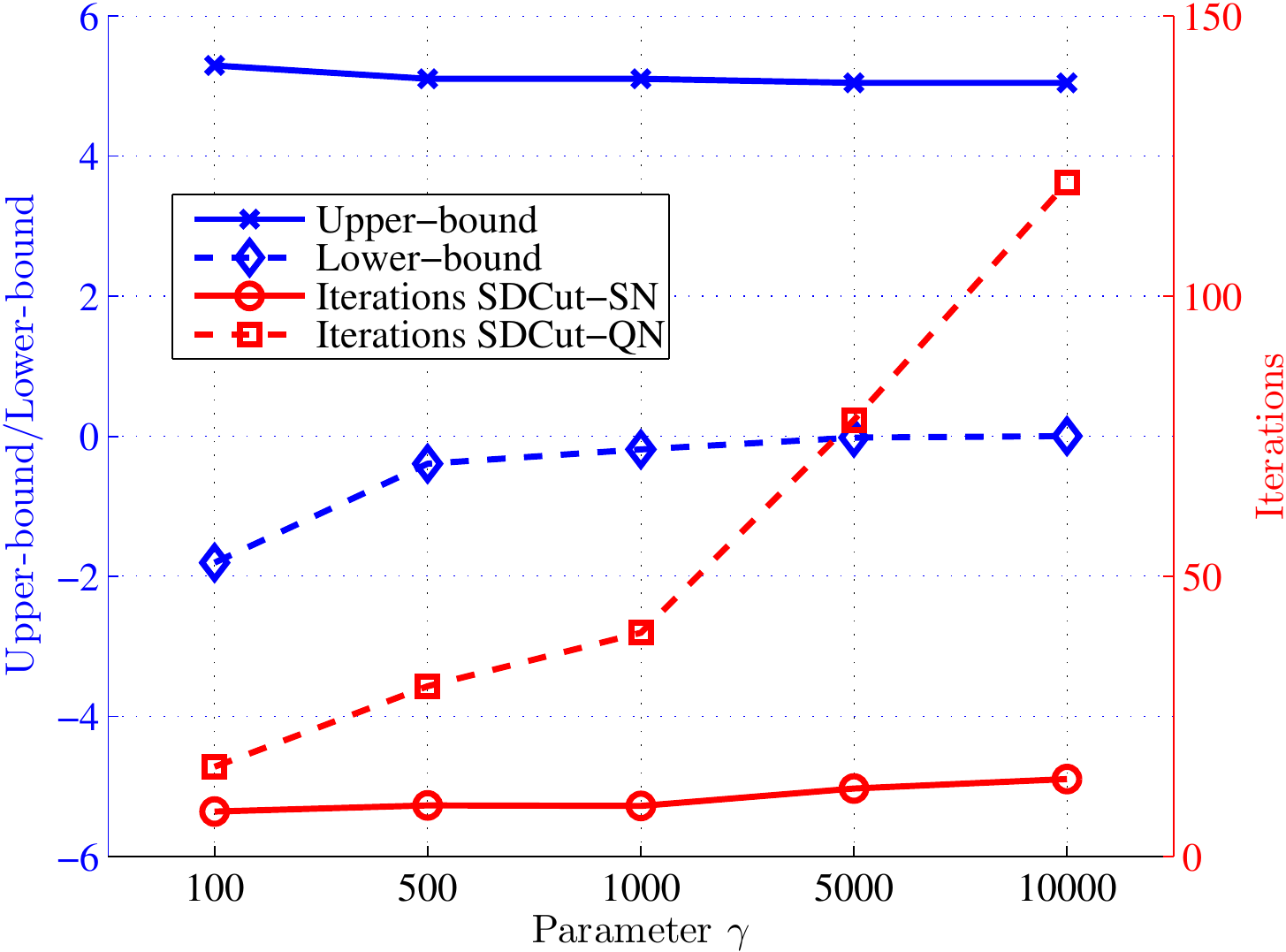} \\
\caption{Results for graph bisection with different values of the parameter $\gamma$.
The illustrated results are averaged over $10$ random graphs.
Upper-bounds and lower-bounds achieved by \lbfgsb are shown in this figure (those of \smooth is very similar and thus omitted).
The relaxation becomes tighter (\ie, upper-bounds and lower-bounds are closer) for larger $\gamma$.
The number of iterations for both \smooth and \lbfgsb grows with the increase of $\gamma$.
}
\label{fig:graph_bisection_t}
\end{figure}

\subsection{Constrained Image Segmentation}
\label{sec:segm}

We consider image segmentation with two types of quadratic constraints (\wrt $\bx$):
partial grouping constraints~\cite{Yu2004segmentation} and histogram constraints~\cite{gorelick2012segmentation}.
The affinity matrix $\bW$ is sparse, so
ARPACK is used by \lbfgsb for eigen-decomposition.

Besides interior-point SDP methods,
we also compare our methods with graph-cuts~\cite{kolmogorov2004energy,kolmogorov2007minimizing,rother2007optimizing} and two constrained spectral clustering method
proposed by Maji \etal~\cite{Maji2011biased} (referred to as BNCut) and Wang and Davidson~\cite{wang2010flexible} (referred to as SMQC).
BNCut and SMQC can encode {\em only} one quadratic constraint,
but it is difficult (if not impossible) to generalize them to multiple quadratic constraints.

\noindent{\bf Partial Grouping Constraints}
The corresponding BQP formulation is Equation~\eqref{equ:app_segm1} in Table~\ref{tab:formulation}.
{%
A feasible solution $\bx$ to \eqref{equ:app_segm1} can obtained from any random sample $\bz$ as follows:
\begin{align}
x_i = \left\{ \begin{array}{ll}
              \mathrm{sign}(z_i - \theta_f) & \mbox{if } (\bs_f)_i > 0, \\
              \mathrm{sign}(z_i - \theta_b) & \mbox{if } (\bs_b)_i > 0, \\
              \mathrm{sign}(z_i) & \mbox{otherwise},
              \end{array} \right. \label{eq:round_partial}
\end{align}
where $\theta_f$ and $\theta_b$ are chosen from $[\min(\{ z_i | (\bs_f)_i > 0 \}), +\infty)$ and $(-\infty, \max(\{ z_i | (\bs_b)_i > 0 \}]$ respectively.
Note that for any sample $\bz$, $\bx$ is feasible if $\theta_f = \min(\{ z_i | (\bs_f)_i > 0 \})$
and $\theta_b = \max(\{ z_i | (\bs_b)_i > 0 \})$.
}

Fig.~\ref{fig:img_segm} illustrates the result for image segmentation with partial grouping constraints on the Berkeley dataset~\cite{MartinFTM01}.
All the test images are over-segmented into about $760$ superpixels.
We find that BNCut did not accurately segment foreground,
as it only incorporates a single set of grouping pixels (foreground).
In contrast, our methods are able to accommodate multiple sets of grouping pixels and
segment the foreground more accurately.
In Table~\ref{tab:img_segm}, we compare the CPU time and the upper-bounds of \fastsdp, SeDuMi and SDPT3.
\lbfgsb achieves objective values similar to that of SeDuMi and SDPT3, yet is over $10$ times faster.

\begin{figure*}[t]
\centering
\begin{minipage}[c]{0.75\textwidth}
\vspace{-0.0cm}
\centering
\subfloat{
\centering
\begin{minipage}[c]{0.01\textwidth}
\begin{turn}{90} {\scriptsize Images} \end{turn}
\end{minipage}
\begin{minipage}[c]{0.99\textwidth}
\includegraphics[width=0.112\textwidth]{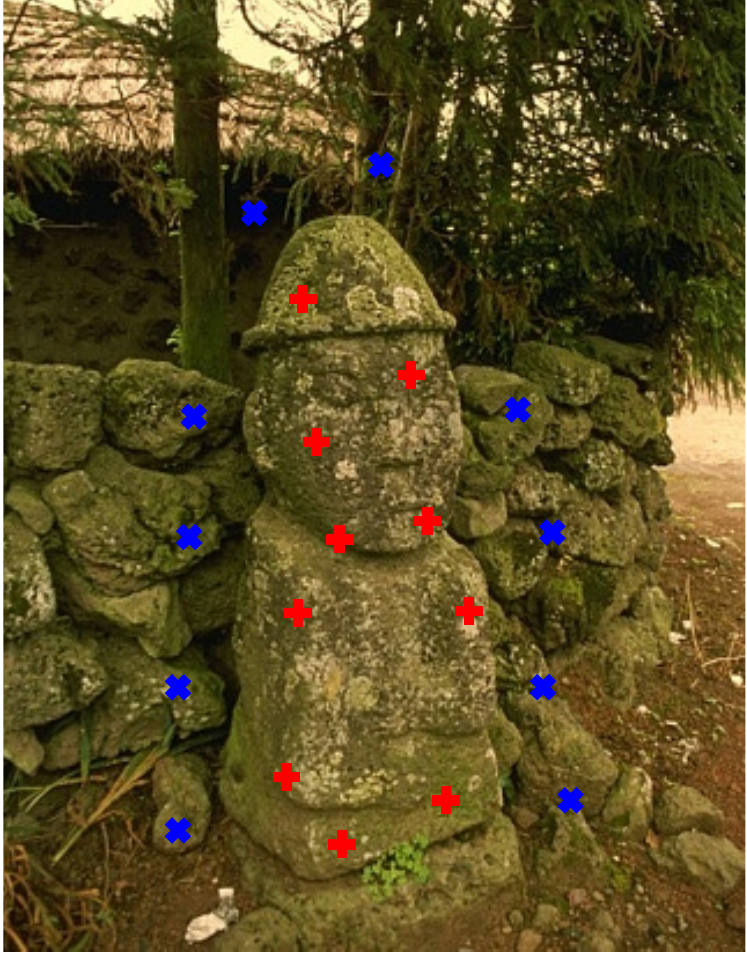}
\includegraphics[width=0.215\textwidth]{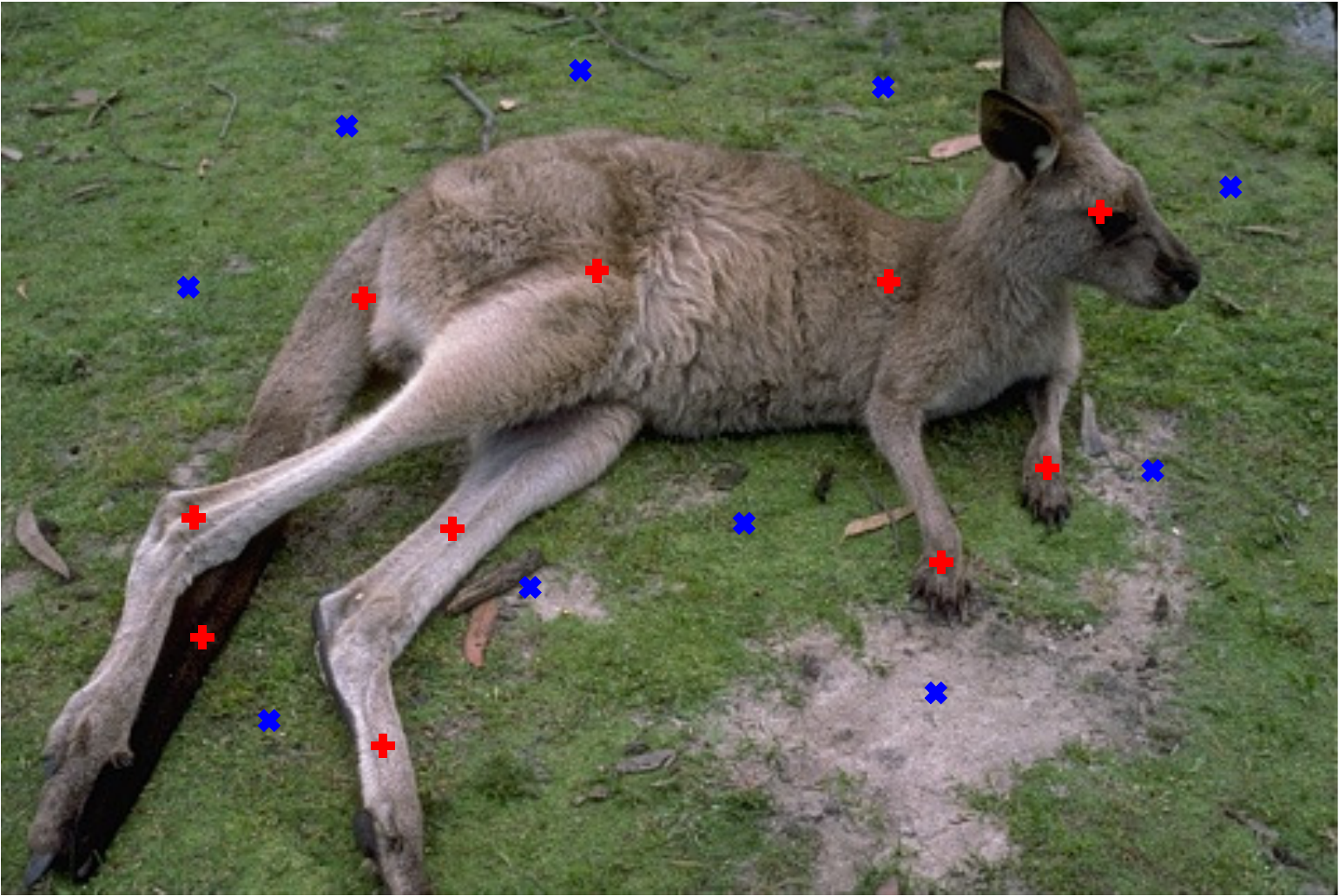}
\includegraphics[width=0.215\textwidth]{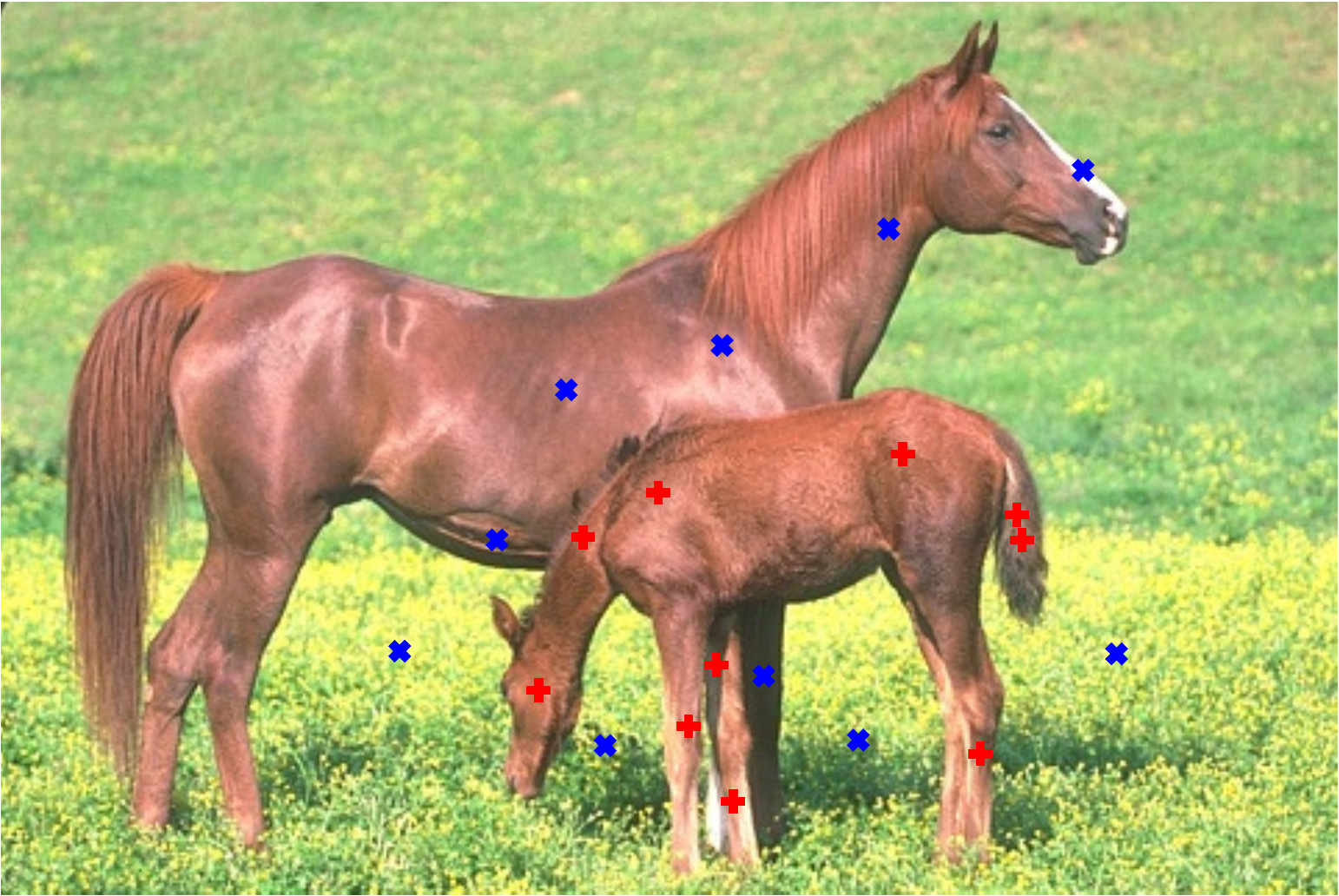}
\includegraphics[width=0.215\textwidth]{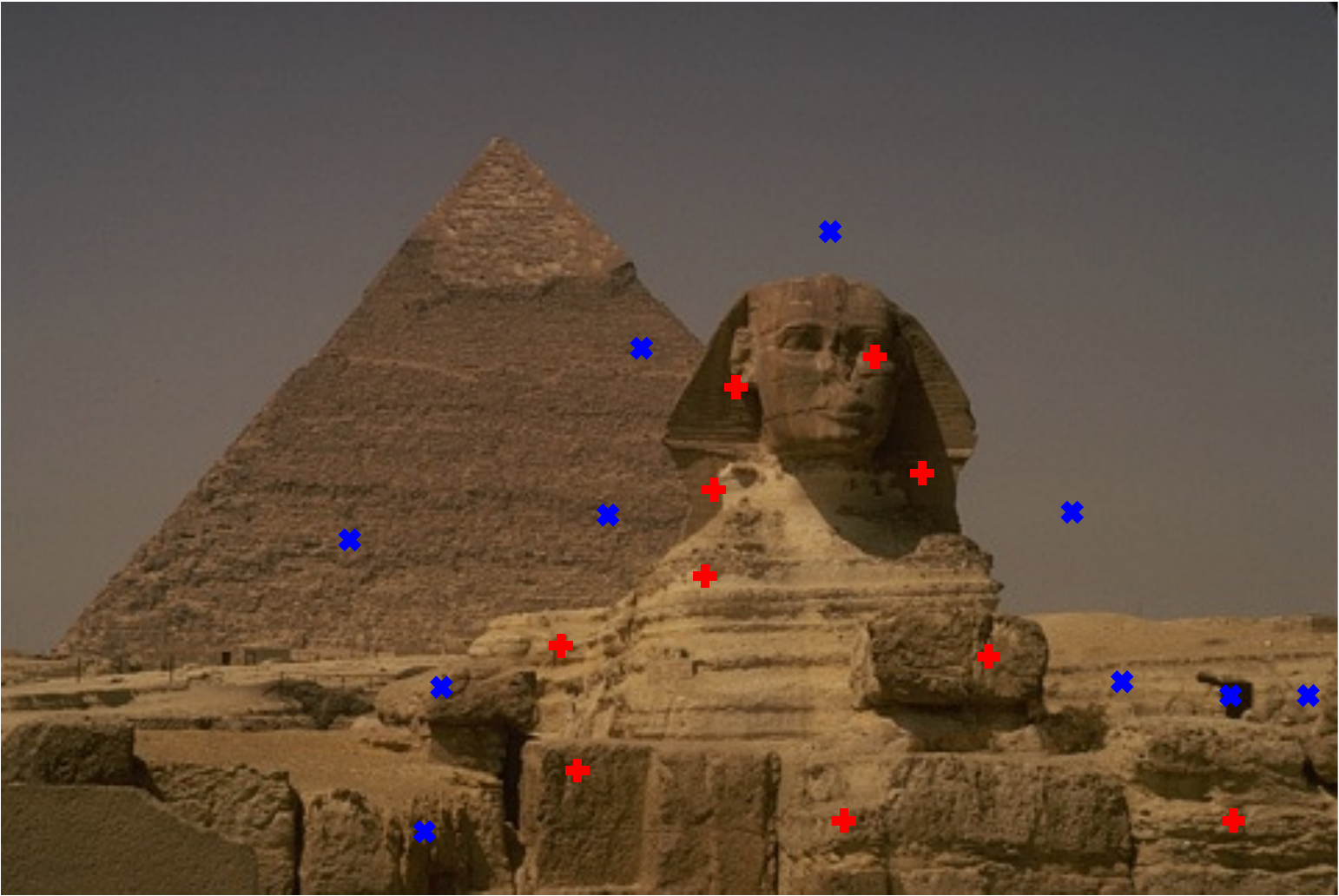}
\includegraphics[width=0.175\textwidth]{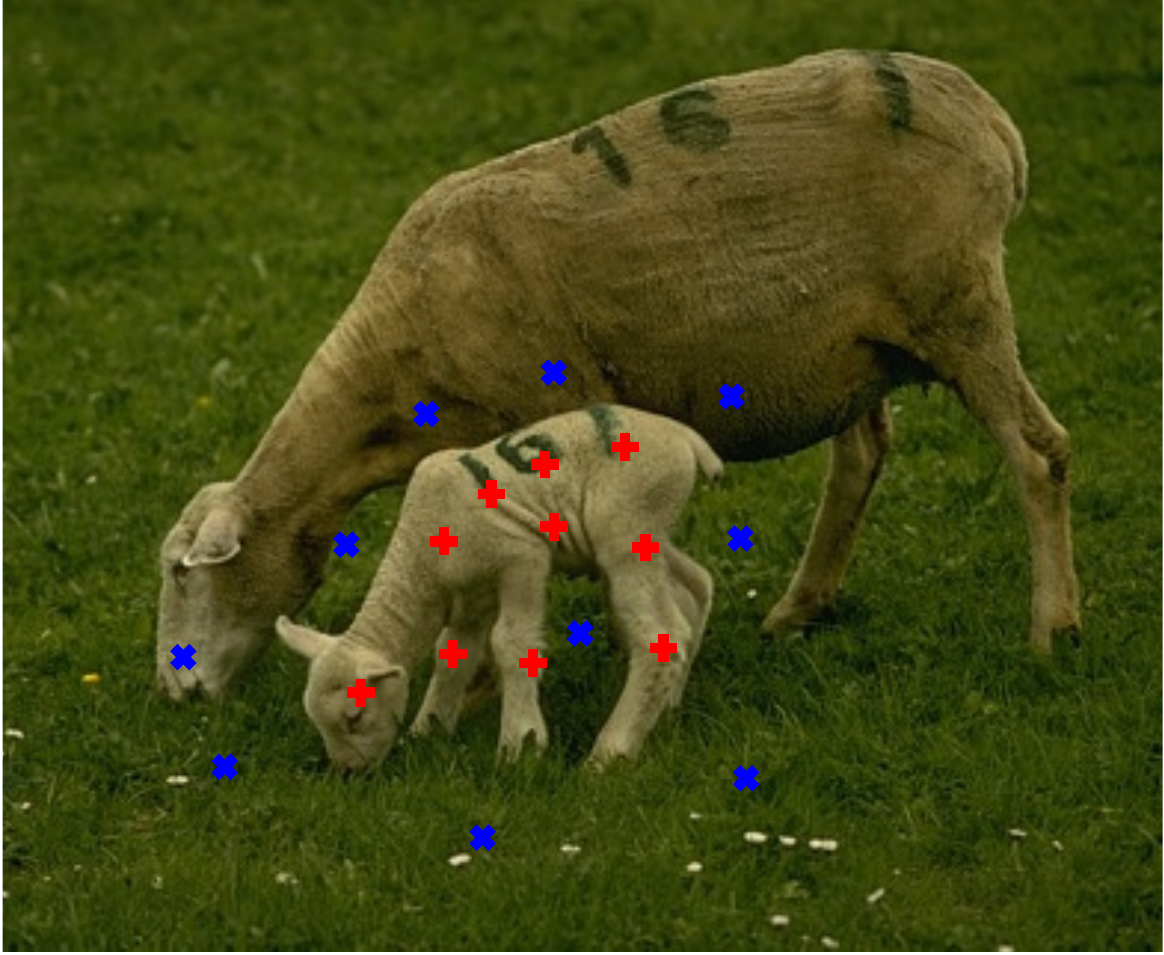}
\end{minipage}
}\\
\vspace{-0.3cm}
\subfloat{
\centering
\begin{minipage}[c]{0.01\textwidth}
\begin{turn}{90} {\scriptsize BNCut} \end{turn}
\end{minipage}
\begin{minipage}[c]{0.99\textwidth}
\includegraphics[width=0.112\textwidth]{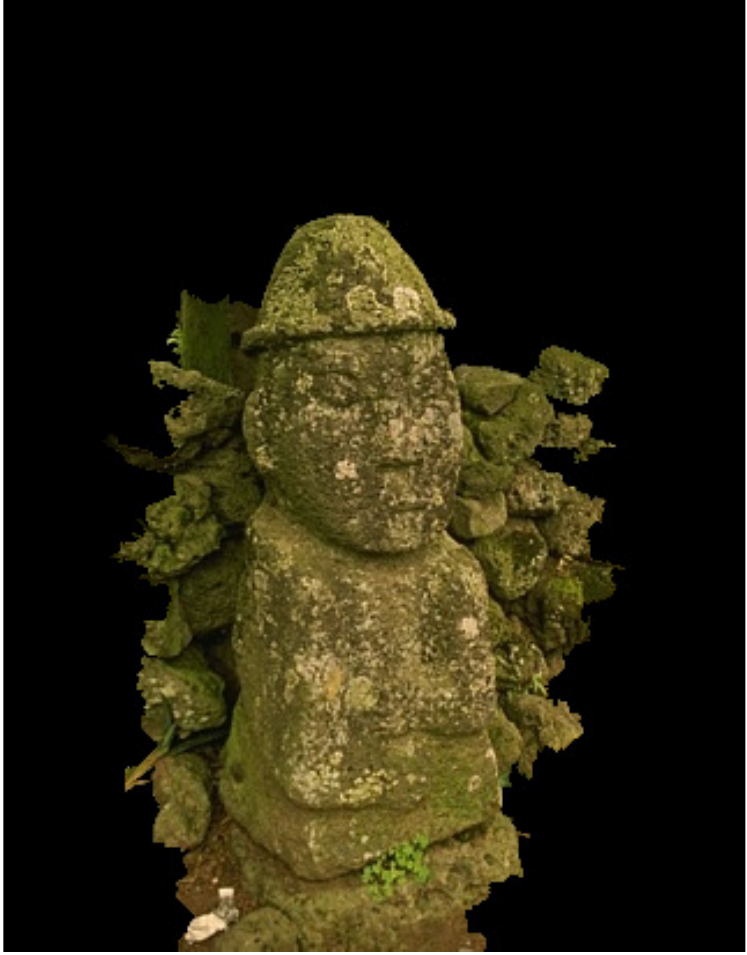}
\includegraphics[width=0.215\textwidth]{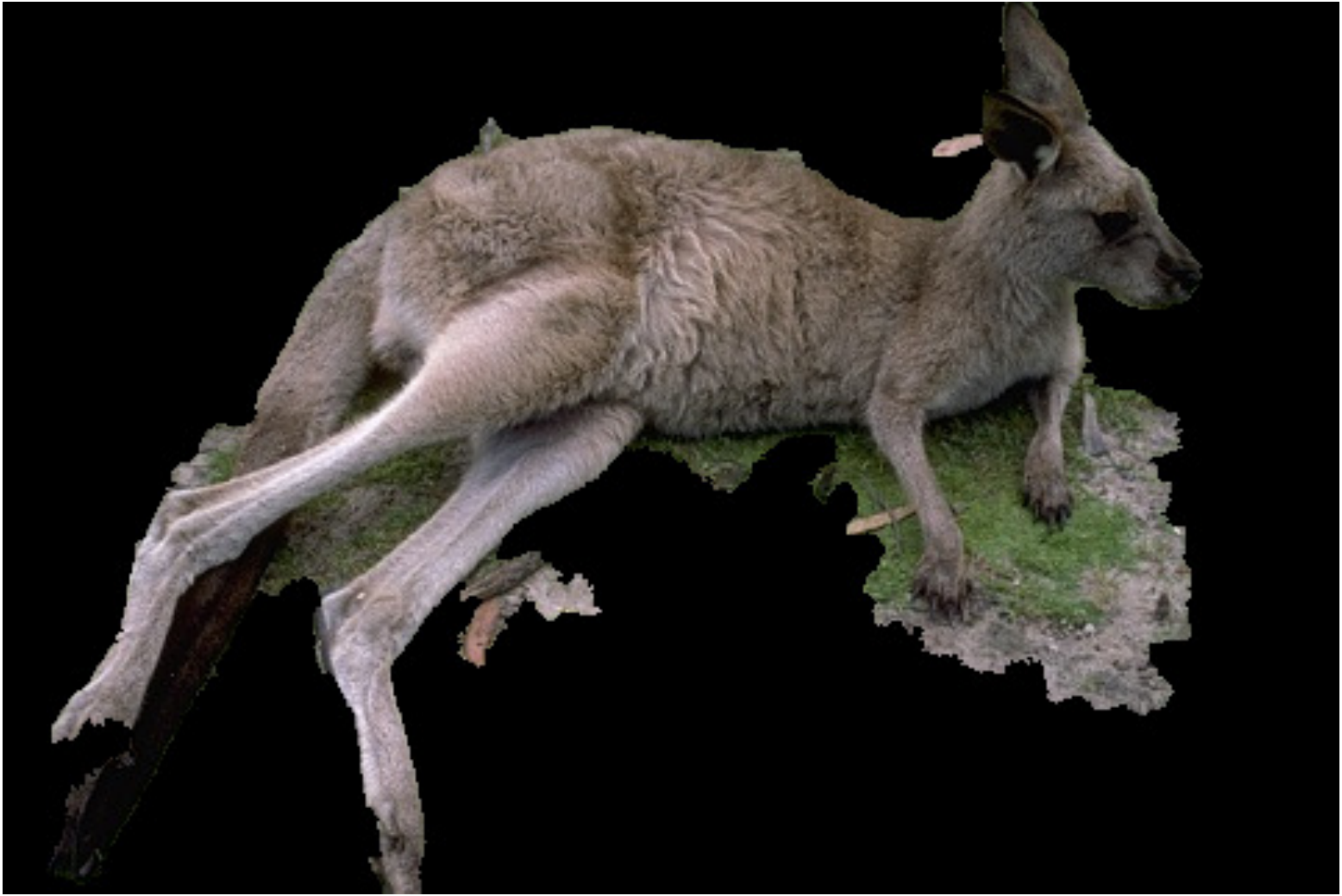}
\includegraphics[width=0.215\textwidth]{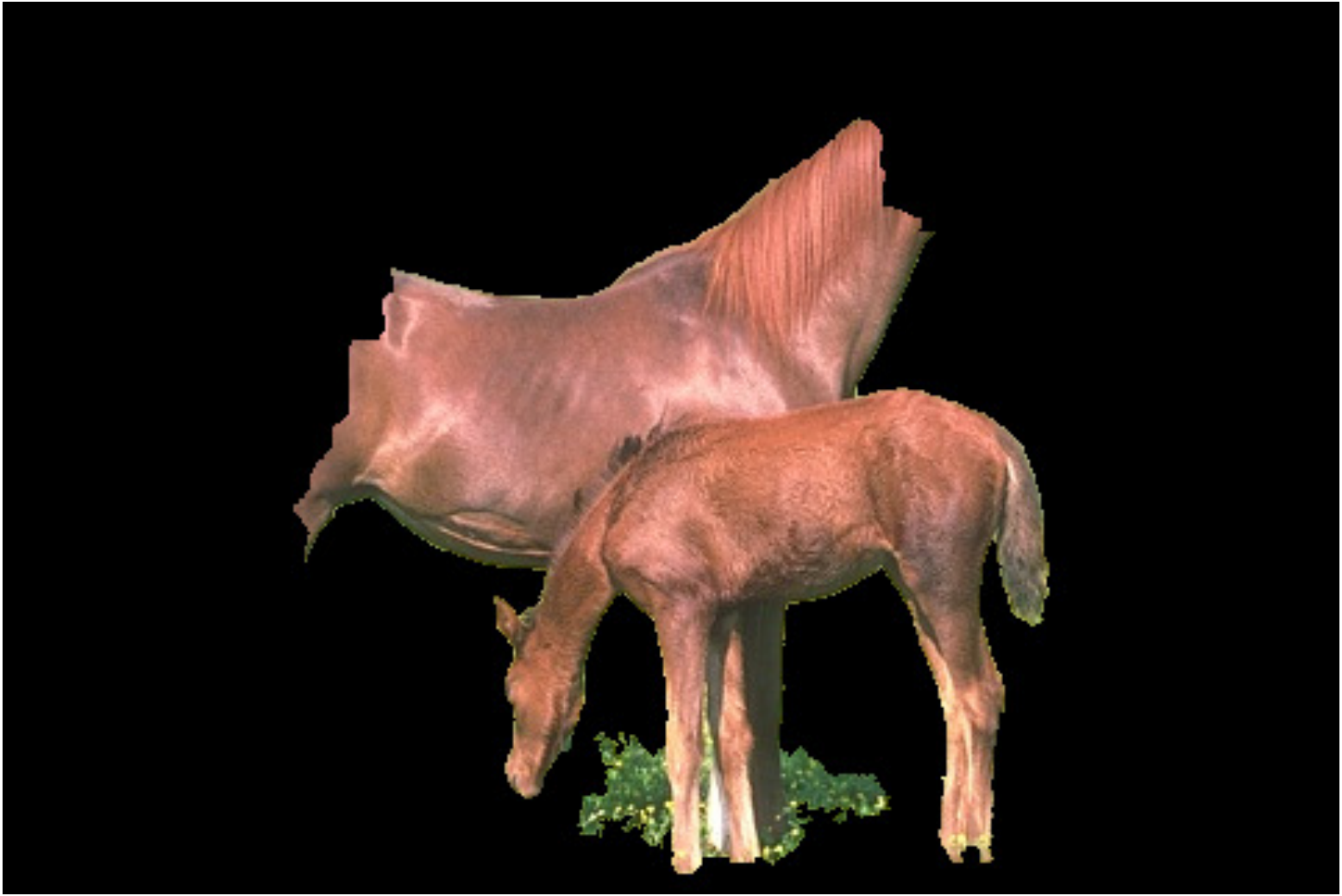}
\includegraphics[width=0.215\textwidth]{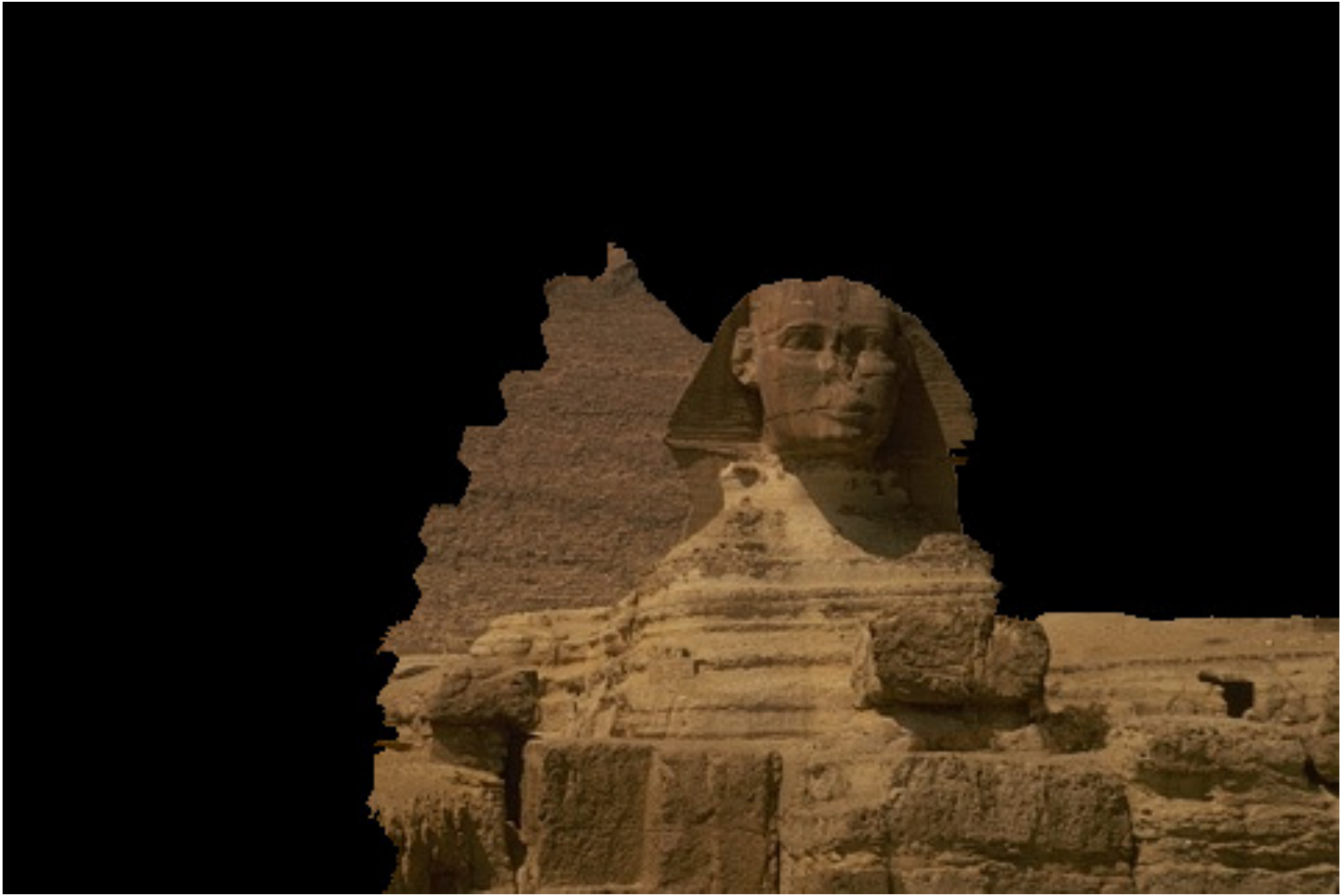}
\includegraphics[width=0.175\textwidth]{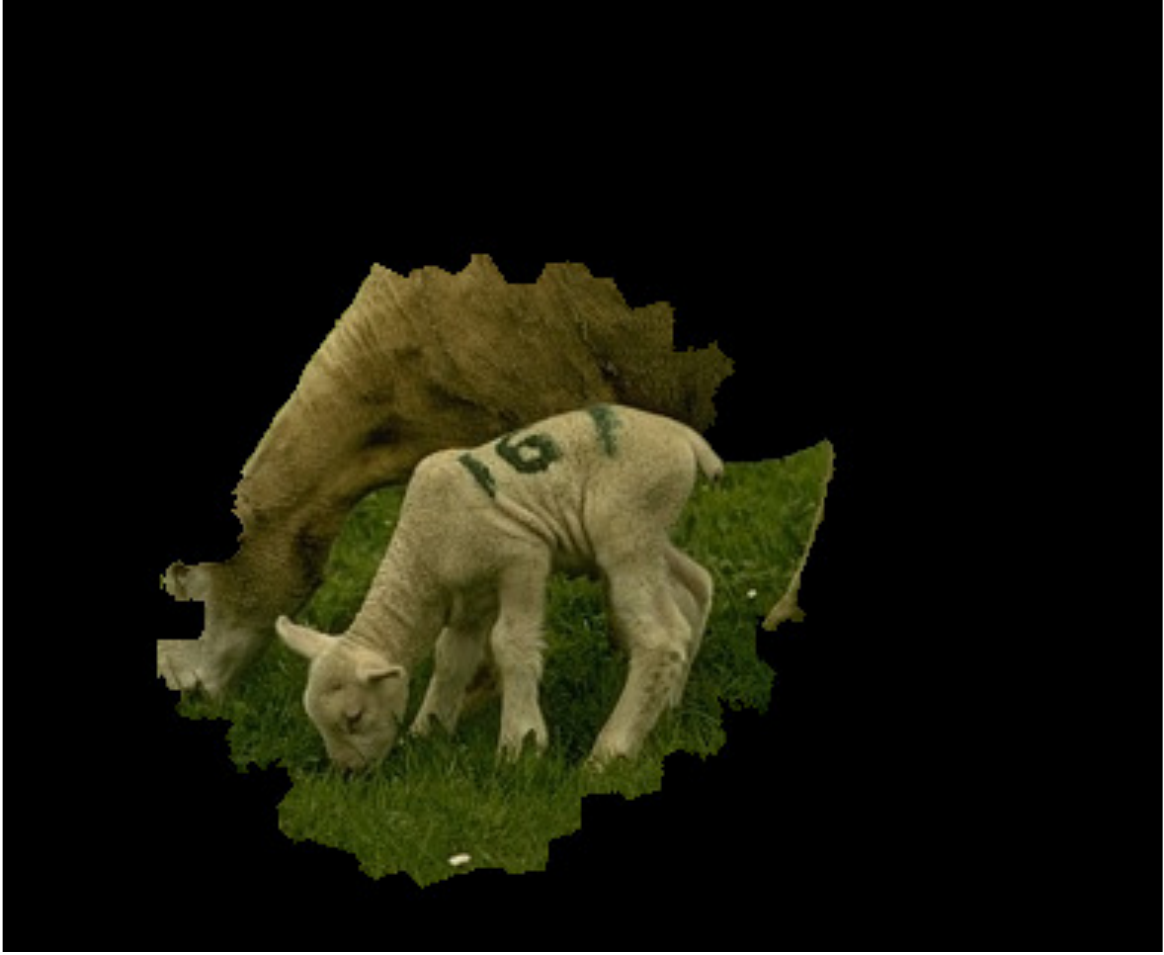}
\end{minipage}
}\\
\vspace{-0.3cm}
\subfloat{
\centering
\begin{minipage}[c]{0.01\textwidth}
\begin{turn}{90} {\scriptsize \lbfgsb} \end{turn}
\end{minipage}
\begin{minipage}[c]{0.99\textwidth}
\includegraphics[width=0.112\textwidth]{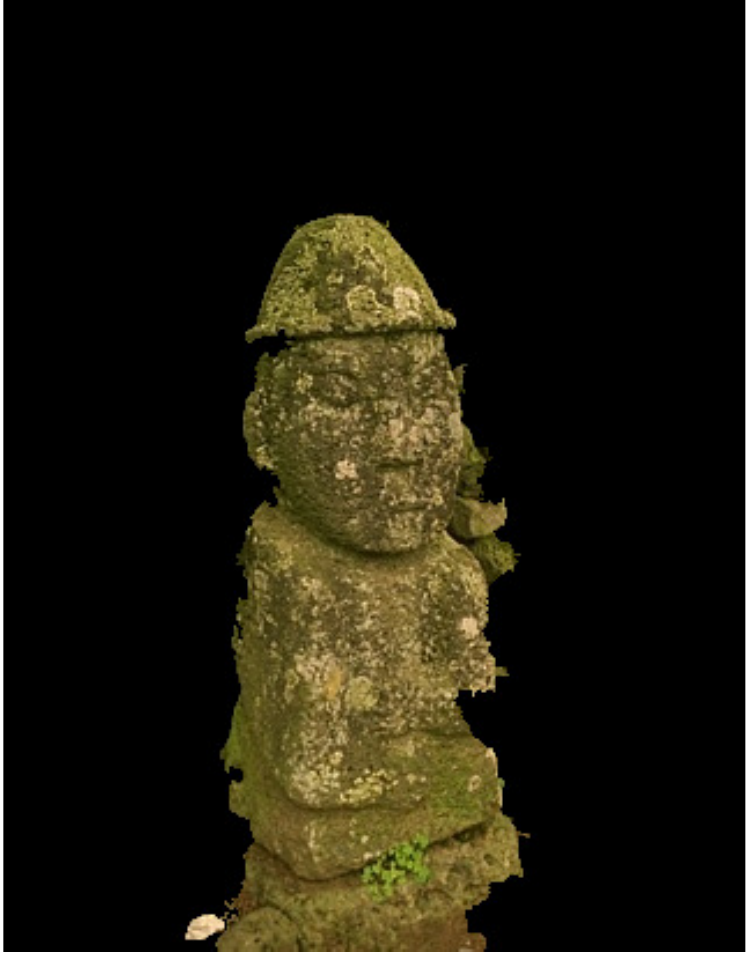}
\includegraphics[width=0.215\textwidth]{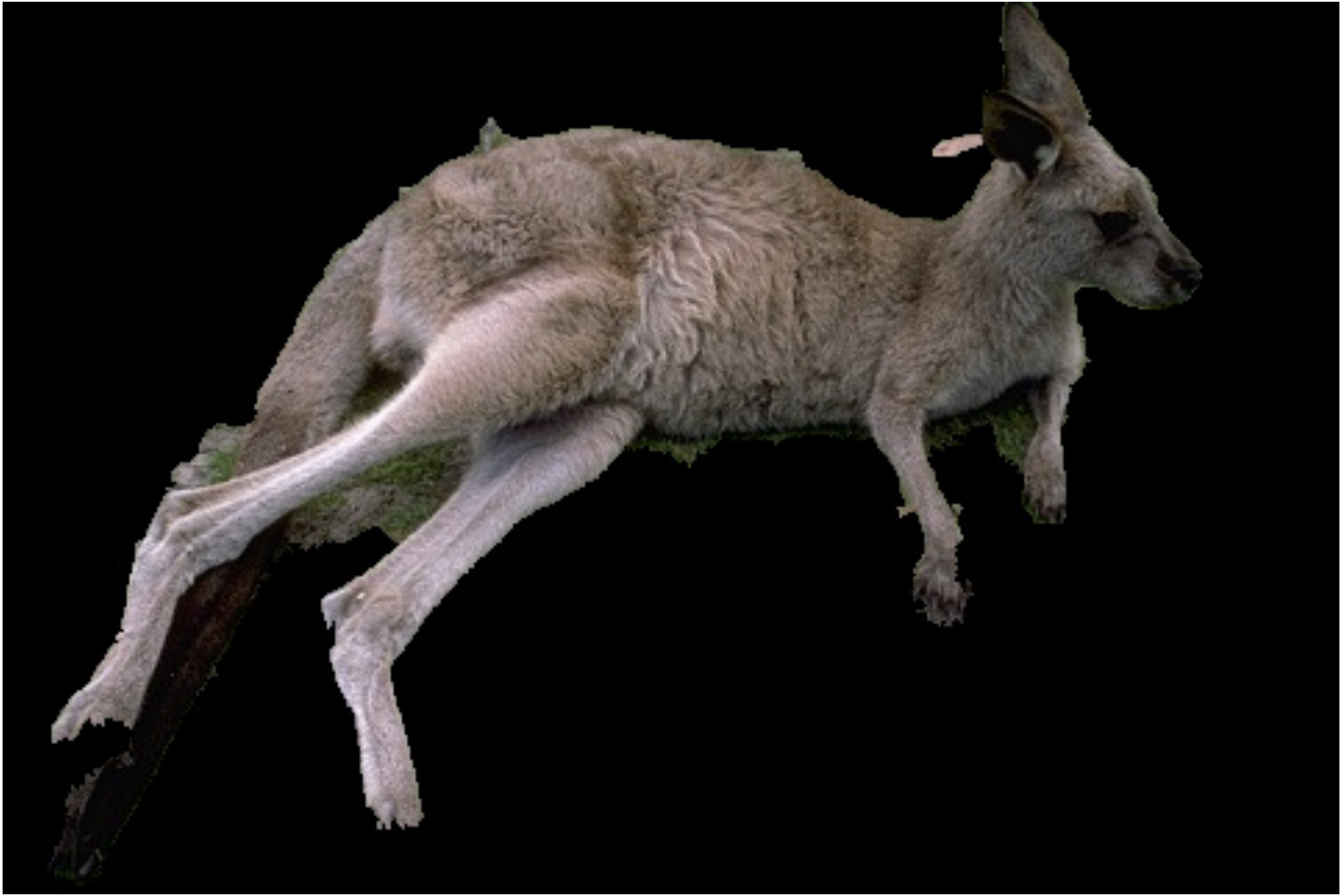}
\includegraphics[width=0.215\textwidth]{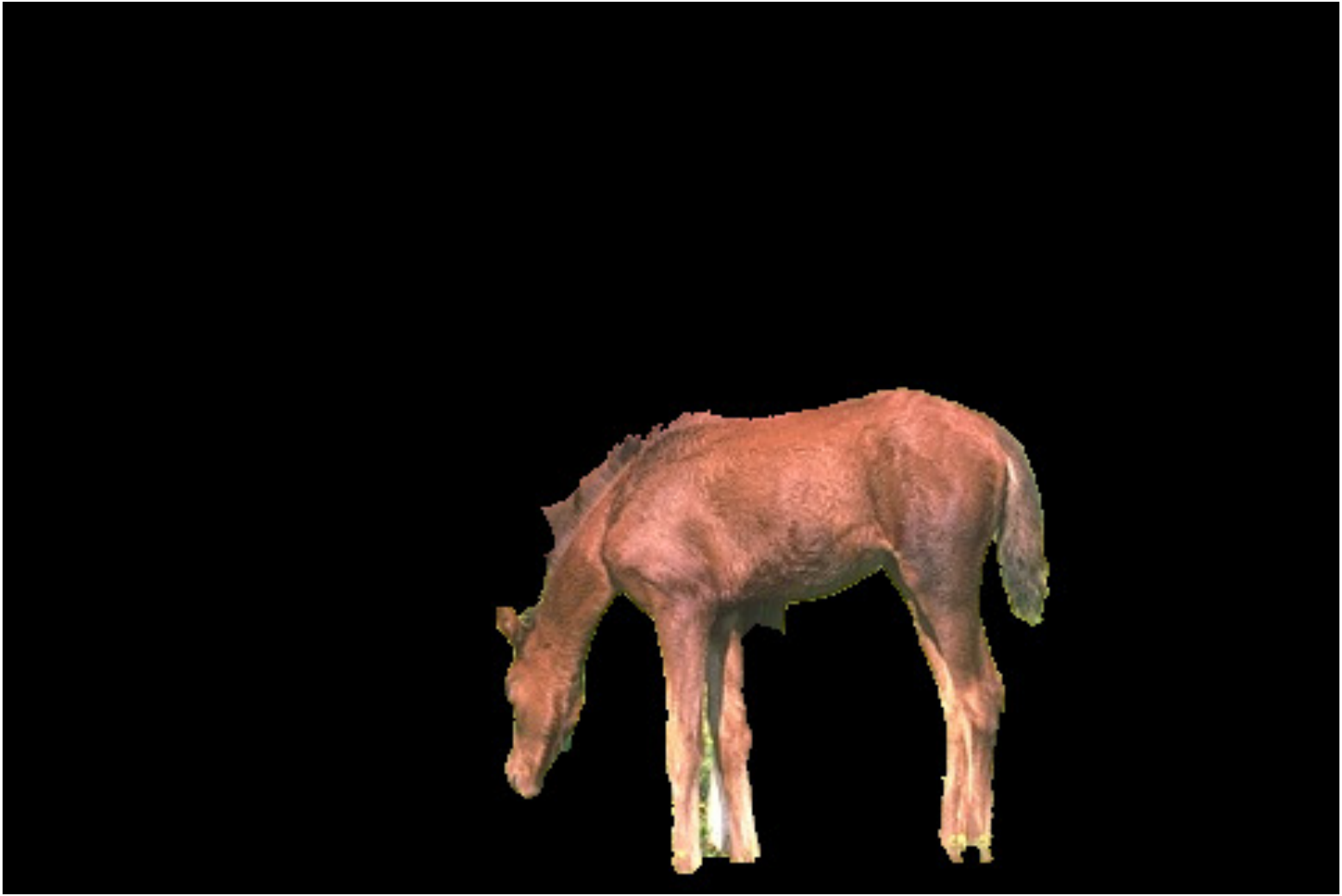}
\includegraphics[width=0.215\textwidth]{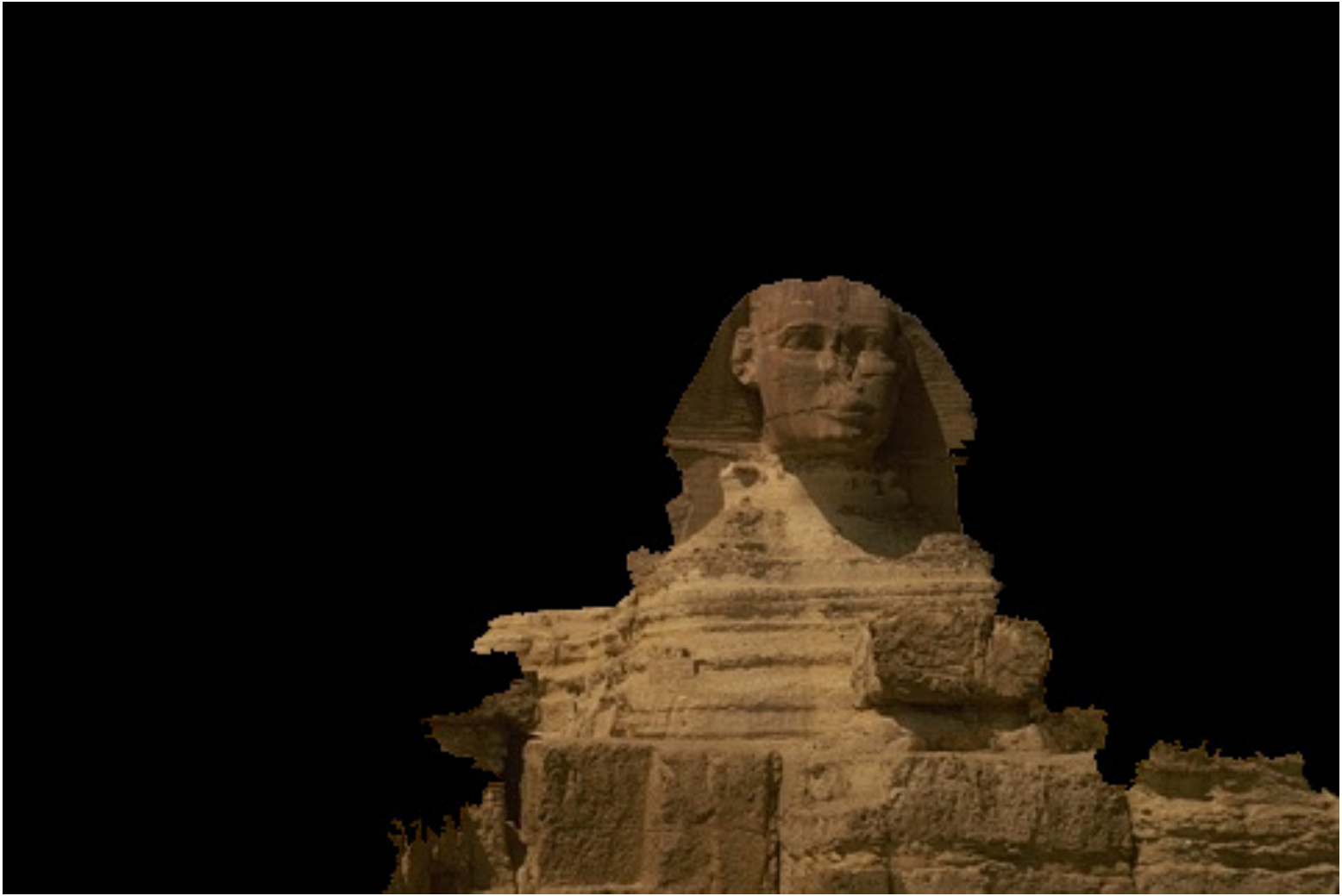}
\includegraphics[width=0.175\textwidth]{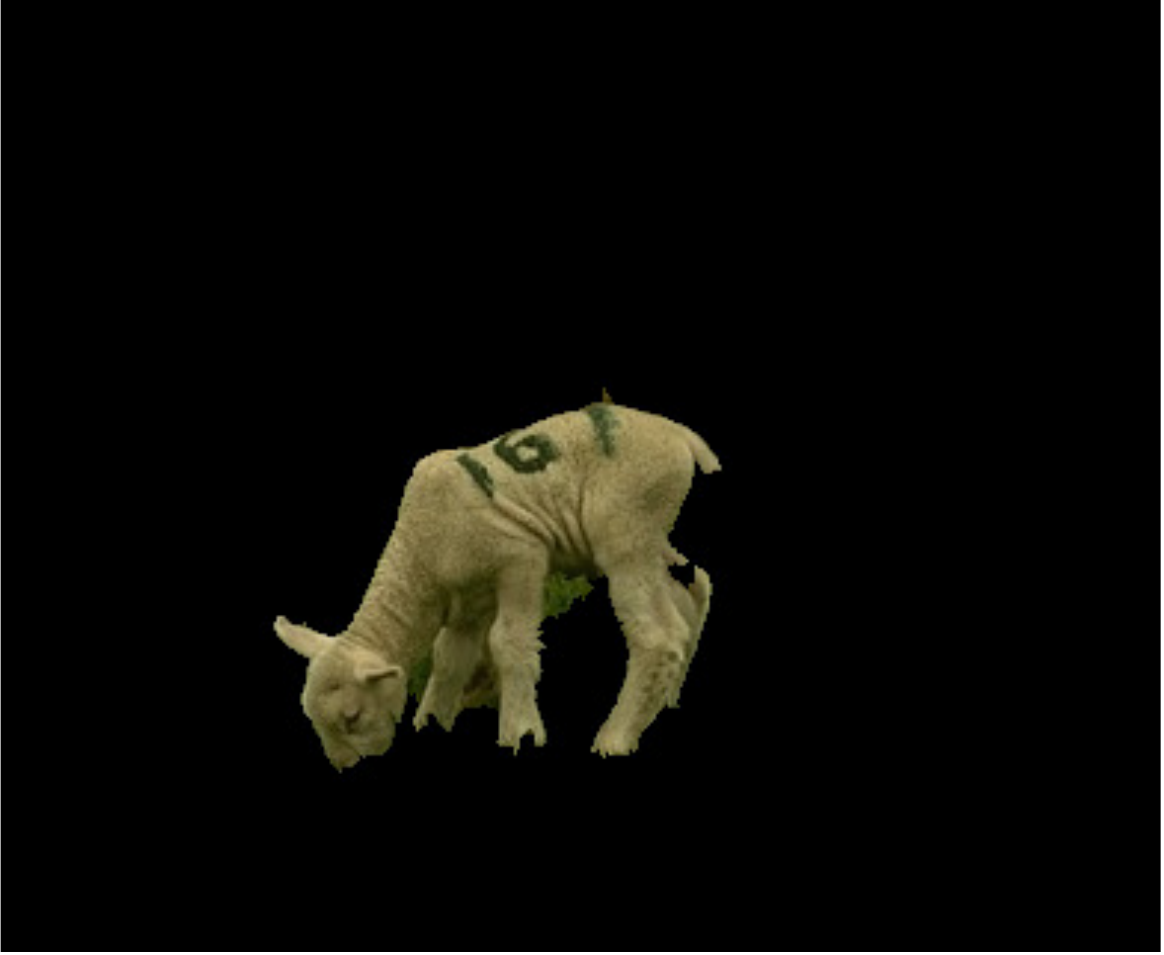}
\end{minipage}
}
\end{minipage}
\caption{Image segmentation with partial grouping constraints.
The top row shows the original images with $10$ labelled foreground (red markers) and $10$ background (blue markers) pixels.
\lbfgsb achieves significantly better results than BNCut.
The results of SeDuMi and SDPT3 are omitted, as they are similar to those of \lbfgsb.
}
\label{fig:img_segm}
\end{figure*}

\begin{table}[t]
  \centering
  \scriptsize
  \begin{tabular}{l|cccc}
  \hline
& & & &\\ [-2ex]
     Methods             & \lbfgsb & SeDuMi & SDPT3  \\
  \hline
  \hline
& & & &\\ [-2ex]
     Time                & $23.7$s   & $6$m$12$s    & $5$m$29$s   \\
     Upper-bound         & $-116.10$ & $-116.30$ & $-116.32$ \\
  \hline
  \end{tabular}
  \caption{Numerical results for image segmentation with partial grouping constraints.
           Time and upper-bound are the means over the five images in Fig.~\ref{fig:img_segm}.
           \fastsdp runs $10$ times faster than SeDuMi and SDPT3, and offers a similar upper-bound.
          } %
  \label{tab:img_segm}
\end{table}

\begin{figure*}[t]
\centering
\subfloat{
\begin{minipage}[c]{0.002\textwidth}
\begin{turn}{90} {\scriptsize Images} \end{turn}
\end{minipage}
\hspace{-0.2mm}
\begin{minipage}[c]{0.996\textwidth}
\includegraphics[width=0.113\textwidth]{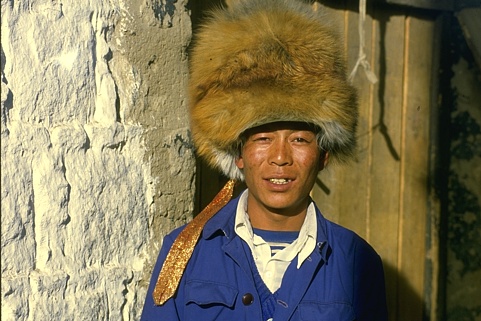}
\includegraphics[width=0.113\textwidth]{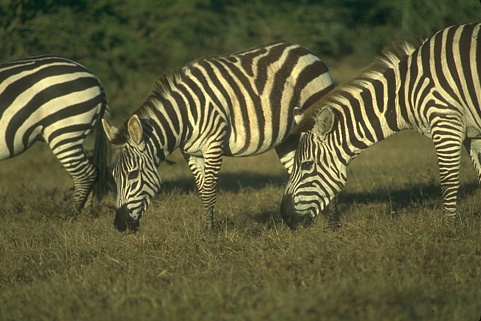}
\includegraphics[width=0.113\textwidth]{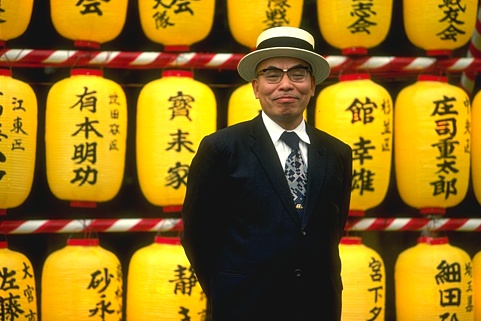}
\includegraphics[width=0.113\textwidth]{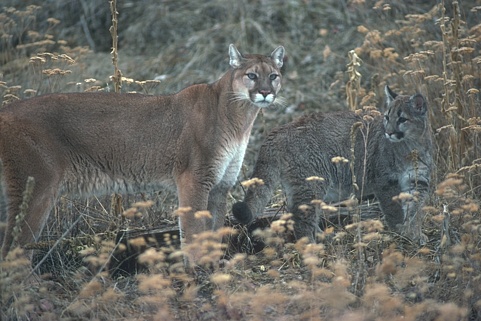}
\includegraphics[width=0.113\textwidth]{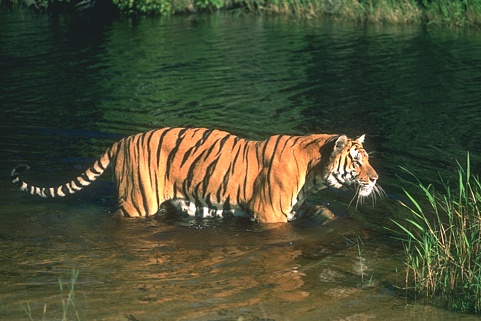}
\includegraphics[width=0.113\textwidth]{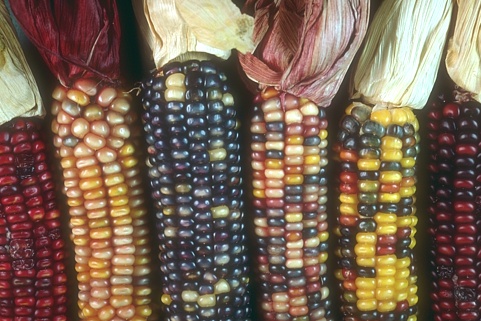}
\includegraphics[width=0.113\textwidth]{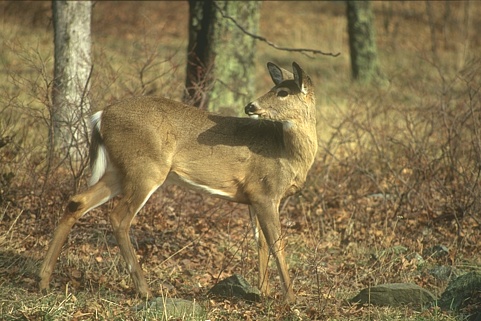}
\includegraphics[width=0.113\textwidth]{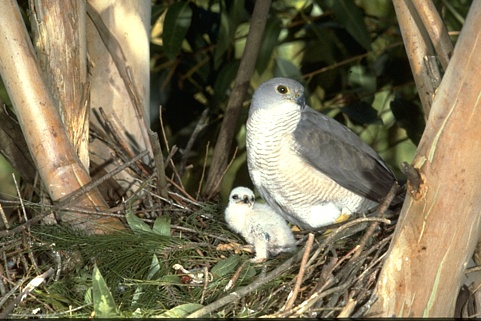}
\centering
\end{minipage}
}\\
\vspace{-2mm}
\subfloat{
\begin{minipage}[c]{0.002\textwidth}
\begin{turn}{90} {\scriptsize GT} \end{turn}
\end{minipage}
\hspace{-0.2mm}
\begin{minipage}[c]{0.996\textwidth}
\includegraphics[width=0.113\textwidth]{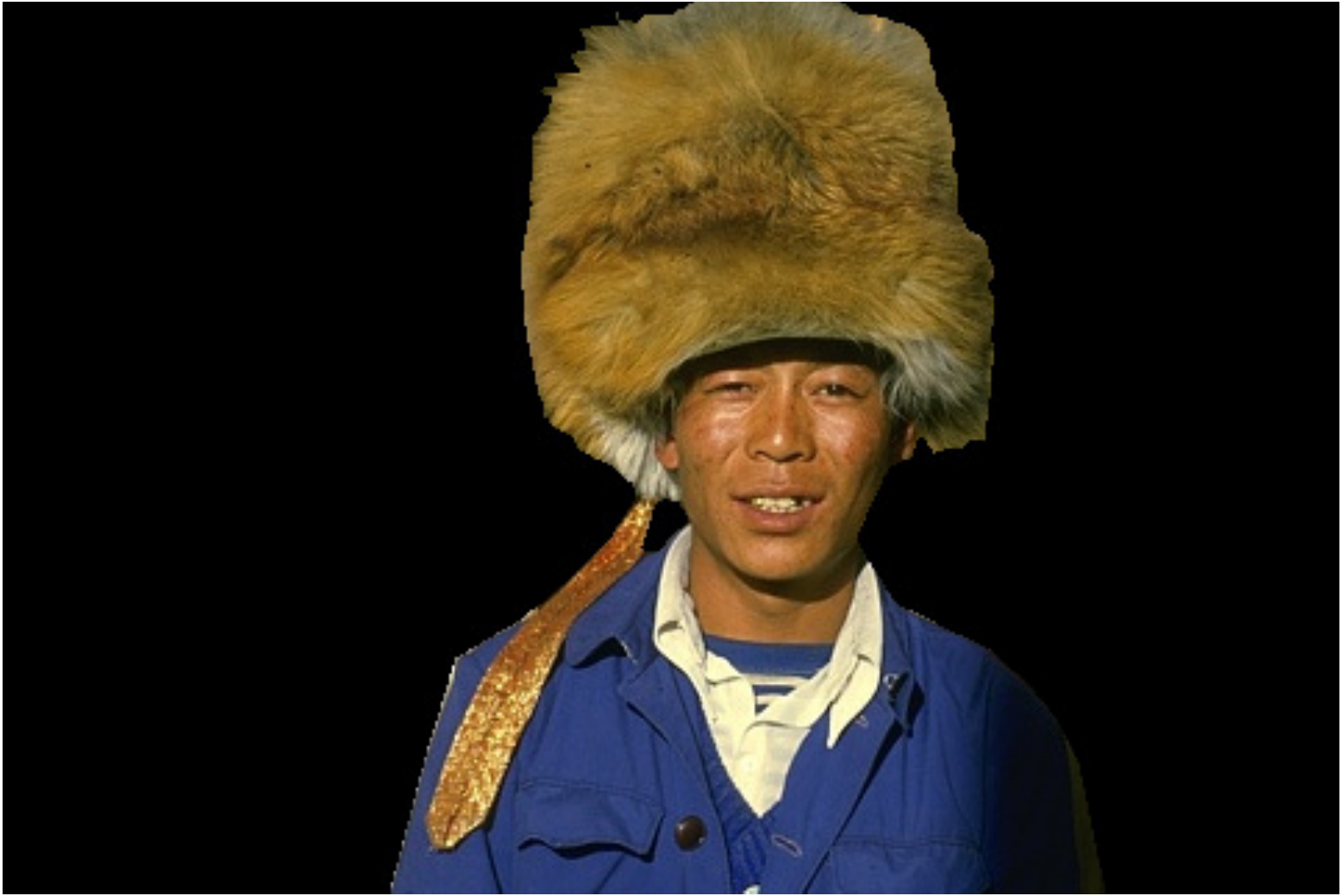}
\includegraphics[width=0.113\textwidth]{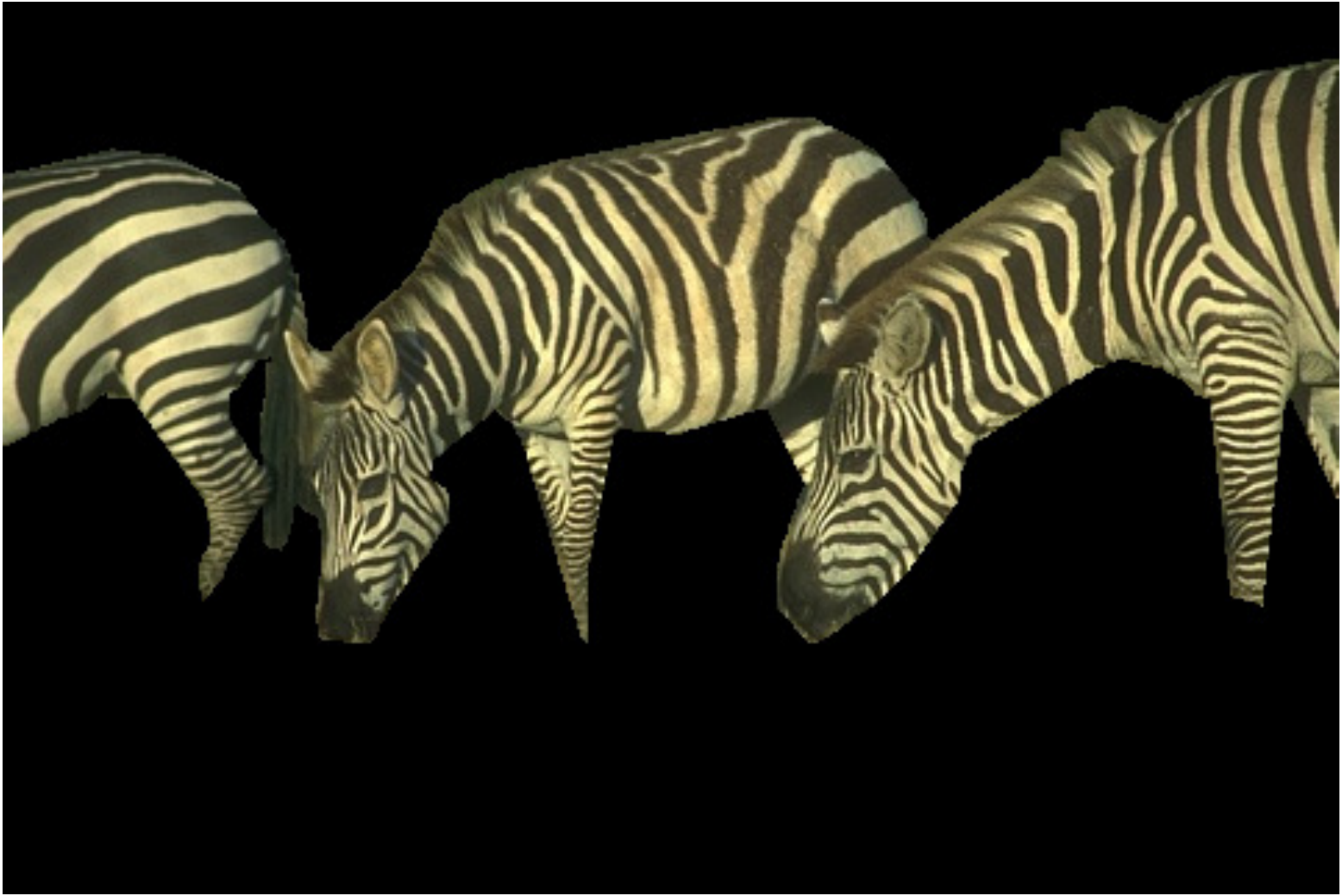}
\includegraphics[width=0.113\textwidth]{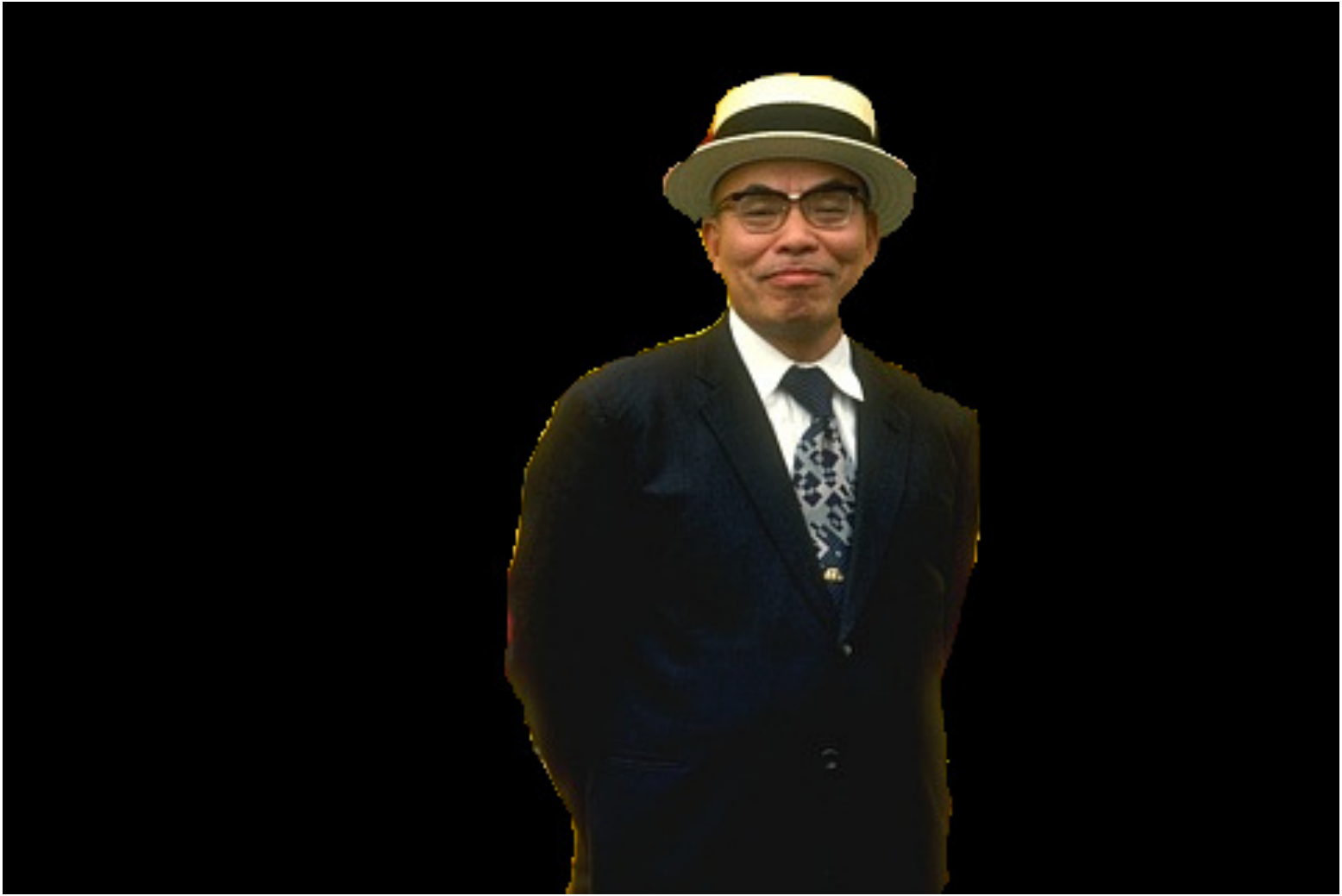}
\includegraphics[width=0.113\textwidth]{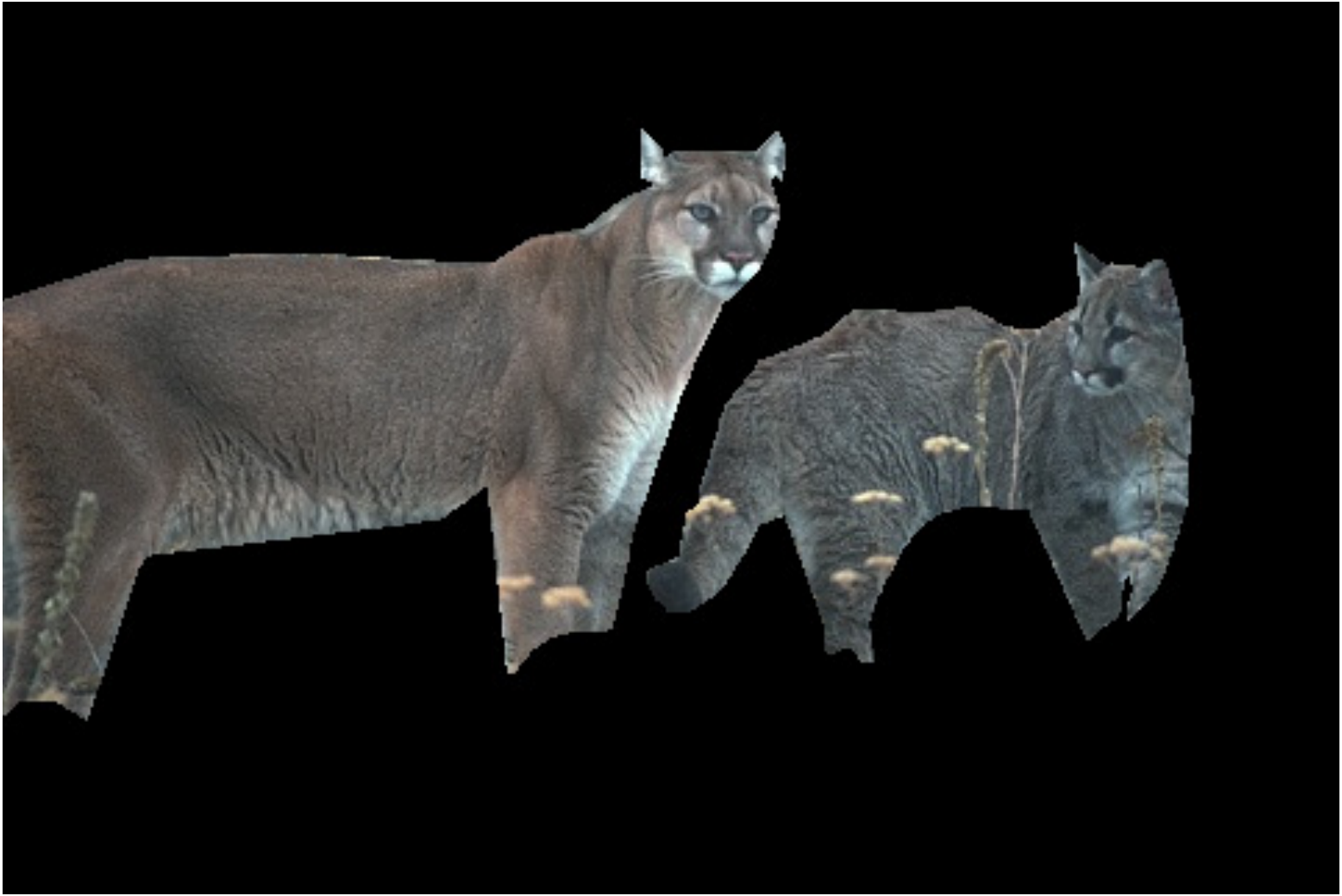}
\includegraphics[width=0.113\textwidth]{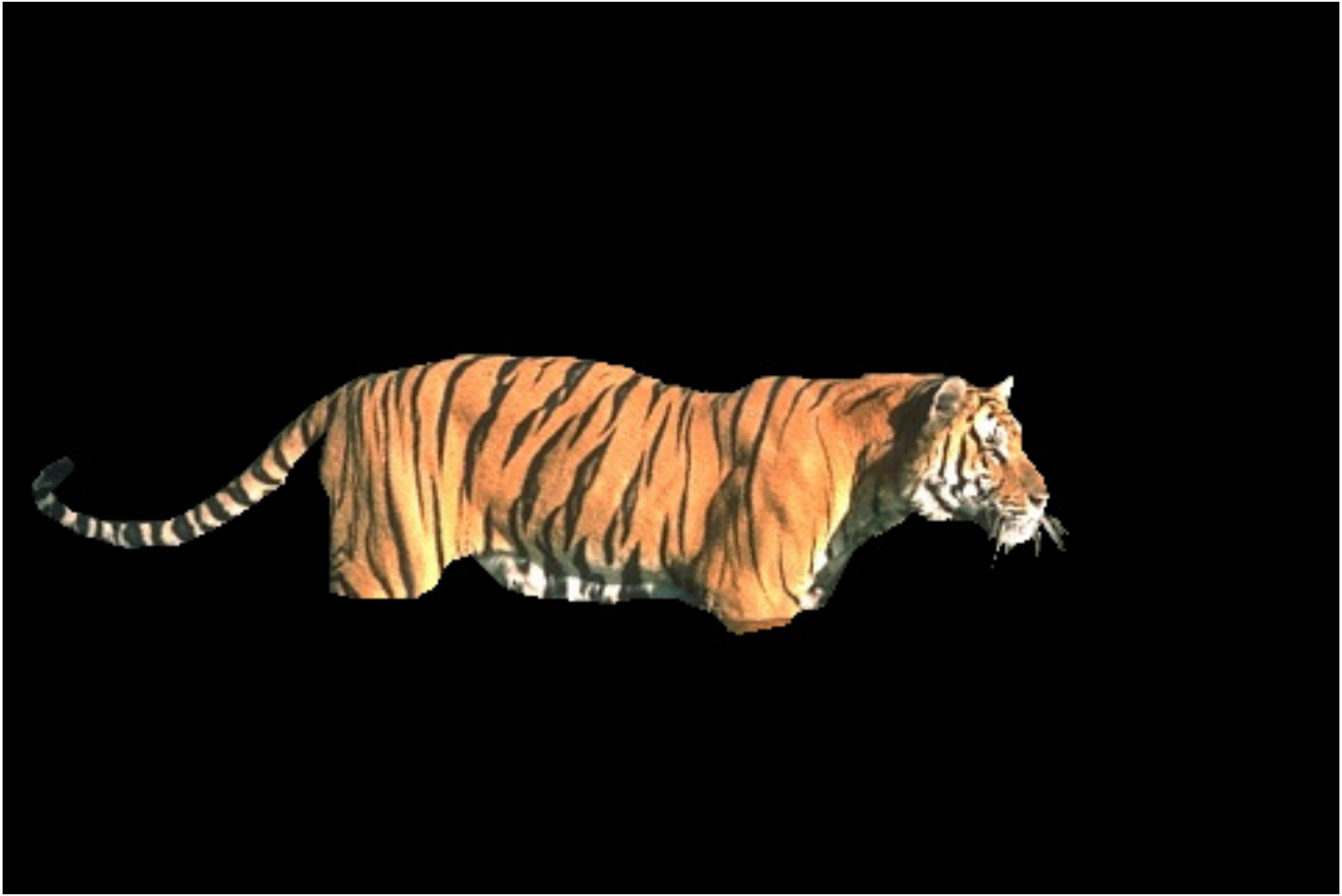}
\includegraphics[width=0.113\textwidth]{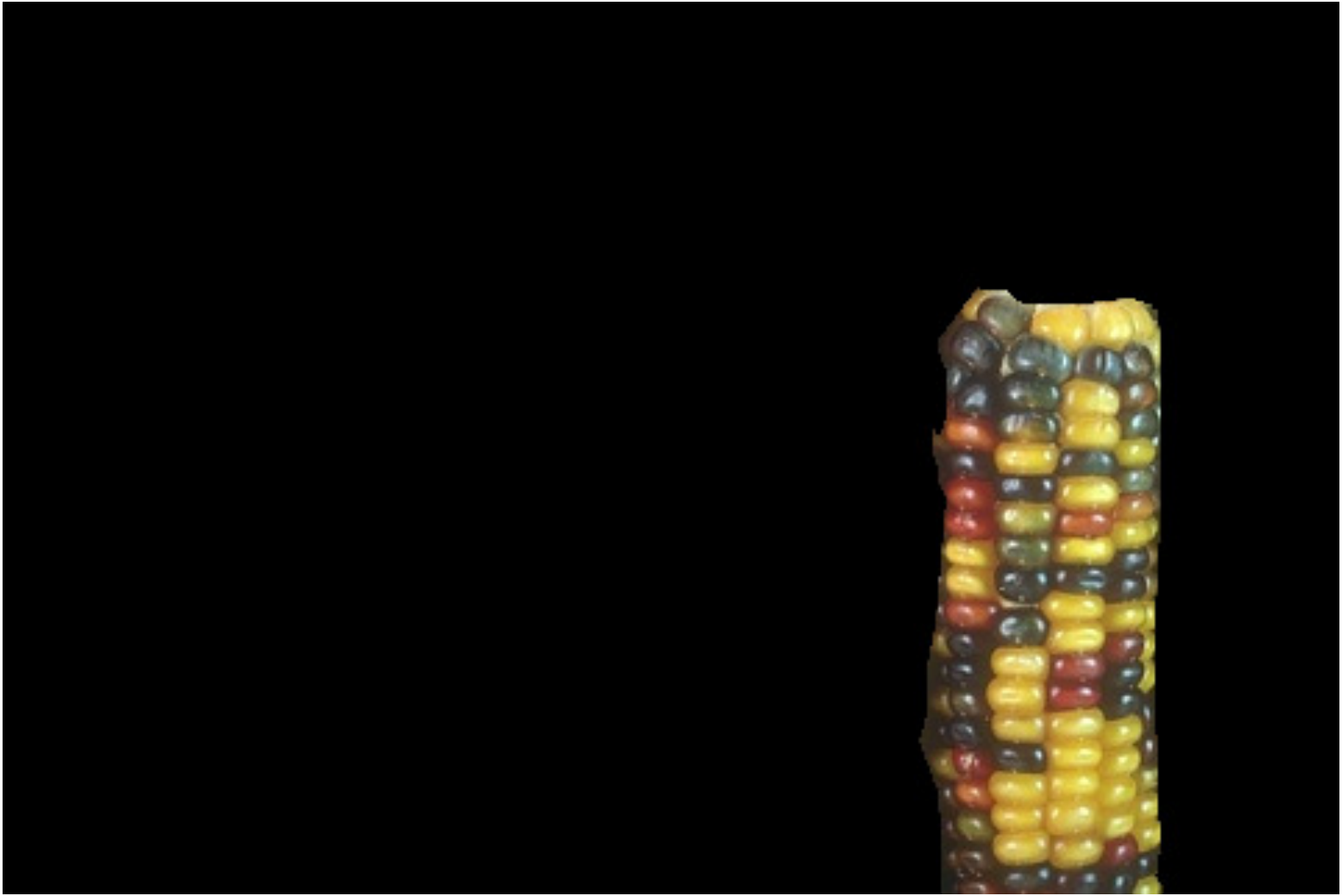}
\includegraphics[width=0.113\textwidth]{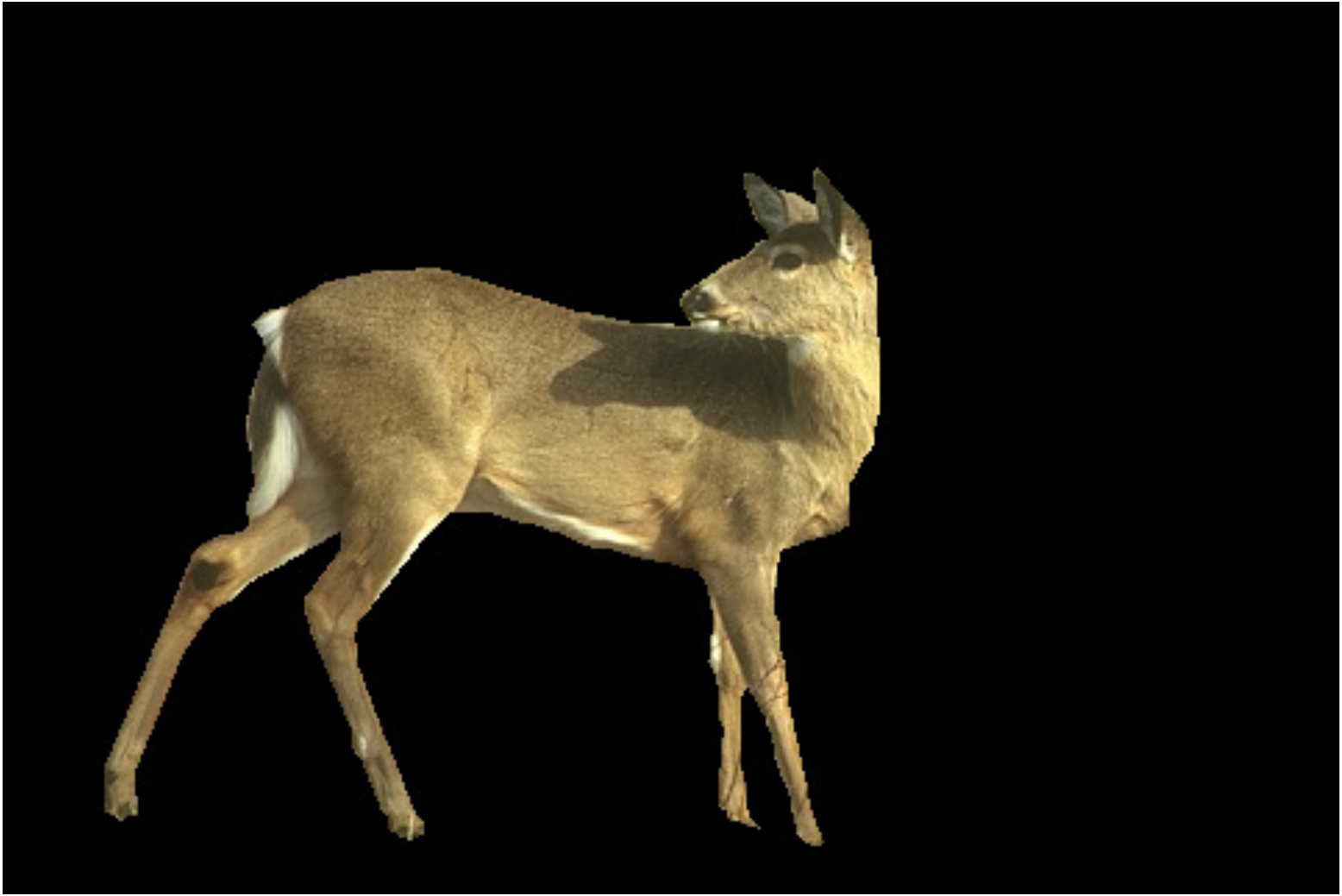}
\includegraphics[width=0.113\textwidth]{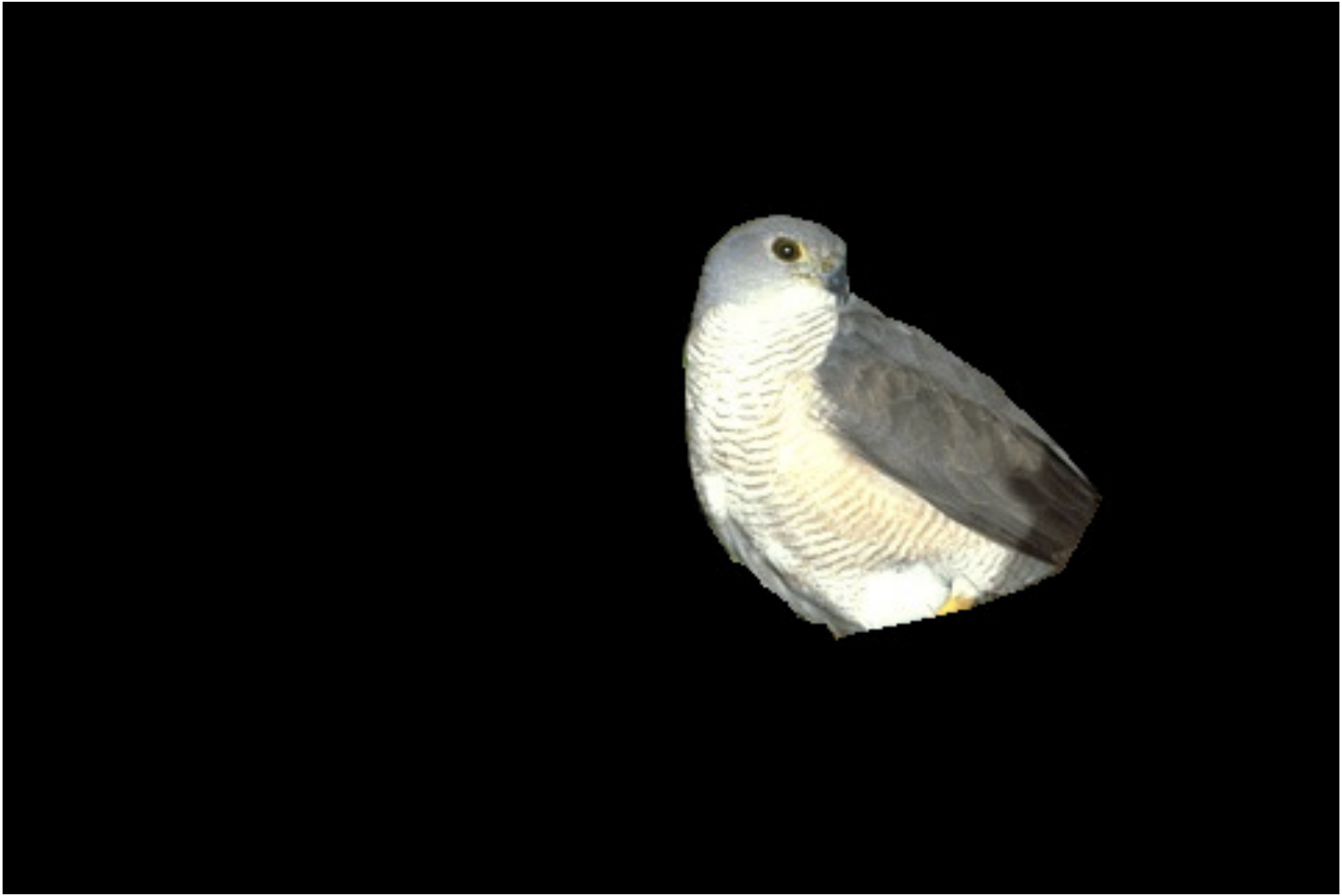}
\centering
\end{minipage}
}\\
\vspace{-2mm}
\subfloat{
\begin{minipage}[c]{0.001\textwidth}
\begin{turn}{90} {\scriptsize Graph cuts} \end{turn}
\end{minipage}
\hspace{0.2mm}
\begin{minipage}[c]{0.001\textwidth}
\begin{turn}{90} {\scriptsize unary} \end{turn}
\end{minipage}
\hspace{-2.3mm}
\begin{minipage}[c]{0.996\textwidth}
\includegraphics[width=0.113\textwidth]{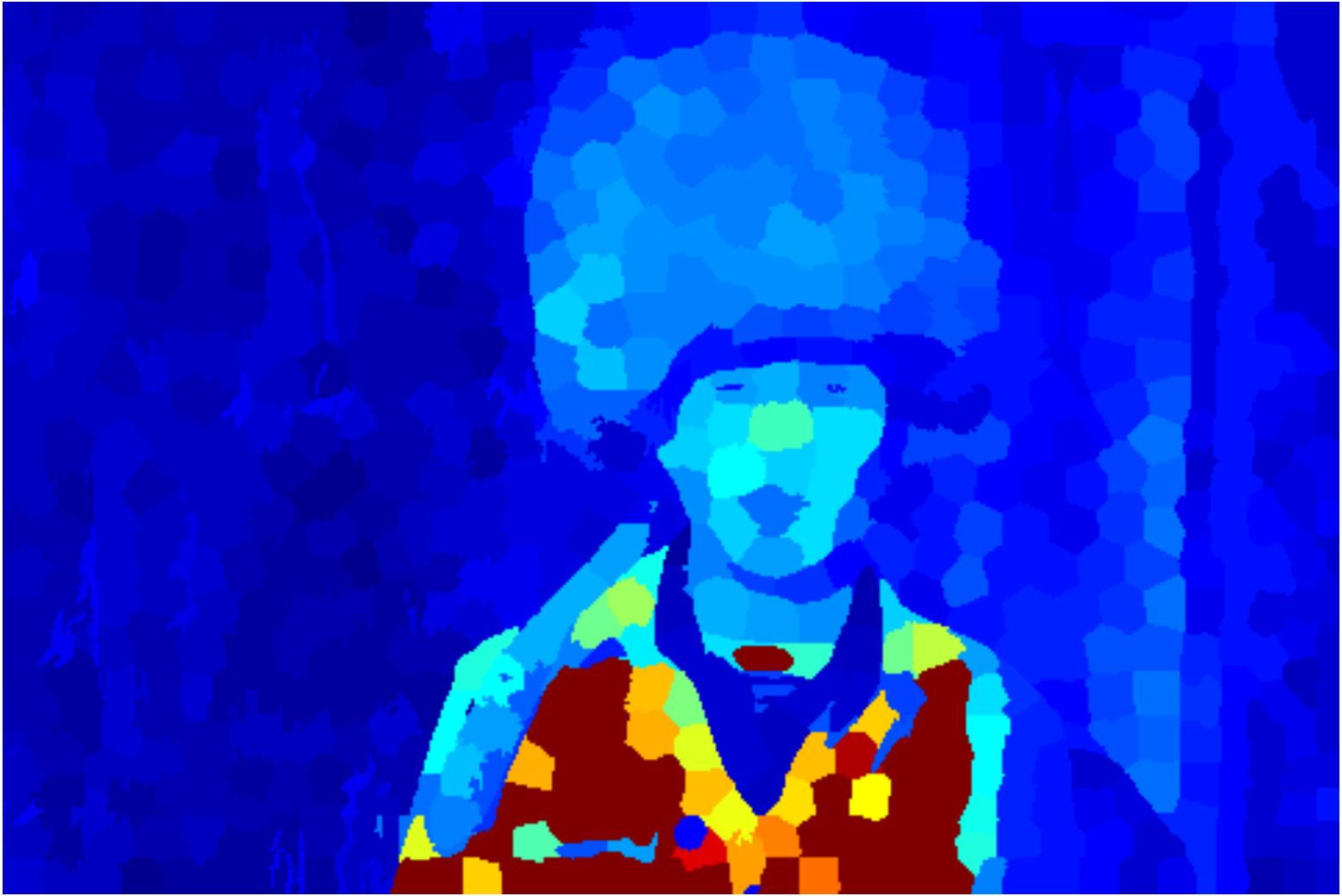}
\includegraphics[width=0.113\textwidth]{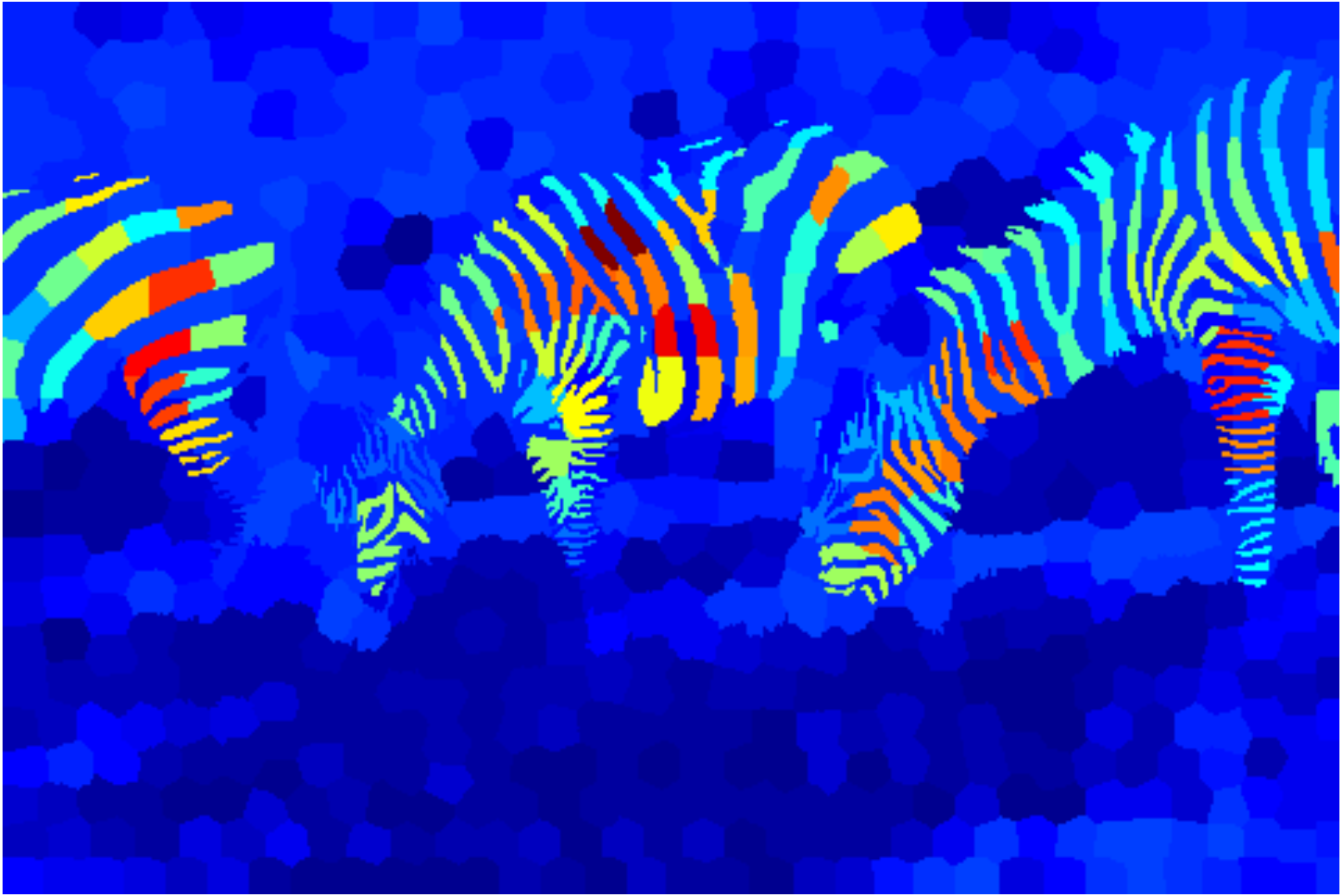}
\includegraphics[width=0.113\textwidth]{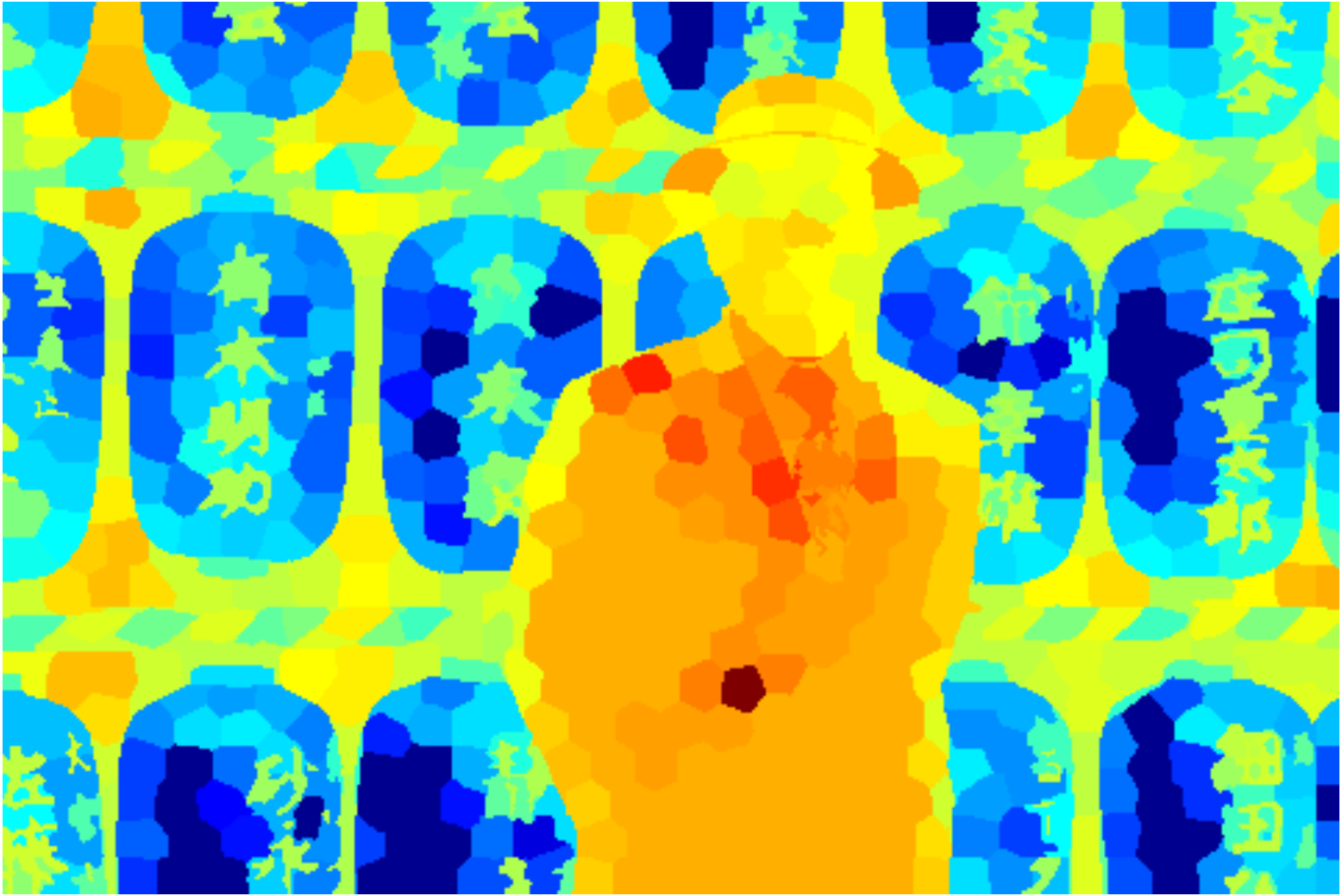}
\includegraphics[width=0.113\textwidth]{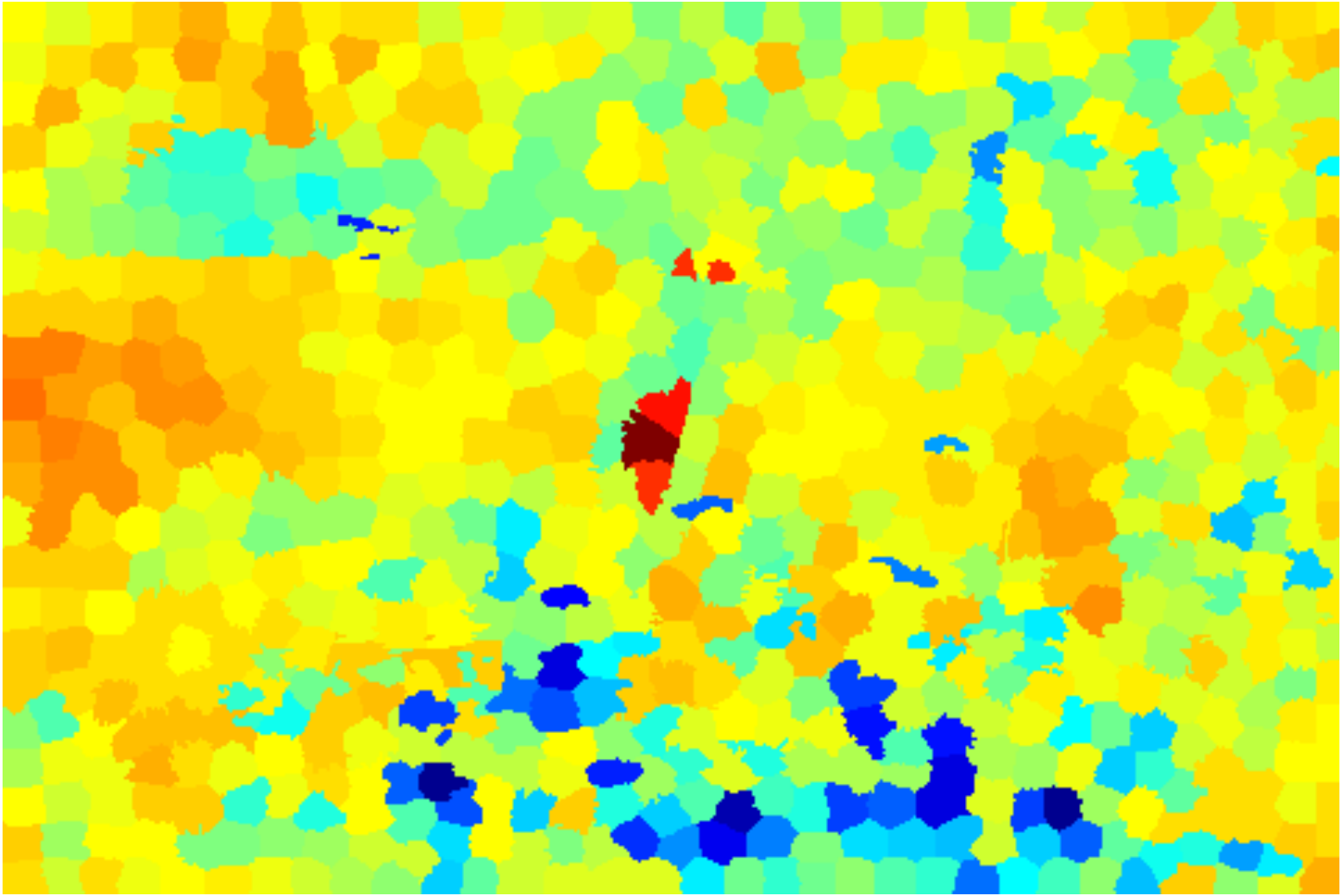}
\includegraphics[width=0.113\textwidth]{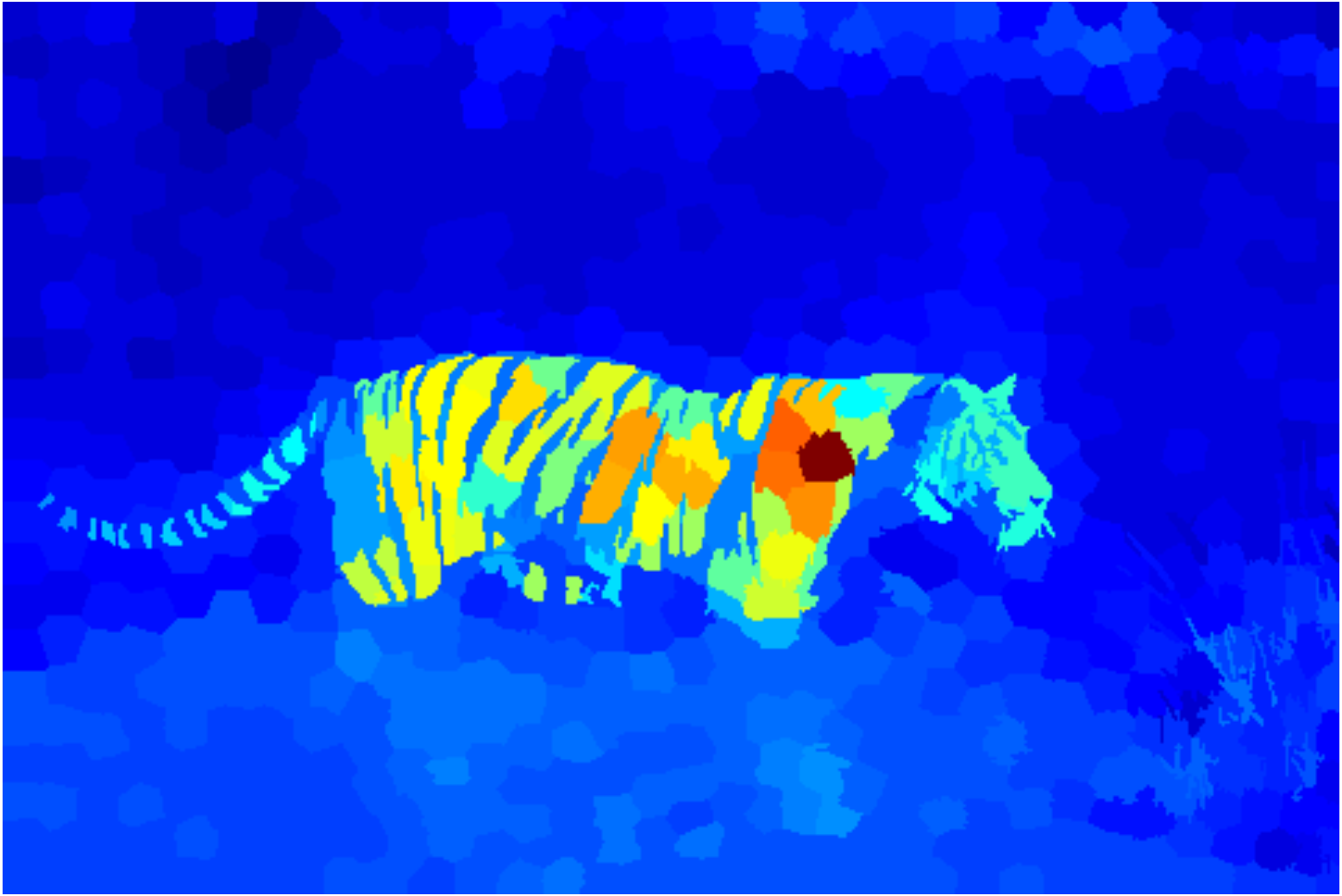}
\includegraphics[width=0.113\textwidth]{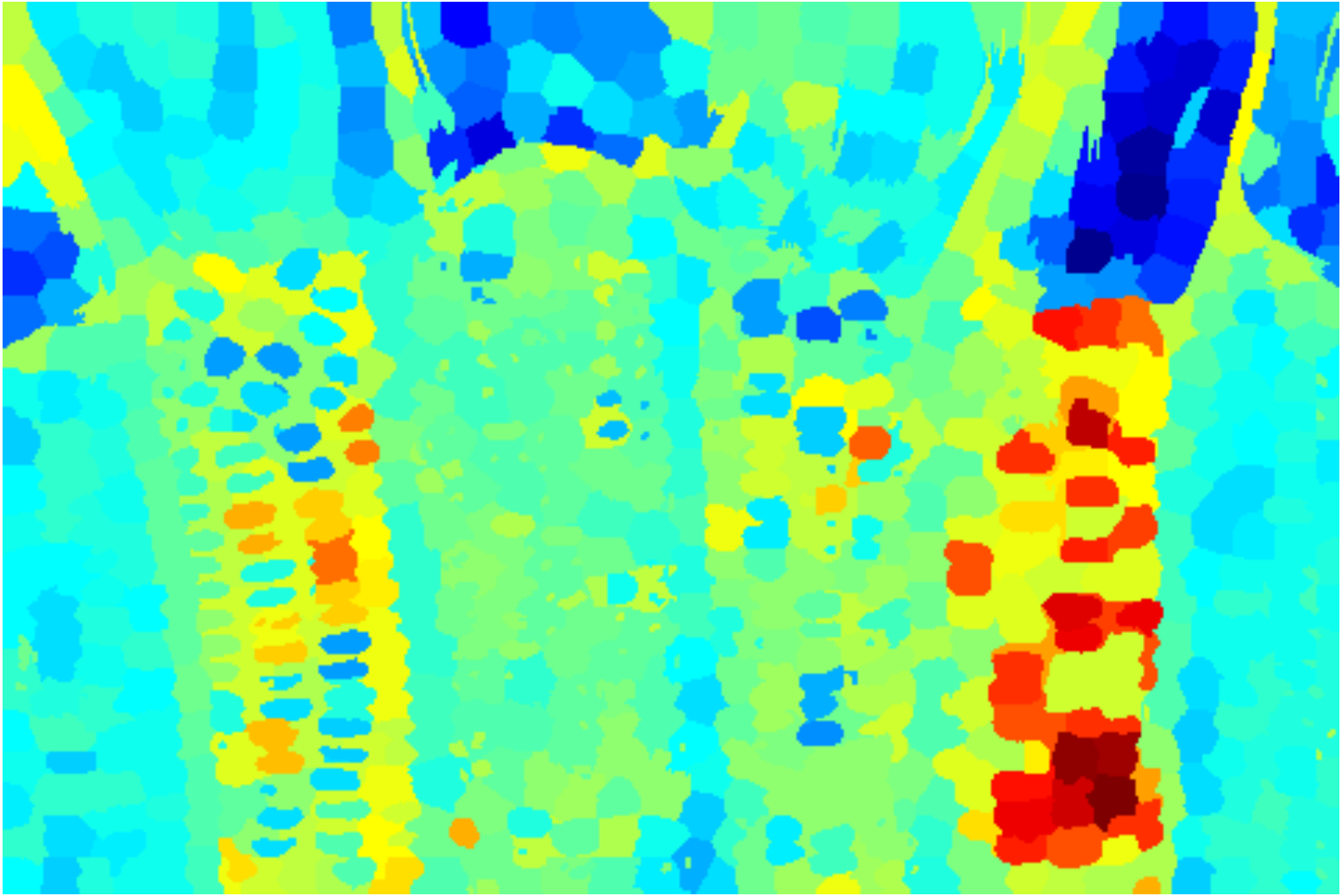}
\includegraphics[width=0.113\textwidth]{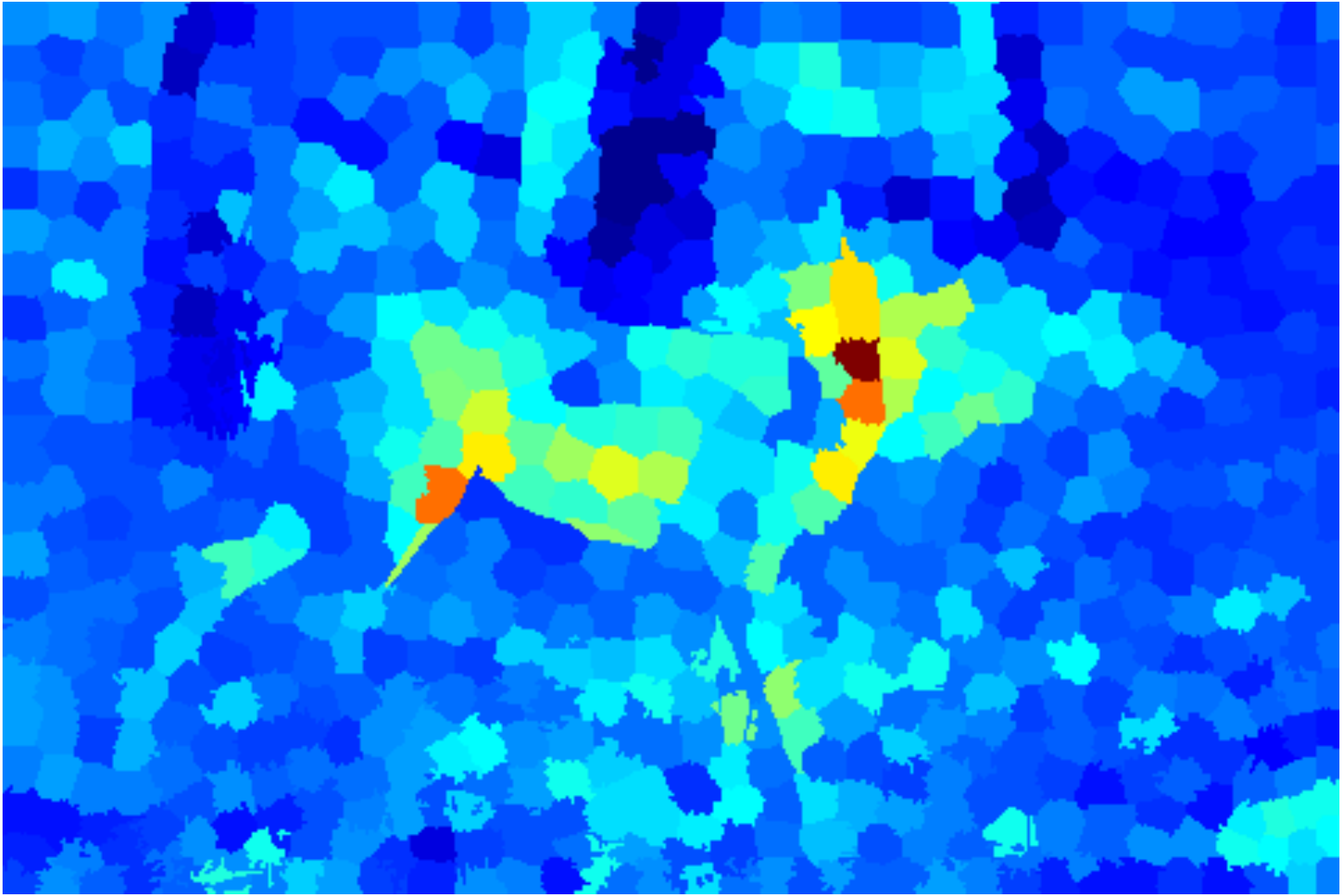}
\includegraphics[width=0.113\textwidth]{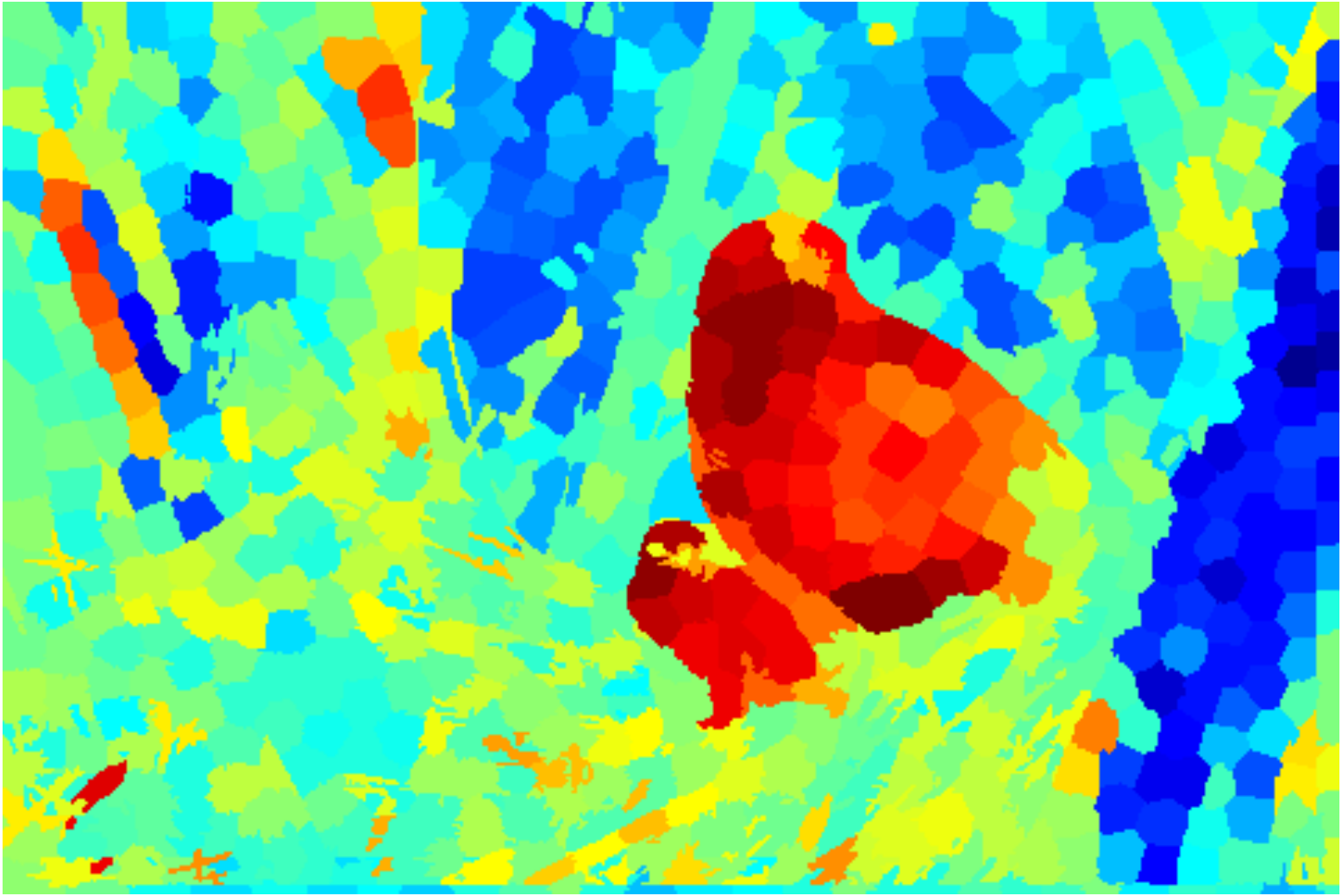}
\centering
\end{minipage}
}\\
\vspace{-2mm}
\subfloat{
\begin{minipage}[c]{0.002\textwidth}
\begin{turn}{90} {\scriptsize Graph cuts} \end{turn}
\end{minipage}
\hspace{-0.2mm}
\begin{minipage}[c]{0.996\textwidth}
\includegraphics[width=0.113\textwidth]{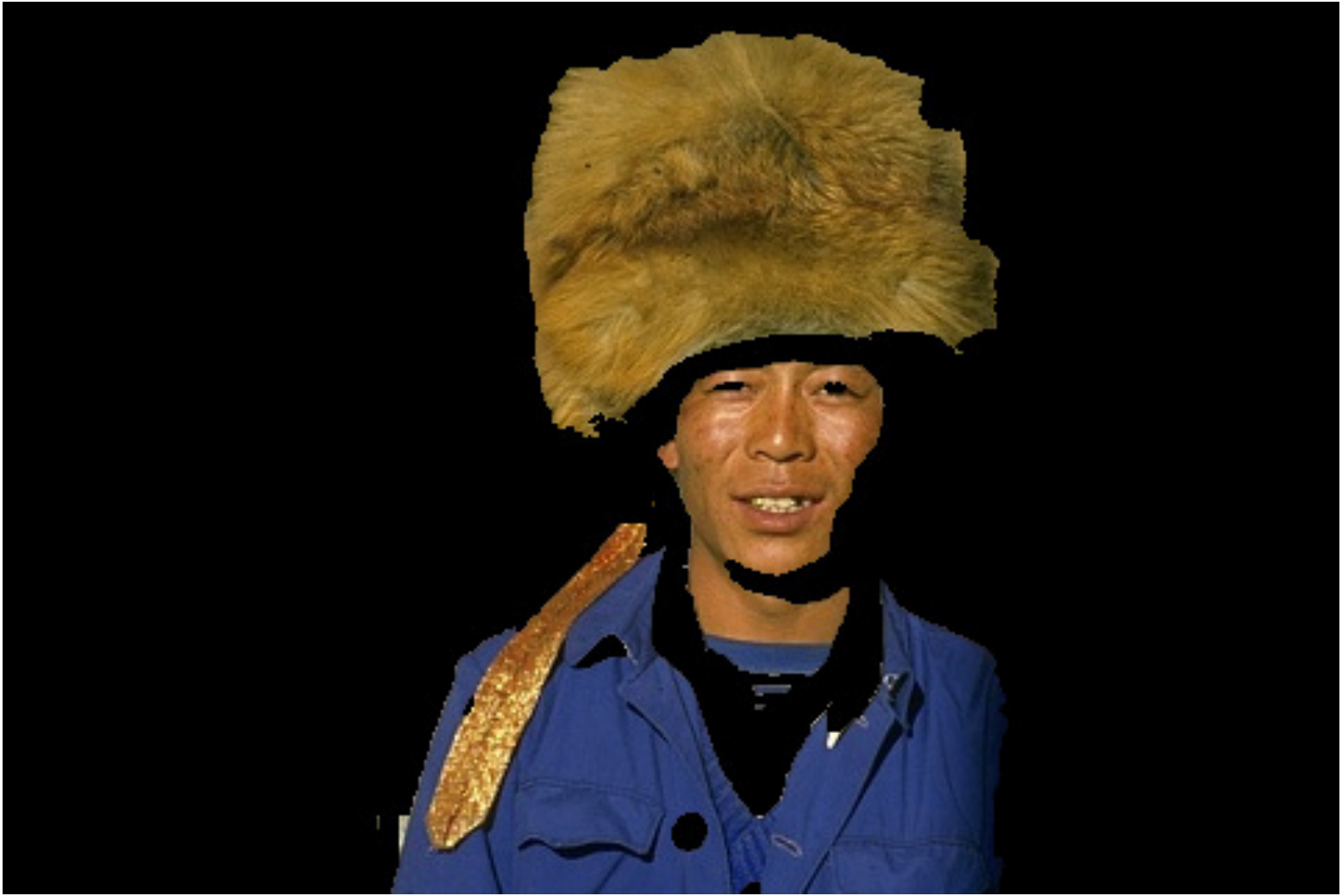}
\includegraphics[width=0.113\textwidth]{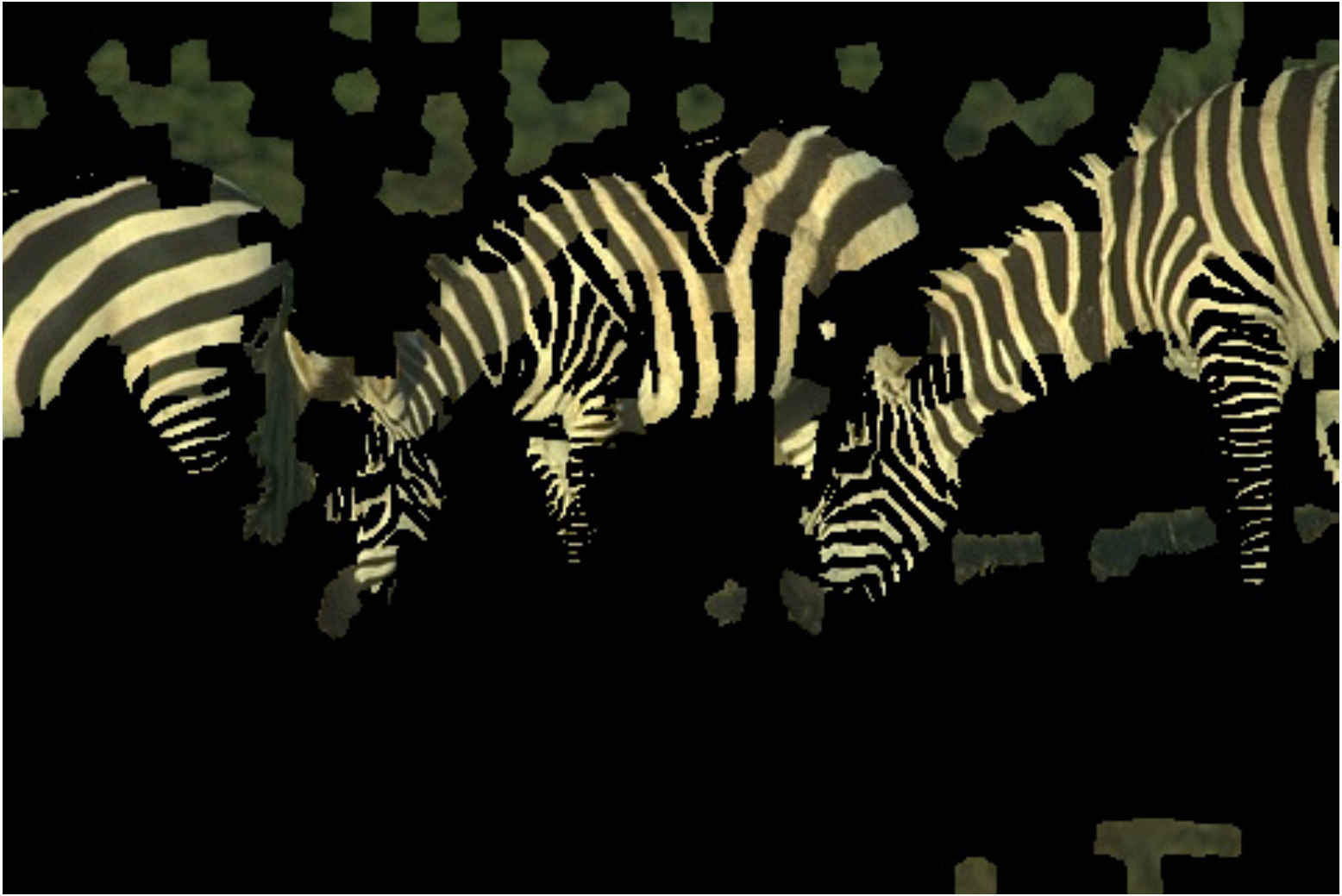}
\includegraphics[width=0.113\textwidth]{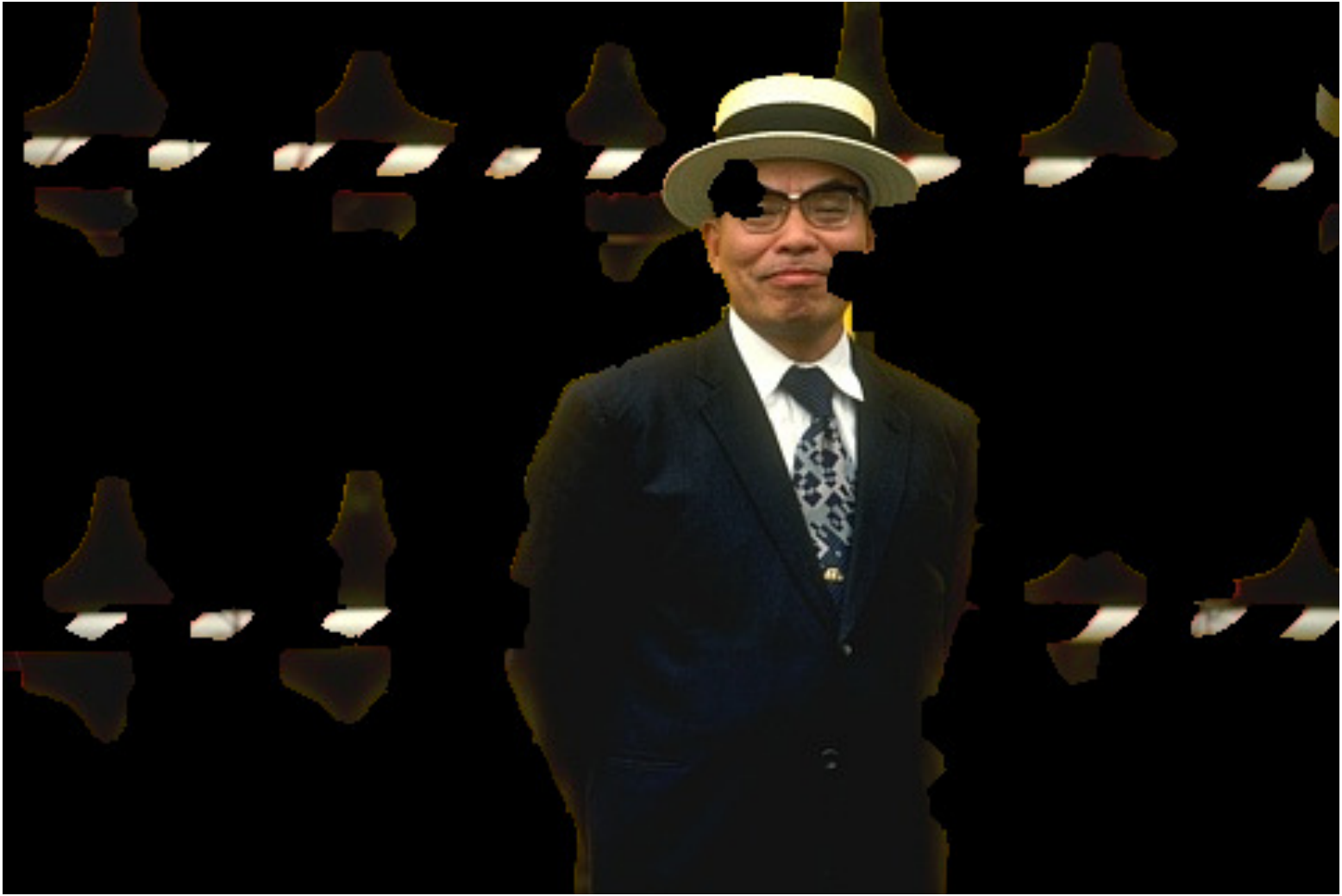}
\includegraphics[width=0.113\textwidth]{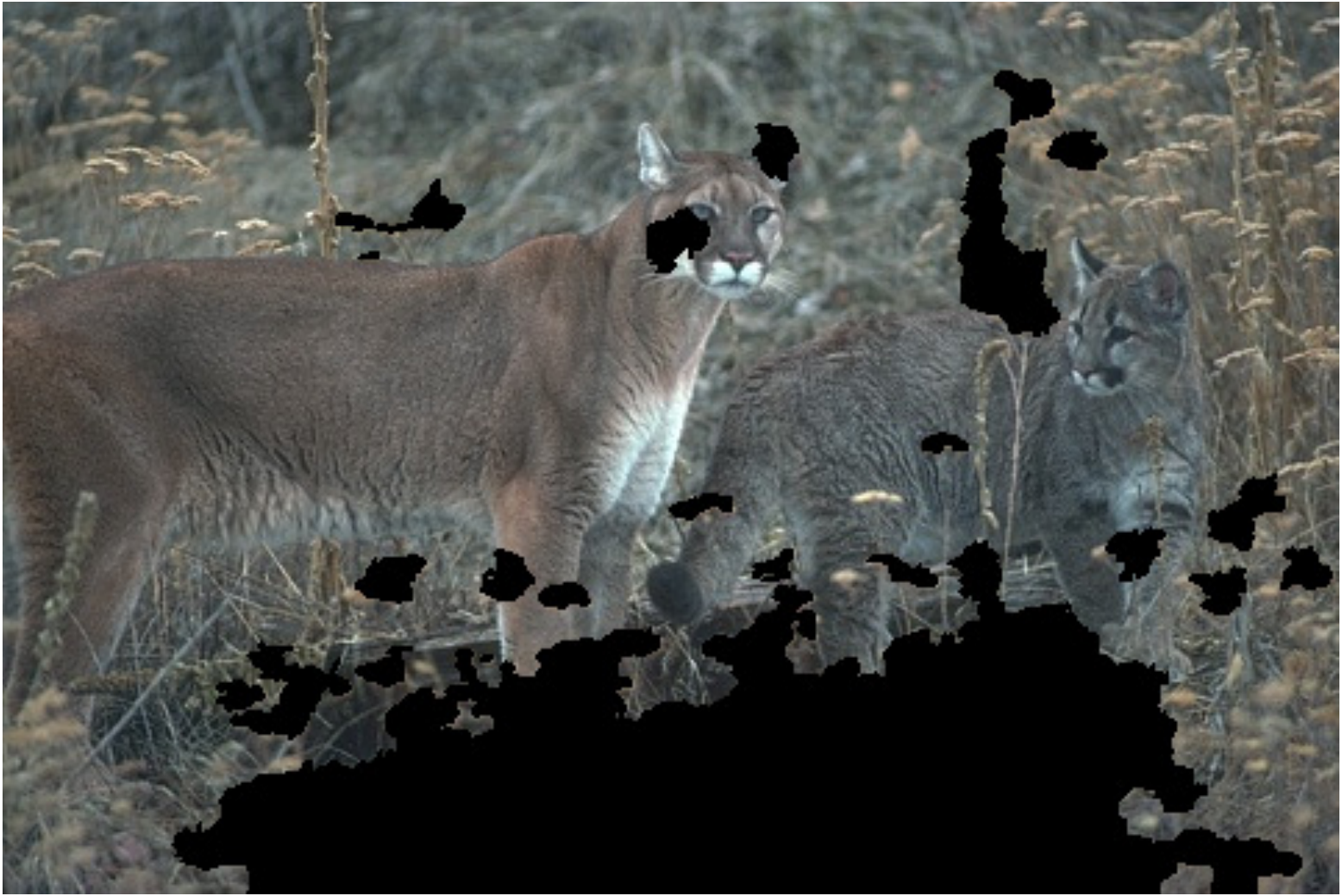}
\includegraphics[width=0.113\textwidth]{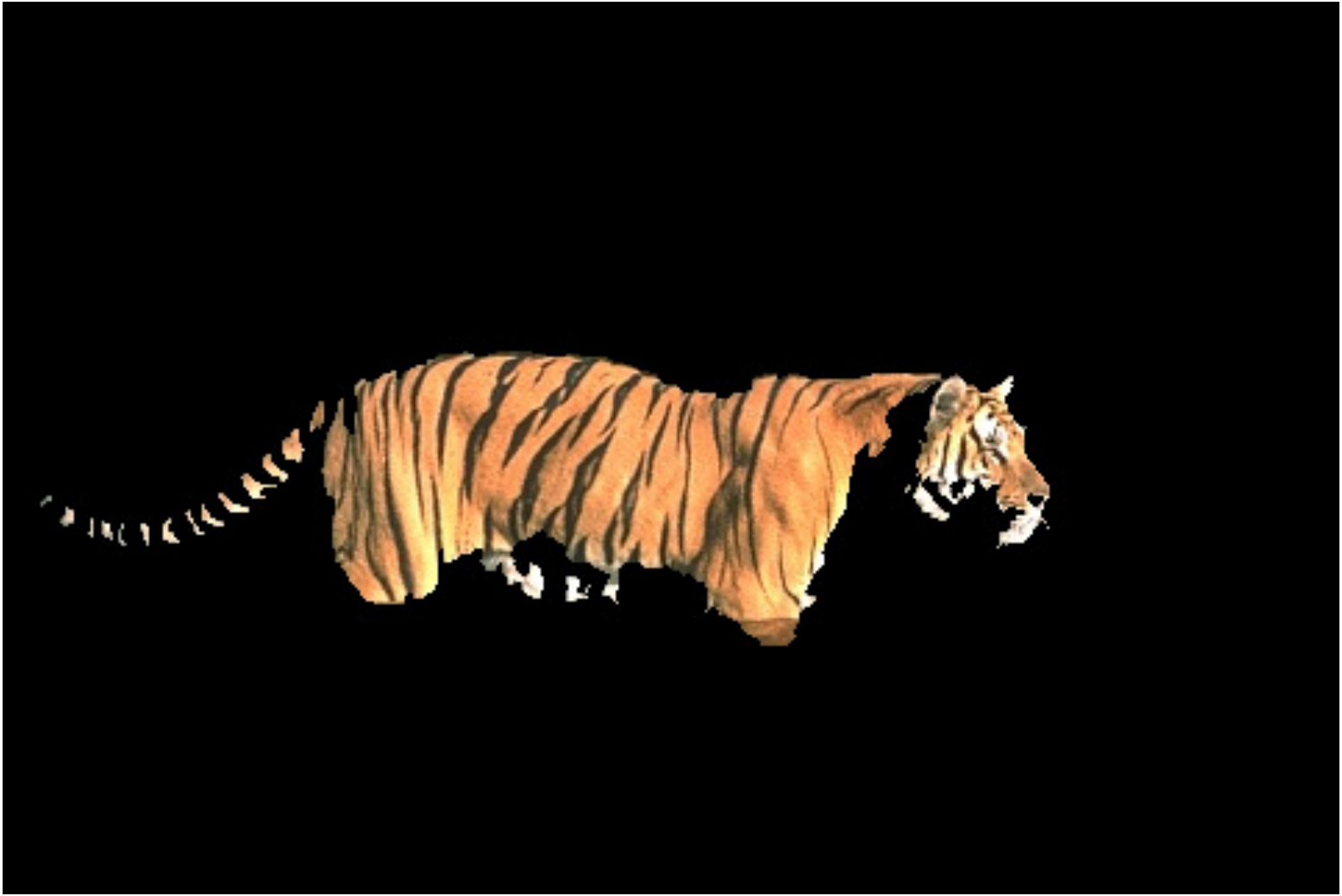}
\includegraphics[width=0.113\textwidth]{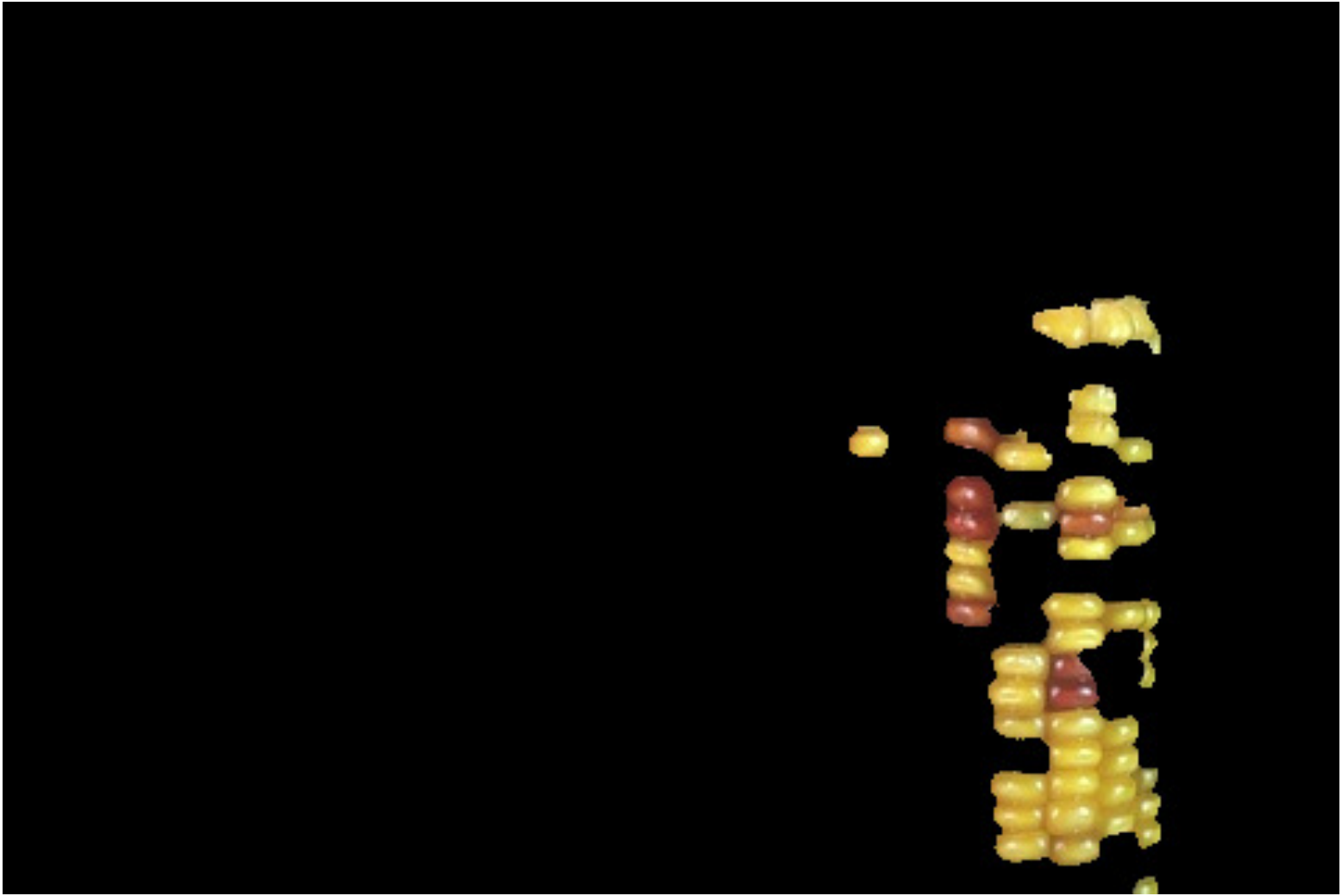}
\includegraphics[width=0.113\textwidth]{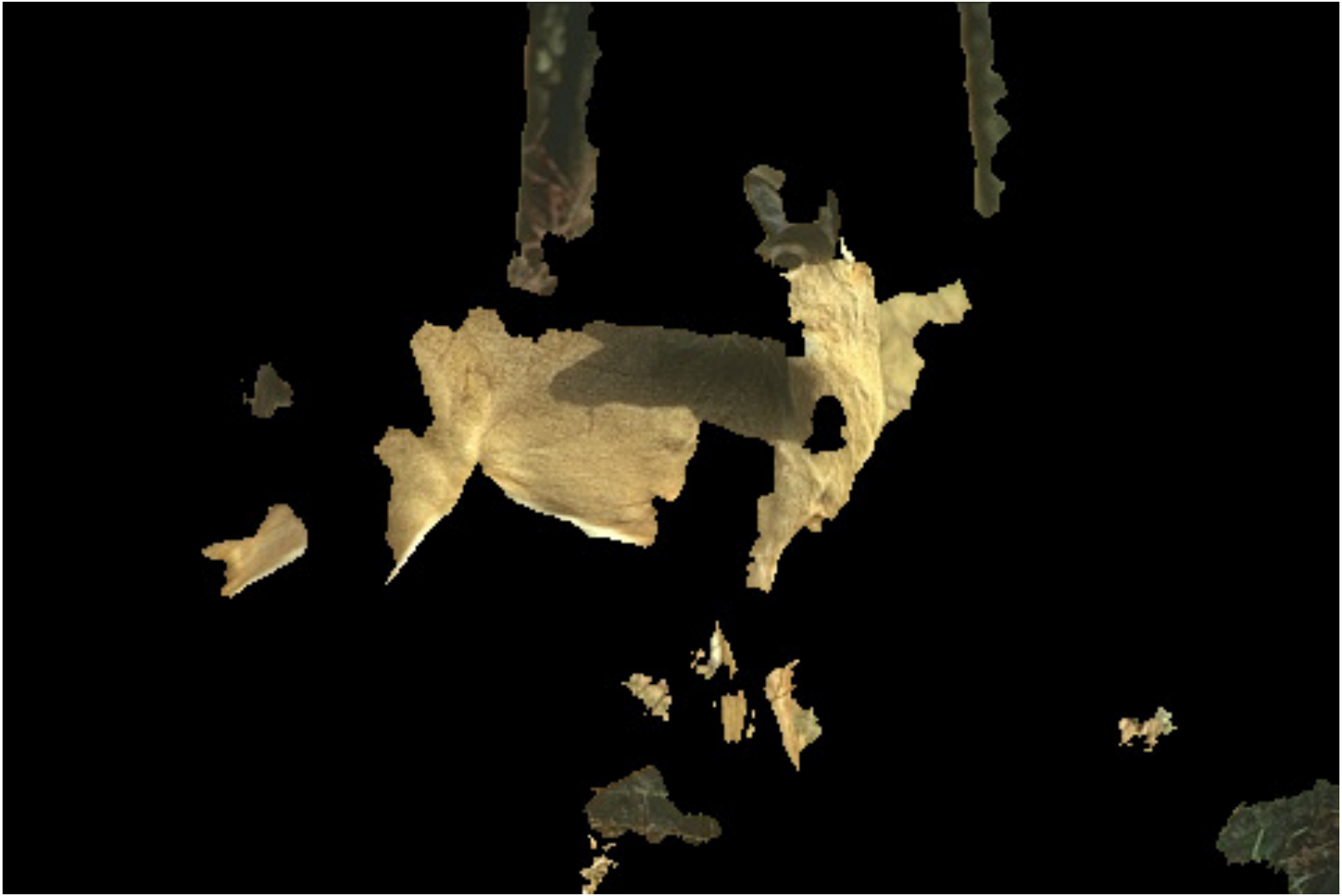}
\includegraphics[width=0.113\textwidth]{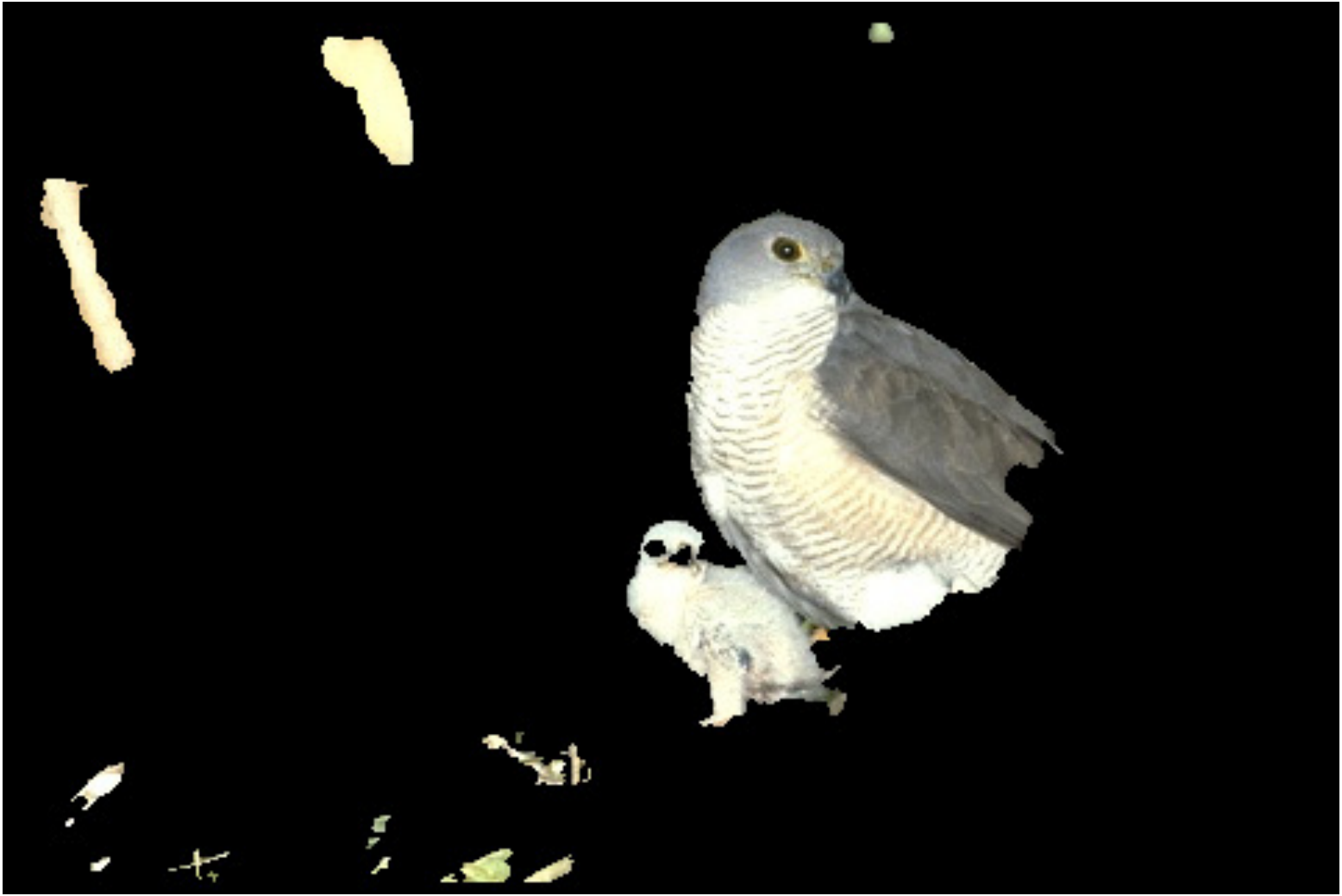}
\centering
\end{minipage}
}\\
\vspace{-2mm}
\subfloat{
\begin{minipage}[c]{0.002\textwidth}
\begin{turn}{90} {\scriptsize SMQC} \end{turn}
\end{minipage}
\hspace{-0.2mm}
\begin{minipage}[c]{0.996\textwidth}
\includegraphics[width=0.113\textwidth]{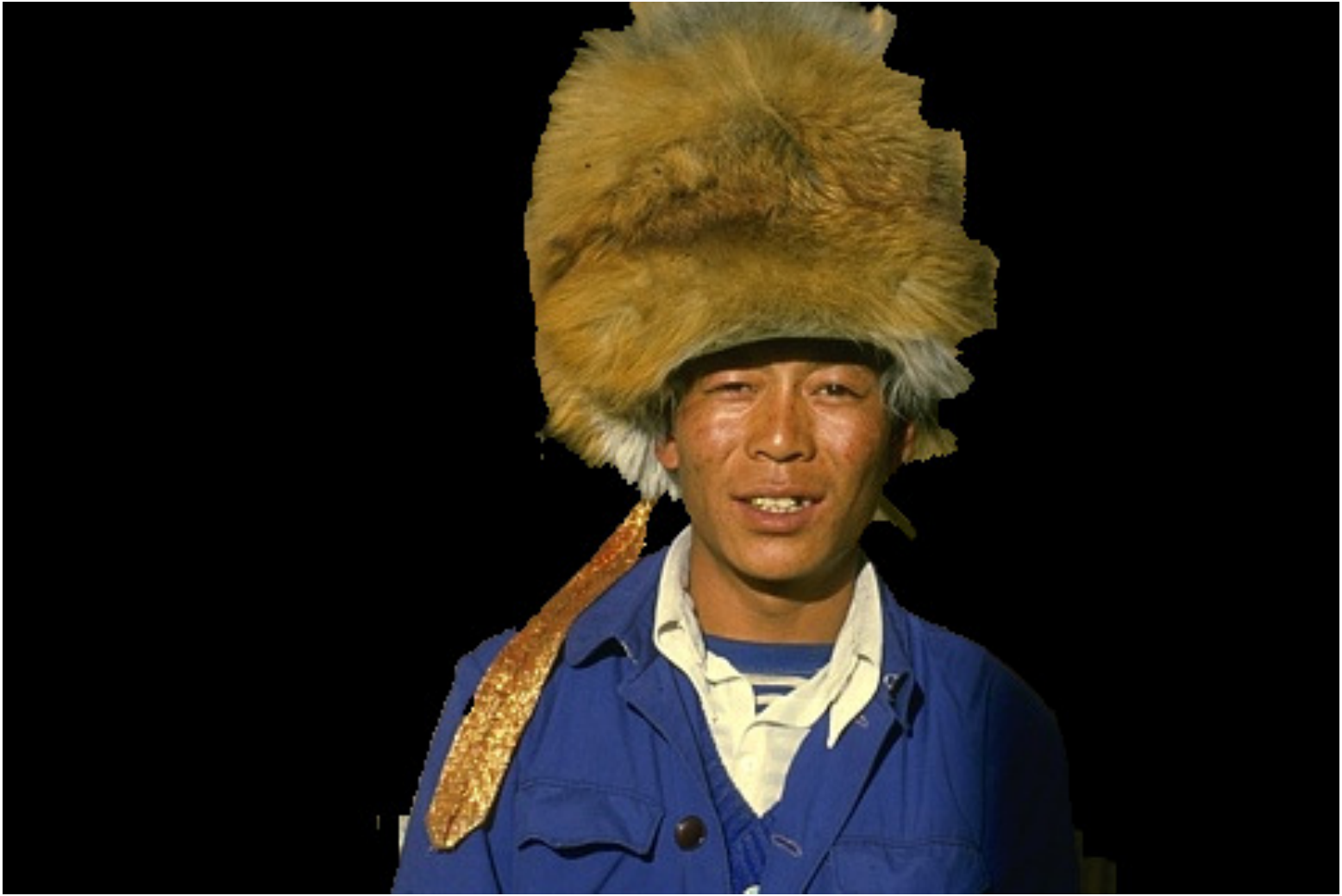}
\includegraphics[width=0.113\textwidth]{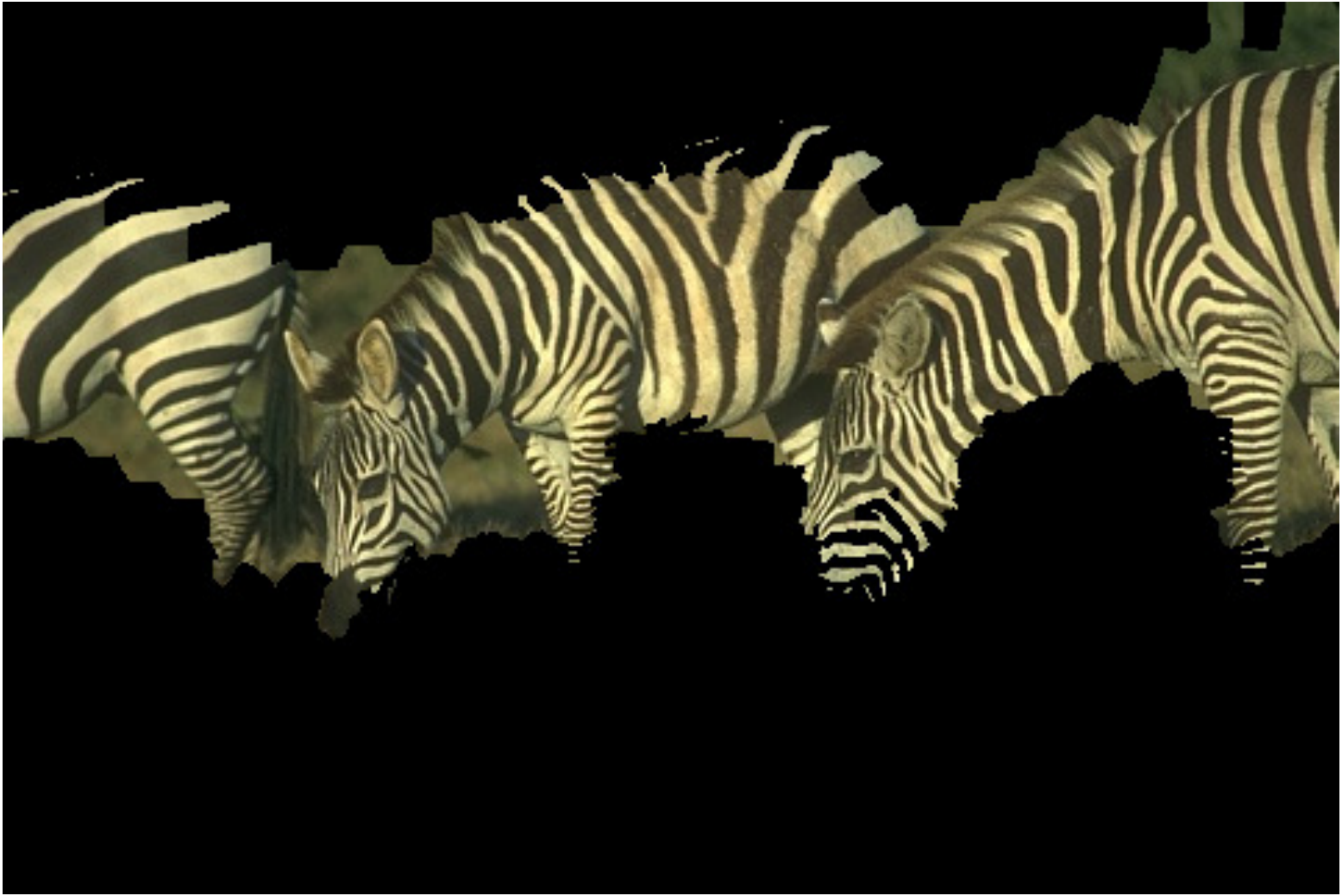}
\includegraphics[width=0.113\textwidth]{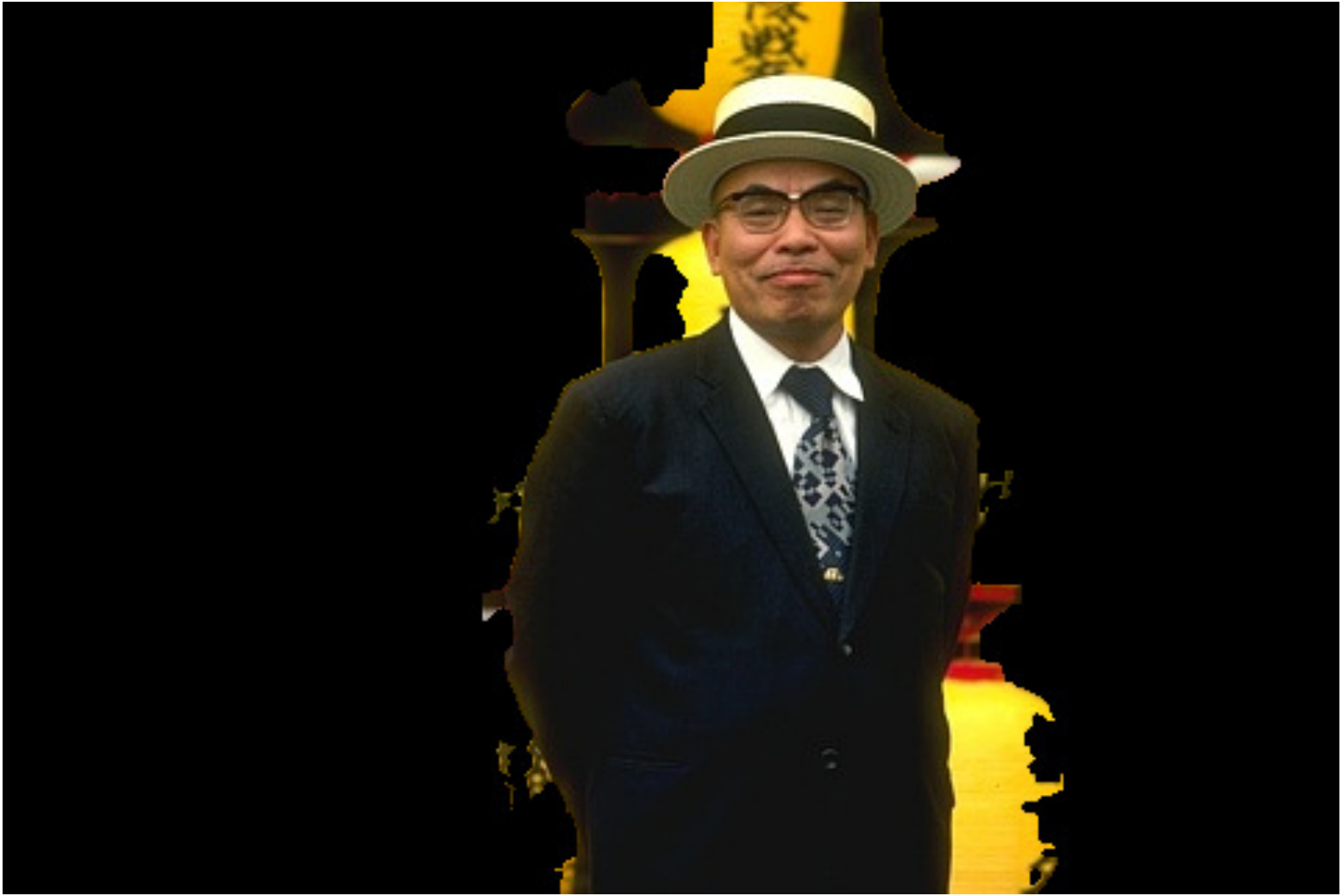}
\includegraphics[width=0.113\textwidth]{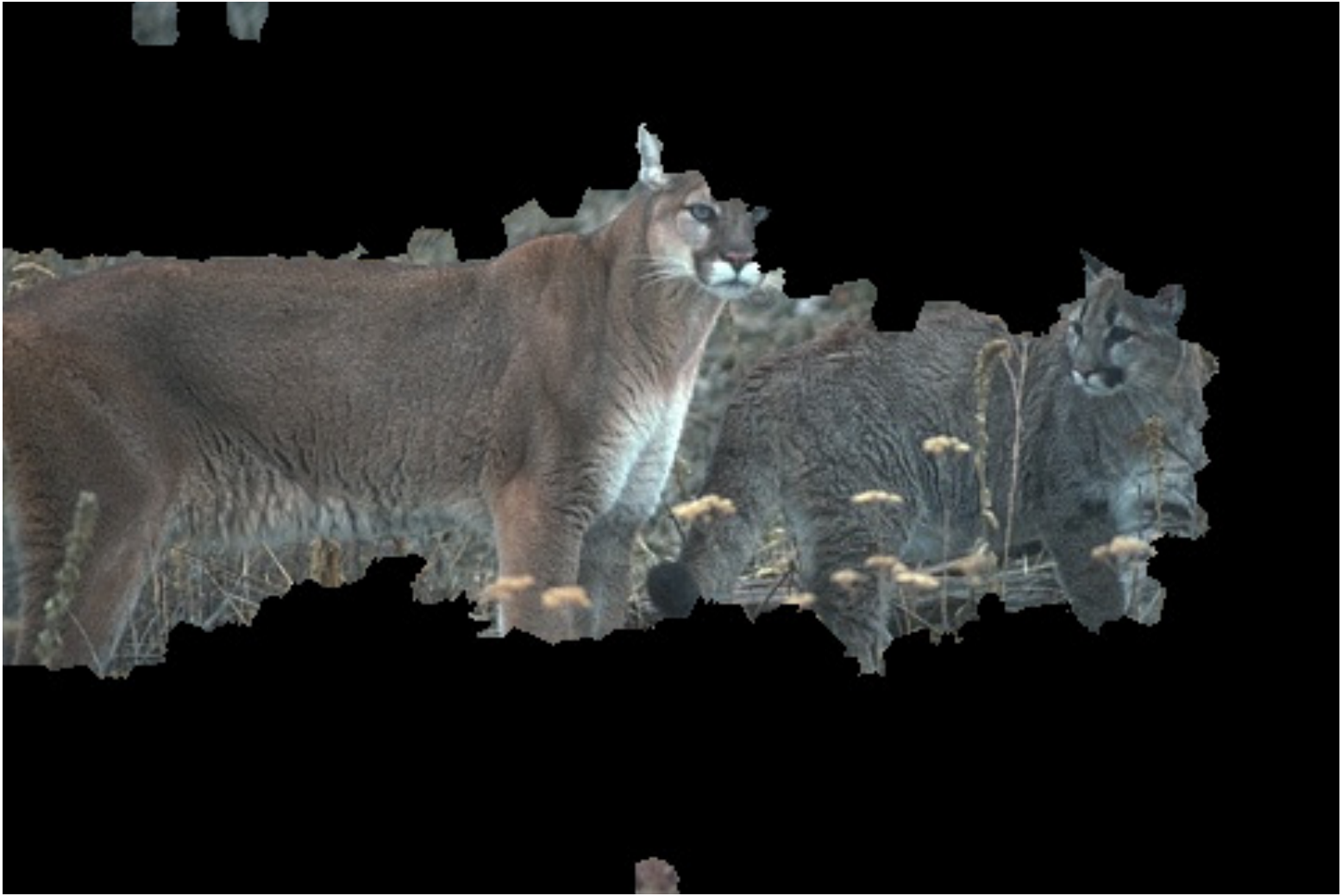}
\includegraphics[width=0.113\textwidth]{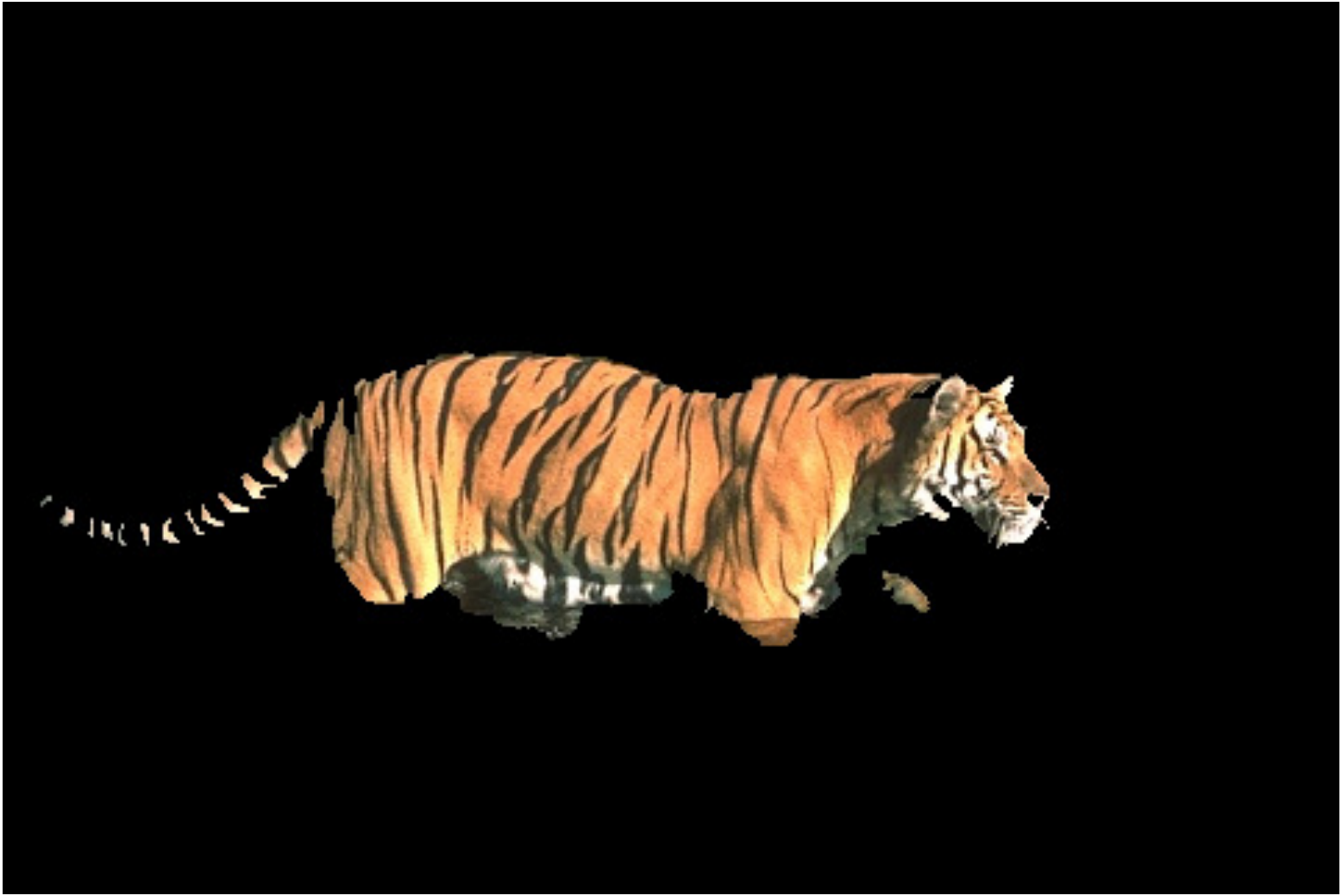}
\includegraphics[width=0.113\textwidth]{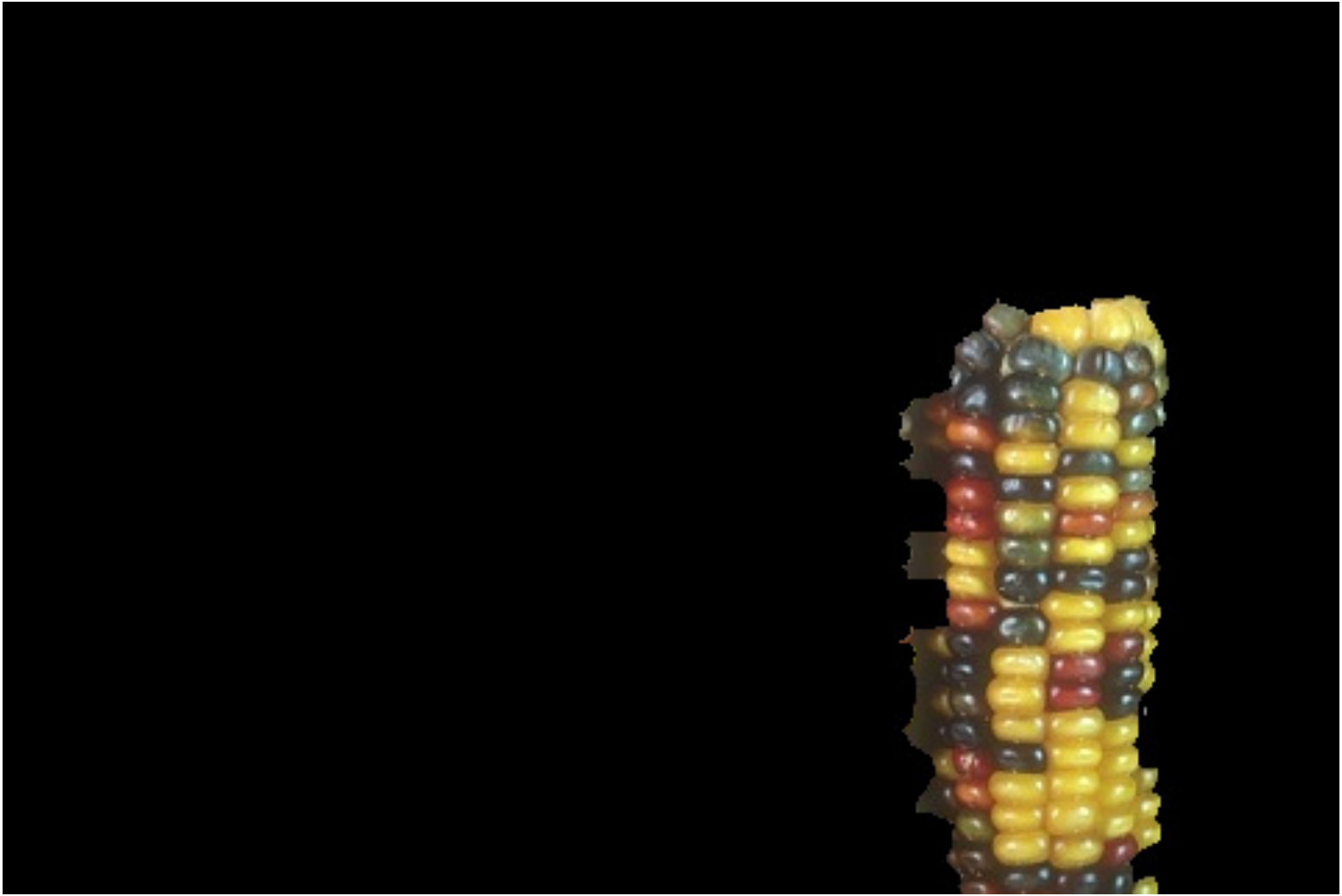}
\includegraphics[width=0.113\textwidth]{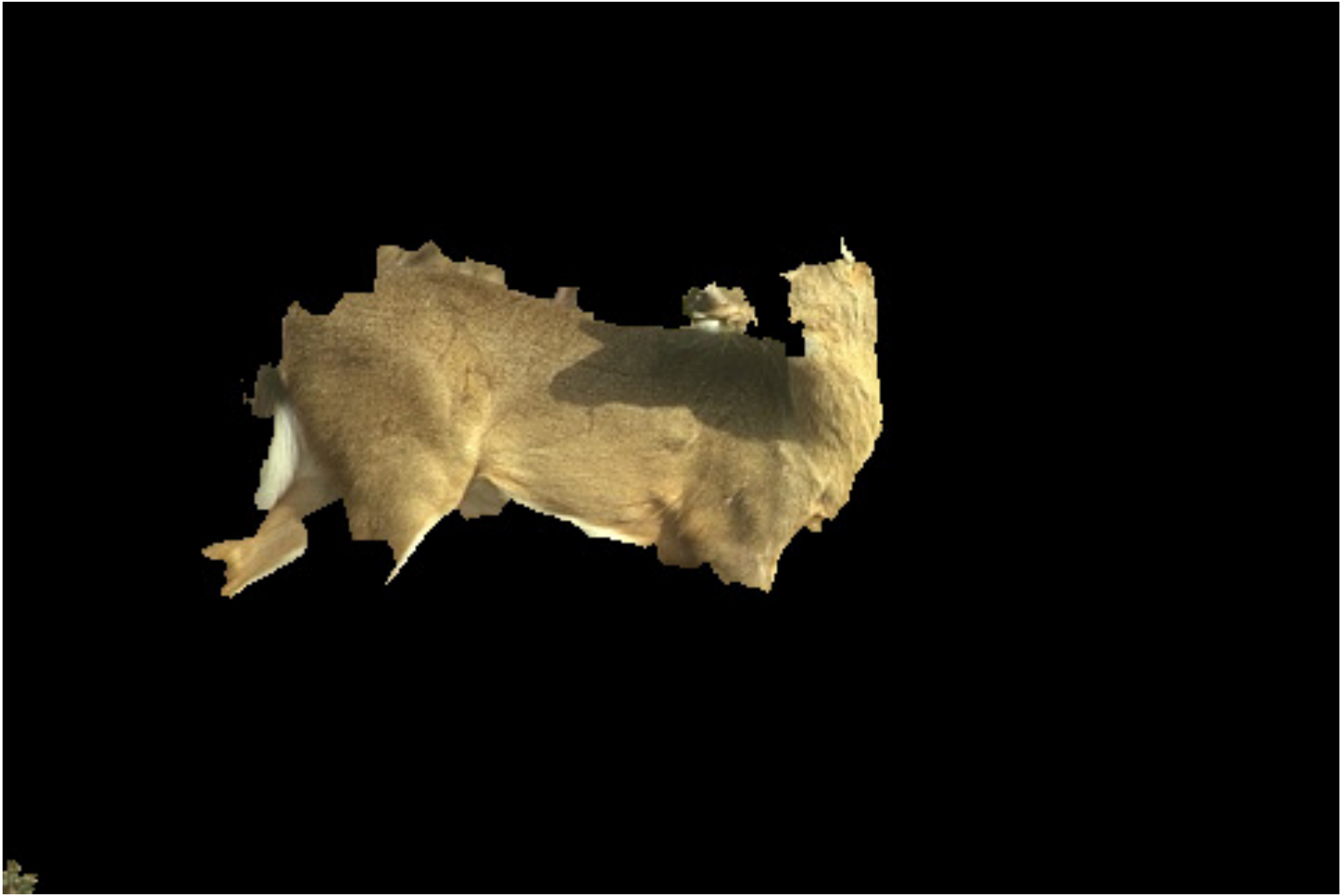}
\includegraphics[width=0.113\textwidth]{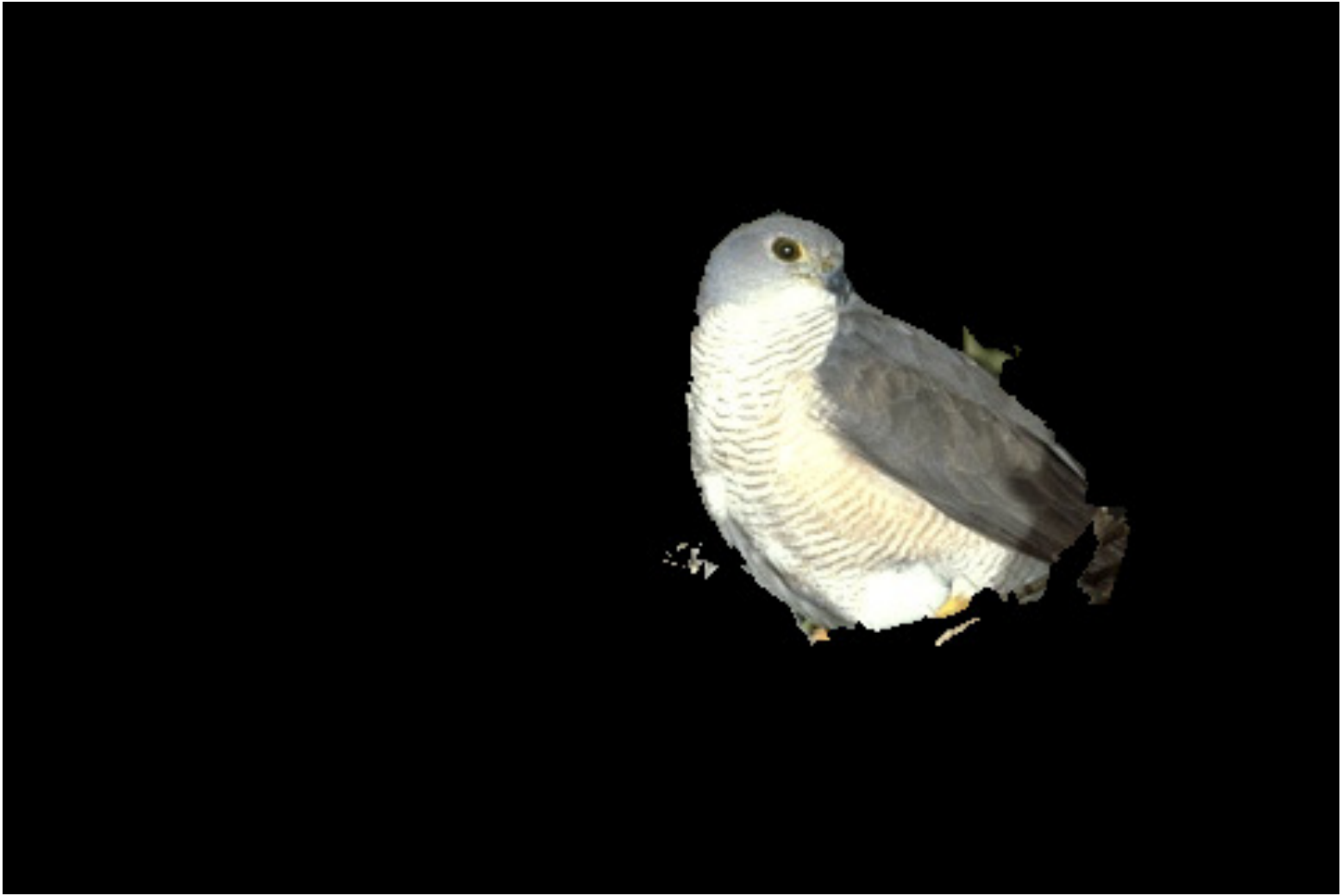}
\centering
\end{minipage}
}\\
\vspace{-2mm}
\subfloat{
\begin{minipage}[c]{0.002\textwidth}
\begin{turn}{90} {\scriptsize \lbfgsb} \end{turn}
\end{minipage}
\hspace{-0.2mm}
\begin{minipage}[c]{0.996\textwidth}
\includegraphics[width=0.113\textwidth]{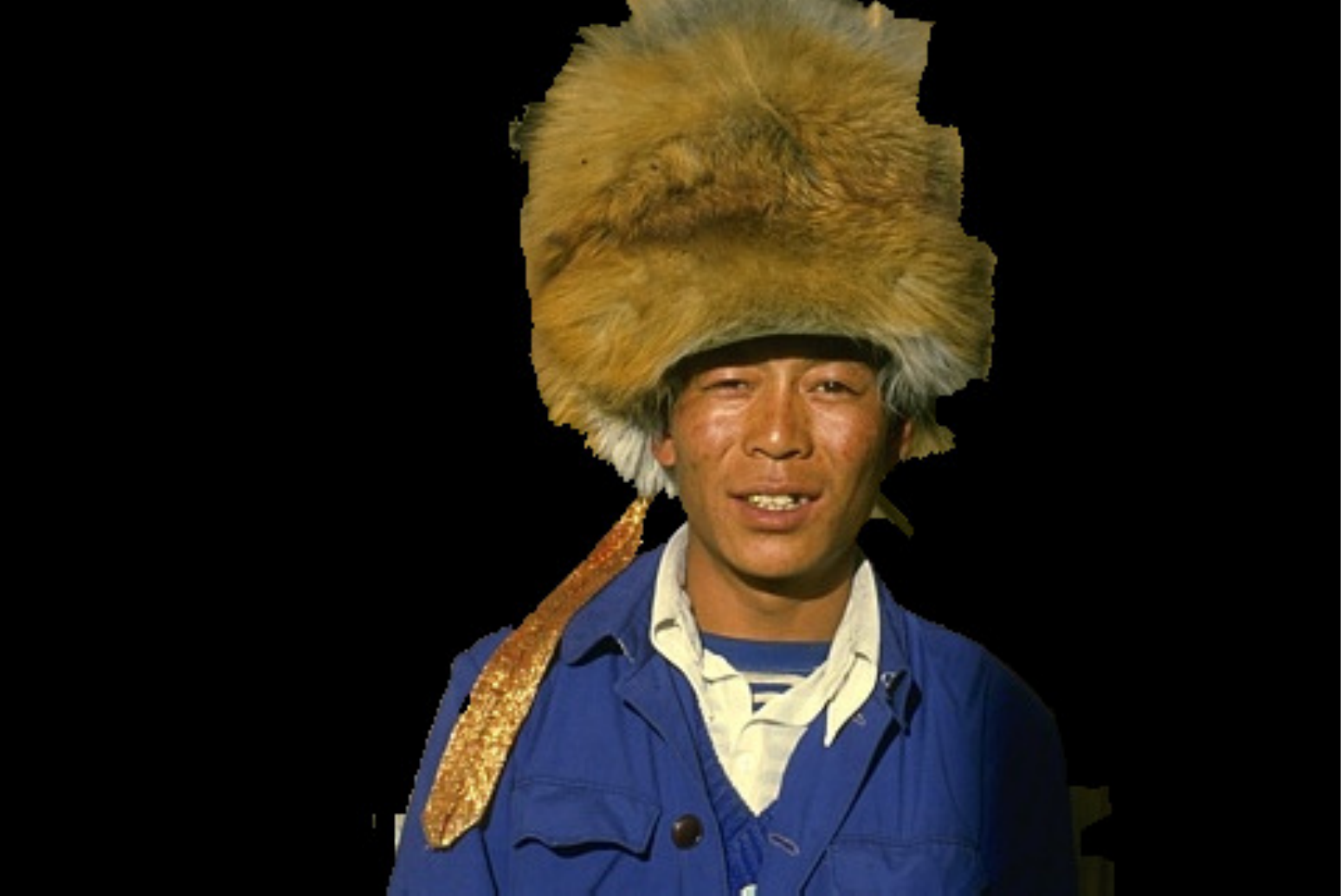}
\includegraphics[width=0.113\textwidth]{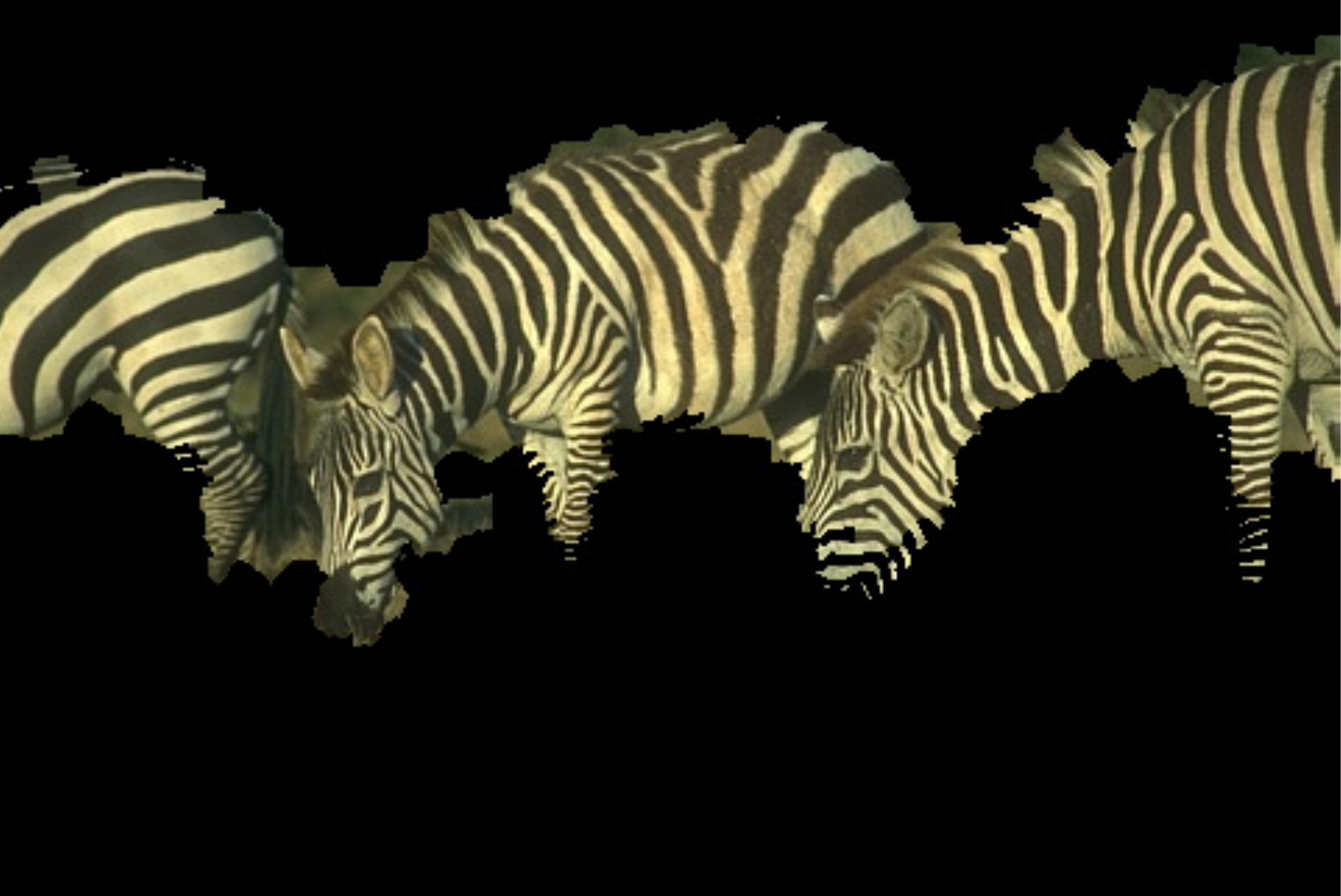}
\includegraphics[width=0.113\textwidth]{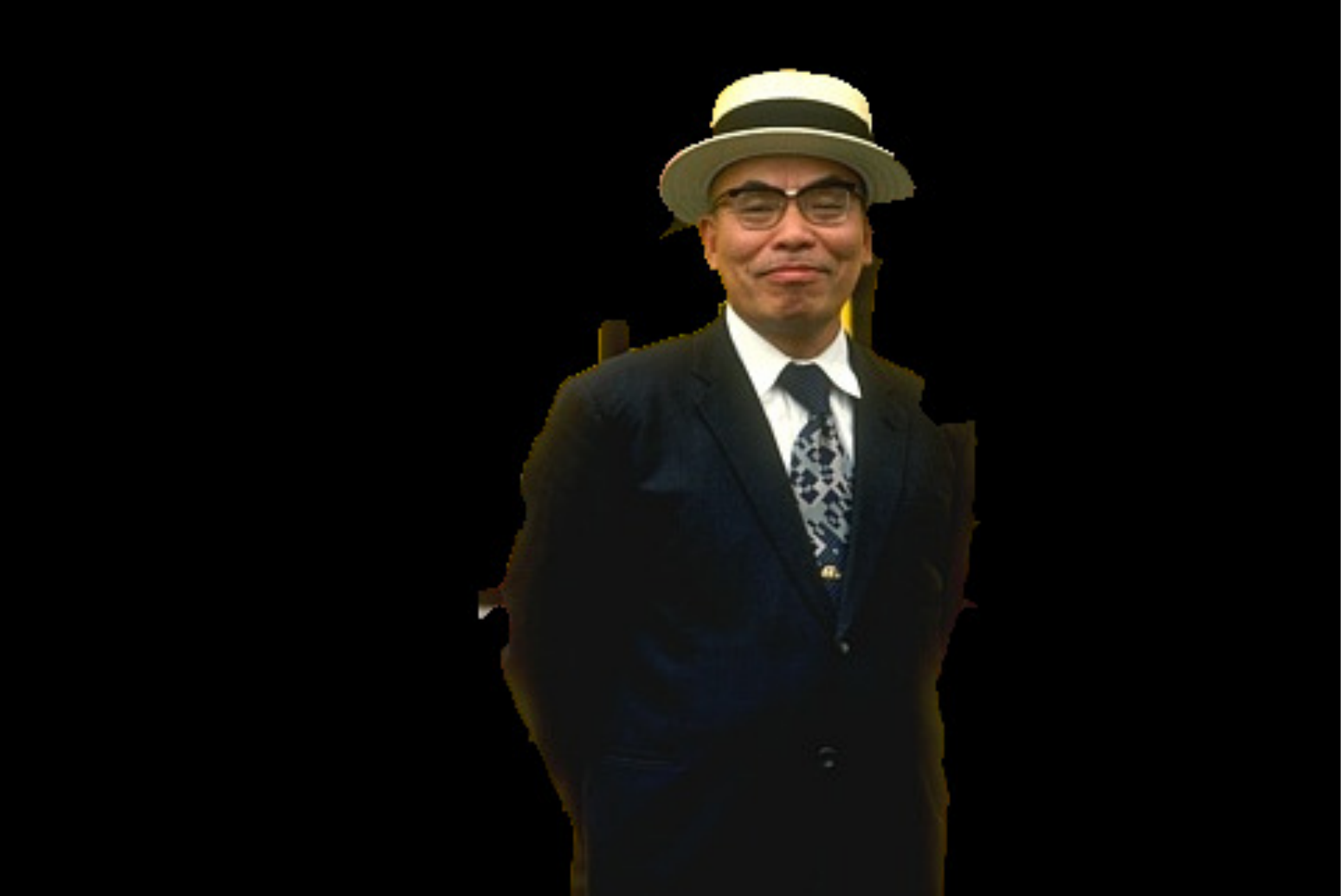}
\includegraphics[width=0.113\textwidth]{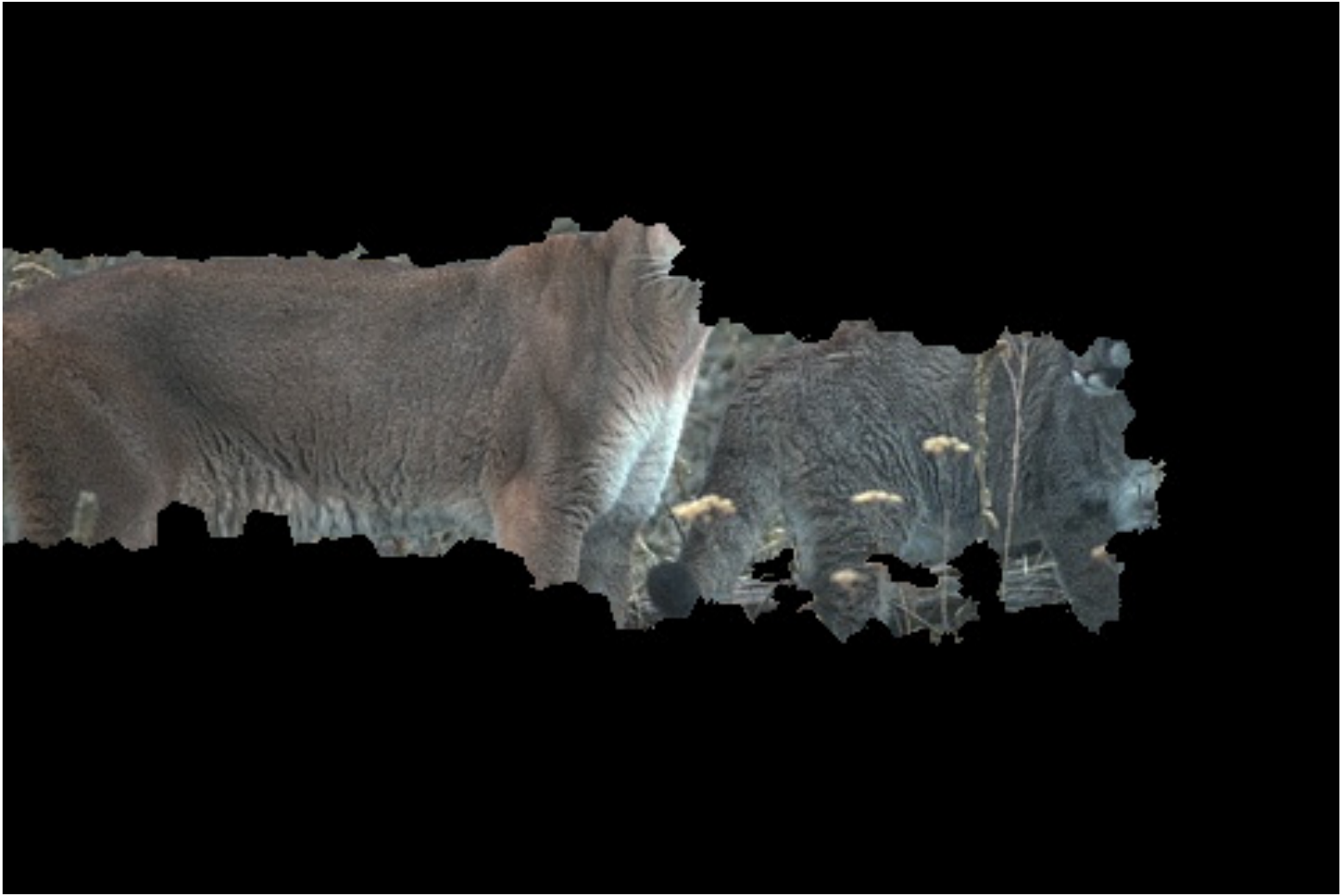}
\includegraphics[width=0.113\textwidth]{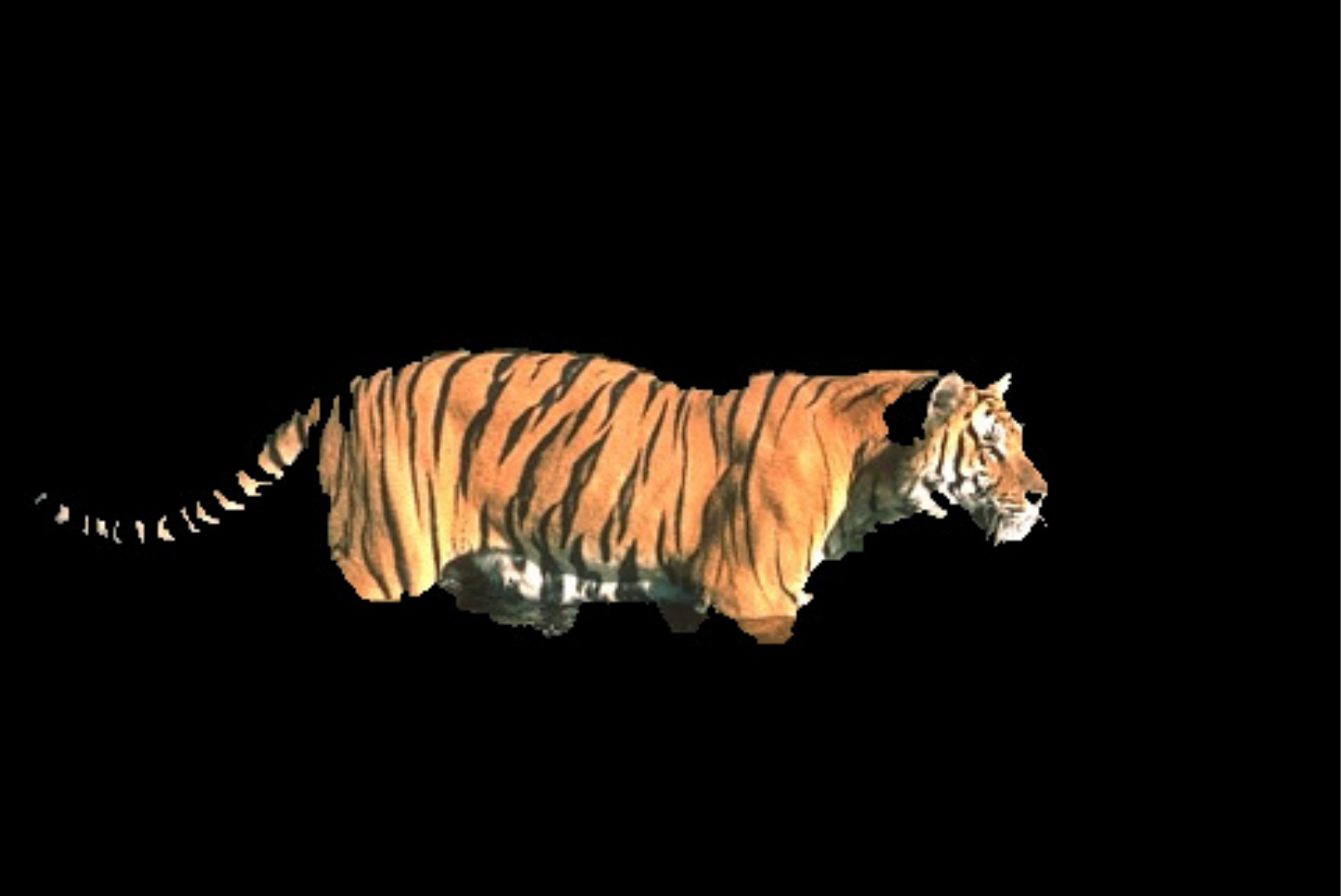}
\includegraphics[width=0.113\textwidth]{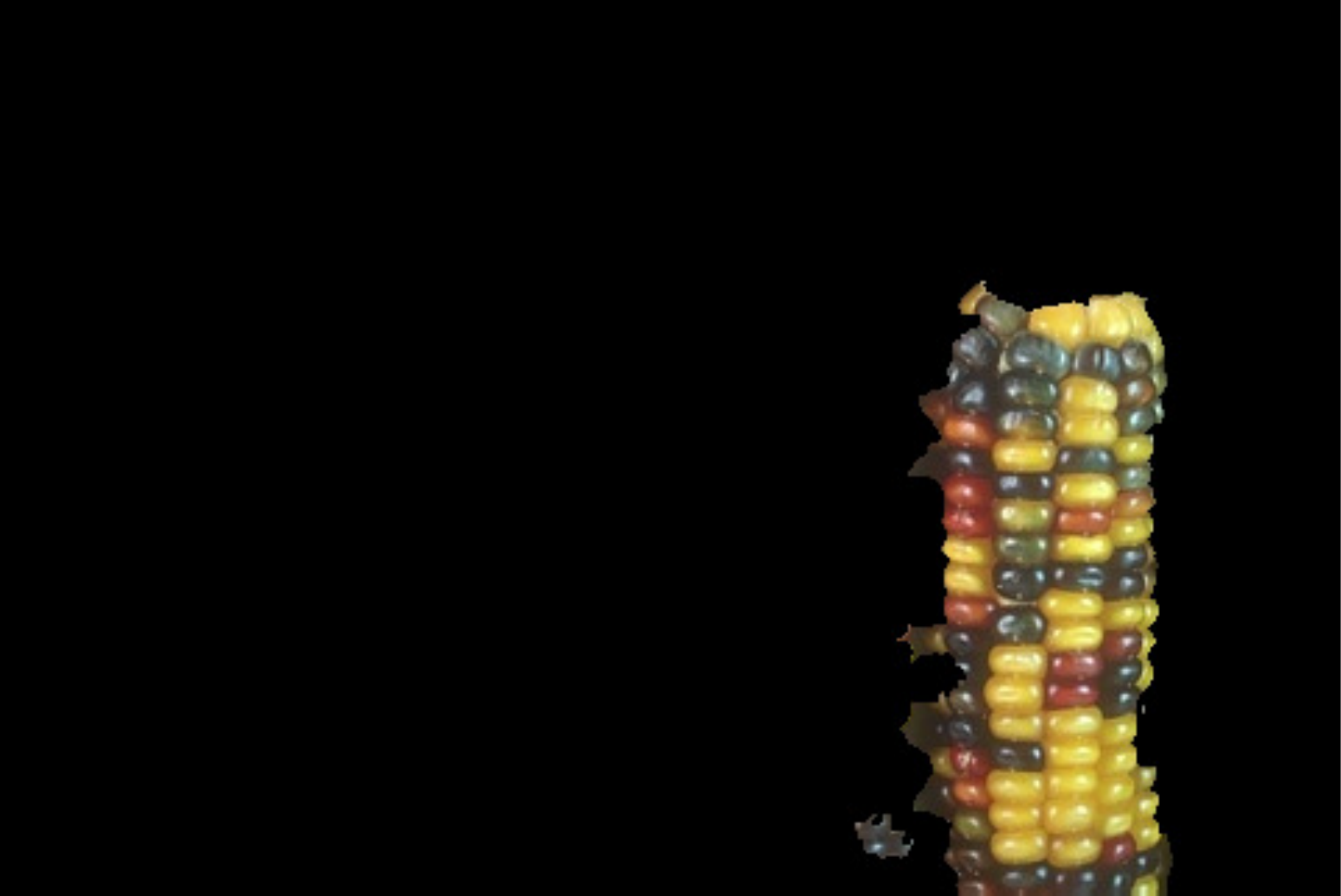}
\includegraphics[width=0.113\textwidth]{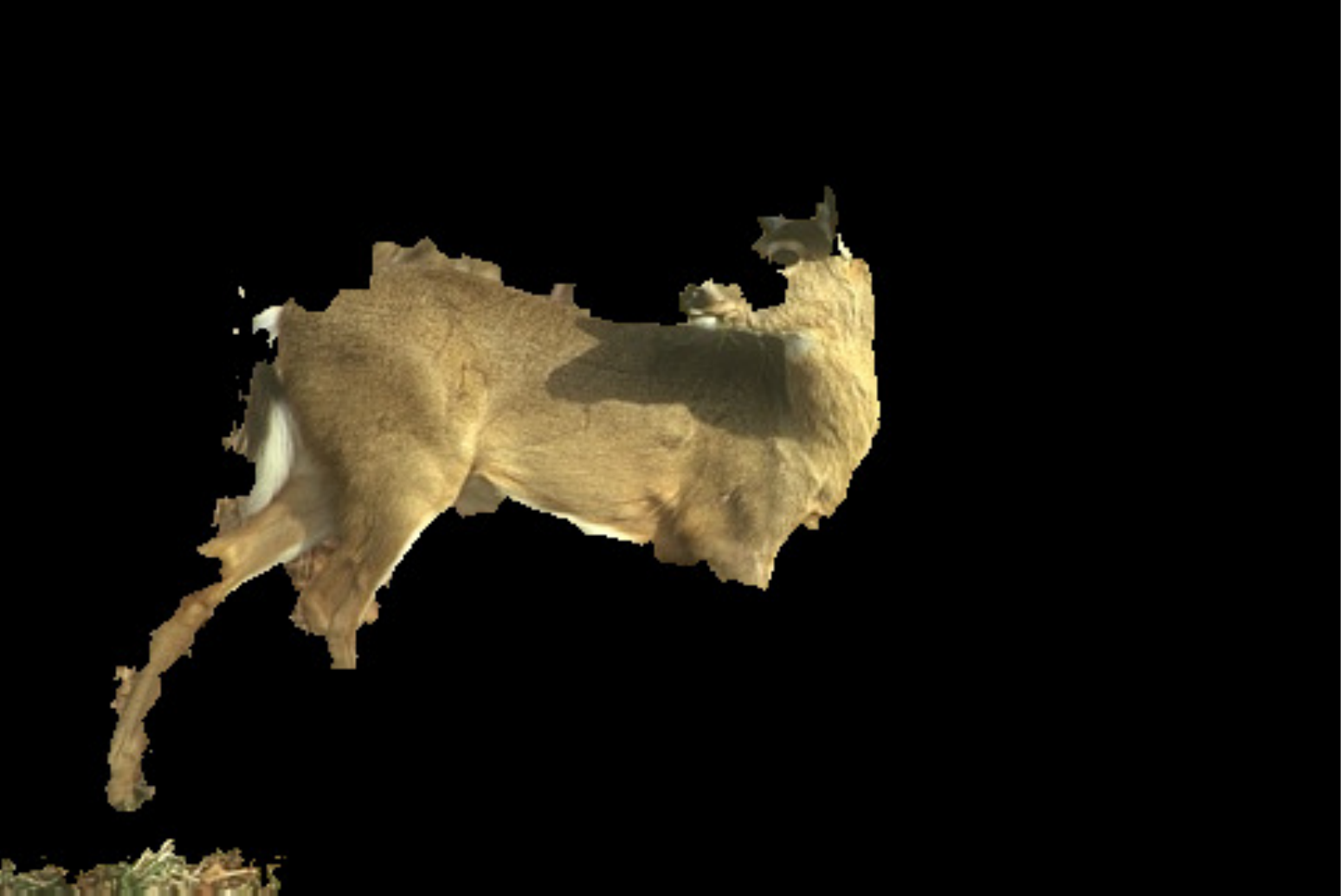}
\includegraphics[width=0.113\textwidth]{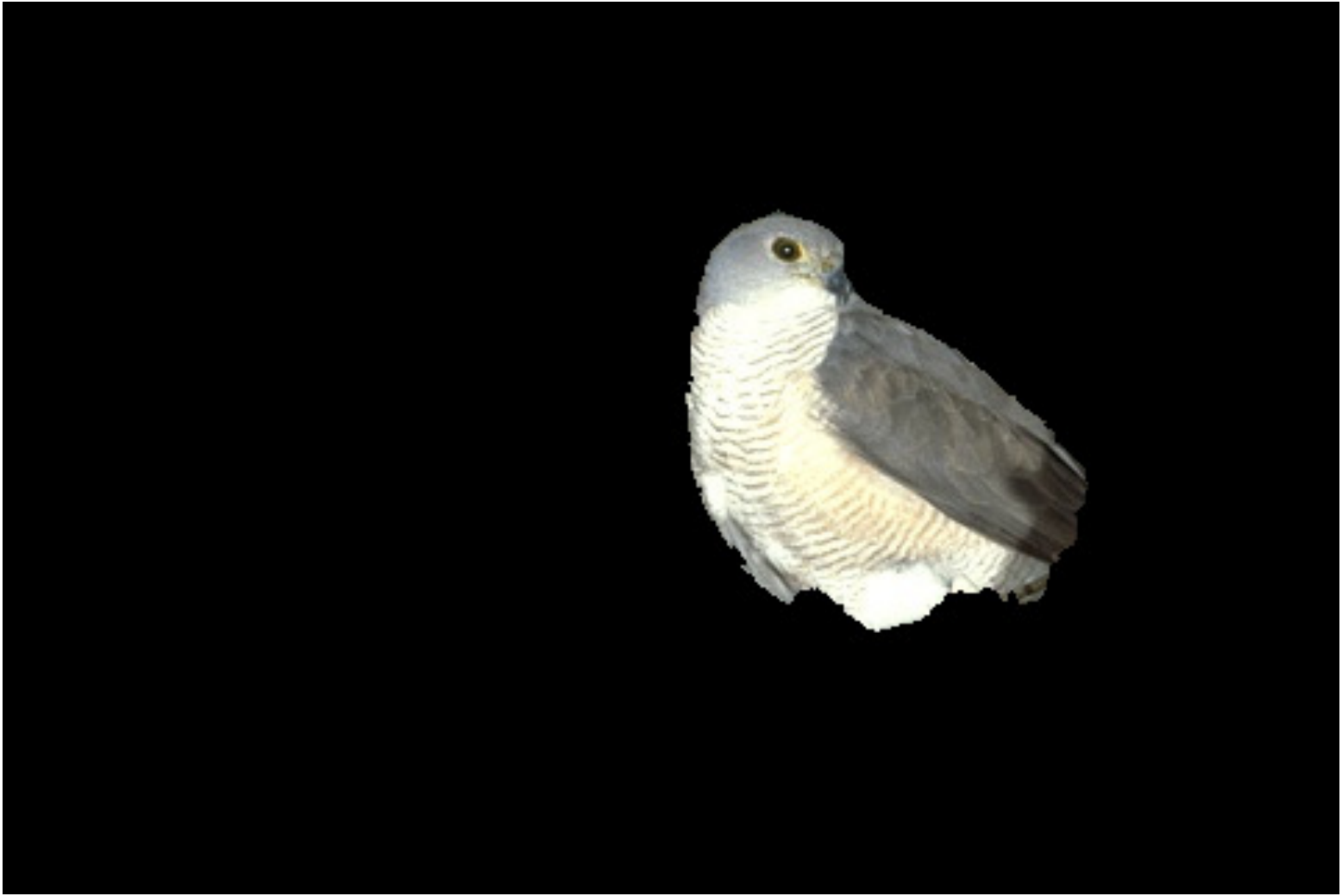}
\centering
\end{minipage}
}\\
\caption{Image segmentation with histogram constraints (coarse over-segmentation).
The number of superpixels is around $726$.
From top to bottom are:
original images, ground-truth (GT), superpixels, unary terms for graph-cuts, results of graph-cuts, SMQC and \lbfgsb.
Results for other SDP based methods are similar to that of \lbfgsb and thus omitted.
Graph cuts tends to mix together the foreground and background with similar color.
\lbfgsb achieves the best segmentation results.}
\label{fig:imgsegmnonlinear}
\end{figure*}

\begin{table}[t]
  \centering
  \scriptsize
  \begin{tabular}{l@{\hspace{0.05cm}}|@{\hspace{0.05cm}}c@{\hspace{0.05cm}}c@{\hspace{0.05cm}}c@{\hspace{0.05cm}}c@{\hspace{0.05cm}}c@{\hspace{0.05cm}}|@{\hspace{0.05cm}}c@{\hspace{0.05cm}}c}
  \hline
& & & & & & &\\ [-2ex]
     Methods  & \lbfgsb  & \smooth & SeDuMi & SDPT3 & MOSEK & GC & SMQC\\
  \hline
  \hline
& & & & & & &\\ [-2ex]
Time/Iters     	& $32.3$s/$248.3$	& $\mathbf{14.8}${\bf s}/$\mathbf{22.9}$	& $2$m$57$s	& $1$m$2$s	& $1$m$19$s	& $0.2$s	& $5.1$s\\
F-measure 	& $0.930$	& $0.925$	& $0.928$	& $0.926$	& $0.928$	& $0.722$	& $0.832$\\
Upper-bound 	& $-120.0$	& $-120.1$	& $-120.1$	& $-119.8$	& $-120.1$	& \NA	& \NA \\%$ 1.40$\\
Lower-bound 	& $-126.8$	& $-126.8$	& $-126.7$	& $-126.7$	& $-126.7$	& \NA	& \NA\\
  \hline
  \end{tabular}
  \caption{Image segmentation with histogram constraints.
           Results are the average of the eight images shown in Fig.~\ref{fig:imgsegmnonlinear}.
           \smooth uses fewer iterations than \lbfgsb and is faster than all other SDP based methods.
           Graph cuts and SMQC exhibit worse F-measure scores than SDP based methods.
          }
  \label{tab:imgsegmnonlinear}
\end{table}

\noindent{\bf Histogram Constraints}
{%
Given a random sample $\bz$, a feasible solution $\bx$ to the corresponding BQP formulation~\eqref{eq:app_hist}
can be obtained through $\bx = \mathrm{sign}(\bz - \theta)$,
where
\begin{align}
 \theta \!=\! \left\{
            \begin{array}{ll}
               \!0
                      \!& \mbox{if }  \lvert \mathbf{1}^\T  \mathrm{sign}(\bz) \rvert \!\leq\! \kappa n, \\
               \!(\tilde{z}_{\lfloor \frac{n + \kappa\! n \!}{2} \rfloor} \!+\! \tilde{z}_{\lfloor \frac{n + \kappa\! n \!}{2} \rfloor +\! 1}) / 2
                      \!& \mbox{if } { \mathbf{1}^\T \mathrm{sign}(\bz) \!>\!   \kappa n} , \\
               \!(\tilde{z}_{\lceil \frac{n - \kappa\! n  \!}{2} \rceil} \!+\! \tilde{z}_{\lceil \frac{n -\kappa\! n  \!}{2}\rceil +\! 1}) / 2
                      \!& \mbox{if } { \mathbf{1}^\T \mathrm{sign}(\bz) \!<\! - \kappa n},
            \end{array}
         \right.
\label{eq:hist-round}
\end{align}
and $\tilde{\bz}$ is obtained by sorting $\bz$ in descending order.
}
For graph cuts methods, the histogram constraint is encoded as unary terms:
$\varphi_i = - \ln \big(  \mathrm{Pr}(\bbf_i|\mbox{fore}) / \mathrm{Pr}(\bbf_i|\mbox{back}) \big)$, $i = 1,2,\dots,n$.
$\mathrm{Pr}(\bbf_i | \mbox{fore})$ and $\mathrm{Pr}(\bbf_i | \mbox{back})$ are
probabilities for the color of the $i$th pixel belonging to foreground and background respectively.

Fig.~\ref{fig:imgsegmnonlinear} and Table~\ref{tab:imgsegmnonlinear}
demonstrate the results for image segmentation with histogram constraints.
We can see that unary terms (the second row in Fig.~\ref{fig:imgsegmnonlinear}) are not ideal especially when the color distribution of foreground and background are overlapped.
For example in the first image, the white collar of the person in the foreground have similar unary terms with the white wall in the background.
The fourth row of Fig.~\ref{fig:imgsegmnonlinear} shows that the unsatisfactory unary terms degrade the segmentation results of graph cuts methods significantly.

The average F-measure of all evaluated methods are reported in Table~\ref{tab:imgsegmnonlinear}.
Our methods outperforms graph cuts and SMQC in terms of F-measure.
As for the running time, \smooth is faster than all other SDP-based methods (\ie, \lbfgsb, SeDuMi, SDPT3 and MOSEK).
As expected, \smooth uses much less ($1/6$) iterations than \lbfgsb.
\lbfgsb and \smooth have comparable upper-bounds and lower-bounds than interior-point methods.

From Table~\ref{tab:imgsegmnonlinear}, we can find that SMQC is faster than our methods.
However, SMQC does not scale well to large problems since it needs to compute full eigen-decomposition.
We also test \lbfgsb and SMQC on problems with a larger number of superpixels ($9801$).
Both of the algorithms achieve similar segmentation results, but \lbfgsb is much faster than SMQC ($23$m$21$s \vs $4$h$9$m).

\subsection{Image Co-segmentation}
\label{sec:cosegm}

\begin{figure*}[t]
\vspace{-0.0cm}
\centering{
\subfloat{
\centering{
\begin{minipage}[c]{0.01\textwidth}
\begin{turn}{90} {\scriptsize Images} \end{turn}
\end{minipage}
\begin{minipage}[c]{0.99\textwidth}
\includegraphics[width=0.113\textwidth,clip]{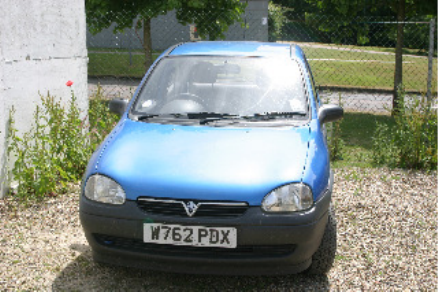}
\includegraphics[width=0.113\textwidth,clip]{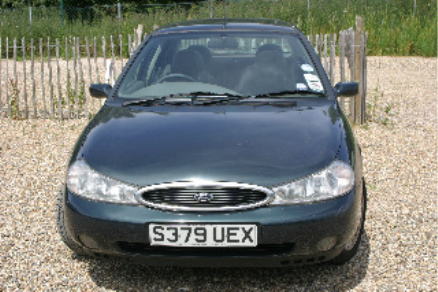}
\includegraphics[width=0.113\textwidth,clip]{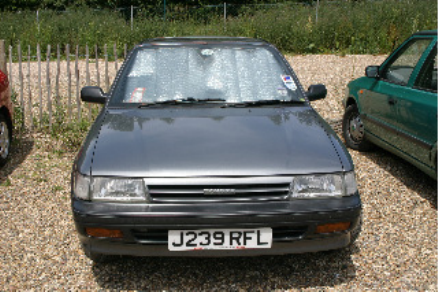}
\includegraphics[width=0.113\textwidth,clip]{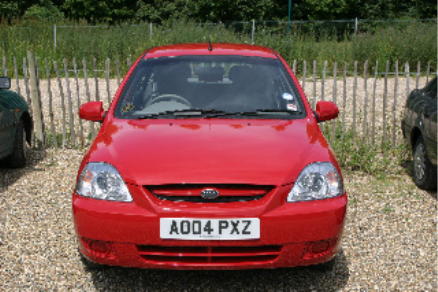}
\includegraphics[width=0.113\textwidth,clip]{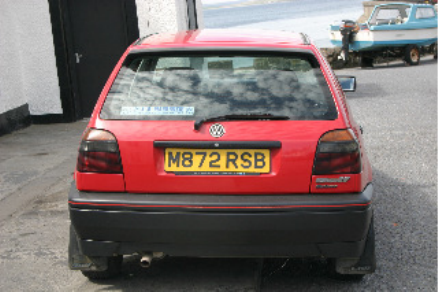}
\includegraphics[width=0.113\textwidth,clip]{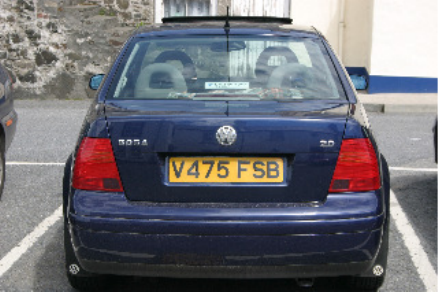}
\includegraphics[width=0.113\textwidth,clip]{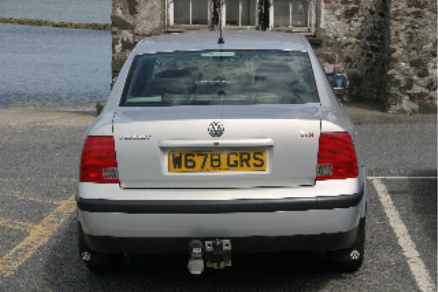}
\includegraphics[width=0.113\textwidth,clip]{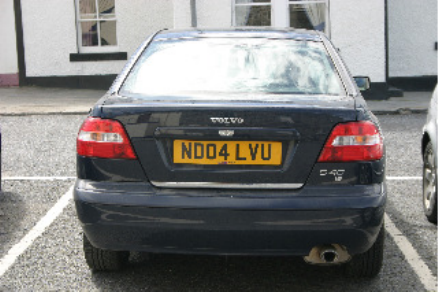}
\end{minipage}
}}
\vspace{-2mm} \\
\subfloat{
\centering{
\begin{minipage}[c]{0.01\textwidth}
\begin{turn}{90} {\scriptsize \lowrank} \end{turn}
\end{minipage}
\begin{minipage}[c]{0.99\textwidth}
\includegraphics[width=0.113\textwidth,clip]{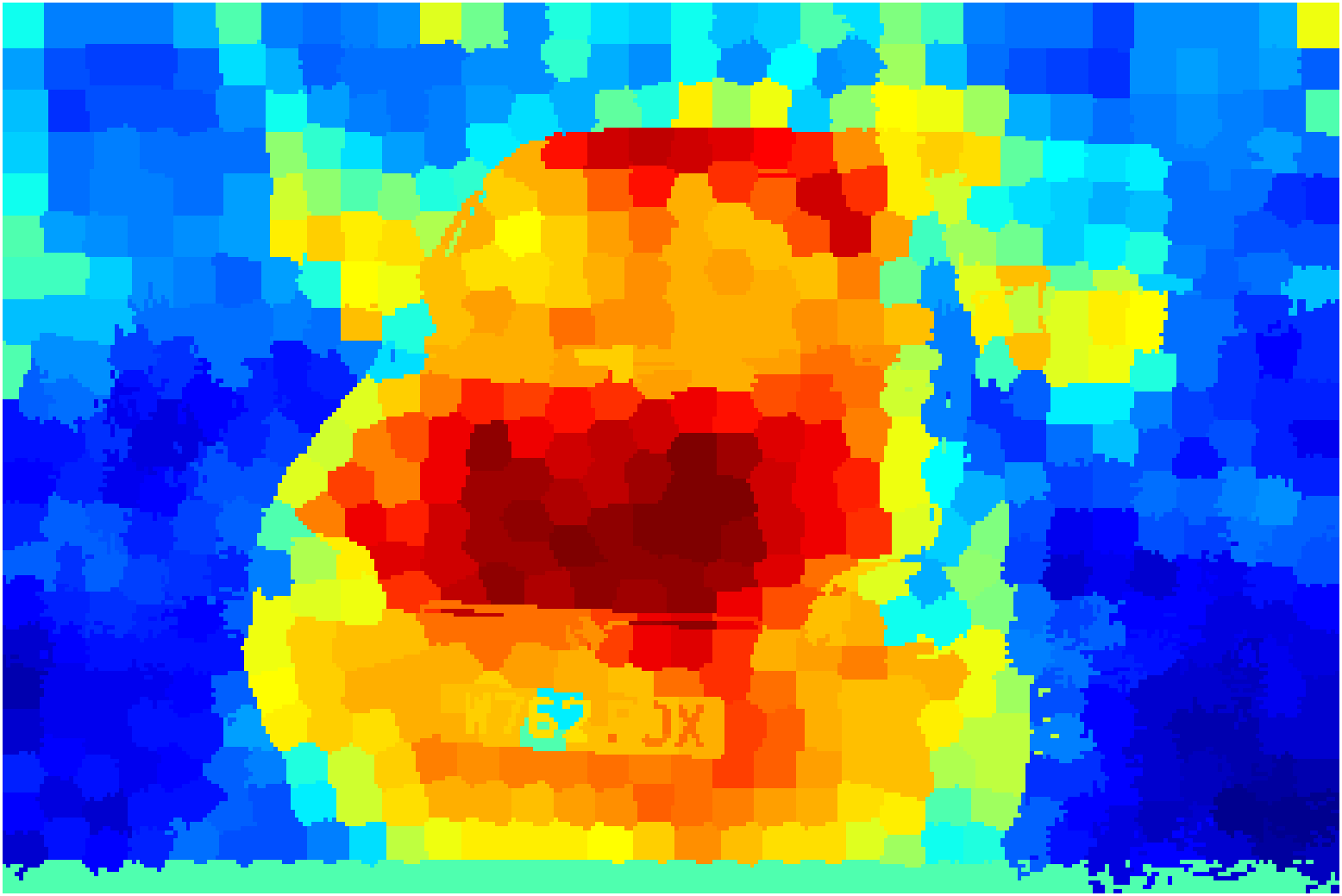}
\includegraphics[width=0.113\textwidth,clip]{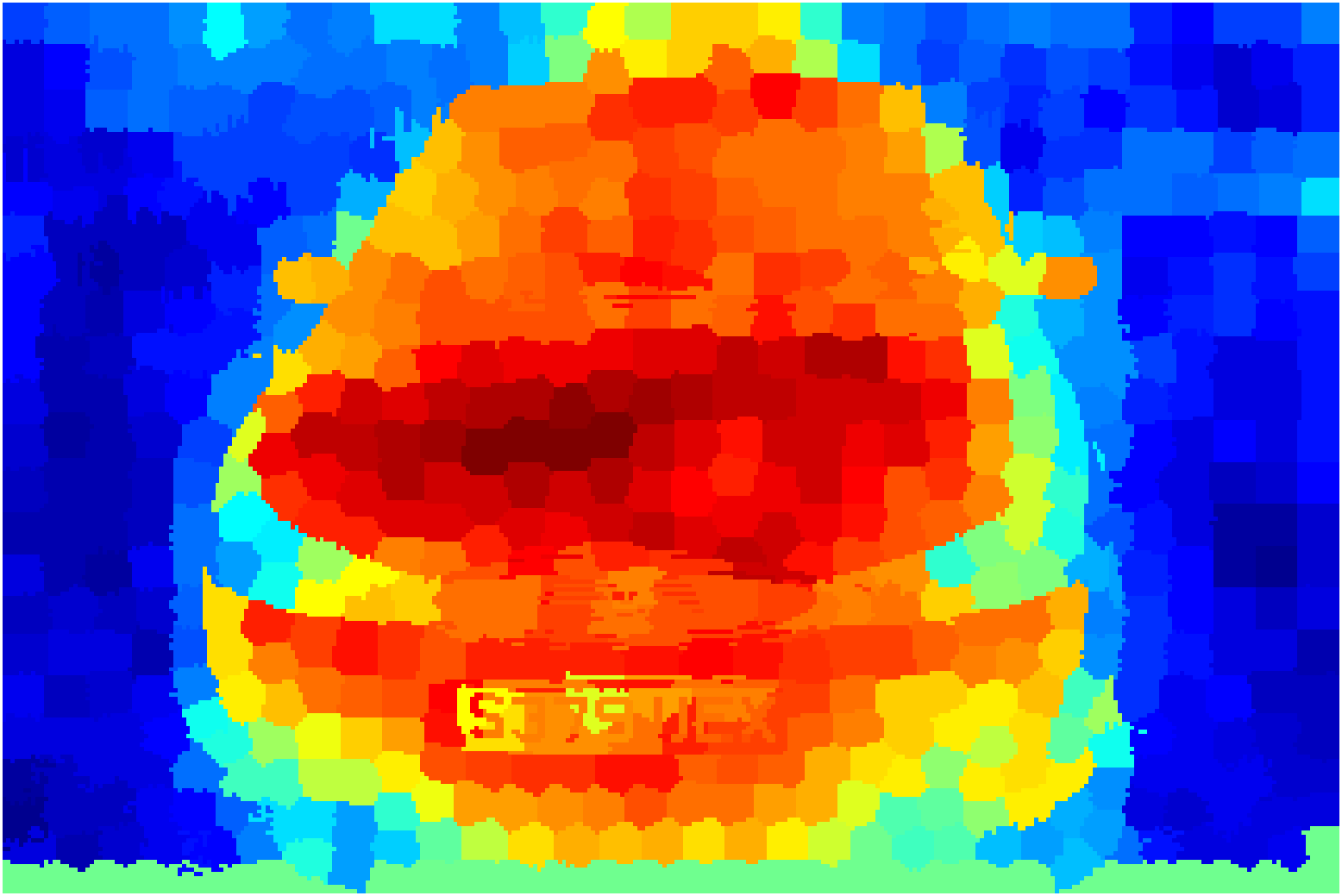}
\includegraphics[width=0.113\textwidth,clip]{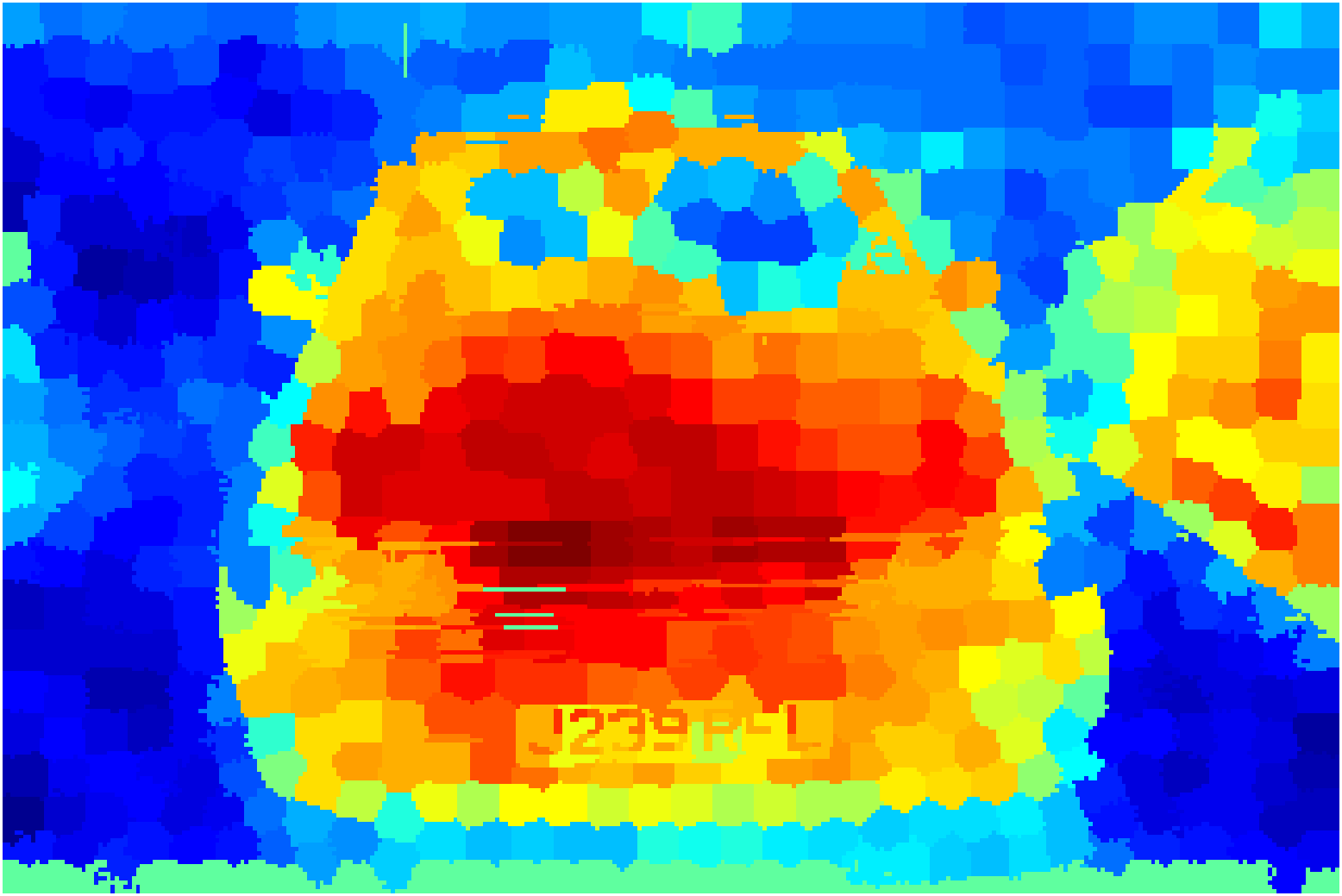}
\includegraphics[width=0.113\textwidth,clip]{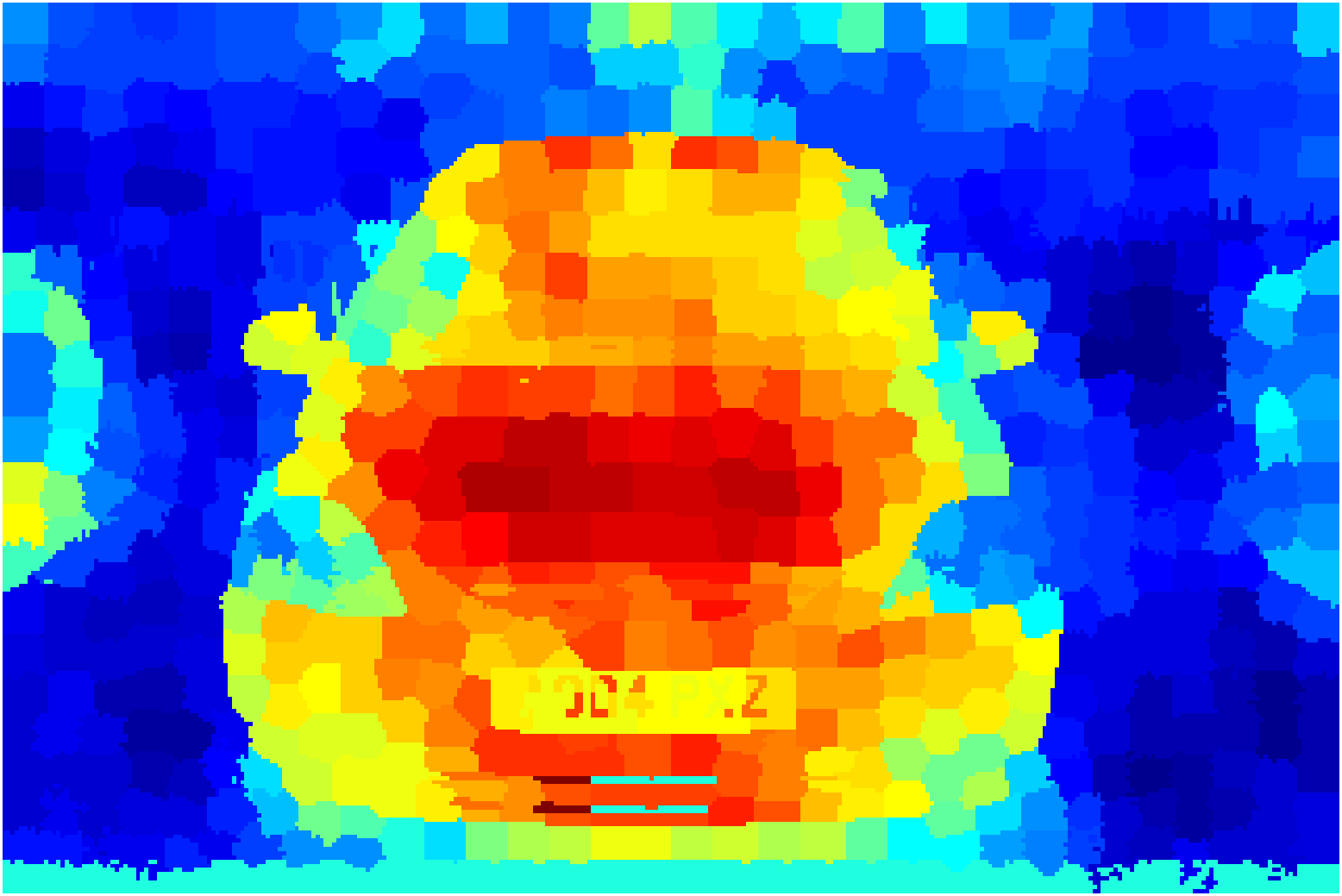}
\includegraphics[width=0.113\textwidth,clip]{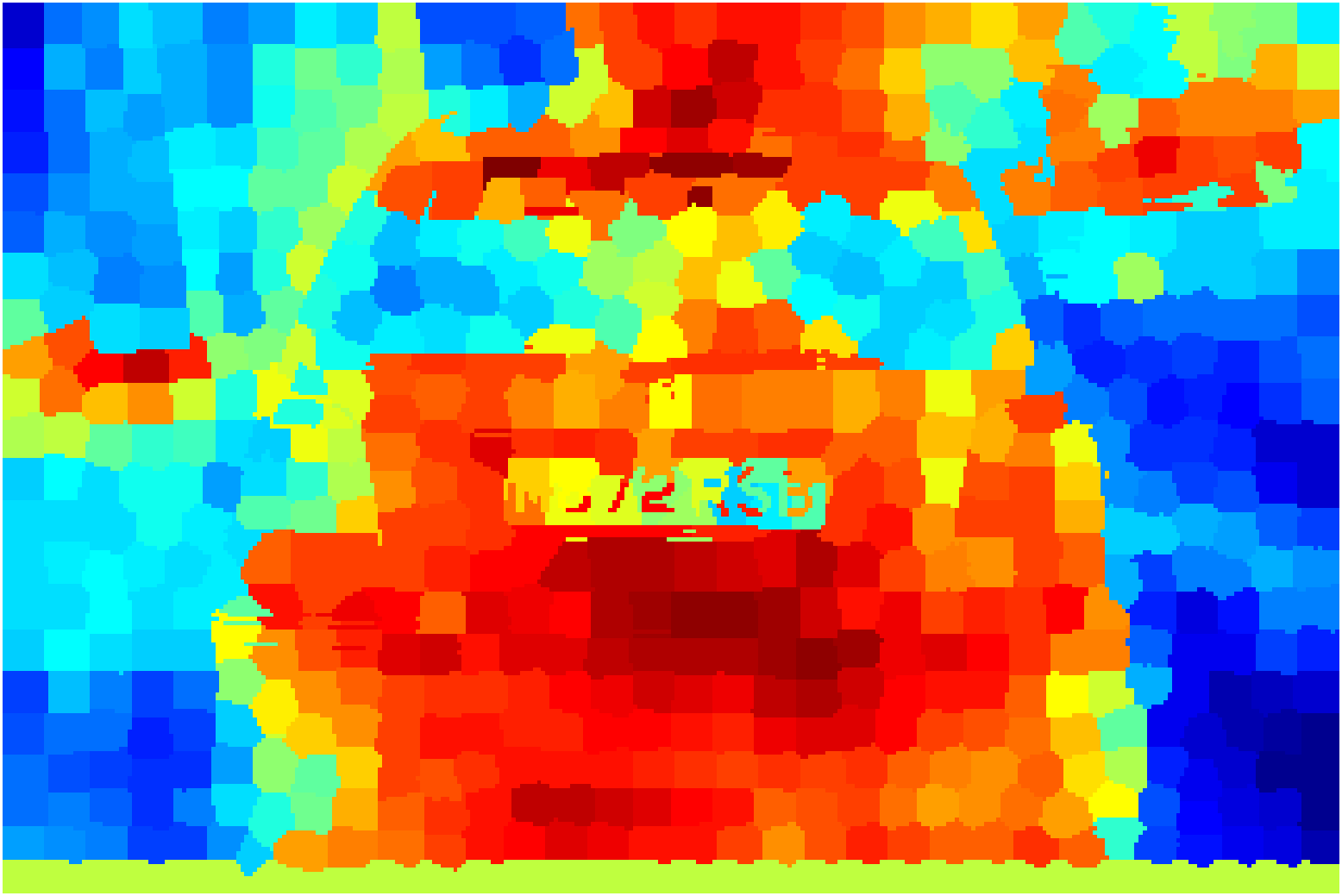}
\includegraphics[width=0.113\textwidth,clip]{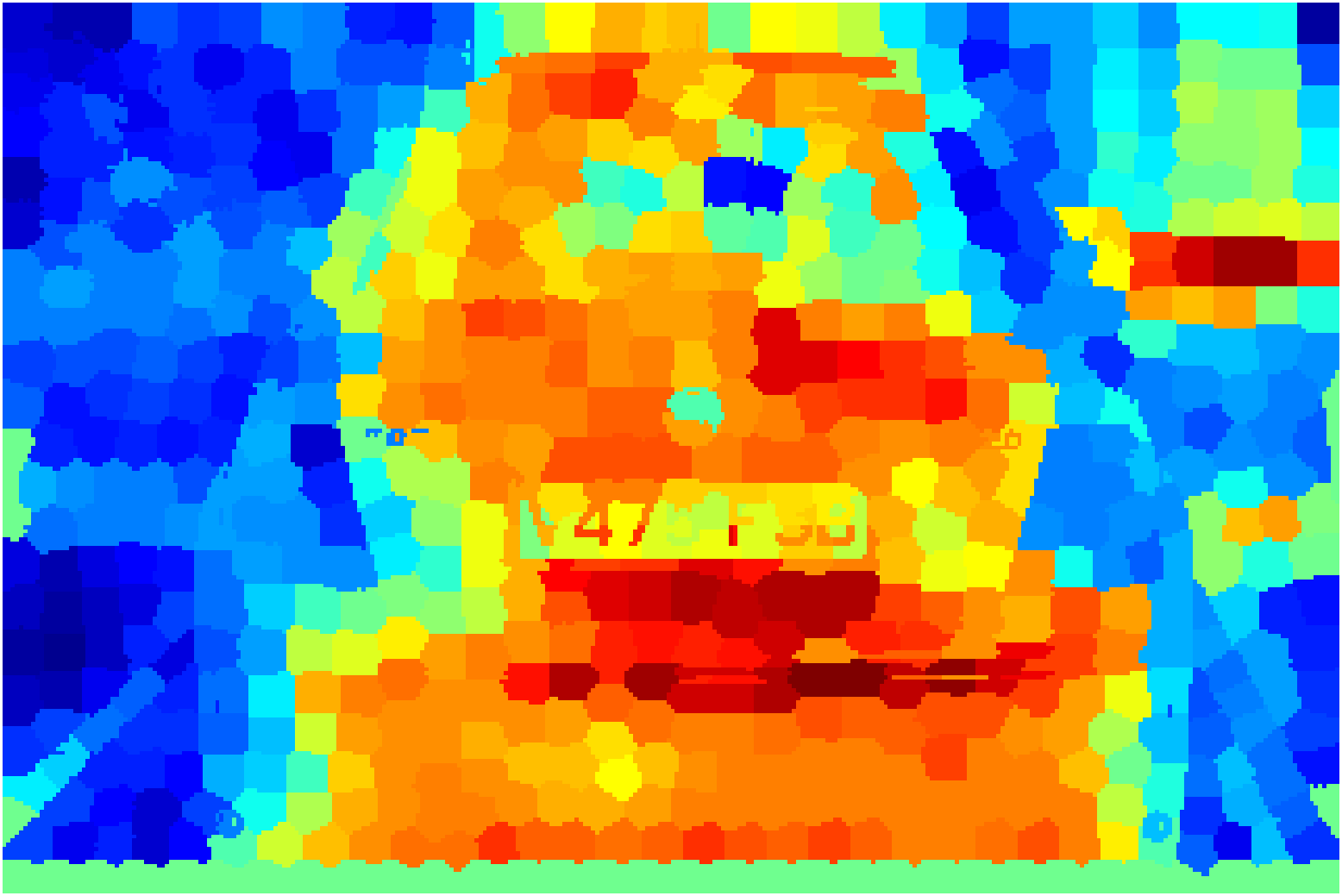}
\includegraphics[width=0.113\textwidth,clip]{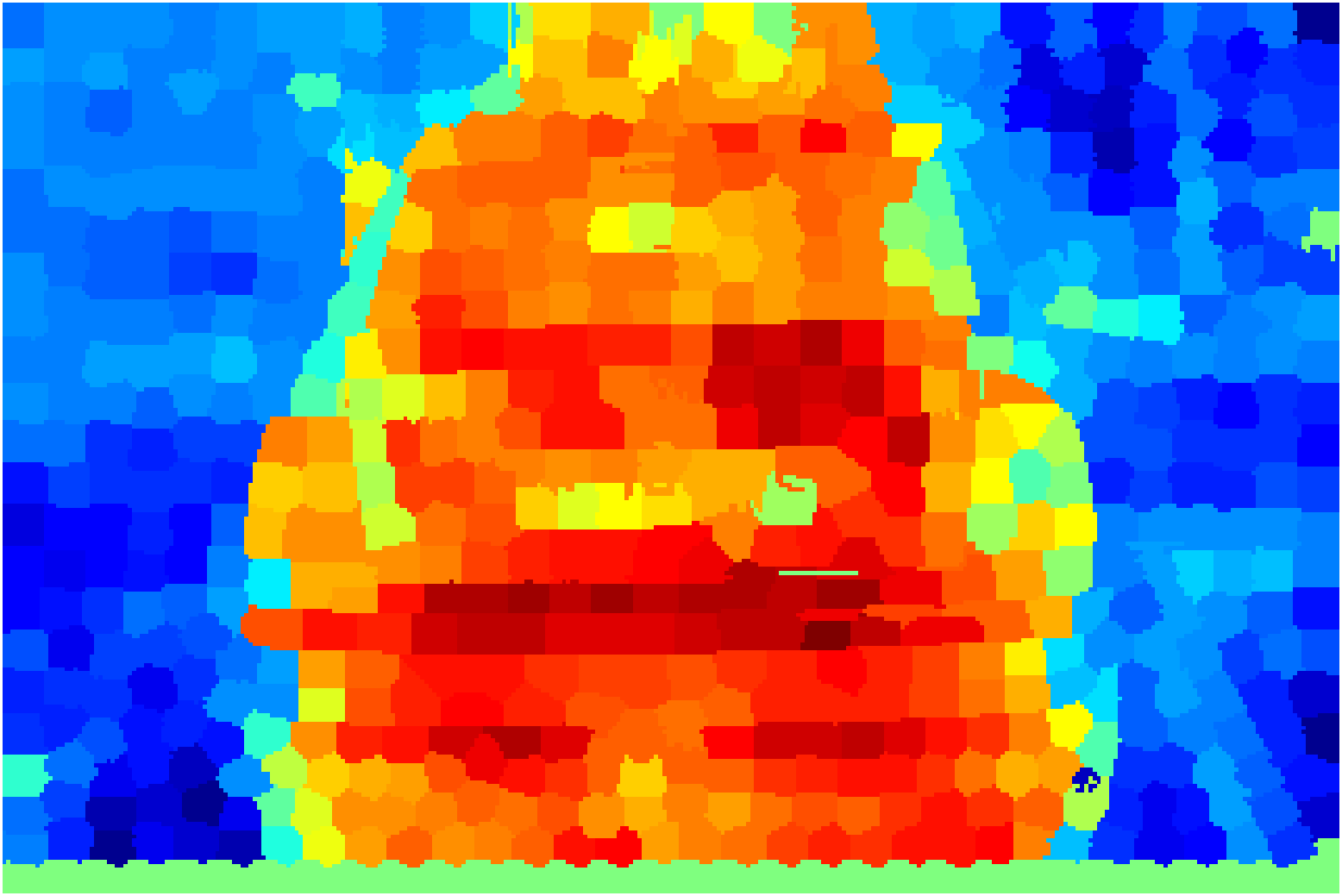}
\includegraphics[width=0.113\textwidth,clip]{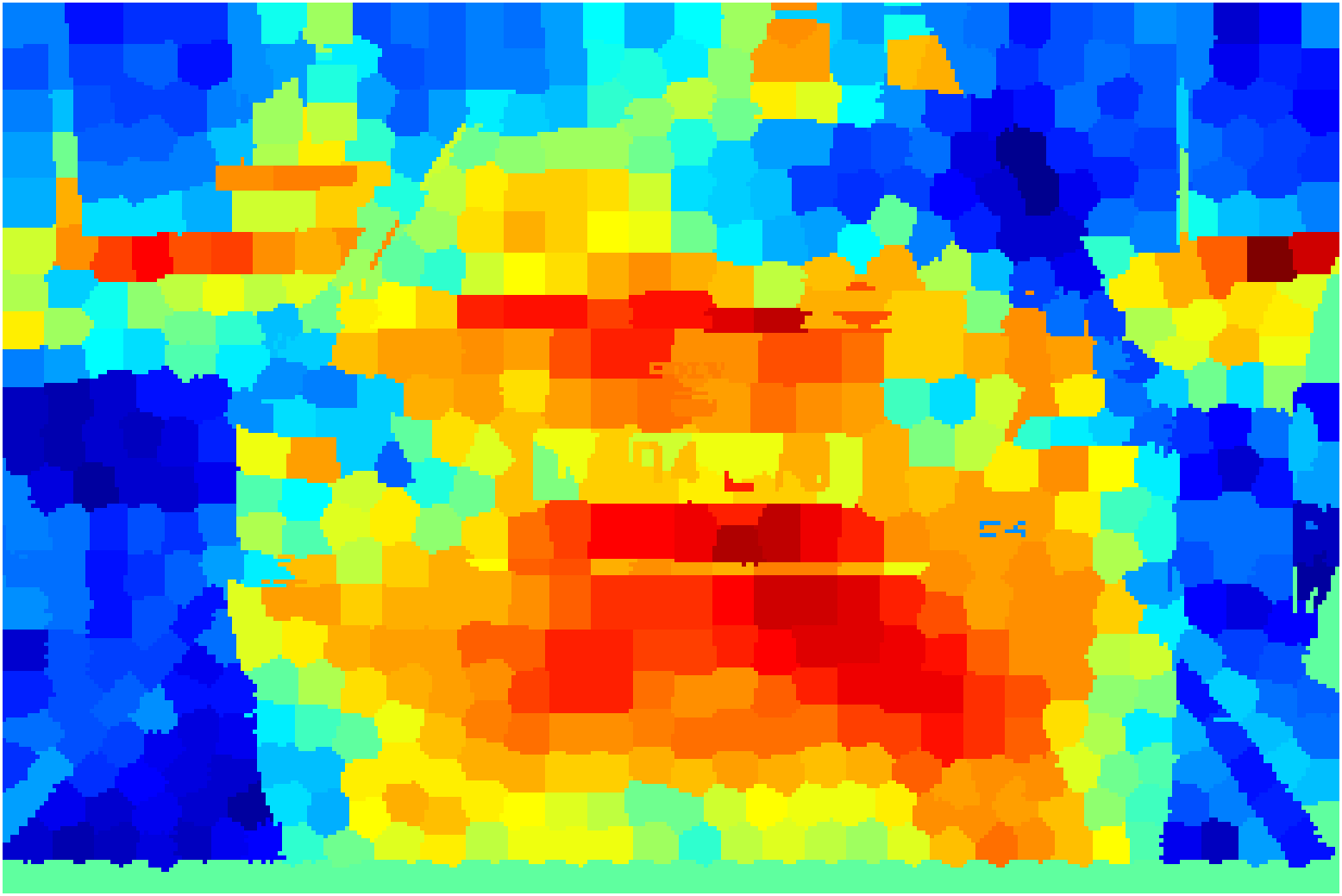}
\end{minipage}
}}\\ \vspace{-2mm}
\subfloat{
\centering{
\begin{minipage}[c]{0.01\textwidth}
\begin{turn}{90} {\scriptsize \lbfgsb} \end{turn}
\end{minipage}
\begin{minipage}[c]{0.99\textwidth}
\includegraphics[width=0.113\textwidth,clip]{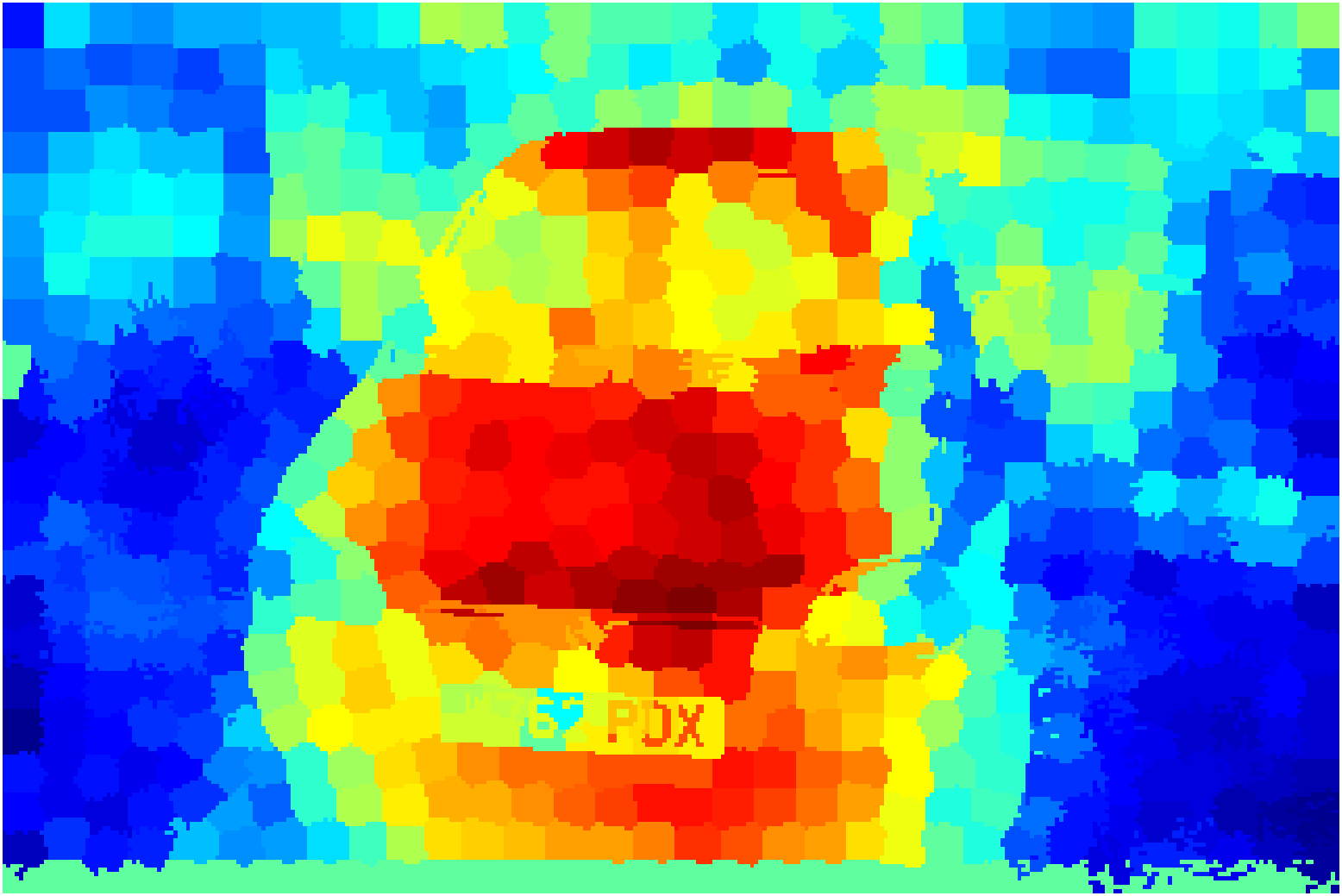}
\includegraphics[width=0.113\textwidth,clip]{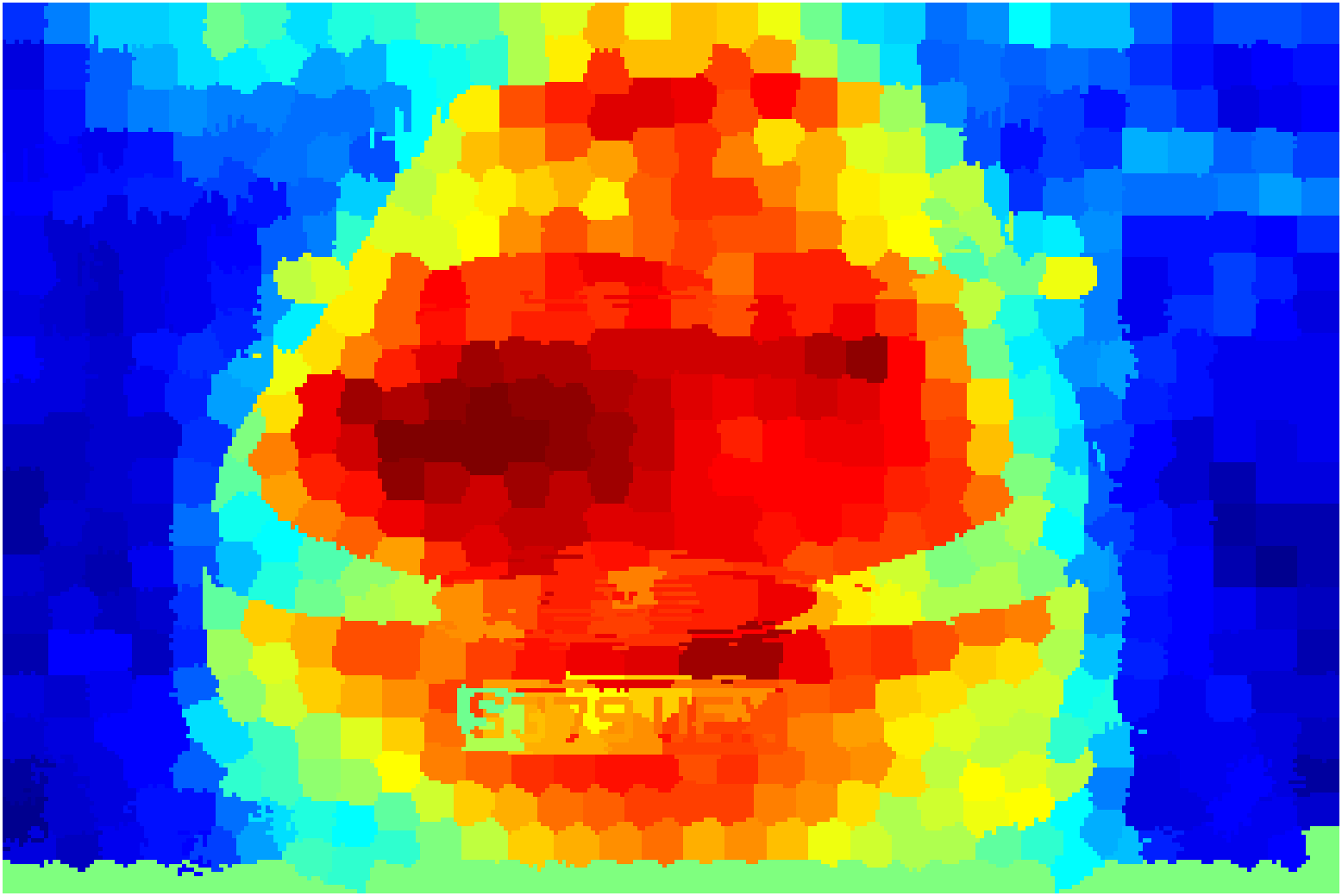}
\includegraphics[width=0.113\textwidth,clip]{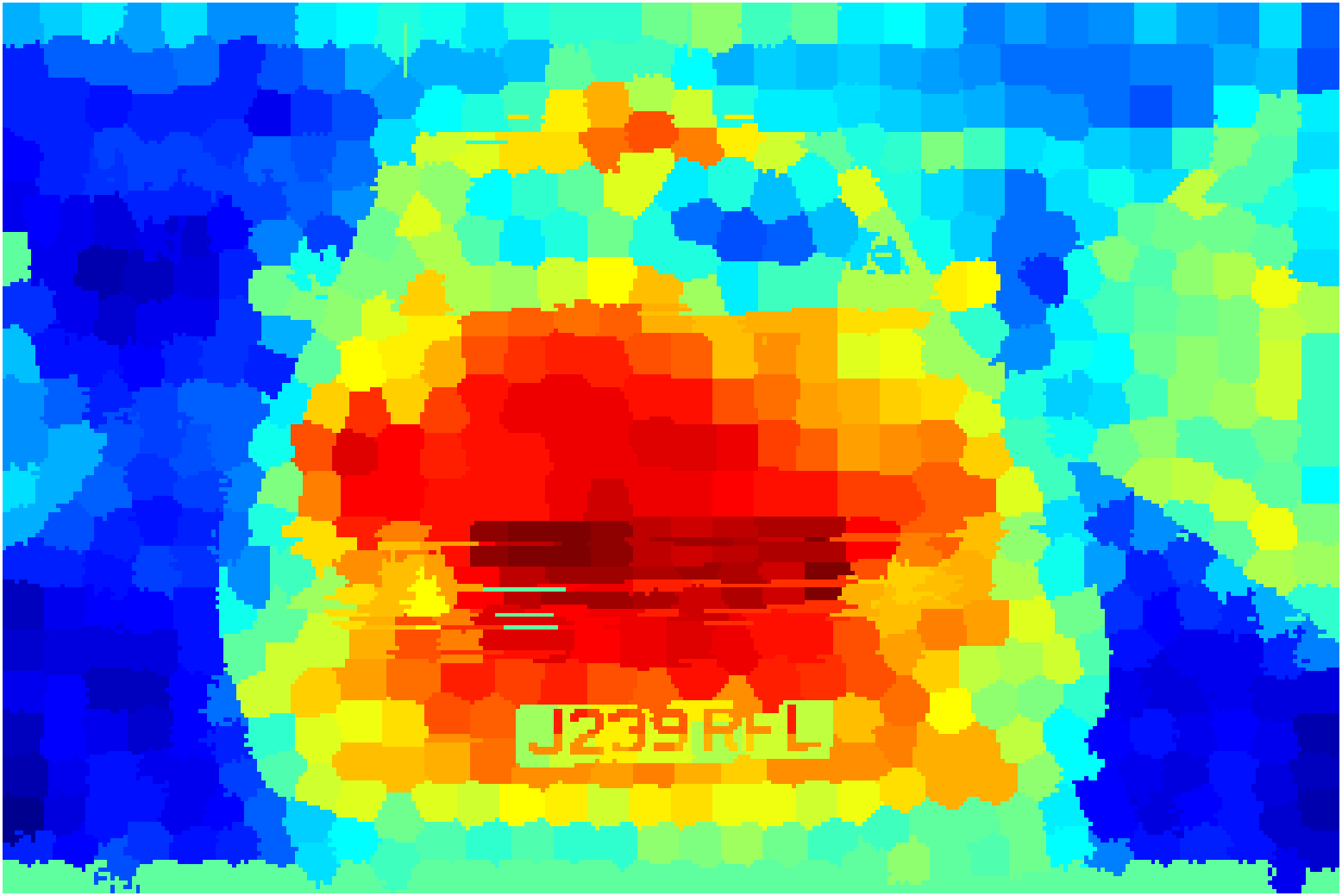}
\includegraphics[width=0.113\textwidth,clip]{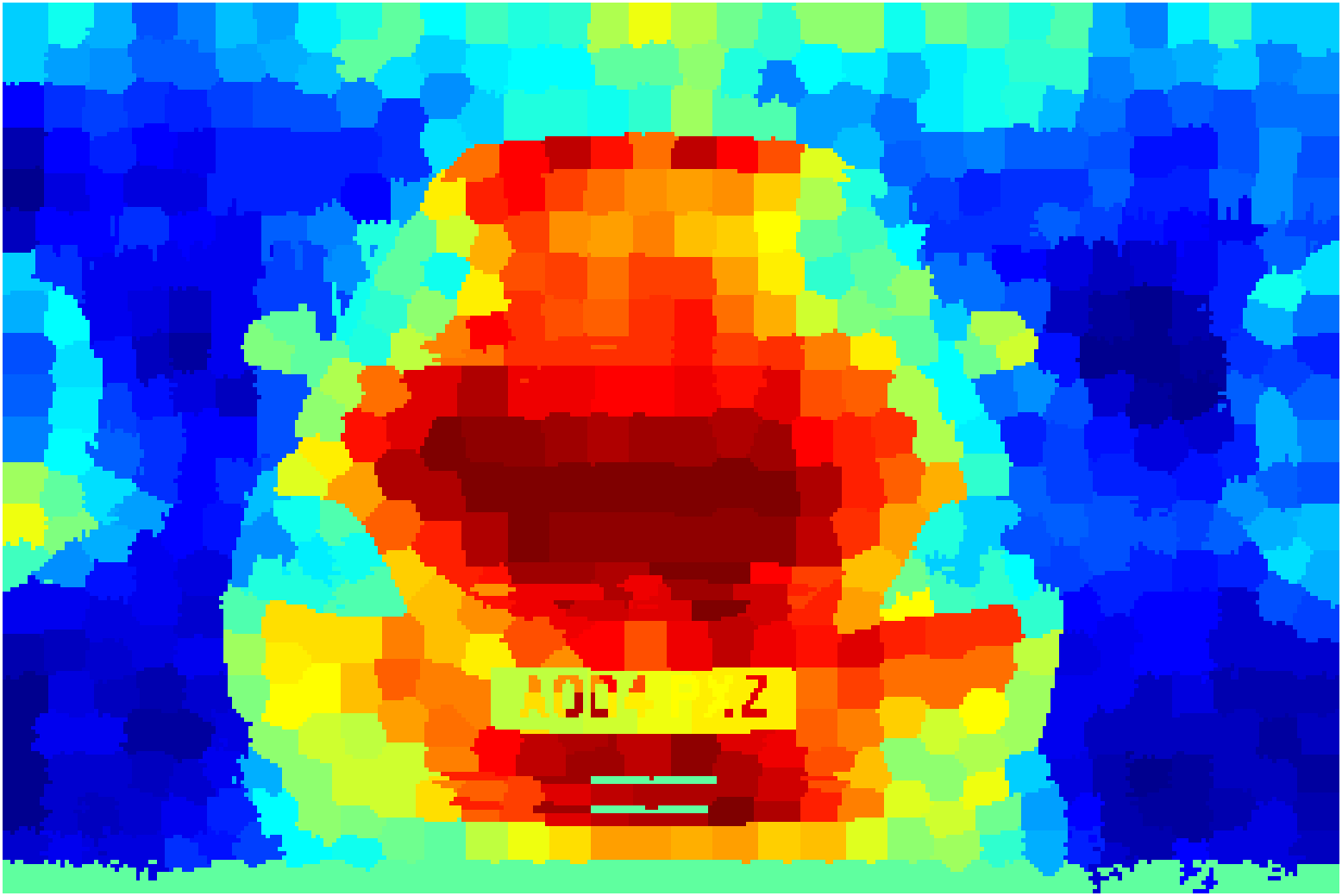}
\includegraphics[width=0.113\textwidth,clip]{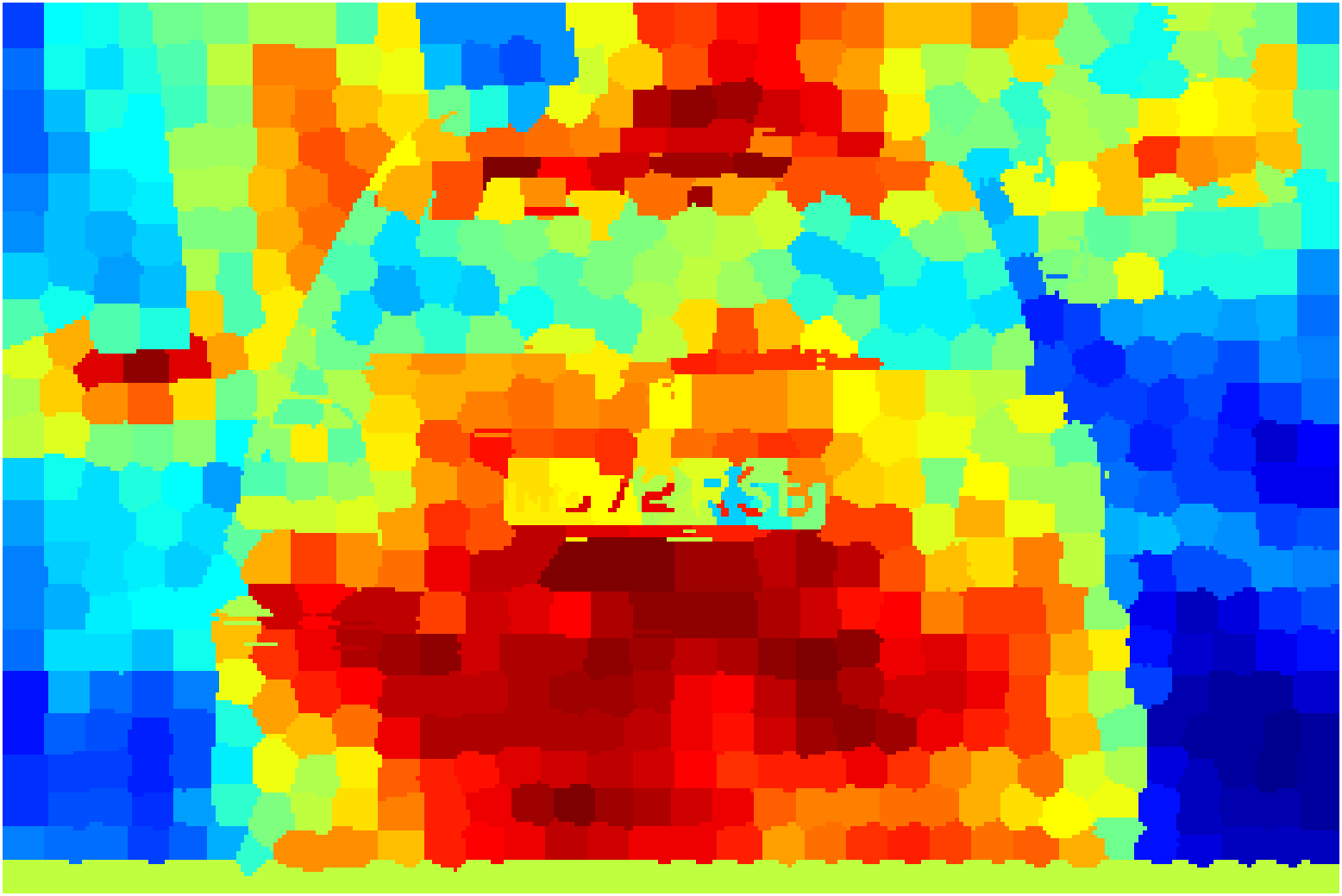}
\includegraphics[width=0.113\textwidth,clip]{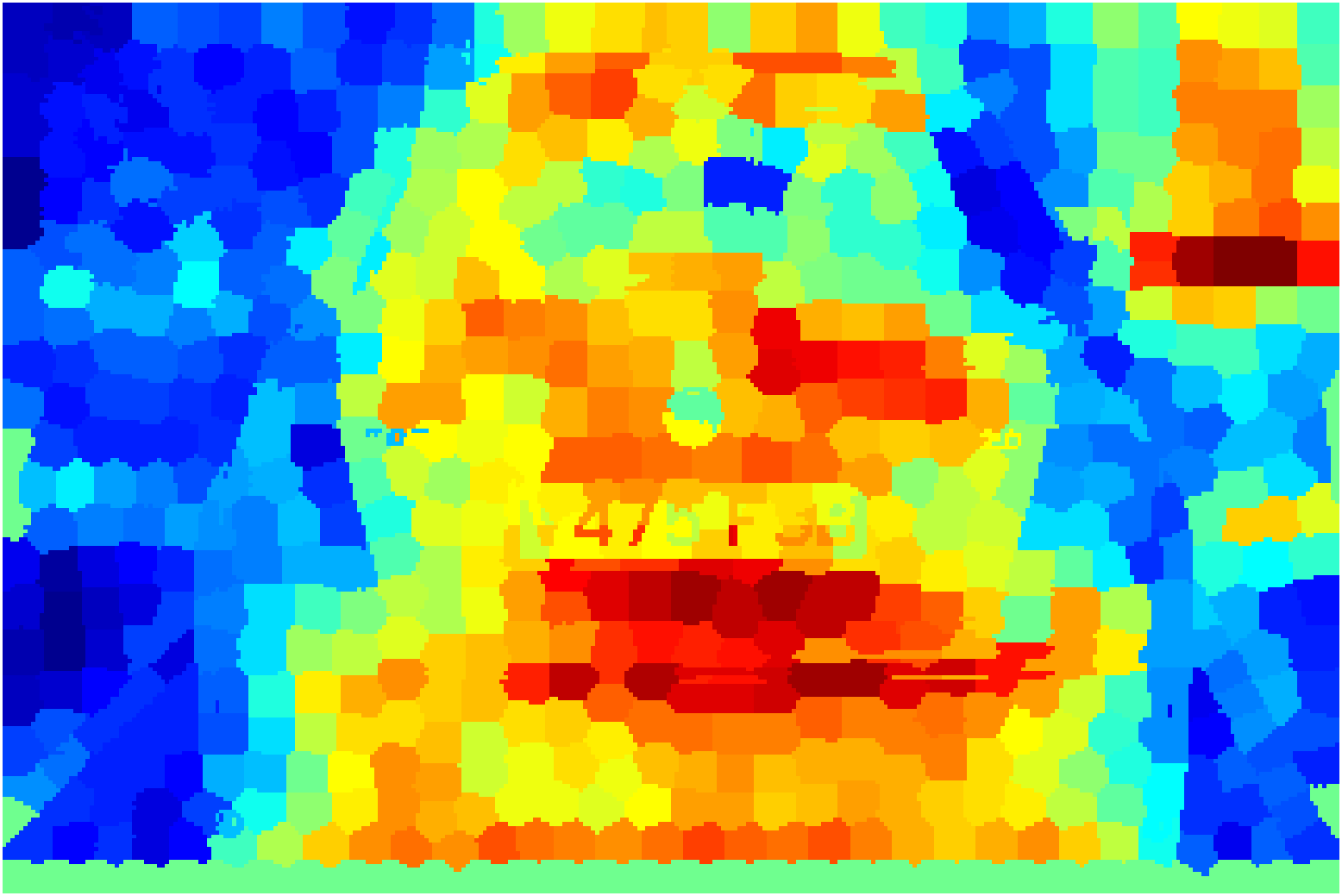}
\includegraphics[width=0.113\textwidth,clip]{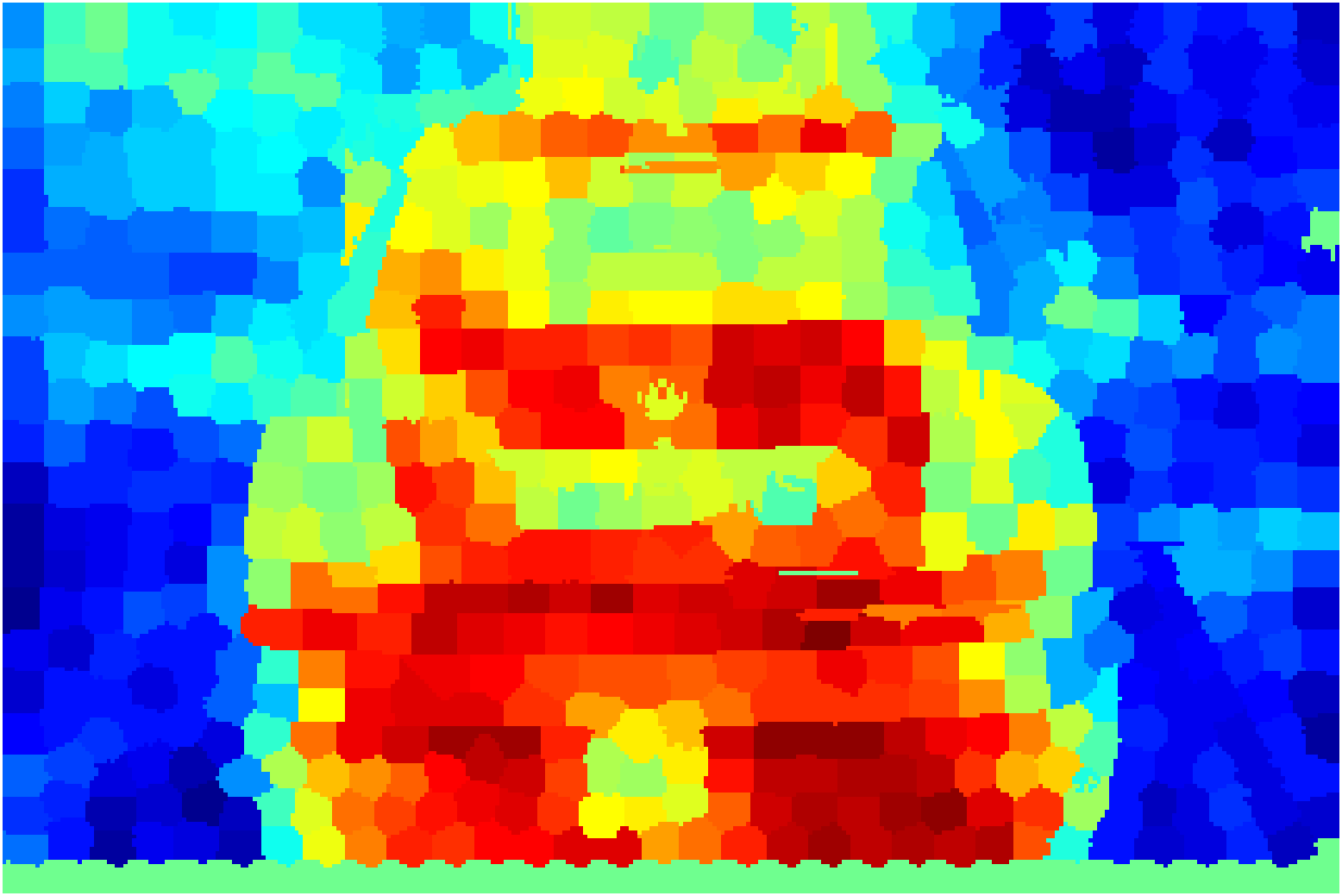}
\includegraphics[width=0.113\textwidth,clip]{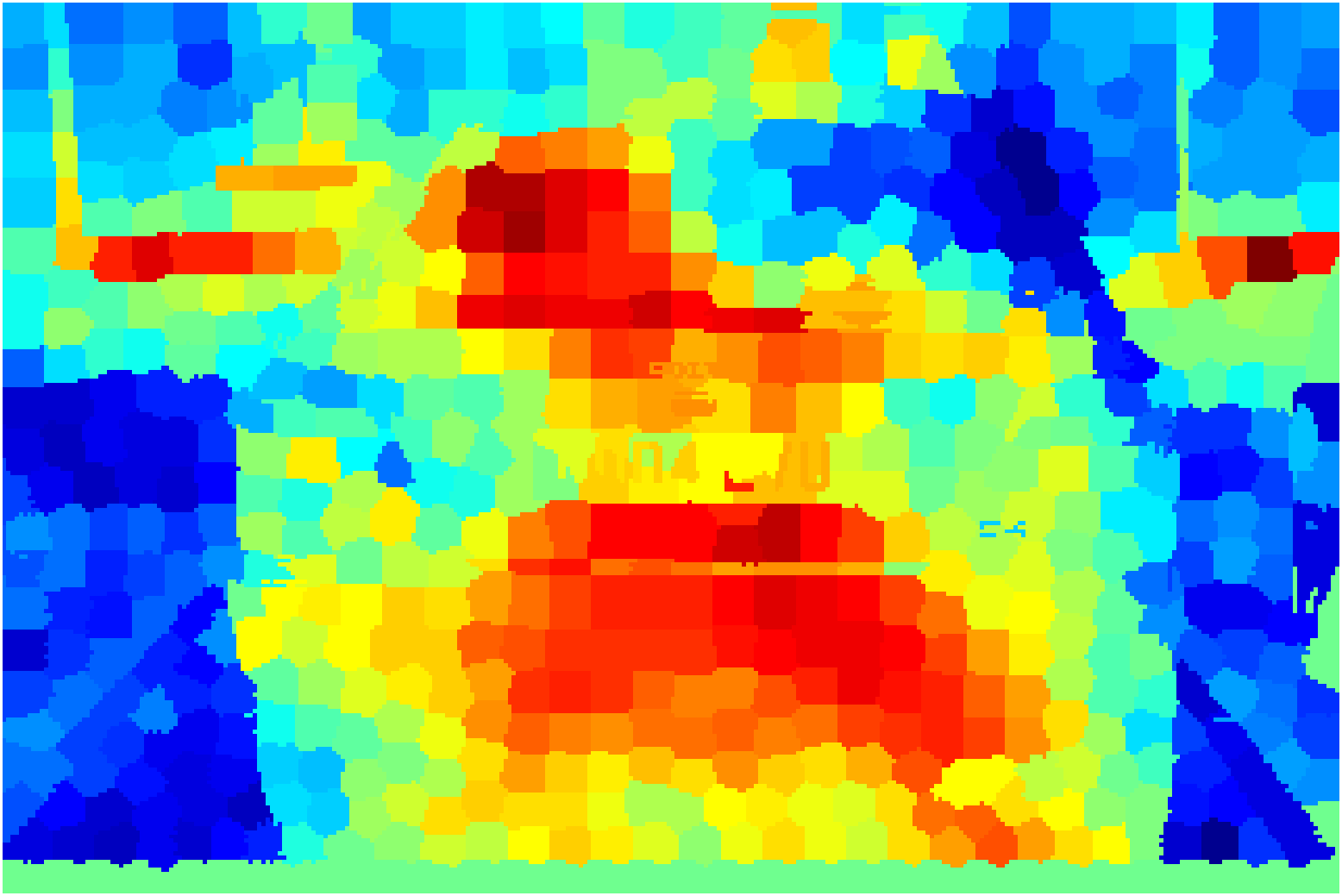}
\end{minipage}
}}\\ \vspace{-2mm}
\subfloat{
\centering{
\begin{minipage}[c]{0.01\textwidth}
\begin{turn}{90} {\scriptsize Images} \end{turn}
\end{minipage}
\begin{minipage}[c]{0.99\textwidth}
\includegraphics[width=0.113\textwidth,clip]{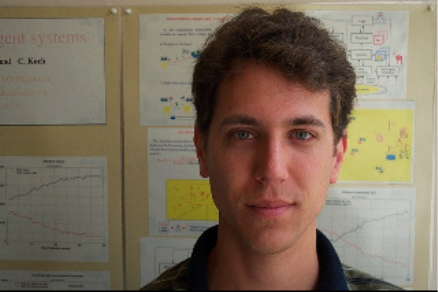}
\includegraphics[width=0.113\textwidth,clip]{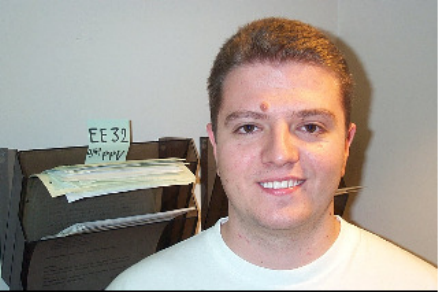}
\includegraphics[width=0.113\textwidth,clip]{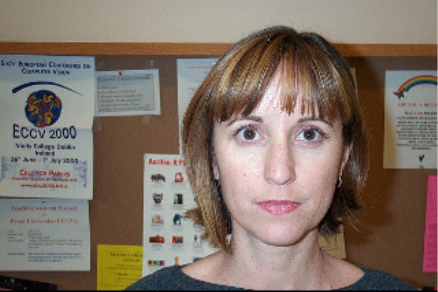}
\includegraphics[width=0.113\textwidth,clip]{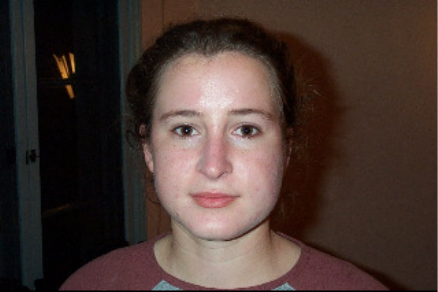}
\includegraphics[width=0.113\textwidth, height = 0.076\textwidth, clip]{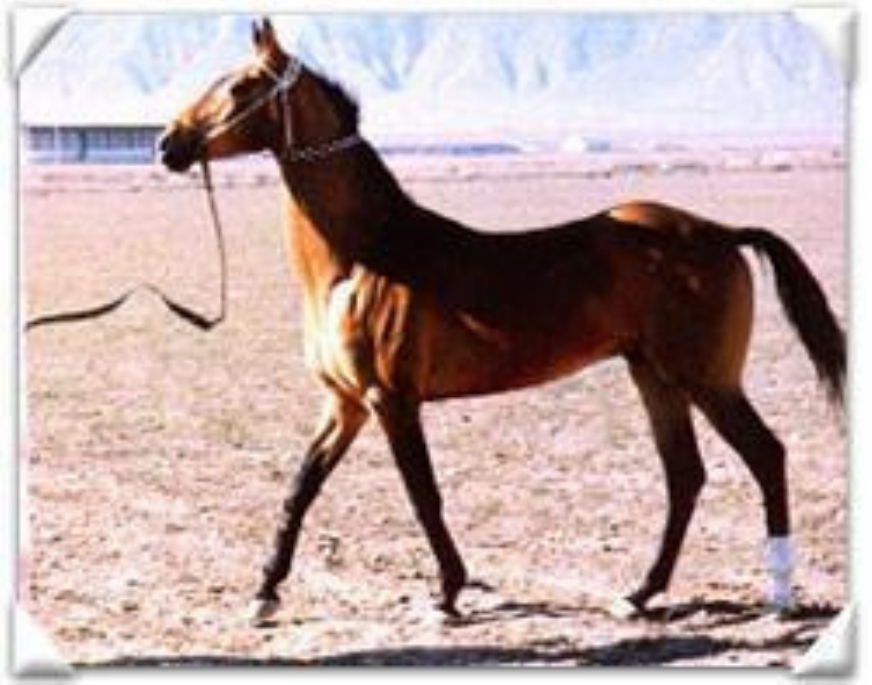}
\includegraphics[width=0.113\textwidth, height = 0.076\textwidth, clip]{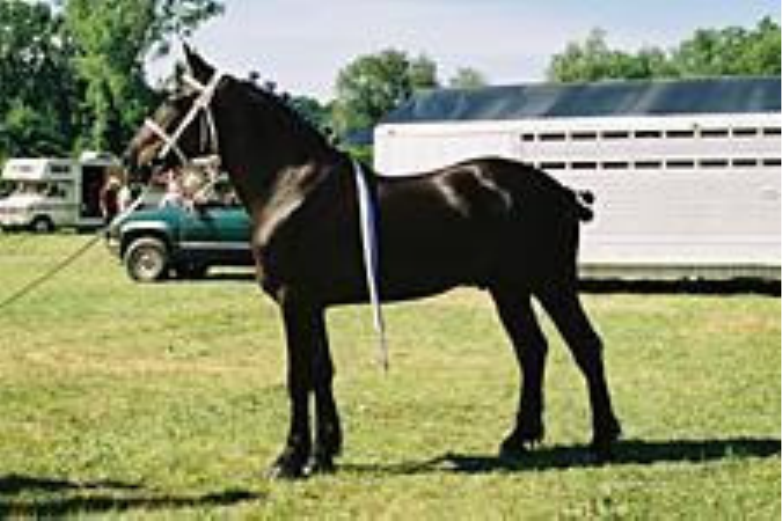}
\includegraphics[width=0.113\textwidth, height = 0.076\textwidth, clip]{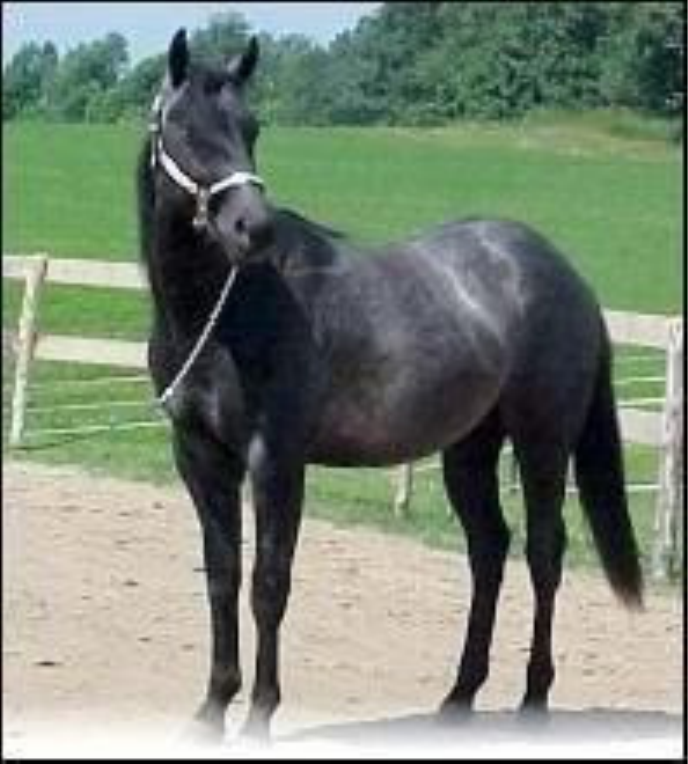}
\includegraphics[width=0.113\textwidth, height = 0.076\textwidth, clip]{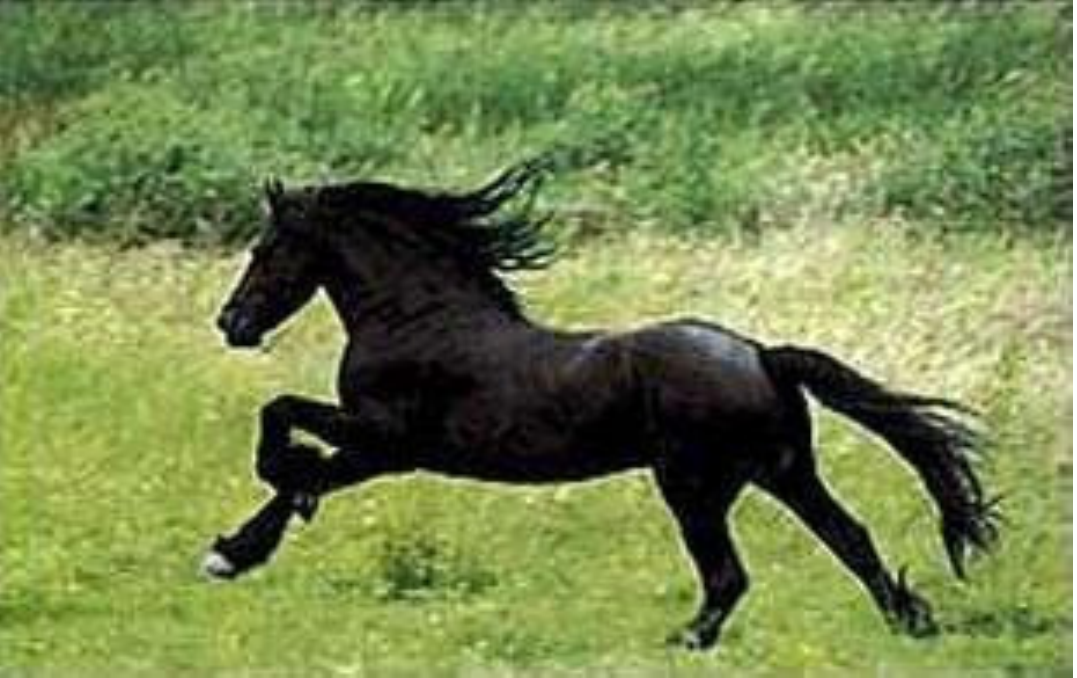}
\end{minipage}
}}\\ \vspace{-2mm}
\subfloat{
\centering{
\begin{minipage}[c]{0.01\textwidth}
\begin{turn}{90} {\scriptsize \lowrank} \end{turn}
\end{minipage}
\begin{minipage}[c]{0.99\textwidth}
\includegraphics[width=0.113\textwidth,clip]{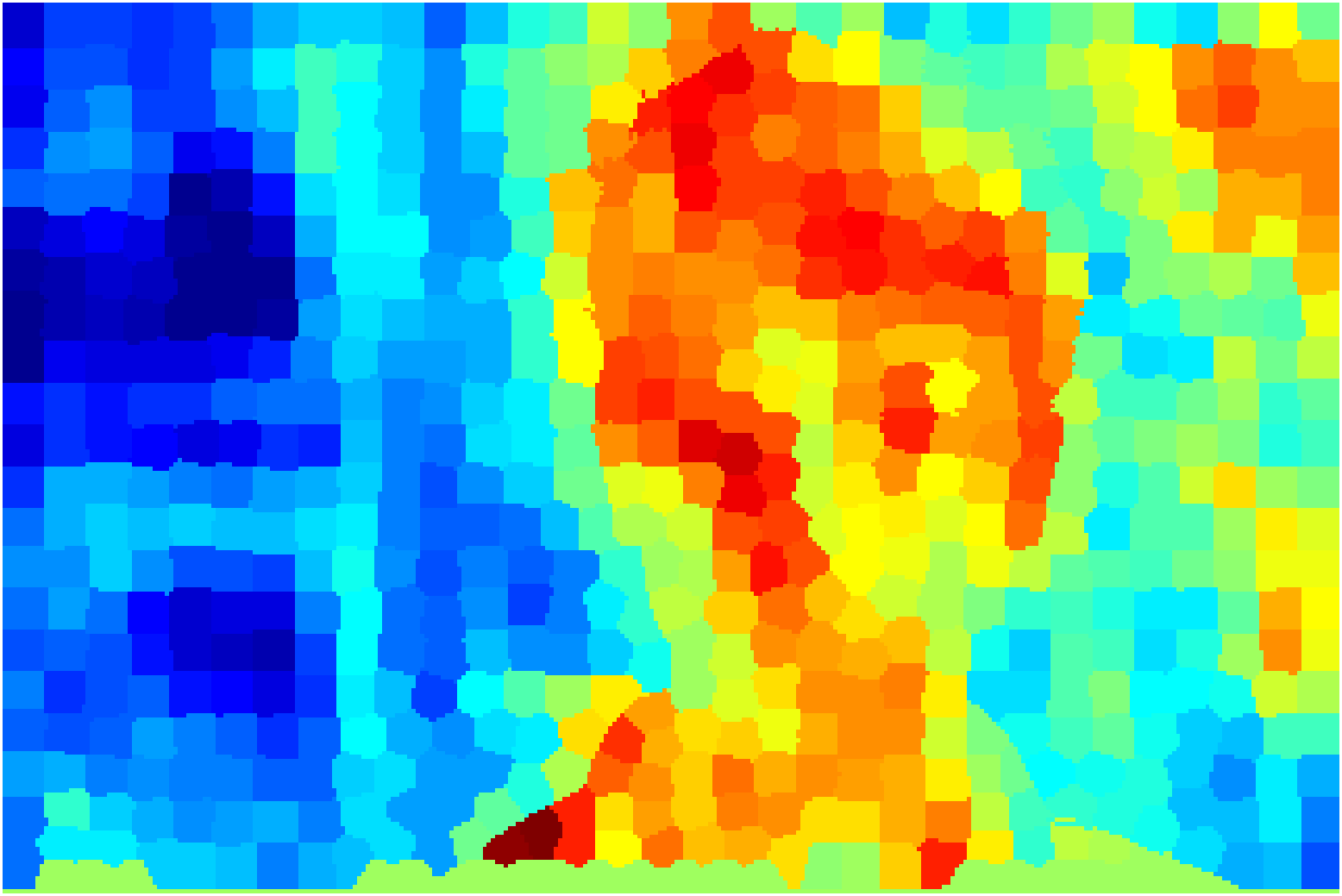}
\includegraphics[width=0.113\textwidth,clip]{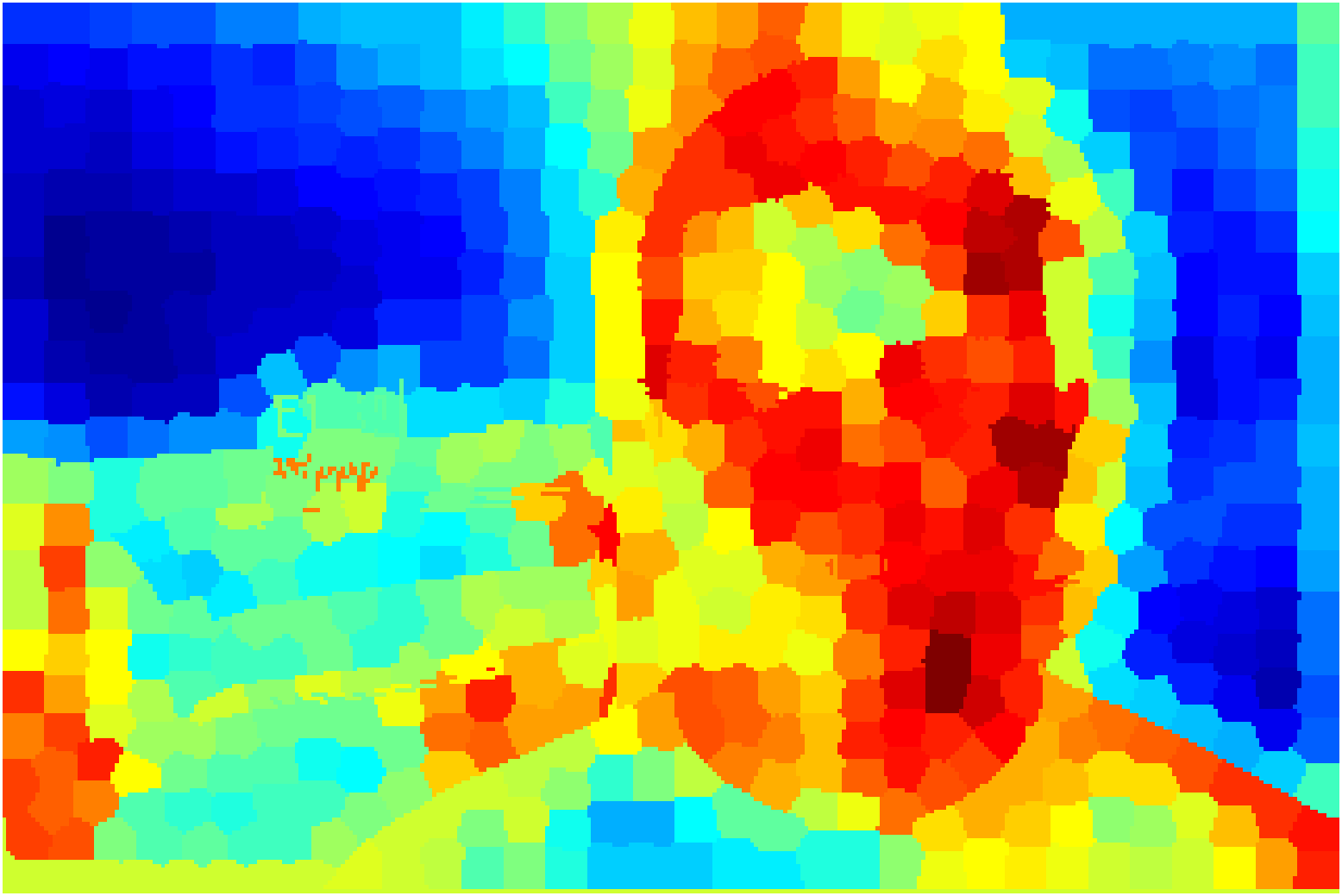}
\includegraphics[width=0.113\textwidth,clip]{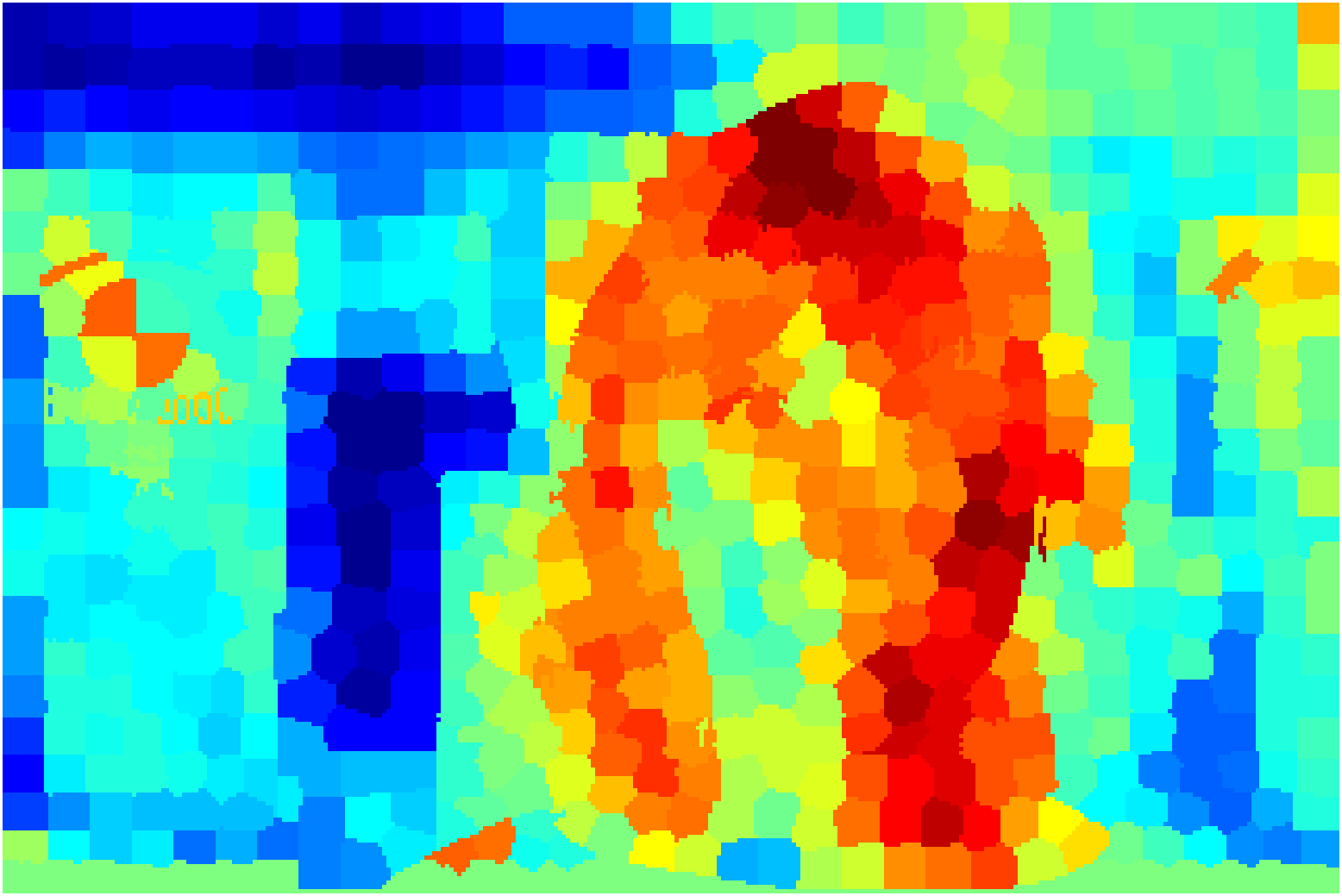}
\includegraphics[width=0.113\textwidth,clip]{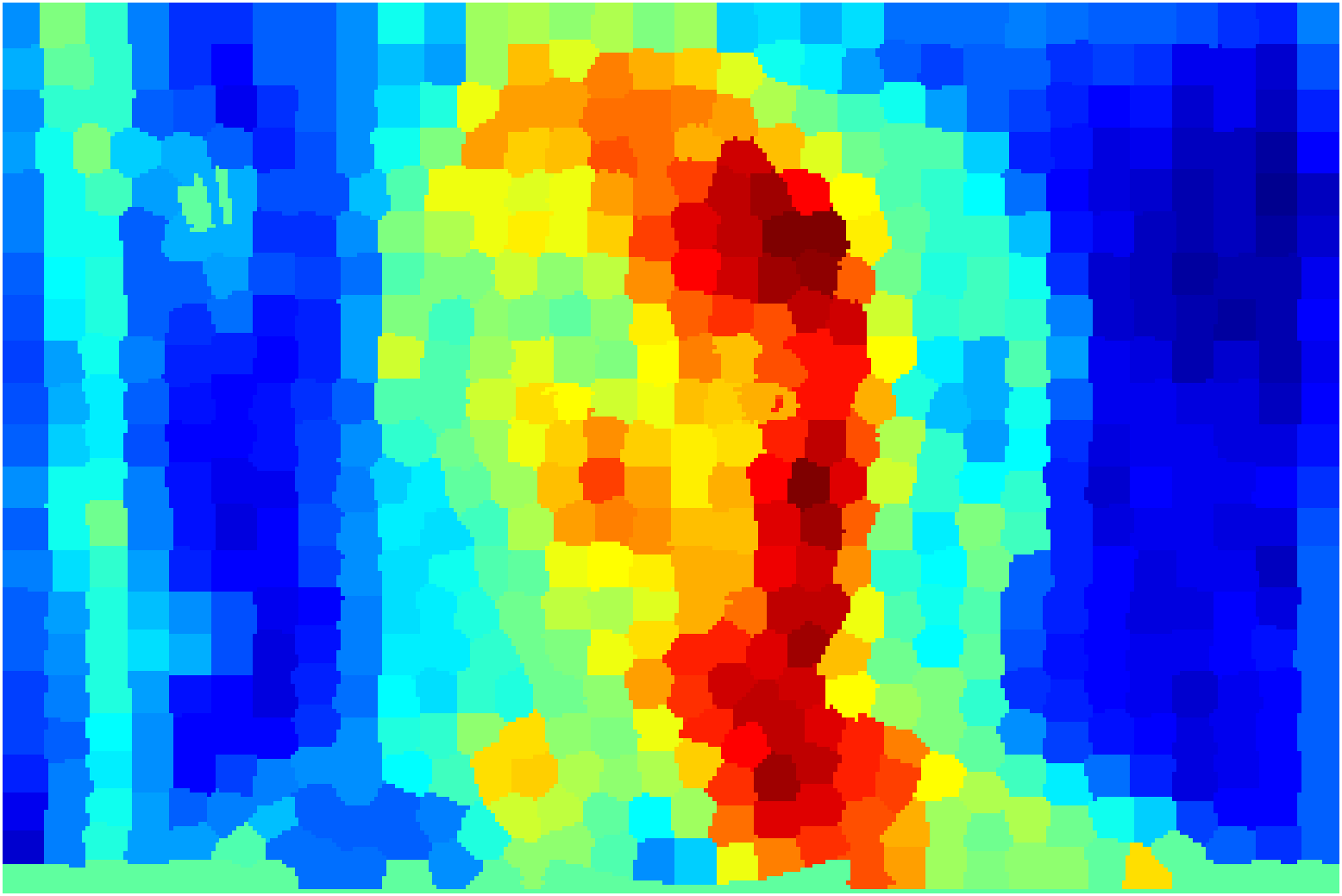}
\includegraphics[width=0.113\textwidth, height = 0.076\textwidth, clip]{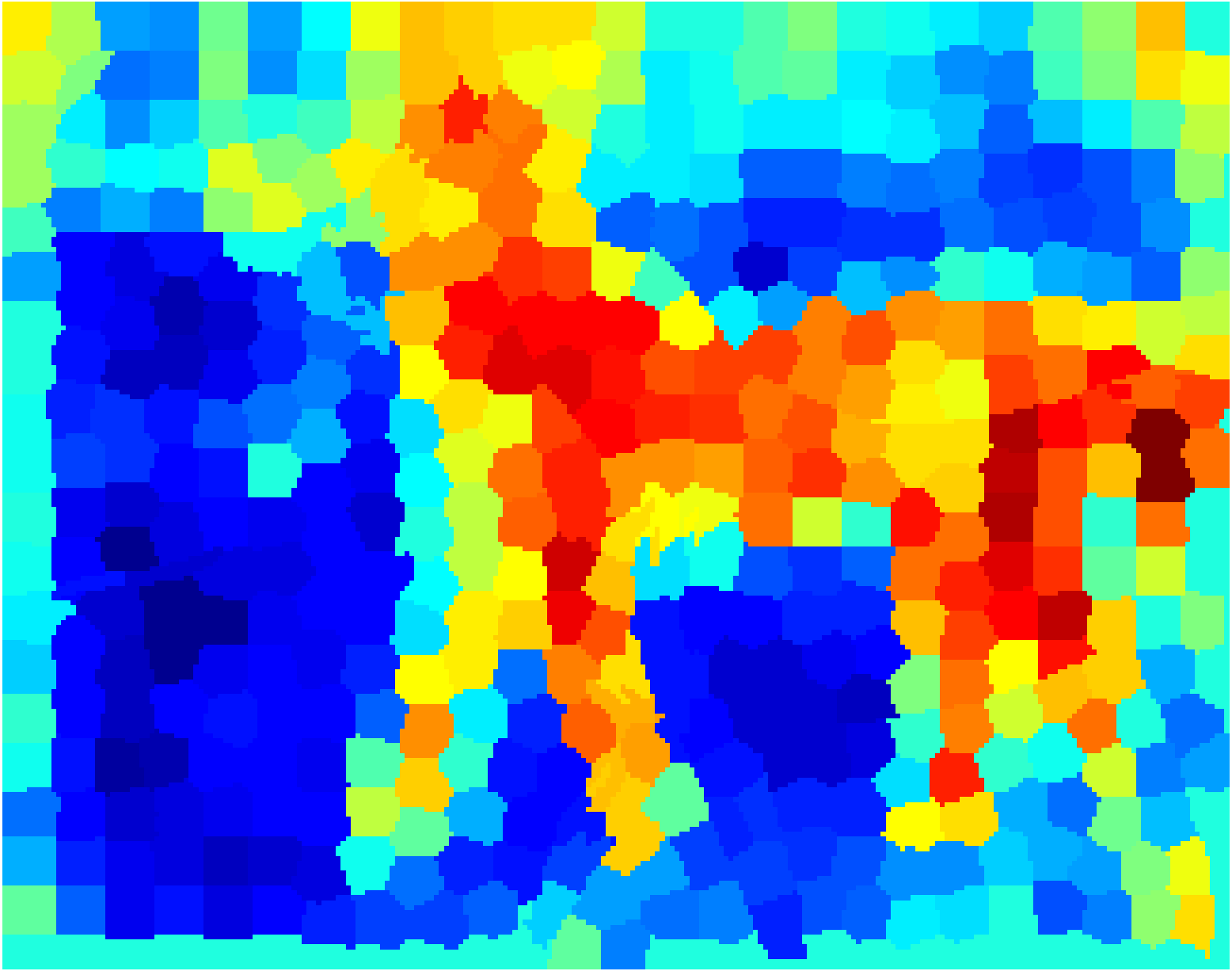}
\includegraphics[width=0.113\textwidth, height = 0.076\textwidth, clip]{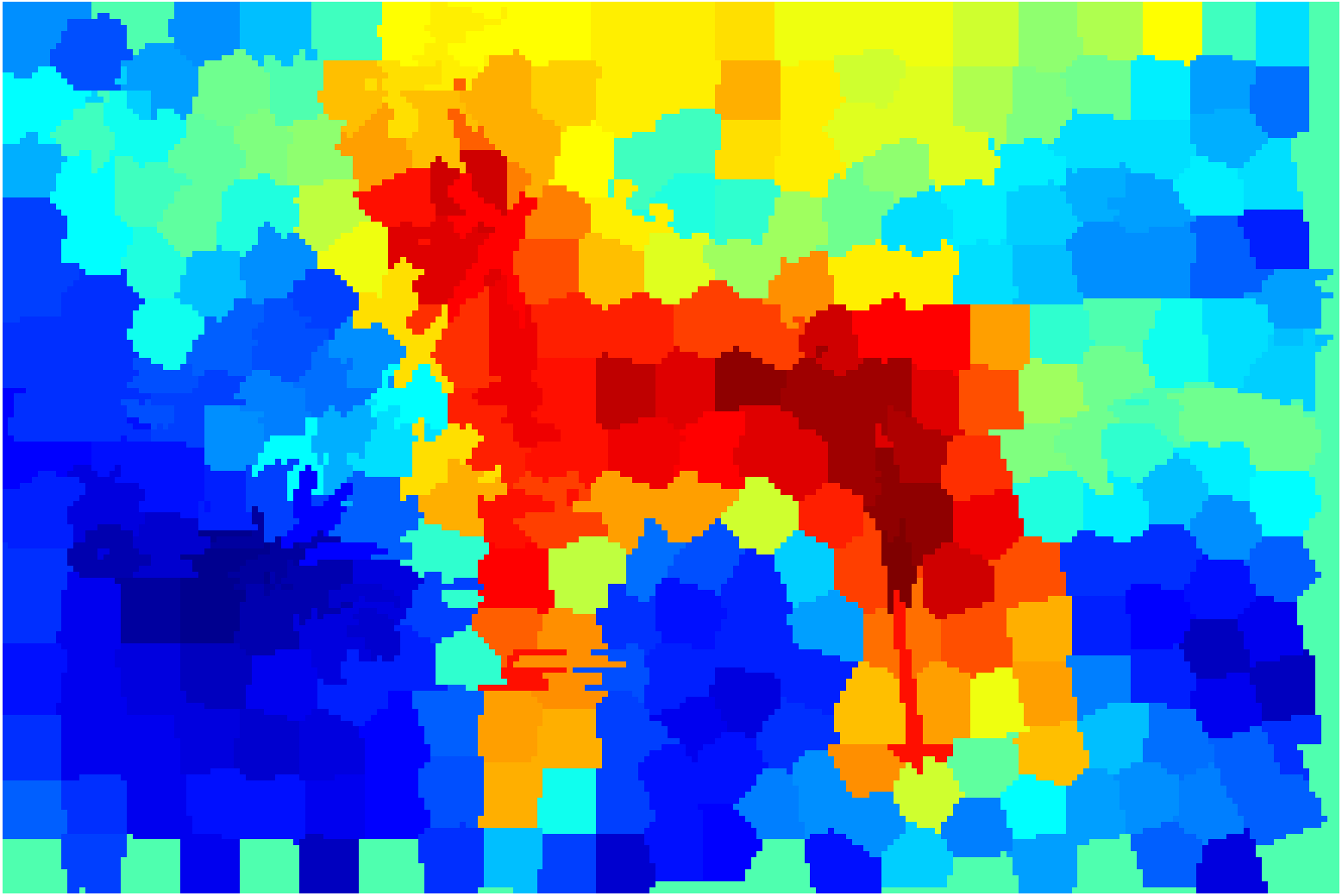}
\includegraphics[width=0.113\textwidth, height = 0.076\textwidth, clip]{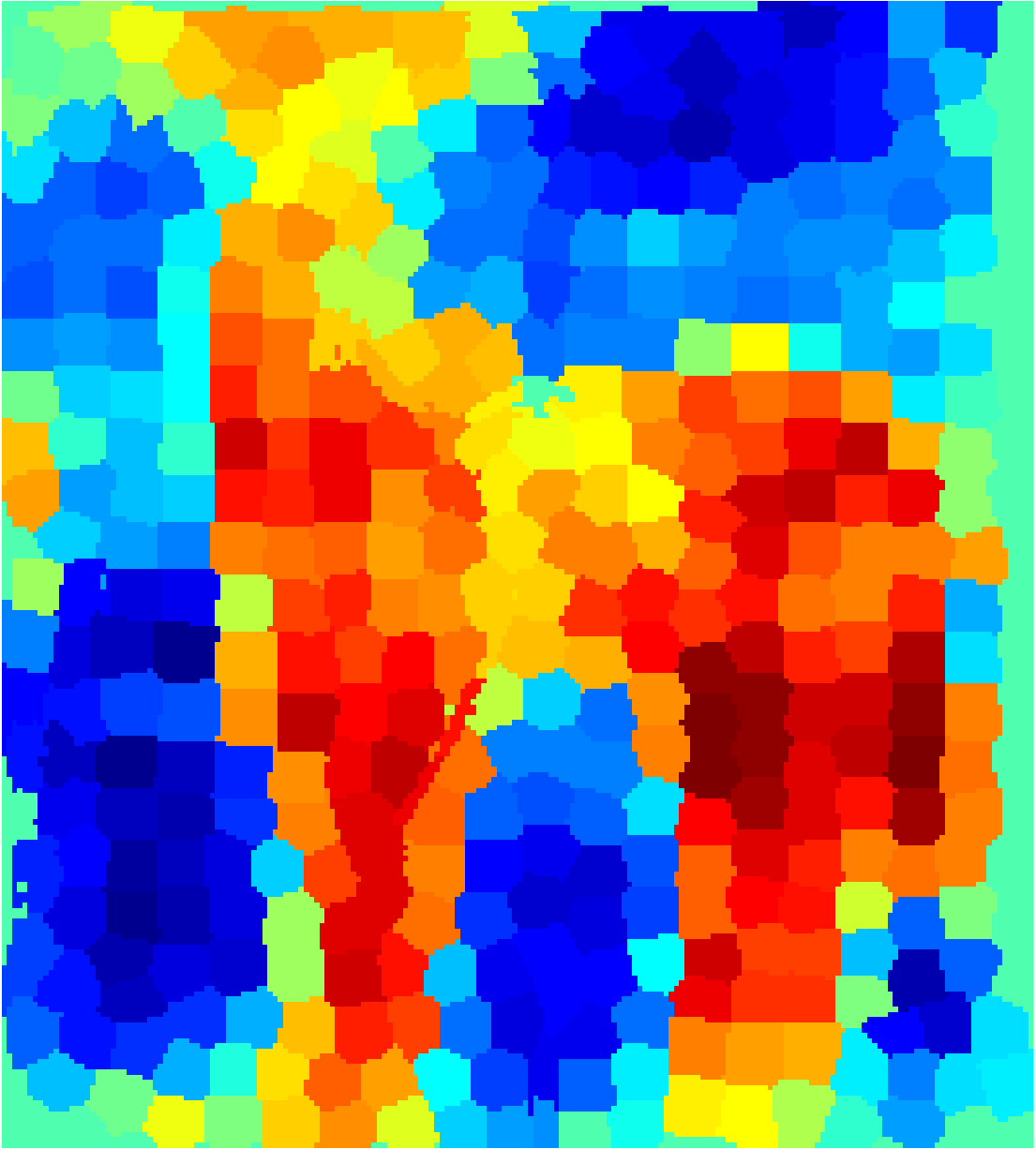}
\includegraphics[width=0.113\textwidth, height = 0.076\textwidth, clip]{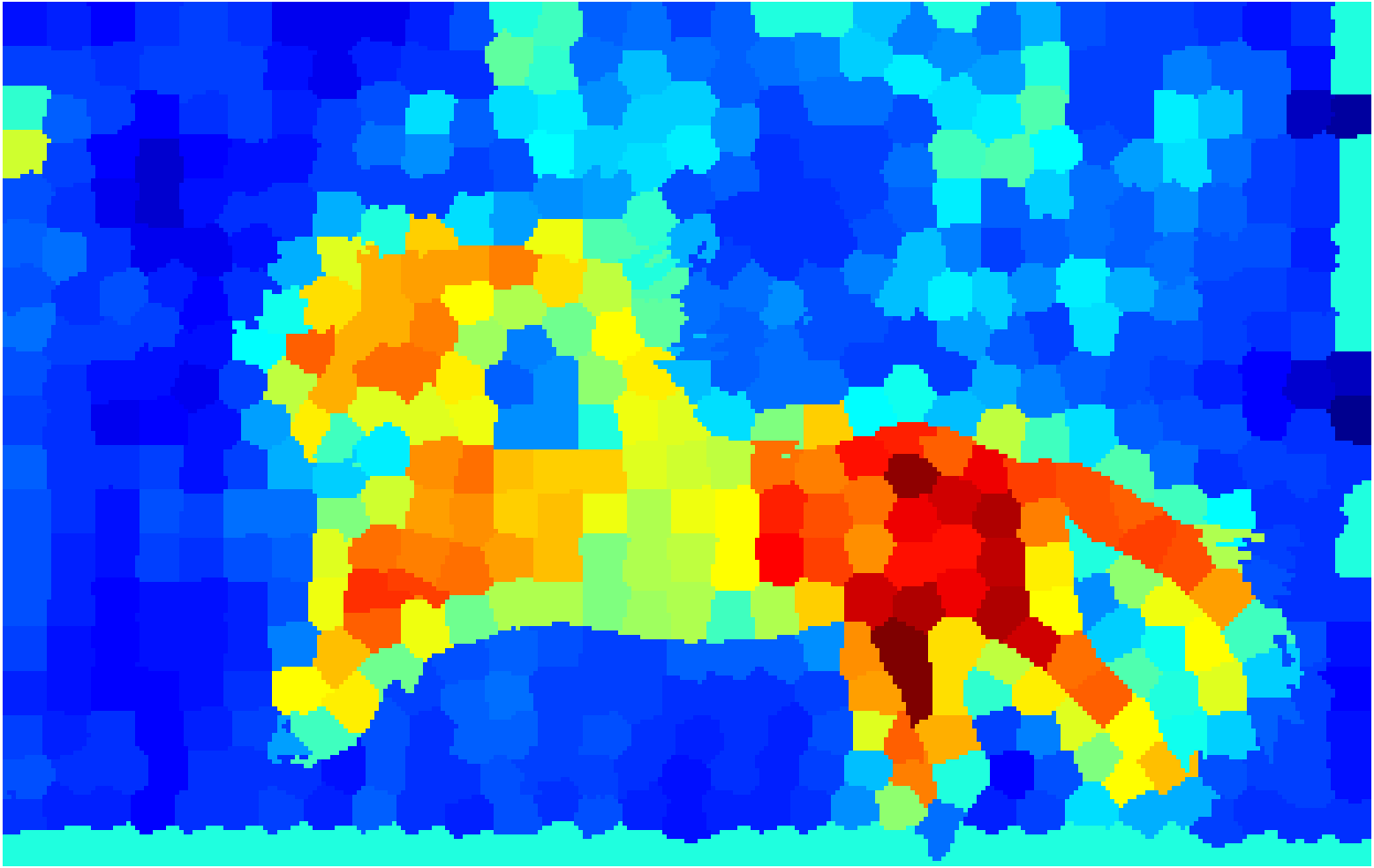}
\end{minipage}
}}\\ \vspace{-2mm}
\subfloat{
\centering{
\begin{minipage}[c]{0.01\textwidth}
\begin{turn}{90} {\scriptsize \lbfgsb} \end{turn}
\end{minipage}
\begin{minipage}[c]{0.99\textwidth}
\includegraphics[width=0.113\textwidth,clip]{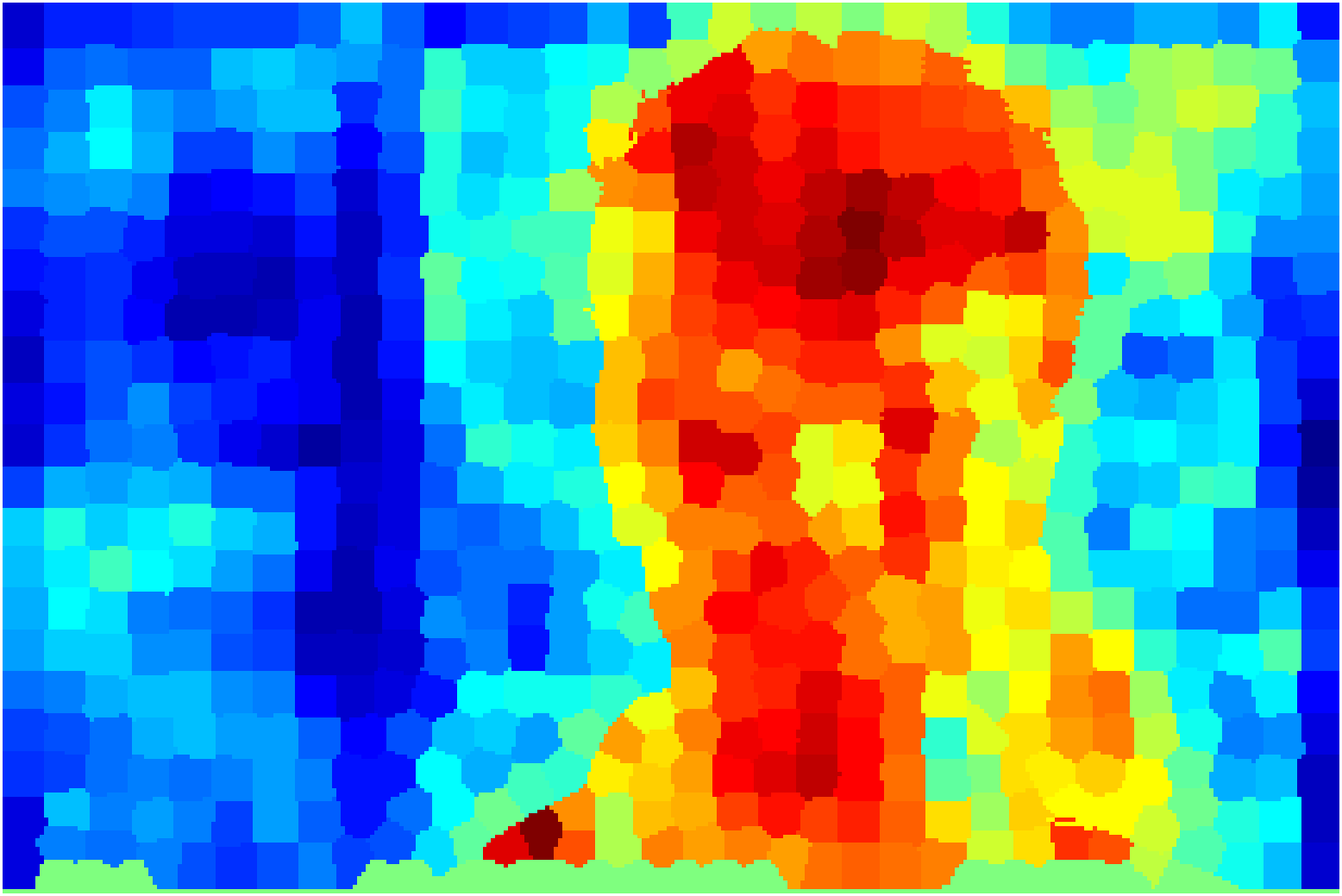}
\includegraphics[width=0.113\textwidth,clip]{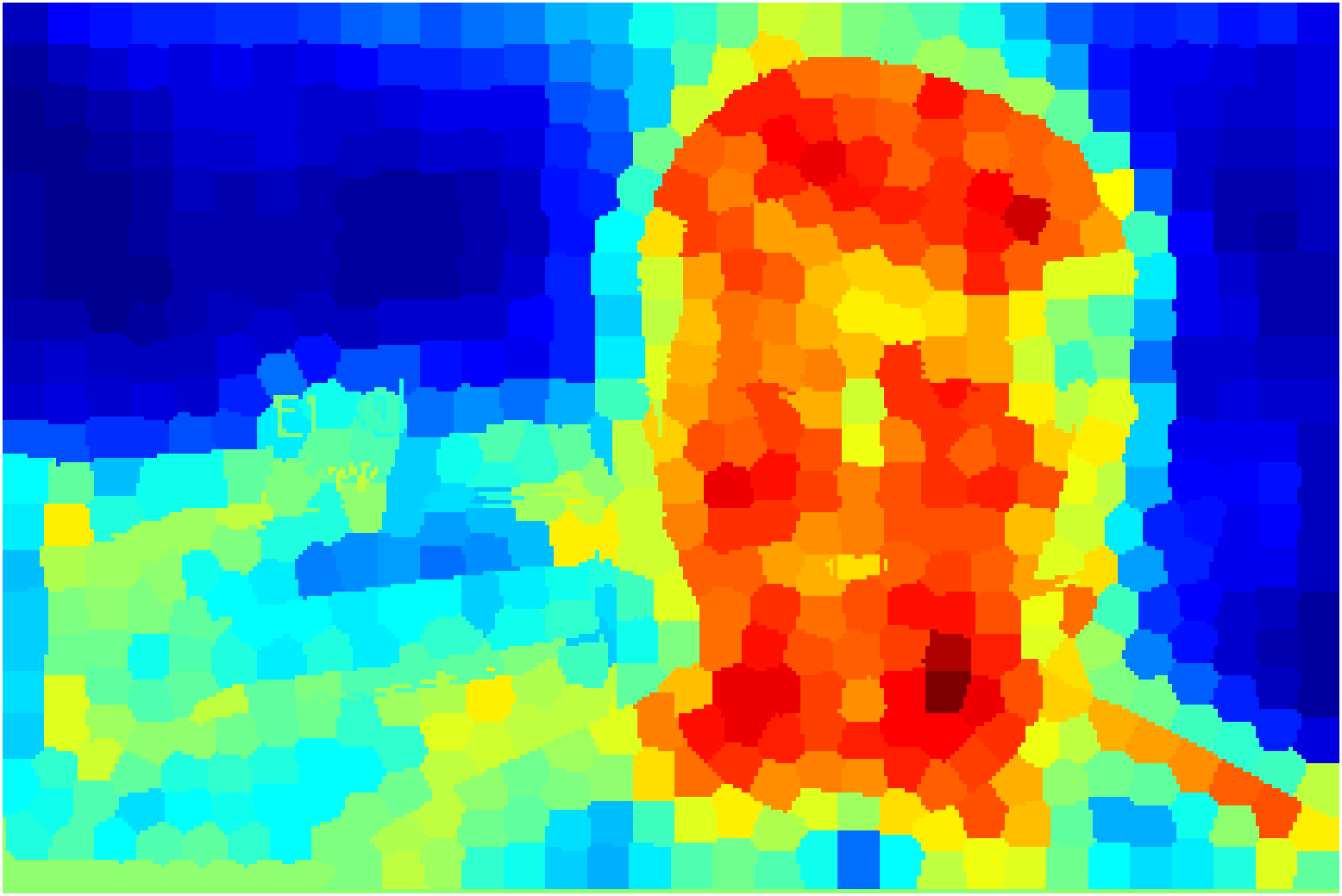}
\includegraphics[width=0.113\textwidth,clip]{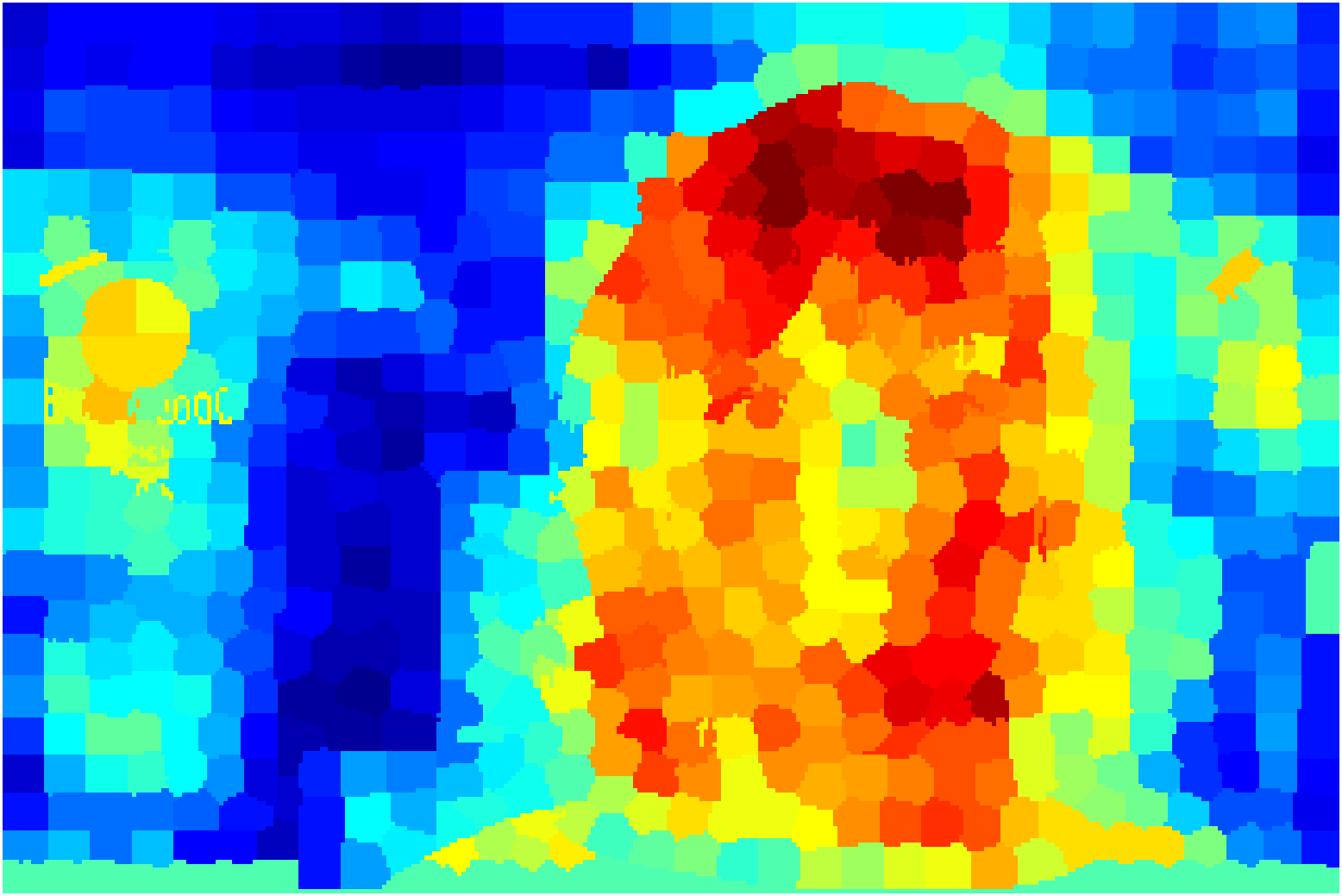}
\includegraphics[width=0.113\textwidth,clip]{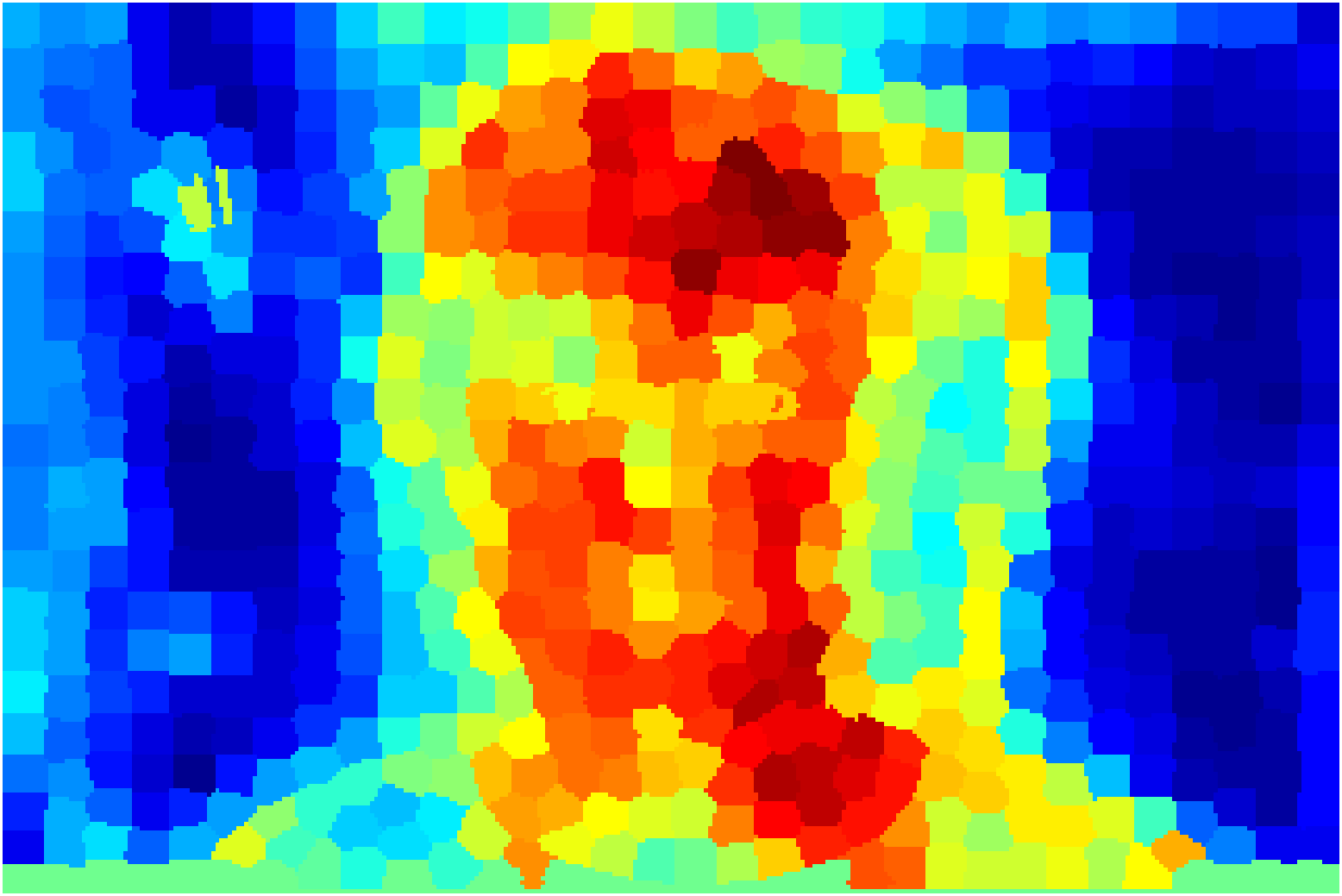}
\includegraphics[width=0.113\textwidth, height = 0.076\textwidth, clip]{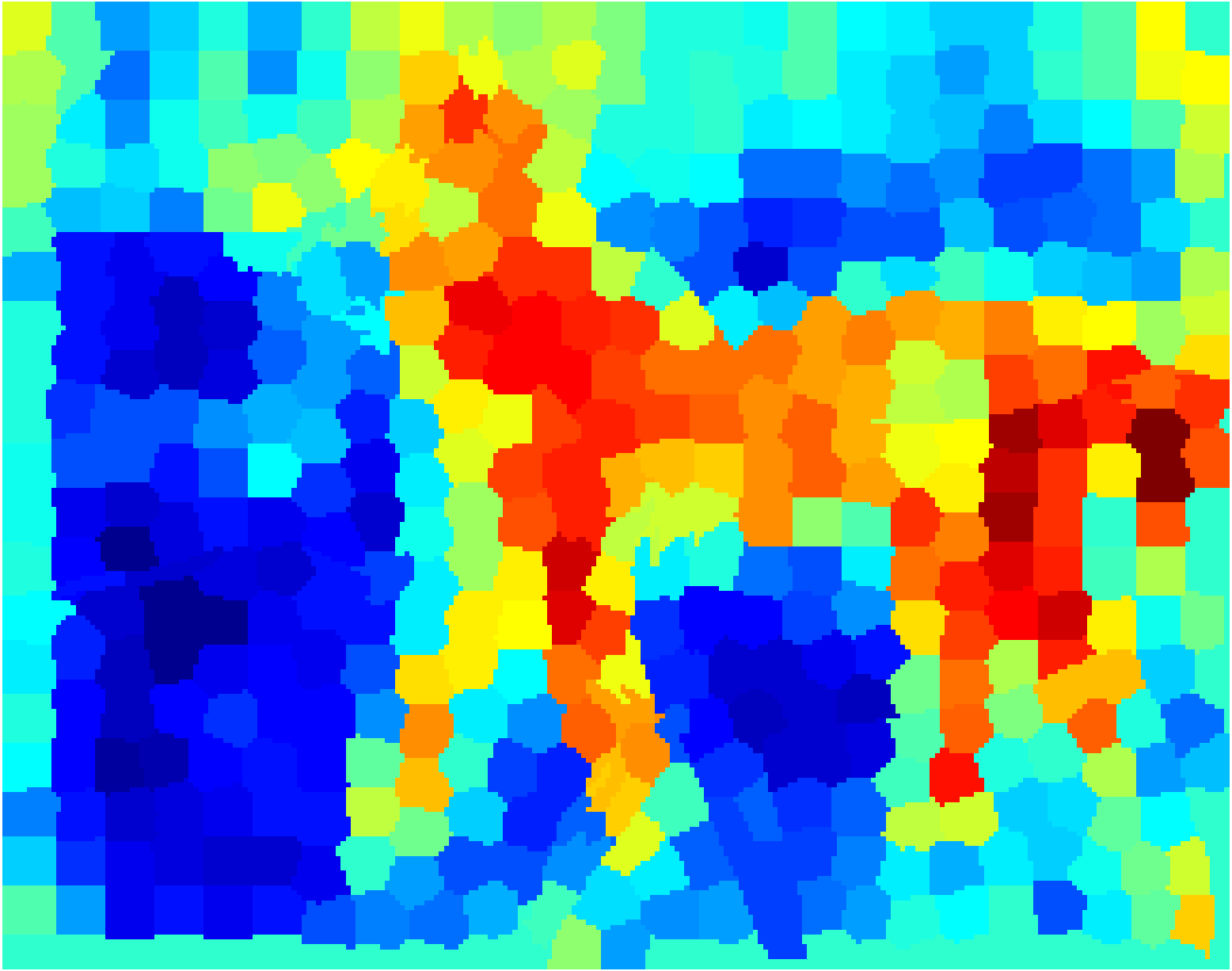}
\includegraphics[width=0.113\textwidth, height = 0.076\textwidth, clip]{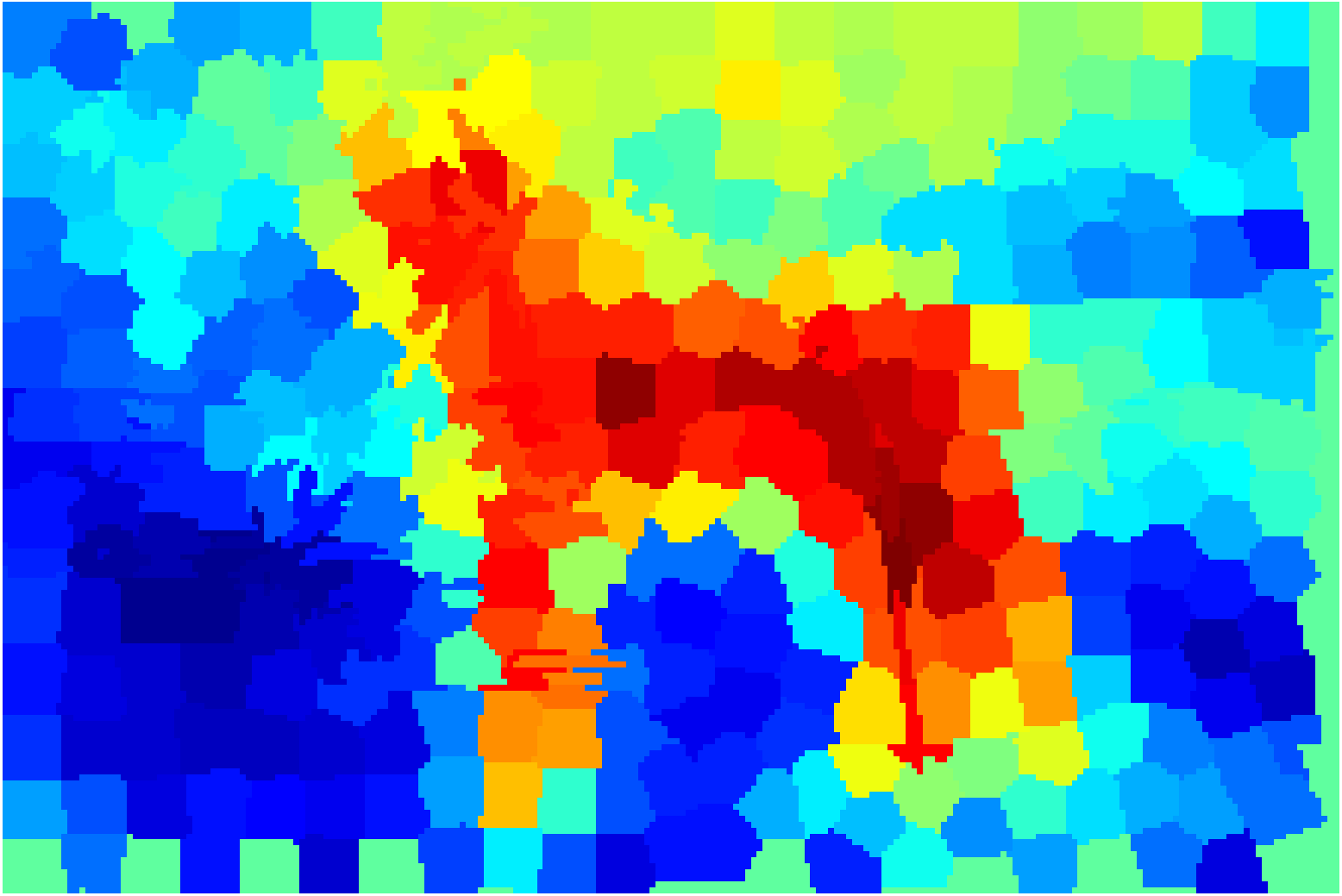}
\includegraphics[width=0.113\textwidth, height = 0.076\textwidth, clip]{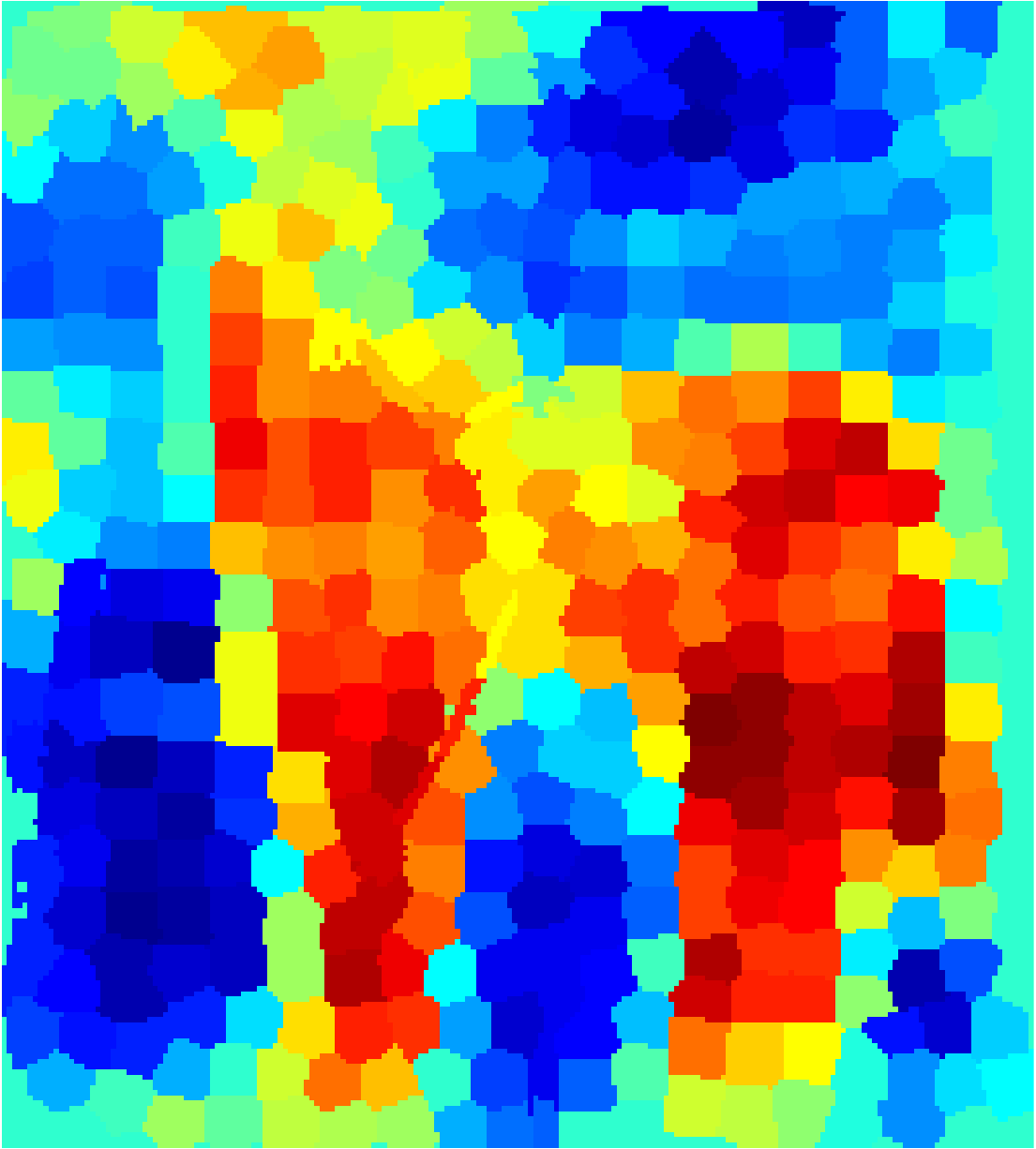}
\includegraphics[width=0.113\textwidth, height = 0.076\textwidth, clip]{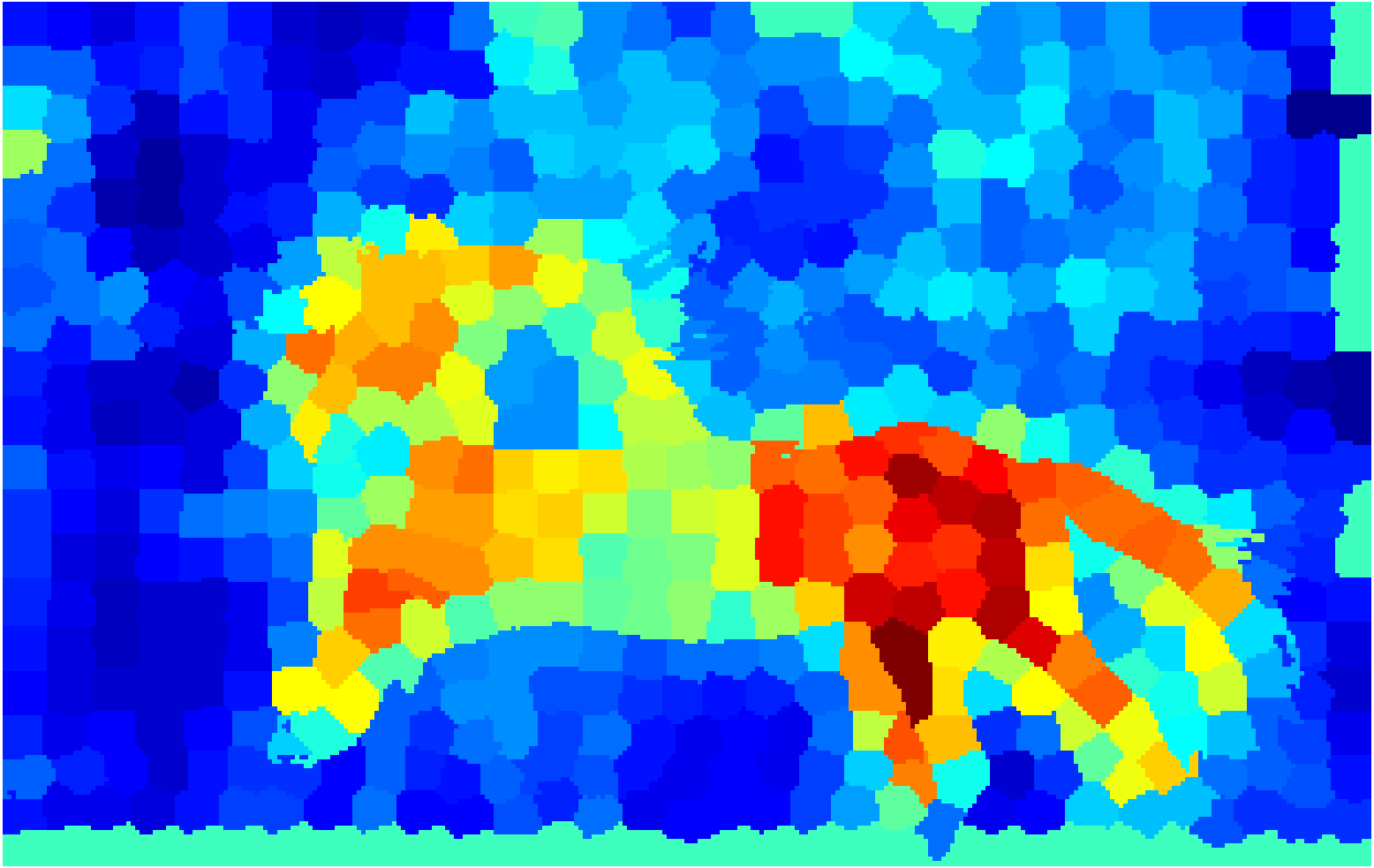}
\end{minipage}
}}\\ \vspace{-2mm}
\vspace{0.2cm}
\caption{Image co-segmentation on Weizman horses and MSRC datasets.
The original images, the results (score vectors) of \lowrank and \lbfgsb are illustrated from top to bottom.
Other methods produce similar segmentation results.}
\label{fig:img_cosegm}
}
\vspace{-0.0cm}
\end{figure*}

The task of image co-segmentation~\cite{Joulin2010dis} aims to partition a common object from multiple images simultaneously.
In this work, the Weizman horses\footnote{\url{http://www.msri.org/people/members/eranb/}}
and
MSRC\footnote{\url{http://www.research.microsoft.com/en-us/projects/objectclassrecognition/}}
datasets are tested.
There are $6\!\thicksim\!10$ images in each of four classes, namely ``car-front'', ``car-back'', ``face'' and ``horse''.
Each image is oversegmented to $400\!\thicksim\!700$ superpixels.
The number of binary variables $n$ is then increased to $4000\!\thicksim\!7000$.

The BQP formulation for image co-segmentation can be found in Table~\ref{tab:formulation} (see \cite{Joulin2010dis} for details).
The matrix $\bA$ can be decomposed into a sparse matrix and a structural matrix,
such that ARPACK can be used by \lbfgsb.
{%
Each vector $\bz = [{\bz^{(1)}}^\T, \cdots, {\bz^{(K)}}^\T ]^\T$ (where $\bz^{i}$ corresponds to the $i$-th image)
randomly sampled from $\mathcal{N}(\mathbf{0}, \bX^\star)$ is discretized to a feasible BQP solution as follows:
\begin{align}
\label{eq:cosegm-round}
\bx = \left[ \mathrm{sign}(\bz^{(1)} - \theta^{(1)})^\T, \cdots, \mathrm{sign}(\bz^{(K)} - \theta^{(K)})^\T \right]^\T,
\end{align}
where $\theta^{(i)}$ can be obtained as \eqref{eq:hist-round}.
}

We compare our methods with the low-rank factorization method~\cite{Joulin2010dis} (referred to as \lowrank) and
interior-point methods.
As we can see in Table~\ref{tab:imgcoseg},
\lbfgsb takes $10$ times more iterations than \smooth, but still runs faster than \smooth especially
when the size of problem is large (see ``face'' data).
The reason is that \lbfgsb can exploit the specific structure of matrix $\bA$ in eigen-decomposition.
\lbfgsb runs also $5$ times faster than \lowrank.
All methods provide similar upper-bounds (primal objective values),
and the score vectors shown in Fig.~\ref{fig:img_cosegm} also show that the evaluated methods achieve similar visual results.

\begin{table*}[t]
  \centering
\begin{minipage}[t]{0.7\textwidth}
  \centering
  \scriptsize
  \begin{tabular}{l|l|cccccc}
  \hline
& & & & & &\\ [-2ex]
     Data, $n$, $m$  & Methods  & \lbfgsb  & \smooth & SeDuMi & MOSEK & \lowrank \\
  \hline
  \hline
& \\ [-2ex]
\multirow{3}{*}{\begin{tabular}{c} $\text{car-back}$, \\ $4012$, $4018$ \end{tabular}}
& Time/Iters
& $\mathbf{06}$\bf{m}$\mathbf{08}$\bf{s}/$140$	& $09$m$59$s/$\mathbf{15}$	& $07$h$02$m		& $02$h$54$m	& $28$m$44$s	\\
& Upper-bound
& $12.71$	& $12.71$	& $12.74$		& $12.74$	& $12.64$	\\
& Lower-bound
& $-7.41$	& $-7.41$	& $-7.30$		& $-7.30$	& \NA	\\
& \\ [-2ex]
\hline\multirow{3}{*}{\begin{tabular}{c} $\text{car-front}$, \\ $4017$, $4023$ \end{tabular}}
& \\ [-2ex]
& Time/Iters
& $\mathbf{07}$\bf{m}$\mathbf{32}$\bf{s}/$188$	& $11$m$25$s/$\mathbf{16}$	& $07$h$04$m		& $02$h$54$m	& $59$m$47$s	\\
& Upper-bound
& $8.16$	& $8.16$	& $8.16$		& $8.16$	& $8.61$	\\
& Lower-bound
& $-7.67$	& $-7.67$	& $-7.56$		& $-7.56$	& \NA	\\
& \\ [-2ex]
\hline\multirow{3}{*}{\begin{tabular}{c} $\text{face}$, \\ $6684$, $6694$ \end{tabular}}
& \\ [-2ex]
& Time/Iters
& $\mathbf{08}$\bf{m}$\mathbf{18}$\bf{s}/$164$	& $43$m$57$s/$\mathbf{16}$	& $>24$hrs		& $12$h$06$m	& $40$m$56$s	\\
& Upper-bound
& $12.65$	& $12.65$	& \NA		& $12.96$	& $20.53$	\\
& Lower-bound
& $-9.73$	& $-9.73$	& \NA		& $-9.53$	& \NA	\\
& \\ [-2ex]
\hline\multirow{3}{*}{\begin{tabular}{c} $\text{horse}$, \\ $4587$, $4597$ \end{tabular}}
& \\ [-2ex]
& Time/Iters
& $\mathbf{06}$\bf{m}$\mathbf{15}$\bf{s}/$167$	& $17$m$01$s/$\mathbf{16}$	& $11$h$03$m		& $04$h$14$m	& $42$m$14$s	\\
& Upper-bound
& $14.78$	& $14.78$	& $14.76$		& $14.76$	& $15.77$	\\
& Lower-bound
& $-6.83$	& $-6.83$	& $-6.69$		& $-6.69$	& \NA	\\
\hline
  \end{tabular}
\end{minipage}\hfill
\begin{minipage}[t]{0.3\textwidth}
\vspace{-21mm}
  \caption{Numerical results for image co-segmentation.
           All the evaluated are SDP based methods and achieve similar upper-bounds and lower-bounds.
           \lbfgsb runs significantly faster than other methods,
           although it needs more iterations to converge than \smooth.}
\label{tab:imgcoseg}
\end{minipage}
\end{table*}

\vspace{0.1cm}
\subsection{Graph Matching}
\label{sec:matching}

In the graph matching problems considered in this work, each of the $K$ source points must be matched to one of the $L$ target points, where $L \geq K$.
The optimal matching should maximize both of the local feature similarities between matched-pairs
and the structure similarity between the source and target graphs.

The BQP formulation of graph matching can be found in Table~\ref{tab:formulation},
which can be relaxed to:
\begin{subequations}
\label{eq:graph_match_sdp}
\begin{align}
\min_{\bx, \bX}
           &\quad \langle \bX, \bH \rangle + \bh^\T \bx \\
\sst
                &\quad \mathrm{diag}(\bX) = \bx, \label{eq:graph_match_cons1} \\
                &\quad \textstyle{\sum_{j=1}^{L}} x_{(i-1)L+j} = 1, \, \forall i \in \mathcal{K}, \label{eq:graph_match_cons2} \\
                &\quad X_{\scriptsize (i-1)L+j,(i-1)L+k} = 0,  \forall i \in \mathcal{K},\,\, j \neq k \in \mathcal{L},  \label{eq:graph_match_cons3} \\
                &\quad X_{(j-1)L+i,(k-1)L+i} = 0,  \forall i \in \mathcal{L},\,\, j \neq k \in \mathcal{K}, \label{eq:graph_match_cons4} \\
                &\quad {\scriptsize \begin{bmatrix} 1 & \bx^\T \\ \bx & \bX \end{bmatrix} }\in \mathcal{S}^{KL+1},
\end{align}
\end{subequations}
where $\mathcal{K} = \{1, \cdots, K \}$ and $\mathcal{L} = \{1, \cdots, L \}$.
A feasible binary solution $\bx$ is obtained by solving the following linear program (see \cite{Schellewald05} for details):
\begin{subequations}
 \label{eq:graph_match_round}
\begin{align}
\max_{\bx \geq \mathbf{0} } \quad &\bx^{\T} \mathrm{diag}(\bX^\star), \\
\ \ \sst \quad
&\ \textstyle{\sum_{j=1}^L} \bx_{\tiny(i-1)\!L+j}=1,  \forall i \in \mathcal{K}, \\
&\ \textstyle{\sum_{i=1}^K} \bx_{\tiny(i-1)\!L+j} \leq 1,  \forall j \in \mathcal{L}.
\end{align}
\end{subequations}

Two-dimensional points are randomly generated for evaluation.
Table~\ref{tab:graphmatching} shows the results for different problem sizes: $n$ ranges from $201$ to $3201$ and $m$ ranges from $3011$ to $192041$.
\smooth and \lbfgsb achieves exactly the same upper-bounds as interior-point methods and comparable lower-bounds.
Regarding the running time, \smooth takes much less number of iterations to converge and is relatively faster (within $2$ times) than \lbfgsb.
Our methods run significantly faster than interior-point methods.
Taking the case $K \times L = 25 \times 50$ as an example,
\smooth and \lbfgsb converge at around $4$ minutes and interior-point methods do not converge within $24$ hours.
Furthermore, interior-point methods runs out of $100$G memory limit when the number of primal constraints $m$ is over $10^5$.
SMAC~\cite{Gour2006}, a spectral method incorporating affine constrains, is also evaluated in this experiment, which provides worse upper-bounds and error ratios.

\begin{table*}[t]
  \centering
  \scriptsize
  \begin{tabular}{l|l|ccccc|c}
  \hline
& & & & & & &\\ [-2ex]
     $K\!\times\!L$, $n$, $m$  & Methods  & \lbfgsb  & \smooth & SeDuMi & SDPT3 & MOSEK & SMAC \\
  \hline
  \hline
\multirow{4}{*}{\begin{tabular}{c} $10 \times 20$, \\ $201$, \\ $3011$ \end{tabular}}
& & & & & & &\\ [-2ex]
& Time/Iters
& $5.8$s/$262.6$	& $\mathbf{1.7}\bf{s}$/$\mathbf{34.2}$	& $02$m$17$s	& $45.0$s	& $30.7$s	& $0.1$s	\\
& Error ratio
& $1/100$	& $1/100$	& $1/100$	& $1/100$	& $1/100$	& $1/100$	\\
& Upper-bound
& $- 1.30\!\times\!10^{-1}$	& $- 1.30\!\times\!10^{-1}$	& $- 1.30\!\times\!10^{-1}$	& $- 1.30\!\times\!10^{-1}$	& $- 1.30\!\times\!10^{-1}$	& $- 1.27\!\times\!10^{-1}$	\\
& Lower-bound
& $- 1.31\!\times\!10^{-1}$	& $- 1.31\!\times\!10^{-1}$	& $- 1.31\!\times\!10^{-1}$	& $- 1.30\!\times\!10^{-1}$	& $- 1.30\!\times\!10^{-1}$	& \NA	\\
& & & & & & &\\ [-2ex]
\hline\multirow{4}{*}{\begin{tabular}{c} $15 \times 30$, \\ $451$, \\ $10141$ \end{tabular}}
& & & & & & &\\ [-2ex]
& Time/Iters
& $22.5$s/$359.7$	& $\mathbf{11.2}\bf{s}$/$\mathbf{35.7}$	& $01$h$34$m	& $15$m$48$s	& $30$m$38$s	& $0.3$s	\\
& Error ratio
& $1/150$	& $1/150$	& $1/150$	& $1/150$	& $1/150$	& $6/150$	\\
& Upper-bound
& $- 3.77\!\times\!10^{-2}$	& $- 3.77\!\times\!10^{-2}$	& $- 3.77\!\times\!10^{-2}$	& $- 3.77\!\times\!10^{-2}$	& $- 3.77\!\times\!10^{-2}$	& $- 2.01\!\times\!10^{-2}$	\\
& Lower-bound
& $- 3.81\!\times\!10^{-2}$	& $- 3.81\!\times\!10^{-2}$	& $- 3.78\!\times\!10^{-2}$	& $- 3.79\!\times\!10^{-2}$	& $- 3.78\!\times\!10^{-2}$	& \NA	\\
& & & & & & &\\ [-2ex]
\hline\multirow{4}{*}{\begin{tabular}{c} $20 \times 40$, \\ $801$, \\ $24021$ \end{tabular}}
& & & & & & &\\ [-2ex]
& Time/Iters
& $01$m$27$s/$405.2$	& $\mathbf{51.2}\bf{s}$/$\mathbf{41.7}$	& $17$h$48$m	& $02$h$09$m	& $04$h$39$m	& $0.2$s	\\
& Error ratio
& $1/200$	& $1/200$	& $1/200$	& $1/200$	& $1/200$	& $6/200$	\\
& Upper-bound
& $ 4.01\!\times\!10^{-2}$	& $ 4.01\!\times\!10^{-2}$	& $ 4.01\!\times\!10^{-2}$	& $ 4.01\!\times\!10^{-2}$	& $ 4.01\!\times\!10^{-2}$	& $ 4.29\!\times\!10^{-2}$	\\
& Lower-bound
& $ 3.93\!\times\!10^{-2}$	& $ 3.93\!\times\!10^{-2}$	& $ 3.99\!\times\!10^{-2}$	& $ 3.98\!\times\!10^{-2}$	& $ 3.99\!\times\!10^{-2}$	& \NA	\\
& & & & & & &\\ [-2ex]
\hline\multirow{4}{*}{\begin{tabular}{c} $25 \times 50$, \\ $1251$, \\ $46901$ \end{tabular}}
& & & & & & &\\ [-2ex]
& Time/Iters
& $04$m$05$s/$384.0$	& $\mathbf{03}$\bf{m}$\mathbf{50}$\bf{s}/$\mathbf{41.0}$	& $>24$hrs	& $>24$hrs	& $>24$hrs	& $0.3$s	\\
& Error ratio
& $0/250$	& $0/250$	& \NA	& \NA	& \NA	& $3/250$	\\
& Upper-bound
& $ 1.04\!\times\!10^{-1}$	& $ 1.04\!\times\!10^{-1}$	& \NA	& \NA	& \NA	& $ 1.06\!\times\!10^{-1}$	\\
& Lower-bound
& $ 1.03\!\times\!10^{-1}$	& $ 1.03\!\times\!10^{-1}$	& \NA	& \NA	& \NA	& \NA	\\
& & & & & & &\\ [-2ex]
\hline\multirow{4}{*}{\begin{tabular}{c} $30 \times 60$, \\ $1801$, \\ $81031$ \end{tabular}}
& & & & & & &\\ [-2ex]
& Time/Iters
& $14$m$43$s/$500.0$	& $\mathbf{10}$\bf{m}$\mathbf{20}$\bf{s}/$\mathbf{50.0}$	& $>24$hrs	& $>24$hrs	& $>24$hrs	& $0.4$s	\\
& Error ratio
& $2/300$	& $2/300$	& \NA	& \NA	& \NA	& $4/300$	\\
& Upper-bound
& $ 1.59\!\times\!10^{-1}$	& $ 1.59\!\times\!10^{-1}$	& \NA	& \NA	& \NA	& $ 1.60\!\times\!10^{-1}$	\\
& Lower-bound
& $ 1.58\!\times\!10^{-1}$	& $ 1.58\!\times\!10^{-1}$	& \NA	& \NA	& \NA	& \NA	\\
& & & & & & &\\ [-2ex]
\hline\multirow{4}{*}{\begin{tabular}{c} $40 \times 80$, \\ $3201$, \\ $192041$ \end{tabular}}
& & & & & & &\\ [-2ex]
& Time/Iters
& $03$h$02$m/$500.0$	& $\mathbf{02}$\bf{h}$\mathbf{26}$\bf{m}/$\mathbf{50.0}$	& Out of mem.	& Out of mem.	& Out of mem.	& $1.2$s	\\
& Error ratio
& $1/400$	& $1/400$	& \NA	& \NA	& \NA	& $9/400$	\\
& Upper-bound
& $ 2.63\!\times\!10^{-1}$	& $ 2.63\!\times\!10^{-1}$	& \NA	& \NA	& \NA	& $ 2.66\!\times\!10^{-1}$	\\
& Lower-bound
& $ 2.61\!\times\!10^{-1}$	& $ 2.61\!\times\!10^{-1}$	& \NA	& \NA	& \NA	& \NA	\\
\hline\end{tabular}
  \caption{Numerical results for graph matching, which are the mean over $10$ random graphs.
           For the fourth and fifth models, interior-point methods including Sedumi, SDPT3 and Mosek do not converge within $24$ hours.
           For the last model with around $2 \times 10^5$ constraints, Sedumi, SDPT3 and Mosek run out of $100$G memory limit.
           \smooth uses fewer iterations than \lbfgsb and achieves the fastest speed over all SDP based methods.
           All SDP based methods achieve the same upper-bounds and error rates.
           The lower-bounds for \smooth and \lbfgsb are slightly worse than interior-point methods.
           SMAC provides worse upper-bounds and error rates than SDP-based methods.
}
\label{tab:graphmatching}
\end{table*}

\subsection{Image Deconvolution}
\label{sec:deconv}

\begin{table}[t]
  \centering
  \scriptsize
{ %
  \begin{tabular}{l|c@{\hspace{0.1cm}}c@{\hspace{0.1cm}}c@{\hspace{0.1cm}}c@{\hspace{0.1cm}}c@{\hspace{0.1cm}}c}
  \hline
& \\ [-2ex]
      Methods  & \smooth & SeDuMi & MOSEK &TRWS &MPLP\\
  \hline
  \hline
& & & & &\\ [-2ex]
 Time/Iters
	        & $\mathbf{2}$\bf{m}$\mathbf{27}$\bf{s}/$\mathbf{20.5}$	& $2$h$48$m	& $21$m$33$s	& $20$m & $20$m		\\
 Error
	& $0.091$	& $\mathbf{0.074}$	& $0.083$	& $0.111$  & $0.112$		\\
 Upper-bound
	& $-988.3$	& $\mathbf{-988.8}$	& $-988.8$		& $-986.9$  & $-987.4$	\\
 Lower-bound
	& $-1054.2$	& $\mathbf{-993.7}$	& $-993.7$		& $-1237.9$  & $-1463.8$	\\
\hline\end{tabular} }
  \caption{{ %
Image deconvolution ($n = 2250$, $m = 2250$). The results are the average over two models in Figure~\ref{fig:deconv}. TRWS and MPLP are stopped at $20$ minutes.
Compared to TRWS and MPLP, \smooth achieves better upper-/lower-bounds and uses less time.
The bounds given by \smooth is comparable to those of interior-point methods.
}}
\vspace{-0.3cm}
\label{tab:binary}
\end{table}

{ %
Image deconvolution with a known blurring kernel is typically equivalent to
solving a regularized linear inverse problem (see \eqref{eq:deconv} in Table~\ref{tab:formulation}).
}
In this experiment, we test our algorithms on two binary $30 \times 75$ images blurred by an $11 \times 11$ Gaussian kernel.
LP based methods such as TRWS~\cite{kolmogorov2006convergent} and MPLP~\cite{Globerson2007} are also evaluated.
Note that the resulting models are difficult for graph cuts or LP relaxation based methods
in that it is densely connected and contains a large portion of non-submodular pairwise potentials.
We can see from Fig.~\ref{fig:deconv} and Table~\ref{tab:binary} that
QPBO~\cite{kolmogorov2004energy,kolmogorov2007minimizing,rother2007optimizing} leaves most of pixels unlabelled
and LP methods (TRWS and MPLP) achieves worse segmentation accuracy.
\smooth achieves a $10$-fold speedup over interior-point methods while keep comparable upper-/lower-bounds.
Using much less runtime, \smooth still yields significantly better upper-/lower-bounds than LP methods.

\begin{figure*}[t]
\centering
\begin{minipage}[c]{0.13\textwidth}
\includegraphics[width=1\textwidth]{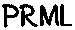}\\
\includegraphics[width=1\textwidth]{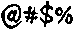}\\
\vspace{-0.8cm}
\begin{center}
{\centering \scriptsize Images}
\end{center}
\end{minipage}
\begin{minipage}[c]{0.13\textwidth}
\includegraphics[width=1\textwidth]{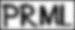}\\
\includegraphics[width=1\textwidth]{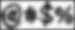}\\
\vspace{-0.8cm}
\begin{center}
{\centering \scriptsize Blurred Images}
\end{center}
\end{minipage}
\begin{minipage}[c]{0.13\textwidth}
\includegraphics[width=1\textwidth]{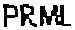}\\
\includegraphics[width=1\textwidth]{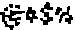}\\
\vspace{-0.8cm}
\begin{center}
{\centering \scriptsize \smooth}
\end{center}
\end{minipage}
\begin{minipage}[c]{0.13\textwidth}
\includegraphics[width=1\textwidth]{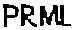}\\
\includegraphics[width=1\textwidth]{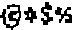}\\
\vspace{-0.8cm}
\begin{center}
{\centering \scriptsize MOSEK}
\end{center}
\end{minipage}
\begin{minipage}[c]{0.13\textwidth}
\includegraphics[width=1\textwidth]{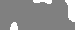}\\
\includegraphics[width=1\textwidth]{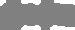}\\
\vspace{-0.8cm}
\begin{center}
{\centering \scriptsize QPBO}
\end{center}
\end{minipage}
\begin{minipage}[c]{0.13\textwidth}
\includegraphics[width=1\textwidth]{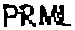}\\
\includegraphics[width=1\textwidth]{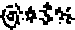}\\
\vspace{-0.8cm}
\begin{center}
{\centering \scriptsize TRWS}
\end{center}
\end{minipage}
\begin{minipage}[c]{0.13\textwidth}
\includegraphics[width=1\textwidth]{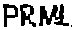}\\
\includegraphics[width=1\textwidth]{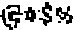}\\
\vspace{-0.8cm}
\begin{center}
{\centering \scriptsize MPLP}
\end{center}
\end{minipage}
\caption{{ %
Image deconvolution. QPBO cannot label most of pixels (grey pixels denote unlabelled pixels), as the MRF models are highly non-submodular.
\smooth and MOSEK have similar segmentation results.
TRWS and MPLP achieve worse segmentation results than our methods.
\label{fig:deconv}} }
\end{figure*}

{ %
\subsection{Chinese Character Inpainting}
\label{sec:dtf}

The MRF models for Chinese character inpainting are obtained from the OpenGM benchmark~\cite{Andres12OpenGM},
in which the unary terms and pairwise terms are learned using decision tree fields~\cite{nowozin2011decision}.
As there are non-submodular terms in these models, they cannot be solved exactly using graph cuts.
In this experiments, all models are firstly reduced using QPBO and different algorithms are compared on the $100$ reduced models.
Our approach is compared to LP-based methods, including TRWS, MPLP.
From the results shown in Table~\ref{tab:dtf}, we can see that \smooth runs much faster than interior-point methods (SeDuMi and MOSEK) and has similar upper-bounds and lower-bounds.
\smooth is also better than TRWS and MPLP  in terms of upper-bound and lower-bound.
An extension of MPLP (refer to as MPLP-C)~\cite{sontag2008,sontag2012efficiently}, which adds violated cycle constraints iteratively, is also evaluated in this experiment.
In MPLP-C, $1000$ LP iterations are performed initially and then $20$ cycle constraints are added at every $20$ LP iteration.
MPLP-C performs worse than \smooth under the runtime limit of $5$ minutes,
and outperforms \smooth with a much longer runtime limit ($1$ hour).
We also find that \smooth achieves better lower-bounds than MPLP-C on the instances with more edges (pairwise potential terms).
Note that the time complexity of MPLP-C (per LP iteration) is proportional to the number of edges,
while the time complexity of \smooth (see Table~\ref{tab:complexity}) is less affected by the edge number.
It should be also noticed that \smooth uses much less runtime than MPLP-C,
and its bound quality can be improved by adding linear constraints (including cycle constraints) as well~\cite{helmberg1998cutting}.

\begin{table*}[t]
  \centering
  \scriptsize
{ %
  \begin{tabular}{l|c@{\hspace{0.1cm}}c@{\hspace{0.1cm}}c@{\hspace{0.1cm}}c@{\hspace{0.1cm}}c@{\hspace{0.1cm}}c@{\hspace{0.1cm}}c@{\hspace{0.1cm}}c}
  \hline
& \\ [-2ex]
      Methods  & \smooth & SeDuMi & MOSEK &TRWS &MPLP &MPLP-C ($5$m) &MPLP-C ($1$hr)\\
  \hline
  \hline
& & & & &\\ [-2ex]
 Time/Iters
	& ${22.7}$s/${18.4}$	& $3$m$50$s	        & $42$s	        & ${23.5}${s} &$24.9$s     & $5$m        & $51$m$24$s	\\
 Upper-bound
	& $-49525.5$	        & ${-49525.5}(31)$	& $-49525.5(30)$	& $-49511.9(79)$ &$-49403.4(100)$  & $-49410.4(85)$	 & $-49495.8(47)$   \\
 Lower-bound
	& $-49683.0$	        & ${-49676.2}(0)$	& $-49676.2(0)$	& $-50119.4(100)$  &$-50119.4(100)$  & $-49908.7(72)$	 & $-49614.0(7)$   \\
\hline\end{tabular} }
  \caption{{ %
Chinese character inpainting using decision tree fields ($n = 191 \sim 1522$, $m = 191 \sim 1522$).
The results are the average over $100$ models.
The numbers of instances on which \smooth performs better are shown in the parentheses.
The upper-/lower-bounds given by \smooth is comparable to those of interior-point methods.
\smooth also achieves better solutions than TRWS, MPLP and MPLP-C ($5$m).
Using a much longer runtime ($55$m$24$s \vs $22.7$s), MPLP-C outperforms \smooth.
}}
\label{tab:dtf}
\end{table*}
}

\section{Conclusion}
In this paper, we have presented a regularized SDP algorithm (SDCut) for BQPs.
SDCut produces bounds comparable to the conventional SDP relaxation, and can be solved much more efficiently.
Two algorithms are proposed based on quasi-Newton methods (\lbfgsb) and smoothing Newton methods (\smooth) respectively.
Both \lbfgsb and \smooth are more efficient than classic interior-point algorithms.
To be specific, \smooth is faster than \lbfgsb for small to medium sized problems.
If the matrix to be eigen-decomposed, $\bC(\bu)$, has a special structure (\eg, sparse or low-rank)
such that matrix-vector products can be computed efficiently,
\lbfgsb is much more scalable to large problems.
The proposed algorithms have been applied to several computer vision tasks,
which demonstrate their flexibility in accommodating different types of constraints.
Experiments also show the computational efficiency and good solution quality of SDCut.
We have made the code
available online\footnote{\url{http://cs.adelaide.edu.au/~chhshen/projects/BQP/}}.

\vspace{2mm}
\noindent{\bf Acknowledgements}
We thank the anonymous reviewers for the constructive comments on Propositions~\ref{thm:prop1} and \ref{Remark:4}.

%This work was in part supported by
%the Data to Decisions Cooperative Research Centre, Australia.

  This work was in part supported by ARC Future
    Fellowship FT120100969.
    This work was also in part supported by
    the Data to Decisions Cooperative Research Centre.

\section{Proofs}
\label{sec:appdx}

\subsection{Proof of Proposition~\ref{thm:prop1}}
\begin{proof}

{%
($i$) Let $t := \frac{1}{2\gamma}$ and $\mathcal{P} := \{\bX \in \mathcal{S}^n_+ | \langle \bB_i, \bX \rangle = b_i, i \in \mathcal{I}_{eq}; \langle \bB_i, \bX \rangle \leq b_i, i \in \mathcal{I}_{in} \}$, we have
\begin{align}
\lvert \mathrm{p}(\bX^\star) - \mathrm{p}(\bX^\star_\gamma) \rvert
\leq & \mathrm{p}(\bX^\star_\gamma) + \frac{1}{2\gamma}\lVert \bX^\star_\gamma \rVert^2_F - \mathrm{p}(\bX^\star) \\
= & \min_{\bX \in \mathcal{P}} \mathrm{p}(\bX) + t \lVert \bX \rVert^2_F - \mathrm{p}(\bX^\star) := \theta (t). \notag
\end{align}
As a pointwise minimum of affine functions of $t$, $\theta (t)$ is concave and continuous.
It is also easy to find that $\theta (0) = 0$ and $\theta (t)$ is monotonically increasing on $\mathbb{R}_+$.
So for any $\epsilon > 0$, there is a $t > 0$ (and equivalently $\gamma = \frac{1}{2t} > 0$) such that
$\lvert \mathrm{p}(\bX^\star) - \mathrm{p}(\bX^\star_\gamma) \rvert  < \epsilon$.
}

\noindent
($ii$)
By the definition of $\bX^\star_{\gamma_1}$ and $\bX^\star_{\gamma_2}$, it is clear that
$\mathrm{p}_{\gamma_1}(\bX_{\gamma_1}^\star) \leq \mathrm{p}_{\gamma_1}(\bX_{\gamma_2}^\star)$ and
$\mathrm{p}_{\gamma_2}(\bX_{\gamma_2}^\star) \leq \mathrm{p}_{\gamma_2}(\bX_{\gamma_1}^\star)$.
Then we have
$\mathrm{p}_{\gamma_1}(\bX_{\gamma_1}^\star) - \frac{\gamma_2}{\gamma_1}\mathrm{p}_{\gamma_2}(\bX_{\gamma_1}^\star)
 = (1 \!-\! \frac{\gamma_2}{\gamma_1}) \cdot \mathrm{p}(\bX^\star_{\gamma_1})
\leq \ \mathrm{p}_{\gamma_1}(\bX_{\gamma_2}^\star) - \frac{\gamma_2}{\gamma_1}\mathrm{p}_{\gamma_2}(\bX_{\gamma_2}^\star)
 = (1 \!-\! \frac{\gamma_2}{\gamma_1}) \cdot \mathrm{p}(\bX^\star_{\gamma_2})$.
Because $\gamma_2 / \gamma_1 > 1$, $\mathrm{p}(\bX^\star_{\gamma_1}) \geq \mathrm{p}(\bX^\star_{\gamma_2})$.
\end{proof}

\subsection{Proof of Proposition~\ref{thm:dual_simple}}
\begin{proof}
The Lagrangian of the primal problem~\eqref{eq:sdcut_sdp1} is:
\begin{align}
\mathrm{L} (\bX, \bu, \bZ)  =& \langle \bX, \bA \rangle - \langle \bX, \bZ \rangle + \frac{1}{2\gamma} \lVert \bX \rVert_F^2 \notag \\
            &+ \textstyle{\sum_{i=1}^m} u_i (\langle \bX, \bB_i\rangle \! - \! b_i),  \label{eq:fastsdp_lagrangian}
\end{align}
where $\bu \in \mathbb{R}^{|\mathcal{I}_{eq}|} \times \mathbb{R}^{|\mathcal{I}_{in}|}_+$ and $\bZ \in \mathcal{S}^n_+$ are Lagrangian multipliers.

Supposing \eqref{eq:sdcut_sdp1} and \eqref{eq:fastsdp_lagrangian} are feasible, strong duality holds and
$\nabla_{\bX} \mathrm{L} (\bX^\star, \bu^\star, \bZ^\star)  = 0$, where $\bX^\star$, $\bu^\star$ and $\bZ^\star$ are optimal solutions.
Then we have that
\begin{align}
\bX^\star \!=\! \gamma (\bZ^\star \!-\! \bA \!-\! \textstyle{\sum_{i=1}^{m}} u_i^\star \bB_i) = \gamma (\bZ^\star \!+\! \bC(\bu^\star)).
\label{eq:primal_dual_solution}
\end{align}
By substituting $\bX^\star$ to \eqref{eq:fastsdp_lagrangian}, we have the dual:
\begin{align}
\max_{\bu \in \mathbb{R}^{|\mathcal{I}_{eq}|} \times \mathbb{R}^{|\mathcal{I}_{in}|}_+, \, \bZ \in \mathcal{S}^n_+}
    &\,\, - \bu^{\T} \bb -\frac{\gamma}{2} \lVert \bZ + \bC(\bu) \rVert_F^2 .  \label{eq:fastsdp_dual_2}
\end{align}

The variable $\bZ$ can be further eliminated as follows.
Given a fixed $\bu$, \eqref{eq:fastsdp_dual_2} can be simplified to:
$\min_{\bZ \in \mathcal{S}^n_+} \  \lVert \bZ + \bC(\bu) \rVert_F^2$,
which is proved to has the solution $\bZ^\star = \Pi_{\mathcal{S}^n_+}(-\bC(\bu))$ (see \cite{higham1988computing} or Section 8.1.1 of \cite{boyd2004convex}).
Note that $\bC(\bu) = \Pi_{\mathcal{S}^n_+}(\bC(\bu)) - \Pi_{\mathcal{S}^n_+}(-\bC(\bu))$,
so $\bZ^\star + \bC(\bu) = \Pi_{\mathcal{S}^n_+}(\bC(\bu))$.
By substituting $\bZ$ into~\eqref{eq:fastsdp_dual_2} and \eqref{eq:primal_dual_solution}, we have
the simplified dual~\eqref{eq:fastsdp_dual} and Equation~\eqref{eq:fastsdp_primal_dual}.
\end{proof}

\subsection{Proof of Proposition~\ref{thm:dual_objective}}
\begin{proof}
Set $\zeta(\bX) := \frac{1}{2}\lVert \Pi_{\mathcal{S}^n_+} (\bX) \rVert_F^2 = \frac{1}{2} \sum_{i=1}^{n} (\max(0,\lambda_i))^2 $,
where $\lambda_i$ denotes the $i$-th eigenvalue of $\bX$.
$\zeta(\bX)$ is  a {\em separable spectral function} associated with the function
$\mathrm{g}(x) = \frac{1}{2} (\max(0,x) )^2$.
$\zeta: \mathcal{S}^n \rightarrow \mathbb{R}$ is continuously differentiable but not necessarily twice differentiable at $\bX \in \mathcal{S}^n$,
as $\mathrm{g}: \mathbb{R} \rightarrow \mathbb{R}$ has the same smoothness property (see \cite{lewis1996derivatives,lewis2001twice,sendov2007higher}).
We also have $\nabla \zeta(\bX) = \Pi_{\mathcal{S}^n_+}(\bX)$.
\end{proof}

\subsection{The Spherical Constraint}

Before proving Proposition~\ref{Remark:4}, we first give the following theorem.
\begin{theorem} (The spherical constraint).
For any $\bX \in \mathcal{S}^n_+$, we have the inequality $\lVert \bX \rVert_F \leq \mathrm{trace}(\bX)$, in which
the equality holds if and only if $\mathrm{rank}(\bX) = 1$.
\label{thm:1}
\end{theorem}
\begin{proof}
The proof given here is an extension of the one in~\cite{Malick2007spherical}.
We have $\lVert \bX \rVert_F^2 =
\mathrm{trace} (\bX \bX^\T )
= \sum_{i=1}^n \lambda^2_i
\leq (\mathrm{trace}(\bX))^2$, where $\lambda_i \geq 0$ denotes the $i$-th eigenvalue of $\bX$.
Note that $\lVert \bX \rVert_F = \mathrm{trace}(\bX)$ (\ie $\sum_{i=1}^n \lambda^2_i = (\sum_{i=1}^n \lambda_i)^2$),
 if and only if there is only one non-zero eigenvalue of $\bX$, \ie, $\mathrm{rank}(\bX) = 1$.
\end{proof}

\subsection{Proof of Proposition~\ref{Remark:4}}

\begin{proof}
Firstly, we have the following inequalities:
\begin{align}
 \mathrm{p}(\bX^\star) = \mathrm{p}_\gamma(\bX^\star) - \frac{\lVert \bX^\star \rVert^2_F}{2\gamma} \geq \mathrm{p}_\gamma(\bX^\star) - \frac{(\mathrm{trace}(\bX^\star))^2}{2\gamma},
\label{eq:inequality0}
\end{align}
where the second inequality is based on Theorem~\ref{thm:1}.
For the BQP~\eqref{eq:bqp} that we consider, it is easy to see $\mathrm{trace}(\bX^\star) = n$.
Furthermore, $\mathrm{p}_\gamma(\bX^\star) \geq \mathrm{p}_\gamma(\bX^\star_\gamma) $ holds by definition.
Then we have that
\begin{align}
 \mathrm{p}(\bX^\star) \geq \mathrm{p}_\gamma(\bX^\star_\gamma) - \frac{n^2}{2\gamma}.
\label{eq:inequality}
\end{align}

It is known that the optimum of the original SDP problem~\eqref{eq:backgd_sdp1}
is a lower-bound on the optimum of the BQP~\eqref{eq:bqp} (denoted by $p^{\star}$):
$\mathrm{p}(\bX^\star) \leq p^{\star}$.
Then according to \eqref{eq:inequality}, we have $\mathrm{p}_\gamma(\bX^\star_\gamma) - \frac{n^2}{2\gamma} \leq \mathrm{p}(\bX^\star) \leq p^{\star}$.
Finally based on the strong duality, the primal objective value is not smaller than the dual objective value in the feasible set
(see \eg \cite{boyd2004convex}):
$\mathrm{d}_\gamma(\bu) \leq \mathrm{p}_\gamma(\bX^\star_\gamma)$,
where $\bu \in \mathbb{R}_{\vphantom{+}}^{|\mathcal{I}_{eq}|} \times \mathbb{R}_{+}^{|\mathcal{I}_{in}|}$, $\gamma > 0$.
In summary, we have:
$\mathrm{d}_\gamma(\bu) - \frac{n^2}{2\gamma} \leq \mathrm{p}_\gamma(\bX^\star_\gamma) - \frac{n^2}{2\gamma} \leq \mathrm{p}(\bX^\star) \leq p^{\star},
\, \forall \bu \in \mathbb{R}_{\vphantom{+}}^{|\mathcal{I}_{eq}|} \times \mathbb{R}_{+}^{|\mathcal{I}_{in}|}, \forall \gamma > 0$.
\end{proof}

{ %
\bibliographystyle{IEEEtran}
\bibliography{IEEEabrv,sdcut1}
}

\begin{IEEEbiographynophoto}{Peng Wang}
received the B.S. degree in electrical engineering
and automation, and the PhD degree in control science and
engineering from Beihang University, China, in 2004 and 2011, respectively.
He is now a post-doctoral researcher at the University of Adelaide.
\end{IEEEbiographynophoto}
\begin{IEEEbiographynophoto}{Chunhua Shen}
is a Professor at School of
Computer Science, The University of Adelaide.
His research interests are in the intersection of
computer vision and statistical machine learning.
He studied at Nanjing University, at Australian
National University, and received his PhD
degree from University of Adelaide. In 2012, he
was awarded the Australian Research Council
Future Fellowship.
\end{IEEEbiographynophoto}
\begin{IEEEbiographynophoto}{Anton van den Hengel}
is a Professor and the Founding Director of the Australian Centre for Visual Technologies,
at the University of Adelaide, focusing on innovation in the production and analysis of visual digital media.
He received the Bachelor of
mathematical science degree, Bachelor of laws degree, Master's degree in computer science,
and the PhD degree in computer vision from The University of Adelaide in 1991, 1993, 1994, and 2000, respectively.
\end{IEEEbiographynophoto}
\begin{IEEEbiographynophoto}{Philip H. S. Torr}
received the Ph.D. (D.Phil.)
degree from the Robotics Research Group at the
University of Oxford, Oxford, U.K., under Prof.
D. Murray of the Active Vision Group. He was a
research fellow at Oxford for another three years. 
He left Oxford to work as a research
scientist with Microsoft Research for six years,
first with Redmond, WA, USA, and then with
the Vision Technology Group, Cambridge, U.K.,
founding the vision side of the Machine Learning
and Perception Group. Currently, he is a professor in computer vision
and machine learning with Oxford University, Oxford, U.K.
\end{IEEEbiographynophoto}

\input{appendix.tex}

\end{document}

%% file: appendix.tex
\clearpage

\onecolumn

\section*{Appendix}

In this section, we present some  computational details.

\setcounter{section}{0}
\setcounter{theorem}{0}
\setcounter{equation}{0}

\renewcommand{\theequation}{A.\arabic{equation}}
\renewcommand{\thetable}{A.\arabic{table}}
\renewcommand{\thefigure}{A.\arabic{figure}}
\renewcommand{\thesection}{A.\arabic{section}}
\renewcommand{\thetheorem}{A.\arabic{theorem}}

\section{Preliminaries}

\subsection{Euclidean Projection onto the {P.S.D\onedot} Cone}

\begin{theorem}
\label{thm:sdp_projection}
The Euclidean projection of a symmetric matrix $\bX \in \mathcal{S}^n$ onto the positive semidefinite cone $\mathcal{S}^n_+$,
is given by
\begin{align}
\Pi_{\mathcal{S}^n_+}(\bX)
:= \mathrm{arg}\min_{\bY \in \mathcal{S}^n_+} \  \lVert \bY - \bX \rVert_F^2
= \sum_{i=1}^n \max(0,\lambda_i) \bp_i \bp_i^\T,
\end{align}
where $\lambda_i, \bp_i, i = 1, \cdots, n$ are the eigenvalues and the corresponding eigenvectors of $\bX$.
\end{theorem}
\begin{proof}
This result is well-known and
its proof can be found in \cite{higham1988computing} or Section 8.1.1 of \cite{boyd2004convex}.
\end{proof}

\subsection{Derivatives of Separable Spectral Functions}
\label{sec:spectral_functions}

A {\em spectral} function $\mathrm{F}(\bX): \mathcal{S}^n \rightarrow \mathbb{R}$ is a function which depends only on the eigenvalues of a symmetric matrix $\bX$,
and can be written as $\mathrm{f}(\blambda)$ for some {\em symmetric} function $\mathrm{f}: \mathbb{R}^n \rightarrow \mathbb{R}$,
where $\blambda = [\lambda_i, \cdots, \lambda_n]^\T$ denotes the vector of eigenvalues of $\bX$.
A function $\mathrm{f}(\cdot)$ is {\em symmetric} means that $\mathrm{f}(\bx) = \mathrm{f}(\mathbf{U} \bx)$
for any permutation matrix $\mathbf{U}$ and any $\bx$ in the domain of $\mathrm{f(\cdot)}$.
Such symmetric functions and the corresponding spectral functions are called {\em separable},
when $\mathrm{f}(\bx) = \sum_{i=1}^{n} \mathrm{g}(x_i)$ for some function $\mathrm{g}: \mathbb{R} \rightarrow \mathbb{R}$.
It is known (see \cite{lewis1996derivatives,lewis2001twice,sendov2007higher}, \eg) that a spectral function has the following properties:

\begin{theorem}
\label{thm:spectral_1}
A separable spectral function $\mathrm{F}(\cdot)$ is $k$-times (continuously) differentiable at $\bX \in \mathcal{S}^n$,
if and only if its corresponding function $\mathrm{g}(\cdot)$ is $k$-times (continuously) differentiable
at $\lambda_i$, $i = 1, \dots, n$,
and the first- and second-order derivatives of $\mathrm{F}(\cdot)$ are given by
\begin{align}
\label{eq:spectral_1}
\nabla \mathrm{F}(\bX) &= \bP \Big( \mathrm{diag} \big(\nabla \mathrm{g}(\lambda_1), \nabla \mathrm{g}(\lambda_2), \dots, \nabla \mathrm{g}(\lambda_n)\big) \Big) \bP^{\T}, \\
\nabla^2 \mathrm{F}(\bX)(\bH) &= \bP \left( \Omega({\blambda}) \circ (\bP^{\T} \bH \bP) \right) \bP^{\T}, \forall \bH \in \mathcal{S}^n
\end{align}
where
$[\Omega(\blambda)]_{ij} := \left\{ \begin{array}{ll}
                                    \frac{\nabla \mathrm{g}(\lambda_i)- \nabla \mathrm{g}(\lambda_j)}{\lambda_i - \lambda_j}
                                           & \mbox{if } \lambda_i \neq \lambda_j, \\
                                    \nabla^2 \mathrm{g}(\lambda_i)
                                           & \mbox{if } \lambda_i = \lambda_j,
                                          \end{array}
                                  \right.
\,\, i,j = 1, \dots, n$.
$\blambda = [\lambda_1, \cdots, \lambda_n]^\T$ and $\bP = [\bp_1, \cdots, \bp_n]$
are the collection of eigenvalues and the corresponding eigenvectors of $\bX$.
\end{theorem}

\section{Inexact Smoothing Newton Methods: Computational Details}

\subsection{Smoothing Function}
In this section, we show how to constrcut a smoothing function of $\mathrm{F}(\bu)$ (see \eqref{eq:smooth_F_equation}).
First, the smoothing functions for $\Pi_{\mathcal{D}}$ and $\Pi_{\mathcal{S}^n_+}$ are written as follows respectively:
\begin{align}
  &\tilde{\Pi}_{\mathcal{D}}(\epsilon,\bv) := \left\{ \begin{array}{ll}
                                    v_i                 & \mbox{if } i \in \mathcal{I}_{eq}, \\
                                    \phi(\epsilon,v_i)  & \mbox{if } i \in \mathcal{I}_{in},
                                  \end{array}
                          \right. \quad (\epsilon, \bv) \in \mathbb{R} \times \mathbb{R}^{m}, \\
  &\tilde{\Pi}_{\mathcal{S}^n_+}(\epsilon,\bX)   :=
                           \sum_{i=1}^n \phi(\epsilon, \lambda_i) \bp_i \bp_i^\T,
                           \quad (\epsilon, \bX) \in \mathbb{R} \times \mathcal{S}^n,
\end{align}
where $\lambda_i$ and $\bp_i$ are the $i$th eigenvalue and the corresponding eigenvector of $\bX$.
$\phi(\epsilon,v)$ is the Huber smoothing function that we adopt here to replace $\max(0,v)$:
\begin{align}
\phi(\epsilon,v) := \left\{ \begin{array}{ll}
                  v & \mbox{if } v > 0.5\epsilon, \\
                  (v+0.5\epsilon)^2 / 2\epsilon, &\mbox{if } -0.5\epsilon \leq v \leq 0.5\epsilon, \\
                  0 & \mbox{if } v < -0.5\epsilon.
                  \end{array} \right.
\end{align}
Note that at $\epsilon = 0$,
$\phi(\epsilon,v) = \max(0,v)$,
$\tilde{\Pi}_{\mathcal{D}}(\epsilon,\bv) = \Pi_{\mathcal{D}}(\bv)$ and
$\tilde{\Pi}_{\mathcal{S}^n_+}(\epsilon,\bX)   = \Pi_{\mathcal{S}^n_+}(\bX)$.
$\phi$, $\tilde{\Pi}_{\mathcal{D}}$, $\tilde{\Pi}_{\mathcal{S}^n_+}$ are Lipschitz continuous on
$\mathbb{R}$, $\mathbb{R} \times \mathbb{R}^m$, $\mathbb{R} \times \mathcal{S}^n$ respectively,
and they are continuously differentiable when $\epsilon \neq 0$.
Now we have a smoothing function for $\mathrm{F}(\cdot)$:
\begin{align}
\tilde{\mathrm{F}}(\epsilon,\bu) := \bu - \tilde{\Pi}_{\mathcal{D}} \left(\epsilon,
   \bu -  \gamma \Phi \left[ \tilde{\Pi}_{\mathcal{S}^n_+} \left( \epsilon,\bC(\bu) \right) \right] - \bb \right),
\quad (\epsilon,\bu) \in \mathbb{R}\times\mathbb{R}^m,
\label{eq:Upsilon_smooth_F}
\end{align}
which has the same smooth property as $\tilde{\Pi}_{\mathcal{D}}$ and $\tilde{\Pi}_{\mathcal{S}^n_+}$.

\subsection{Solving the Linear System \eqref{eq:cg_equation}}
\label{sec:solve_linear_system}

The linear system~\eqref{eq:cg_equation} can be decomposed to two parts:
\begin{subequations}
\label{eq:equations_two_part}
\begin{align}
\eqref{eq:cg_equation} \Leftrightarrow
\left[ \begin{array}{c} \epsilon_k \\ \tilde{\mathrm{F}}(\epsilon_k,\bu_k) \end{array}\right] +
\left[ \begin{array}{cc} 1 & 0 \\ (\nabla_\epsilon \tilde{\mathrm{F}})(\epsilon_k,\bu_k) & (\nabla_{\bu} \tilde{\mathrm{F}})(\epsilon_k,\bu_k) \end{array}\right],
\label{eq:equations_two_part_2}
\end{align}
\begin{numcases}{\Leftrightarrow }
\Delta \epsilon_k = \bar{\epsilon} - \epsilon_k \label{eq:smooth_E_equation_decomp1}\\
\nabla_{\bu} \tilde{\mathrm{F}}(\epsilon_k,\bu_k) ( \Delta \bu_k ) =
- \tilde{\mathrm{F}}(\epsilon_k,\bu_k) - \nabla_\epsilon \tilde{\mathrm{F}}(\epsilon_k,\bu_k) ( \Delta \epsilon_k), \label{eq:smooth_E_equation_decomp2}
\end{numcases}
\end{subequations}
where $\nabla_\epsilon \tilde{\mathrm{F}}$ and $\nabla_{\bu} \tilde{\mathrm{F}}$ denote the partial derivatives of $\tilde{\mathrm{F}}$ \wrt $\epsilon$ and $\bu$ respectively.
One can firstly obtain the value of $\Delta \epsilon_k$ by \eqref{eq:smooth_E_equation_decomp1}
and then solve the linear system~\eqref{eq:smooth_E_equation_decomp2} using CG-like algorithms.

Since the Jacobian matrix $\nabla_{\bu} \tilde{\mathrm{F}}(\epsilon_k,\bu_k) \in \mathbb{R}^{m \times m}$
is nonsymmetric when inequality constraints exist,
biconjugate gradient stabilized (BiCGStab) methods~\cite{van1992bi} are used for \eqref{eq:smooth_E_equation_decomp2} \wrt $|\mathcal{I}_{in}| \neq 0$,
and classic conjugate gradient methods are used when $|\mathcal{I}_{in}| = 0$.

The computational bottleneck of CG-like algorithms is on the Jacobian-vector products at each iteration.
We discuss in the following the computational complexity of it in our specific cases.
Firstly, we give the partial derivatives of smoothing functions
$\phi(\epsilon,v): \mathbb{R}\times\mathbb{R} \rightarrow \mathbb{R}$,
$\tilde{\Pi}_{\mathcal{D}}(\epsilon,\bv): \mathbb{R}\times\mathbb{R}^m \rightarrow \mathbb{R}^m$
and $\tilde{\Pi}_{\mathcal{S}^n_+}(\epsilon,\bX): \mathbb{R}\times\mathcal{S}^n \rightarrow \mathcal{S}^n$:
\begin{subequations}
\begin{align}
\nabla_{\epsilon} \phi(\epsilon,v)
                  &= \left\{ \begin{array}{ll}
                  0.125 - 0.5(v/\epsilon)^2 &\mbox{if } -0.5\epsilon \leq v \leq 0.5\epsilon, \\
                  0 & \mbox{otherwise,}
                  \end{array} \right. \\
\nabla_{v} \phi(\epsilon,v)
                  &= \left\{ \begin{array}{ll}
                  1 & \mbox{if } v > 0.5\epsilon, \\
                  0.5 + v/\epsilon &\mbox{if } -0.5\epsilon \leq v \leq 0.5\epsilon, \\
                  0 & \mbox{if } v < -0.5\epsilon,
                  \end{array} \right.
\end{align}
\end{subequations}
\begin{subequations}
\begin{align}
\left[\nabla_{\epsilon} \tilde{\Pi}_{\mathcal{D}}(\epsilon,\bv)\right]_i &= \left\{ \begin{array}{ll}  
0                 & \mbox{if } i \in \mathcal{I}_{eq}, \\
\nabla_{\epsilon}\phi(\epsilon,v_i)  & \mbox{if } i \in \mathcal{I}_{in} ,
\end{array}  
\right. \\
\left[\nabla_{\bv} \tilde{\Pi}_{\mathcal{D}}(\epsilon,\bv)\right]_{ij} &= \left\{ \begin{array}{ll}  
1                 & \mbox{if } i = j \in \mathcal{I}_{eq}, \\
\nabla_{v_i}\phi(\epsilon,v_i)  & \mbox{if } i = j \in \mathcal{I}_{in}, \\
0 & \mbox{if } i \neq j,
\end{array}  
\right. 
\end{align}
\end{subequations}
\begin{subequations}
\begin{align}
\nabla_{\epsilon}\tilde{\Pi}_{\mathcal{S}^n_+}(\epsilon,\bX)   &= \bP \mathrm{diag}\left( \nabla_{\epsilon} \phi(\epsilon,\blambda) \right) \bP^{\T}, \label{eq:derivative_5}\\
\nabla_{\bX}\tilde{\Pi}_{\mathcal{S}^n_+}(\epsilon,\bX) (\bH) &= \bP \left( \Omega(\epsilon,{\blambda}) \circ (\bP^{\T} \bH \bP) \right) \bP^{\T}, \label{eq:derivative_6}
\end{align}
\end{subequations}
where
$\blambda$ and $\bP$
are the collection of eigenvalues and the corresponding eigenvectors of $\bX$.
$\nabla_{\epsilon} \phi(\epsilon,\blambda) := \left[ \nabla_{\epsilon} \phi(\epsilon,\lambda_i) \right]_{i=1}^n$ and
$\Omega(\epsilon,\blambda): \mathbb{R}\times\mathbb{R}^n \rightarrow \mathcal{S}^n$ is defined as
\begin{align}
[\Omega(\epsilon,\blambda)]_{ij} := \left\{ \begin{array}{ll}
                                    \frac{\phi(\epsilon,\lambda_i)-\phi(\epsilon,\lambda_j)}{\lambda_i - \lambda_j}
                                           & \mbox{if } \lambda_i \neq \lambda_j, \\
                                    \nabla_{\lambda_i}\phi(\epsilon,\lambda_i)
                                           & \mbox{if } \lambda_i = \lambda_j,
                                          \end{array}
                                  \right.
\,\, i,j = 1, \dots, n.
\end{align}
Equations~\eqref{eq:derivative_5} and \eqref{eq:derivative_6} are derived based on Theorem~\ref{thm:spectral_1}.

Then we have the partial derivatives of $\tilde{\mathrm{F}}(\epsilon,\bu): \mathbb{R}\times\mathbb{R}^m \rightarrow \mathbb{R}^m$ \wrt $\epsilon$ and $\bu$:
\begin{subequations}
\begin{align}
\nabla_{\epsilon} \tilde{\mathrm{F}}(\epsilon,\bu) &= - \nabla_{\epsilon} \tilde{\Pi}_{\mathcal{D}} \left(\epsilon, \bw \right)
- \nabla_{\bw} \tilde{\Pi}_{\mathcal{D}} \left(\epsilon, \bw \right) (\nabla_{\epsilon} \bw), \notag \\
&= - \nabla_{\epsilon} \tilde{\Pi}_{\mathcal{D}} \left(\epsilon, \bw \right)
+  \nabla_{\bw} \tilde{\Pi}_{\mathcal{D}} \left(\epsilon, \bw \right)
\left( \gamma \Phi \left[ \bP \mathrm{diag}\left( \nabla_{\epsilon} \phi(\epsilon,\blambda)\right) \bP^{\T} \right]\right),
\label{eq:Upsilon_derivative_1}\\
\nabla_{\bu} \tilde{\mathrm{F}}(\epsilon,\bu) (\bh)
&= \bh - \nabla_{\bw} \tilde{\Pi}_{\mathcal{D}} \left(\epsilon, \bw \right) (\nabla_{\bu} \bw), \notag \\
&= \bh - \nabla_{\bw} \tilde{\Pi}_{\mathcal{D}} \left(\epsilon, \bw \right)
\left(\bh + \gamma
\Phi\left[ \bP \left( \Omega(\epsilon,{\blambda}) \circ (\bP^{\T} \Psi[\bh] \bP) \right) \bP^{\T}\right] \right) ,
\label{eq:Upsilon_derivative_2}
\end{align}
\end{subequations}
where $\bw := \bu -  \gamma \Phi \left[ \tilde{\Pi}_{\mathcal{S}^n_+} \left( \epsilon,\bC(\bu) \right) \right] - \bb$;
$\bC(\bu) := \!-\!\bA \!-\! \Psi[\bu]$;
$\Phi(\bX) := [\langle \bB_1, \bX\rangle, \cdots, \langle \bB_m, \bX\rangle]^{\T}$;
$\Psi(\bu) := \sum_{i=1}^m u_i \bB_i$;
$\blambda$ and $\bP$
are the collection of eigenvalues and the corresponding eigenvectors of $\bC(\bu)$.

In general cases, computing \eqref{eq:Upsilon_derivative_1} and \eqref{eq:Upsilon_derivative_2}
needs $\mathcal{O}(mn^2+n^3)$ flops.
However, based on the observation that most of $\bB_i,i = 1, \dots, m$ contain only $\mathcal{O}(1)$ elements and
$r = \mathrm{rank}(\Pi_{\mathcal{S}^n_+}(\bC(\bu))) \ll n$,
the computation cost can be dramatically reduced.
Firstly, super sparse $\bB_i$s lead to the computation cost of $\Phi$ and $\Psi$
reduced from $\mathcal{O}(mn^2)$ to $\mathcal{O}(m+n)$.
Secondly, note that $[\Omega(\epsilon,\blambda)]_{ij} = 0, \forall \lambda_i, \lambda_j < 0$.
Given $r \ll n$ and $\epsilon$ is small enough,
the matrix $\Omega$ only contains non-zero elements in the first $r$ columns and rows.
Thus the matrix multiplication in \eqref{eq:Upsilon_derivative_1}, \eqref{eq:Upsilon_derivative_2} and \eqref{eq:Upsilon_smooth_F}
can be computed in $\mathcal{O}(n^2r)$ flops rather than the usual $\mathcal{O}(n^3)$ flops.

In summary, the computation cost of the right hand side of Equ.~\eqref{eq:smooth_E_equation_decomp2}
and the Jacobian-vector product~\eqref{eq:Upsilon_derivative_2} can be
reduced from $\mathcal{O}(mn^2+n^3)$ to $\mathcal{O}(m+n^2r)$ in our cases.